\documentclass[10pt,twocolumn]{article}

\usepackage{graphics} % for pdf, bitmapped graphics files
\usepackage{epsfig} % for postscript graphics files
\usepackage{amsmath} % assumes amsmath package installed
\usepackage{amssymb}  % assumes amsmath package installed
\usepackage{amsthm}
\theoremstyle{definition}
\usepackage{txfonts}
\usepackage{bbm}
\usepackage[singlelinecheck=false,font=small]{caption}
\usepackage{subcaption}
\graphicspath{{Images/}}
\usepackage{color}  % For Highlighting
\usepackage[dvipsnames]{xcolor}
\usepackage{indentfirst}
\usepackage{multirow}
\usepackage{algorithm}
\usepackage{algpseudocode}
\usepackage{ulem}
\usepackage{natbib}
\usepackage[margin=0.75in]{geometry}
\usepackage{hyperref}
\usepackage{appendix}

\usepackage{courier}

\title{Bridging the gap between safety and real-time performance in receding-horizon trajectory design for mobile robots}

\author{Shreyas Kousik\thanks{These authors contributed equally to this work},~Sean Vaskov\footnotemark[1],~Fan Bu,\\ Matthew Johnson-Roberson,~Ram Vasudevan
\footnote{S. Kousik, S. Vaskov, F. Bu, and R. Vasudevan are with University of Michigan Department of Mechanical Engineering.
M. Johnson-Roberson is with the University of Michigan Department of Naval Architecture and Marine Engineering.
This work is supported by the Ford Motor Company via the Ford-UM Alliance under award N022977; the National Science Foundation under Contract CNS-1239037; and by the Office of Naval Research under Award Number N00014-18-1-2575.
Corresponding author: Shreyas Kousik (\texttt{skousik@umich.edu}).}}

\newtheorem{defn}{Definition}
\newtheorem{rem}[defn]{Remark}
\newtheorem{lem}[defn]{Lemma}

\newtheorem{assum}[defn]{Assumption}
\newtheorem{ex}[defn]{Example}
\newtheorem{thm}[defn]{Theorem}

\newtheorem{innercustomthm}{Theorem}
\newenvironment{customthm}[1]
  {\renewcommand\theinnercustomthm{#1}\innercustomthm}
  {\endinnercustomthm}

\newtheorem{innercustomlem}{Lemma}
\newenvironment{customlem}[1]
  {\renewcommand\theinnercustomlem{#1}\innercustomlem}
  {\endinnercustomlem}

\providecommand{\R}{\ensuremath \mathbb{R}}
\providecommand{\N}{\ensuremath \mathbb{N}}

\providecommand{\X}{\ensuremath \mathcal{X}}
\renewcommand{\P}{\ensuremath \mathcal{P}}

\newcommand{\Lf}{\mathcal{L}_f}
\newcommand{\Lg}{\mathcal{L}_g}
\newcommand{\norm}[1]{\left\Vert#1\right\Vert}

\newcommand{\defemph}[1]{\emph{#1}}

\newcommand{\vep}{\varepsilon}

\newcommand{\inv}{^{-1}}

\newcommand{\bd}[1]{\partial #1}

\newcommand{\ts}[1]{\textsuperscript{#1}}
\newcommand{\z}{\zeta}
\newcommand{\kp}{\kappa}

\newcommand{\regtext}[1]{\mathrm{\textnormal{#1}}}

\newcommand{\hi}{_\regtext{hi}}
\newcommand{\hio}{_{\regtext{hi},0}}
\newcommand{\hii}{_{\regtext{hi},i}}

\providecommand{\frs}{_\mathrm{FRS}}
\providecommand{\obs}{_\mathrm{obs}}
\providecommand{\pln}{_\mathrm{plan}}
\providecommand{\stp}{_\mathrm{stop}}
\providecommand{\safe}{_\mathrm{safe}}
\newcommand{\plan}{_\regtext{plan}}
\newcommand{\sense}{_\regtext{sense}}
\newcommand{\brk}{_\regtext{brake}}

\newcommand{\brkmax}{_\regtext{brake,max}}

\newcommand{\move}{_\regtext{move}}
\newcommand{\des}{_\regtext{des}}
\newcommand{\mc}[1]{\mathcal{#1}}
\providecommand{\obsi}{_{\regtext{obs},i}}
\providecommand{\obsj}{_{\regtext{obs},j}}

\newcommand{\tfin}{T}

\newcommand{\rbar}{{\overline{r}}}
\newcommand{\bbar}{{\overline{b}}}

\newcommand{\idx}{\regtext{proj}_X}
\newcommand{\idv}{\regtext{proj}_V}
\newcommand{\vmax}{{v_\mathrm{max}}}

\newcommand{\irbar}{I_{\rbar}}
\newcommand{\pirbar}{P_{\!I_{\rbar}}}

\newcommand{\dels}{{\delta_\pm}}
\newcommand{\delh}{{\delta_x}}
\newcommand{\delv}{{\delta_y}}

\newcommand{\al}{\alpha}

\newcommand{\gm}{\gamma}

\newcommand{\ta}{\theta}

\newcommand{\proj}{\mathrm{proj}}

\newcommand{\SE}{\mathsf{SE}}

\begin{document}

% \keywords{autonomous robots, reachability analysis, guaranteed safety, real-time nonlinear model control}

\maketitle

\begin{abstract}
To operate with limited sensor horizons in unpredictable environments, autonomous robots use a receding-horizon strategy to plan trajectories, wherein they execute a short plan while creating the next plan.
However, creating safe, dynamically-feasible trajectories in real time is challenging; and, planners must ensure persistent feasibility, meaning a new trajectory is always available before the previous one has finished executing.
Existing approaches make a tradeoff between model complexity and planning speed, which can require sacrificing guarantees of safety and dynamic feasibility.
This work presents the Reachability-based Trajectory Design (RTD) method for trajectory planning.
RTD begins with an offline Forward Reachable Set (FRS) computation of a robot's motion when tracking parameterized trajectories; the FRS provably bounds tracking error.
At runtime, the FRS is used to map obstacles to parameterized trajectories, allowing RTD to select a safe trajectory at every planning iteration.
RTD prescribes an obstacle representation to ensure that obstacle constraints can be created and evaluated in real time while maintaining safety.
Persistent feasibility is achieved by prescribing a minimum sensor horizon and a minimum duration for the planned trajectories.
A system decomposition approach is used to improve the tractability of computing the FRS, allowing RTD to create more complex plans at runtime.
RTD is compared in simulation with Rapidly-Exploring Random Trees and Nonlinear Model-Predictive Control.
RTD is also demonstrated in randomly-crafted environments on two hardware platforms: a differential-drive Segway, and a car-like Rover.
The proposed method is safe and persistently feasible across thousands of simulations and dozens of real-world hardware demos.
\end{abstract}
\section{Introduction}\label{sec:intro}

\begin{figure*}%[ht]
\centering
    \begin{subfigure}[t]{0.95\columnwidth}%{0.45\textwidth}
        \centering
        \includegraphics[width=0.95\columnwidth]{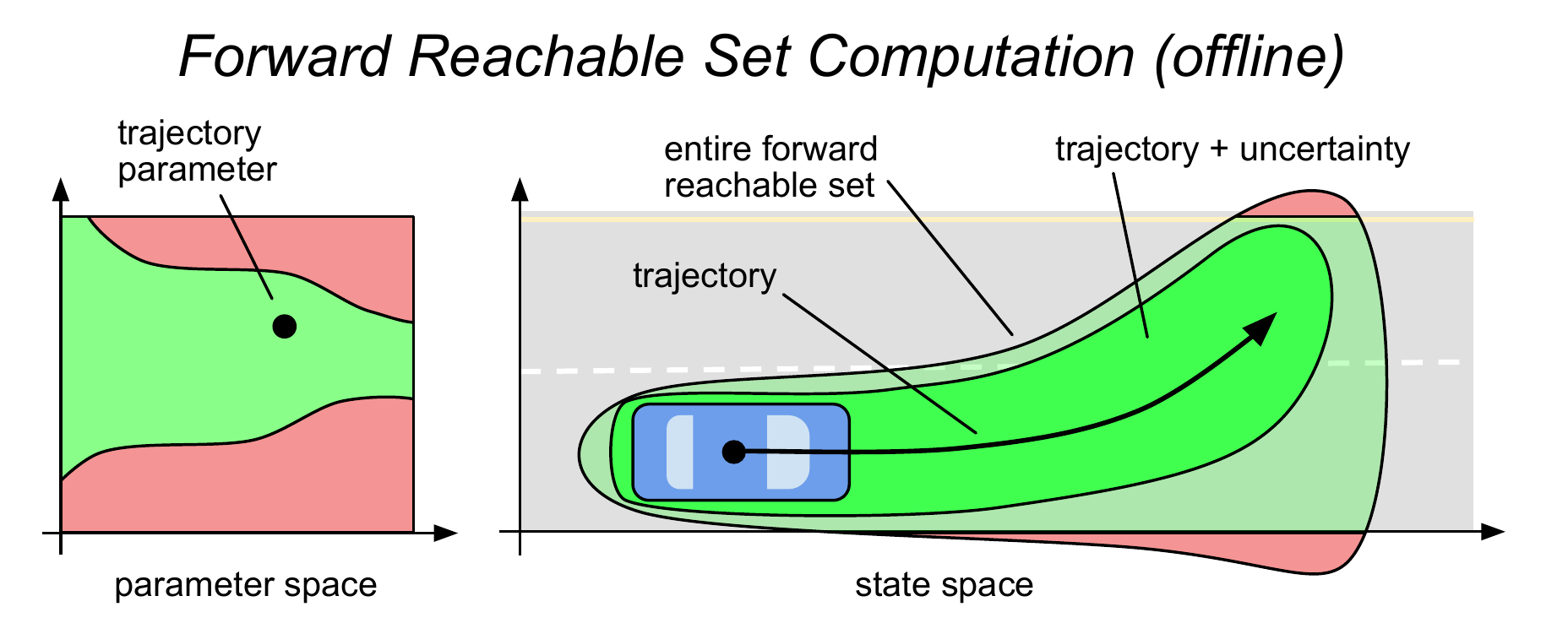}
        \caption{\centering}
        \label{subfig:intro_FRS_overview}
    \end{subfigure}
    %\hspace{1cm}
    \begin{subfigure}[t]{0.95\columnwidth}%{0.45\textwidth}
        \centering
        \includegraphics[width=0.95\columnwidth]{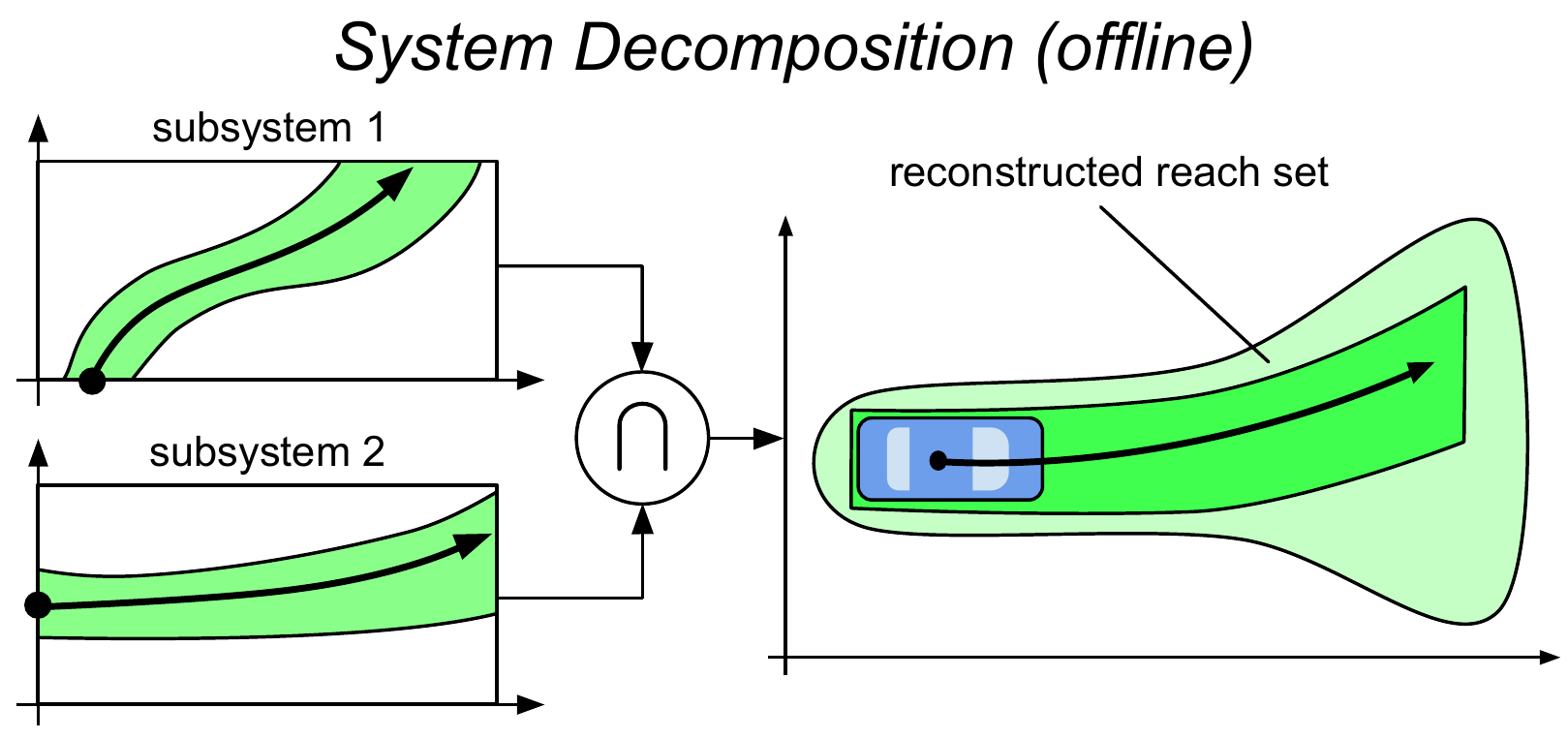}
        \caption{\centering}
        \label{subfig:intro_system_decomp}
    \end{subfigure}
    \begin{subfigure}[t]{0.95\columnwidth}%{0.45\textwidth}
        \centering
        \includegraphics[width=0.95\columnwidth]{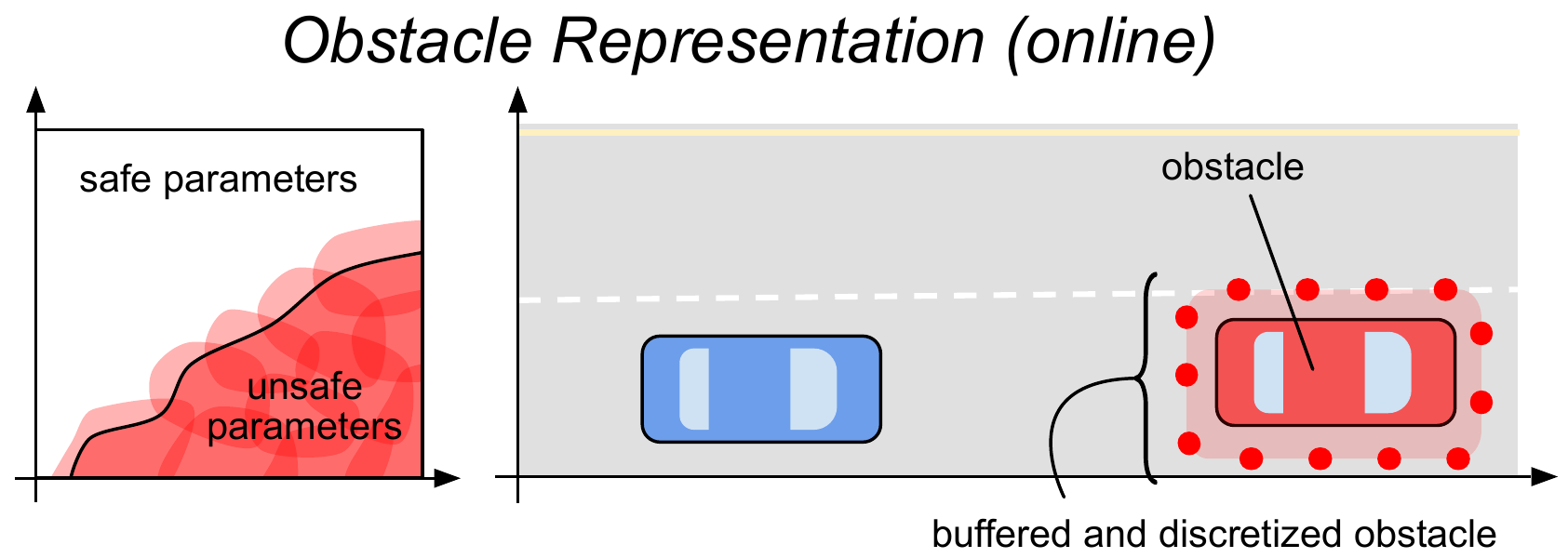}
        \caption{\centering}
        \label{subfig:intro_obstacle_representation}
    \end{subfigure}
    %\hspace{1cm}
    \begin{subfigure}[t]{0.95\columnwidth}%{0.45\textwidth}
        \centering
        \includegraphics[width=0.95\columnwidth]{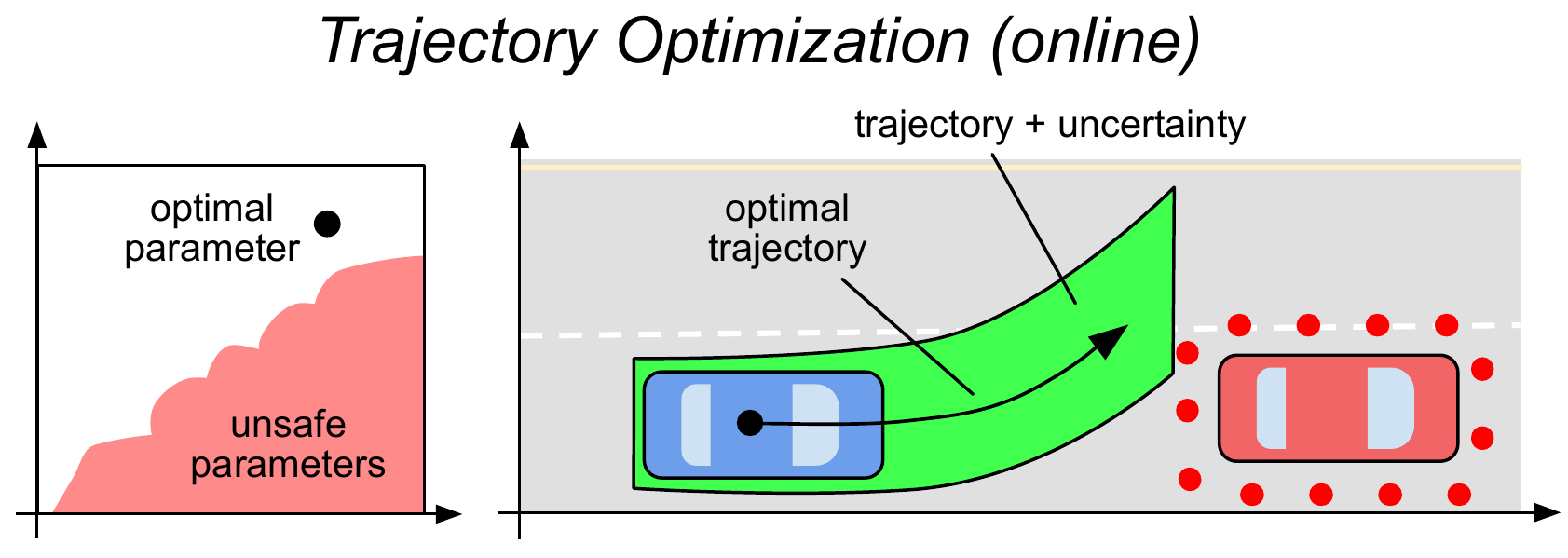}
        \caption{\centering}
        \label{subfig:intro_traj_opt}
    \end{subfigure}
    \caption{
    The contributions of this paper.
    Figure \ref{subfig:intro_FRS_overview} shows a parameterized trajectory space on the left, and the state space of the mobile robot on the right.
    The large bell-shaped contour on the right is the forward reachable set (FRS) corresponding to the robot attempting to track any trajectory from the parameter space.
    The sections of the FRS that leave the road correspond to unsafe trajectory parameters on the left.
    A single trajectory parameter is chosen and shown as a trajectory in the state space, plus the uncertainty in the robot's model, which results in a subset of the FRS corresponding to that parameter.
    The proposed method extends the existing RTD method from \citet{kousik2017safe} to higher-dimensional systems via a system decomposition approach adapted from \citet{chen2016exact} and \citet{chen2016journal}, shown in Figure \ref{subfig:intro_system_decomp}.
    For online trajectory optimization, this paper presents a method of representing obstacles discretely (Figure \ref{subfig:intro_obstacle_representation}) that allows real-time computation without sacrificing safety in Section \ref{sec:obstacle_representation}.
    Each point in the discretized obstacle representation is mapped to the subset of all trajectory parameters that could cause the robot to reach that point.
    At run-time, this mapping is expressed as a finite list of nonlinear constraints for online trajectory optimization (Figure \ref{subfig:intro_traj_opt}), which enables real-time operation.
    The unsafe parameters corresponding to the discretization are a superset of the exact set of unsafe parameters corresponding to the obstacle.
    Therefore, the discretized representation defines the feasible trajectory parameter space and ensures that the online trajectory optimization is safe as we prove in Section \ref{subsec:proving_X_p_works}.
    These two contributions enable provably safe, real-time trajectory planning (Figure \ref{subfig:intro_traj_opt}) in Section \ref{sec:trajectory_optimization}.}
    \label{fig:overview}
\end{figure*}

\begin{figure*}
    \centering
    \begin{subfigure}[h]{0.465\textwidth}
        \centering
        \includegraphics[height=5cm]{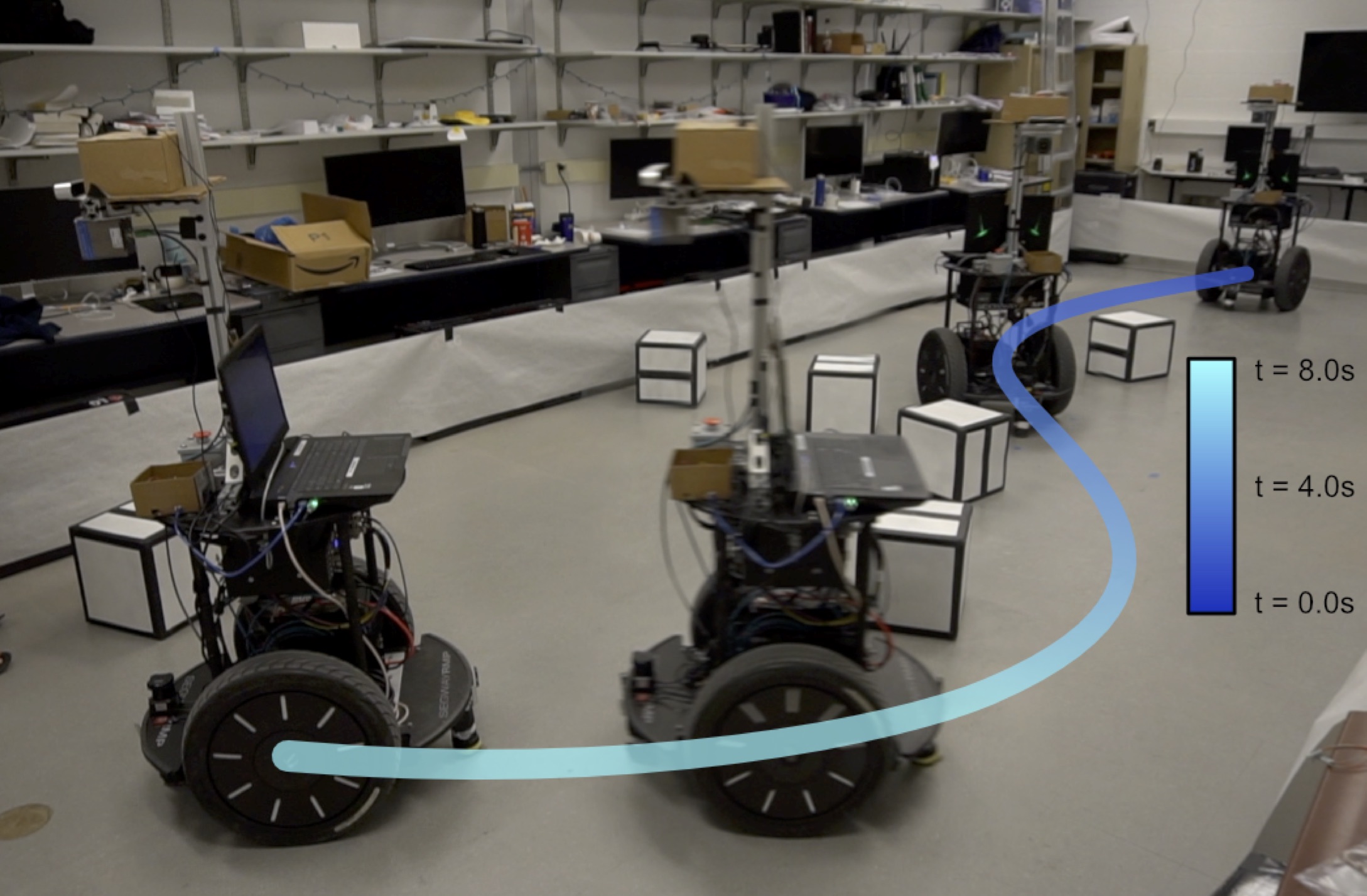}
        \caption{\centering}
        \label{subfig:segway_time_lapse}
    \end{subfigure}
    \begin{subfigure}[h]{0.465\textwidth}
        \centering
        \includegraphics[height=5cm]{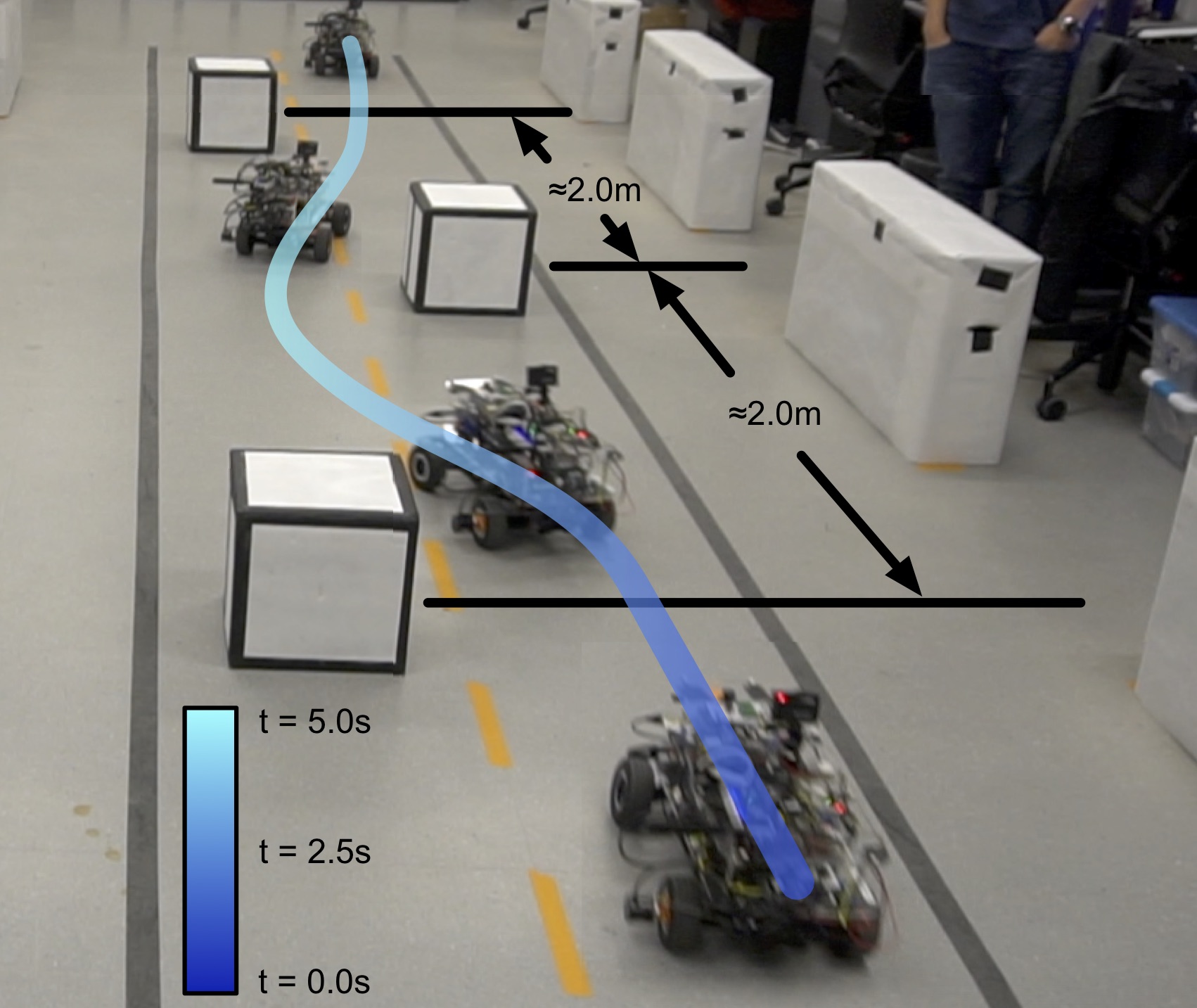}
        \caption{\centering}
        \label{subfig:rover_time_lapse}
    \end{subfigure}
    \caption{
    Example trajectories demonstrating RTD on two hardware platforms: the differential-drive Segway in Figure \ref{subfig:segway_time_lapse} and the car-like Rover in Figure \ref{subfig:rover_time_lapse}.
    The Segway travels at up to $1.25$ m/s around box obstacles distributed randomly around a rectangular room.
    The Rover travels at $1.5$ m/s on a mock road with randomly spaced box obstacles.
    Both use a planar lidar to sense the obstacles in real time.
    The Segway replans its trajectory every $0.5$ s, and the Rover every $0.375$ s, using Algorithm \ref{alg:trajopt} from Section \ref{sec:trajectory_optimization}.
    Both robots safely traverse their respective scenarios despite error in each robot's ability to track planned trajectories.
    Videos of the robots are available at \url{https://youtu.be/FJns7YpdMXQ} for the Segway and \url{https://youtu.be/bgDEAi_Ewfw} for the Rover.}
    \label{fig:hardware_time_lapse}
\end{figure*}

Autonomous mobile robots, such as autonomous cars, unmanned ground vehicles, and drones, are required to operate in unpredictable environments with limited sensor horizons.
To do so, they typically employ a \defemph{receding-horizon} strategy, wherein the robot simultaneously plans a short trajectory, then executes it while planning a subsequent trajectory. 
This strategy is necessary because the robot receives new sensor information as it moves through the environment.

To successfully perform receding-horizon planning, the robot must guarantee safety and persistent feasibility.
Planning is \defemph{safe} if the robot avoids collision with obstacles in the environment while executing a planned trajectory.
Planning is \defemph{persistently feasible} if there always exists a safe trajectory or stopping maneuver before the robot completes executing the previously-planned trajectory.
Therefore, the receding horizon strategy imposes a real-time requirement on trajectory planning, because the time required to generate a plan is less than or equal to the duration of the plan.
The main \textbf{contribution} of this work is a provably safe and persistently feasible receding-horizon trajectory planner for ground (planar) mobile robots in static environments.

This introduction section presents an overview of the literature and challenges in receding-horizon planning (Section \ref{subsec:related_work}); a statement of contributions (Section \ref{subsec:contributions}); and notation used throughout the paper (Section \ref{subsec:notation}).

\subsection{Related Work}\label{subsec:related_work}

To address the requirements of safety and persistent feasibility, a three level hierarchical control architecture is often used \citep{buehler2009darpa,falcone2007predictive,gonzalez2016review,gray2012predictive,boss2008urbanchallenge}.
At the top of the hierarchy, a high-level planner performs coarse route planning on a map using, e.g., Dijkstra's algorithm.
To construct paths rapidly, the high-level planner typically does not use a dynamic model of the robot, and thus cannot make safety guarantees.
At the bottom of the hierarchy, a low-level tracking controller translates kinematic commands into actuator torques.
This controller is not concerned with the robot's environment, and real-time applications of these controllers are widely used and well-studied.
The middle level is called a \defemph{trajectory planner}, which is the focus of this paper.
Trajectory planners take in high-level route guidance and local environmental constraints (e.g. walls, other robots, lane boundaries), and output a trajectory for the low-level controller to track.
Typically, since the trajectory planner uses knowledge of both the environment and a robot's dynamic model, it is used for planning obstacle-avoidance maneuvers in real time.
In this work, we propose a trajectory planner that is both safe and persistently feasible.

We now discuss the literature to show that existing trajectory planners incur a tradeoff: they typically must attempt to encourage either safety and persistent feasibility, or performance (meaning, quickly and successfully completing a task).
The proposed RTD method enables strict safety guarantees without a severe performance penalty.
Here, we discuss three general classes of trajectory planners: sample-based methods, model predictive control methods, and reachability-based methods.

\subsubsection{Sample-Based Methods}
Sample-based methods plan trajectories by drawing samples from a robot's control input and/or state space, resulting in temporal and/or spatial discretization of a robot's dynamic model.
A finer discretization typically enables stronger statements about the safety of such approaches, but with increased computational cost, and therefore a performance impact \citep{lavalle_textbook}.
Here, we first discuss several examples, then discuss how they attempt to enforce safety and persistent feasibility.

There are many different sample-based methods.
A widely-used example is the Rapidly-exploring Random Tree (RRT) algorithm, which plans trajectories by sampling the control input space or the state space to generate nodes in a graph representing a tree that explores the state space \citep{lavalle2001randomized}.
One can guarantee that this method will eventually find a path to a goal location, and even that  variants of that method can construct a path that is optimal with respect to an arbitrary cost function \citep{karaman2011sampling}, though such optimality may not be critical to ensure safety or real-time performance.
Other examples, with similar methods and guarantees, include Probabilistic Road Maps (PRM), which builds a graph that can includes loops \citep{kavraki1996probabilistic}, and Fast Marching Trees (FMT) \citep{janson2015fast}, which combine the tree structure of RRTs with dynamic programming to rapidly find paths.

Sample-based methods attempt to achieve safety in the following ways.
Since plans must incorporate the dynamics of a robot to certify safety, these methods must either have an explicit solution for a robot's trajectory to use for sampling, or must numerically integrate a dynamic model \citep{lavalle2001randomized,elbanhawi2014sampling}.
In addition to representing the dynamics, plans must not pass through obstacles; sample-based methods check if nodes, and potentially the edges between them, are in collision with obstacles, and then omit those nodes and edges \citep{lavalle_textbook}.
Since collision checking is challenging when the edges represent trajectories of a dynamic model, these methods typically linearize the robot's model to rapidly produce edges for collision checking \citep{elbanhawi2014sampling}; furthermore, since dynamic models typically only represent a robot's center of mass dynamics, obstacles in the environment must be buffered (i.e., padded or dilated) to compensate for the robot's shape.
If the dynamic model is not accurate, some sample-based approaches buffer obstacle to compensate \citep{lavalle_textbook}; others treat the robot's dynamics as linear, and propagate Gaussian distributions that can be used for collision checking \citep{Luders2010_RRTchanceconstraints}.
When sampling-based methods are used for receding-horizon planning, uncertainty can be mitigated by repropagating the tree from the robot's current position each time step \citep{kuwata2009rrt}.

Sample-based methods attempt to achieve persistent feasibility in the following ways.
Ensuring a plan always exists can be achieved by repropagating a preexisting tree in each planning iteration, and by attempting to end every plan with a braking maneuver \citep{kuwata2009rrt}; however, planning the stopping trajectory requires additional computation and still suffers from the above trade-offs.
To achieve real-time performance, a balance must be struck between the dimensionality of the dynamic model, the number of obstacles considered, and the discretization fineness, which is typically done by applying a heuristic \citep{elbanhawi2014sampling, kuwata2009rrt}.
The representation of constraints to attempt to ensure real-time performance and persistent feasibility, such as enforcing a minimum time horizon or distance, can be easy to check, but may impact performance.
Altogether, it is challenging to enforce persistent feasibility because it requires real-time performance, but this comes at the expense of simplifying the dynamics and collision checking; this means that one must either lose safety guarantees, or buffer obstacles more (which reduces performance by reducing the free space available for planning).

\subsubsection{Nonlinear Model Predictive Control}
Nonlinear Model Predictive Control (NMPC) methods plan trajectories by formulating an optimization program over a robot's control inputs, with the dynamics and obstacles treated as constraints.
They typically discretize time to make the optimization program tractable, and therefore incur the same tradeoff as sample-based methods.
We now present several examples, then discuss safety and persistent feasibility.

While there are many NMPC methods, most share a strategy of discretizing time and linearizing  the robot's dynamics at each discrete point in time \citep{falcone2007predictive,falcone2008low,howard2007roughterrainplanning,boss2008urbanchallenge,wurts2018collision}.
To avoid linearization, recent pseudo-spectral methods approximate the NMPC program with polynomial functions \citep{gpopsii}.
This approach can increase performance and computational efficiency over linearization methods, but still requires discretization.
An alternative to these types of discretizations and simplifications of the dynamics is Sequential Action Control.
This NMPC method uses a single control input, applied for a short duration, as a decision variable; optimality is checked by directly forward-simulating the system dynamics without discretizing or linearizing them \textit{a priori}, sacrificing a longer ``lookahead'' time for real-time planning speed.

NMPC methods attempt to enforce safety as follows.
Obstacles are represented as polygons or ellipses; collision-checking  is performed by evaluating if a discrete set of points (along a planned trajectory) lies within an obstacle (see, e.g., the work of \citet{wurts2018collision} for a focus on collision avoidance).
These methods typically represent only the center of mass dynamics of the robot, and must therefore buffer obstacles to compensate for the robot's nonzero volume.
Besides representing the robot's dynamics faithfully, safety also requires handling uncertainty.
A variety of methods exist to let NMPC handle different types of uncertainty.
For example, Robust NMPC treats the nonlinear parts of a robot's dynamic model as a bounded disturbance \citep{Gao2014_robustMPC,gao2014tube}, but relies on linearizing about a precomputed reference trajectory that can be difficult to generate for complex environments.
Sequential Action Control can be used to estimate uncertain parameters, and then plan with the estimate incorporated into the dynamics \citep{wilson2015SAC}; however, this has not been shown for safe control of mobile robots in arbitrary environments.
In general, because NMPC approaches must use simplified representations of a robot's dynamics to ensure fast computation, they cannot make safety guarantees without potentially buffering obstacles by a large amount.

NMPC methods attempt to enforce persistent feasibility as follows.
In general, persistent feasibility requires real-time solving.
But, the more complex (i.e., higher fidelity) the robot's dynamics, the slower an algorithm runs \citep{howard2007roughterrainplanning,boss2008urbanchallenge}.
This can be addressed by tuning hyperparameters (such as discretization fineness) \citep{wurts2018collision} and by linearizing the dynamics \citep{howard2007roughterrainplanning}.
Other ways to improve solving speed are to precompute a dynamically-feasible reference trajectory, then attempt to adjust it \citep{Frash2013_ACADO_MPC}; to exploit environment structure \citep{boss2008urbanchallenge}; to use a lookup table of initial guesses for the nonlinear solver \citep{howard2007roughterrainplanning}; or to use Sequential Action Control \citep{ansari2017SAC,wilson2015SAC}.
Viability kernels have been computed to establish persistent feasibility for MPC \citep{liniger2017real}; however, these kernels have to be computed offline while assuming the environment is known.
In general, it is unclear when persistently feasible planning is possible with arbitrary obstacle configurations.

To summarize, NMPC methods suffer the same tradeoff between safety and real-time performance (which is required for persistent feasibility) as sample-based methods.

\subsubsection{Reachability-based Methods}
Reachability-based methods precompute a reachable set containing the motion of the robot, then use these reachable sets to ensure collision avoidance at runtime.
The precomputed reachable sets can be used to synthesize safe tracking controllers and incorporate uncertainty in the dynamics.
These methods are focused on real-time planning with guaranteed safety, at the expense of some performance.
This means that they introduce some conservatism into trajectory planning, which may cause a robot to stop safely instead of reaching a goal.
We first discuss several methods for computing reachable sets, then discuss how reachability-based methods address safety and persistent feasibility.

A variety of methods exist to compute reachable sets.
Sums-Of-Squares (SOS) programming can be used to find reachable sets and associated tracking controllers, for a single trajectory \citep{majumdar2016funnel}, or a set of operating points \citep{majumdar2014convex}.
These methods use semi-algebraic set representations, and compute polynomial controllers, but require polynomial system dynamics.
To avoid a polynomial representation, one can use Hamilton-Jacobi-Bellman (HJB) reachability \citep{ding2011reachability,herbert2017fastrack}, which solves a partial differential equation by gridding the robot's state space and time (and therefore suffers the curse of dimensionality).
To avoid gridding, one can use zonotope reachable sets \citep{althoff_cora}; here, the user provides dynamics, a tracking controller, a reference trajectory \citep{althoff2014online}.
The reachable set is produced by partitioning time into small intervals, linearizing the dynamics in each time interval, then overapproximating the reachable set with a zonotope for each time interval.
This incurs a tradeoff where a coarser time partition or more nonlinearity in the dynamics results in the reachable set becoming conservative quickly, so these reachable sets can be difficult to compute for high-fidelity models of robots.

Reachability-based methods address safety as follows.
The SOS and zonotope approaches compute overapproximations of the reachable sets of robots in state space \citep{majumdar2016funnel,althoff_cora}.
Therefore, when planning, one seeks to ensure that the reachable set corresponding to any plan lies outside of obstacles.
Unfortunately, these approaches must pre-specify a finite set of trajectories for the offline reachability analysis; this can limit performance if the finite set is not rich enough to plan in arbitrary scenarios.
The HJB approach, on the other hand, poses its offline reachability analysis as a differential game between a high-fidelity model of a robot and a simplified planning model, which allows the planning model to choose from a continuum of possible plans.
The reachability analysis computes the tracking error between the high-fidelity and planning models, and an associated controller to keep the error within the computed bound at runtime.
At runtime, one buffers obstacles by this bound, then ensures that the planning model can only plan outside of the buffered obstacles.
Though this approach is conservative in theory because the planning model attempts to escape from the high-fidelity model \citep{herbert2017fastrack}, the numerical solution of this offline reachability analysis is not provably overapproximative \citep{mitchell2005time}.
Nevertheless, this approach can be excessively conservative in practice; since the planning model is allowed to act arbitrarily, the reachability analysis must bound tracking error for ``too many'' trajectories (as opposed to ``too few'' trajectories for \citet{majumdar2014convex} and \citet{althoff2014online}).

Reachability-based methods address persistent feasibility as follows.
For the SOS approach, with a finite library of reachable sets, one attempts to compose the reachable sets sequentially at runtime \citep{majumdar2016funnel}, though it is unclear how to proceed when no reachable sets are available.
In the zonotope case, since this approach is used to validate a single maneuver \citep{althoff2014online}, one simply does not execute an invalid maneuver; however, it is unclear how to always generate valid maneuvers.
Finally, for the HJB approach, one can simultaneously plan exploration trajectories and trajectories that return the robot to a previously-known safe location \citep{fridovich2019safely}; however, the robot may become ``stuck'' with this approach due to the underlying conservatism of the reachability analysis, which limits the free space available to the robot.

To summarize, reachability analysis enables strict safety guarantees, and some persistent feasibility guarantees.
However, the existing methods potentially suffer from either pre-specifying particular trajectories \citep{majumdar2016funnel,althoff2014online}, or allowing the trajectory planning model to select from completely arbitrary trajectories \citep{herbert2017fastrack,fridovich2019safely}.
% In other words, the performance of reachability-based methods can suffer from conservatism.

Next, we introduce our proposed approach to address the challenge of the tradeoff between safety and performance in real-time mobile robot trajectory planning.

\subsection{Proposed Method}\label{subsec:proposed method}

This work proposes Reachability-based Trajectory Design (RTD), which addresses the limitations in the literature discussed above.
We discuss how RTD relates to the literature, then present an overview of the method.

\subsubsection{Proposed Method in Context}
RTD uses reachability analysis, which allows the user to tune hyperparameters arbitrarily without losing safety guarantees, thereby addressing the tradeoff that sample-based and NMPC methods face between safety and performance.
RTD avoids the challenge faced by other reachability-based methods in choosing either too few or too many possible trajectories, by using a continuum of parameterized trajectories.
Furthermore, RTD enables strict persistent feasibility guarantees by ensuring the existence of a fail-safe maneuver in every trajectory plan.

Note that RTD was originally introduced in our prior work \citep{kousik2017safe}, which ensured safety, but could not perform real-time planning.
Additionally, that method was limited to low-dimensional systems such as Dubins cars.
The present work builds significantly upon the prior work to enable provably safe, persistently feasible, real-time trajectory planning for higher-dimensional representations of ground mobile robots.

\subsubsection{Proposed Method Overview}
RTD begins with an offline reachability analysis.
First, we specify a high-fidelity model describing a robot's dynamics, and a simplified trajectory-producing model that generates parameterized trajectories used real-time planning at runtime.
Second, we conservatively estimate the tracking error between the high-fidelity model and the parameterized trajectories.
Finally, we use the tracking error and the trajectory-producing model to compute a Forward Reachable Set (FRS) that provably contains the motion of the high-fidelity model (i.e., the robot) when tracking any of the parameterized trajectories.

Online (at runtime), RTD plans safe trajectories in an iterative, receding-horizon manner.
Assume a safe plan exists in the first planning iteration.
In each subsequent planning iteration, RTD operates as follows.
First, obstacles are intersected with the FRS, which contains all reachable points corresponding to all parameterized trajectories, to identify the set of unsafe trajectory parameters (i.e., those that could cause a collision).
Second, RTD performs trajectory optimization over the set of all safe trajectory parameters.
% If no new safe parameterized trajectory is found this way, then RTD commands the robot to continue executing the safe trajectory found in the previous planning iteration, ensuring persistent feasibility.
By designing each trajectory to be long enough for the robot to stop safely, in case no new trajectory can be found, we ensure safety and persistent feasibility for all time.

\subsection{Contributions}\label{subsec:contributions}

As depicted in Figure \ref{fig:overview}, the present work contains \textbf{three contributions}, which address the shortcomings of prior methods and demonstrate the present work's application.
\textbf{First}, we adapt system decomposition techniques, which have been effectively applied to reduce computational memory requirements for Backwards Reachable Sets \citep{chen2016exact, chen2016journal}, to RTD and forward reachability.
\textbf{Second}, we present a method for representing obstacles with discrete, finite sets, which enables provably safe, real-time planning.
\textbf{Third}, we demonstrate RTD on two autonomous mobile robots to illustrate that the method is safe and persistently feasible. %without being conservative.
The platforms are a differential-drive ``Segway'' and a car-like ``Rover,'' depicted in Figures \ref{subfig:segway_time_lapse} and \ref{subfig:rover_time_lapse} respectively.
Code used for the reachable set computation and simulations is available at \url{https://github.com/skvaskov/RTD}.

\subsection{Organization and Notation}

\subsubsection{Paper Structure}
% The paper is organized as follows.
Section \ref{sec:dynamic_models} introduces dynamic models used to describe the robot and generate plans.
Section \ref{sec:FRSmethod} describes a general method for offline computation of the Forward Reachable Set (FRS).
Section \ref{sec:system_decomp} presents a system decomposition technique to compute the FRS for higher-dimension systems.
Section \ref{sec:pers_feas} prescribes conditions required to ensure safety and persistent feasibility.
Section \ref{sec:obstacle_representation} presents an obstacle representation that enables safe, real-time planning.
Section \ref{sec:trajectory_optimization} describes the online receding-horizon trajectory optimization procedure.
Section \ref{sec:application} describes the application of RTD to the Segway and Rover.
Section \ref{sec:simulation_results} presents results of a simulation comparing RTD to RRT and NMPC for the Segway and Rover.
Section \ref{sec:hardware_demo} describes the hardware demonstrations on the FRS on the Segway and Rover.

\subsubsection{Notation}\label{subsec:notation}

The real numbers are $\R$.
The natural numbers are $\N$.
Euclidean space in $n \in \N$ dimensions is $\R^n$.
The special Euclidean group associated with $\R^2$ is $\SE(2)$.

Given a set $K$, its boundary is $\partial K$, its closure is $\regtext{cl}(K)$, its complement is $K^C$, its interior is $\text{int}(K)$, and its cardinality is $|K|$.
The power set of $K$ is $\P(K)$.
The set of continuous (resp. absolutely continuous) functions on a compact set $K$ is $C(K)$ (resp. $AC(K)$).
The Lebesgue measure on $K$ is denoted by $\lambda_K$, and the volume of $K$ is $\text{vol}(K) = \int_K\lambda_K$.
% The $i$-th component of a vector $v \in \R^n$ is denoted by $v_i$.

The ring of polynomials in $x$ is $\R[x]$, and the degree of a polynomial is the degree of its largest multinomial; 
the degree of the multinomial $x^\alpha,\,\alpha\in \N$ is $|\alpha|=\|\alpha\|_1$.
$\R_d[x]$ is the set of polynomials in $x$ with degree $d$.
The vector of coefficients of a polynomial $p$ is denoted $\regtext{vec}(p)$.
For a pair of vector-valued functions $f$ and $g$ with domain $\R^n$, the notation $\circ$ denotes the elementwise (Hadamard) product: $f\circ g = [f_1\cdot g_1,\ \cdots,\ f_n\cdot g_n]^\top$.

For points, sets, and functions, subscripts are used to indicate an index, subspace or subset.
Let $\R_{> 0}$ (resp. $\R_{\geq 0}$) denote the set $(0,\infty)$ (resp. $[0,\infty)$).
For a state space $Z$ with state variable $\z \in Z$, $Z_0$ indicates a set of initial conditions and $Z_j$ indicates a lower-dimensional subspace $j$ of $Z$, where $j$ can be an index or a coordinate of the state $\z$.
When referring to states, subscripts are used to indicate a particular index or subspace the state belongs to.
For example $\z_i$ can be used to indicate the $i^\mathrm{th}$ component of $\z$, or if $Z_j$ indicates a lower-dimensional subspace $j$ of $Z$ then $\z_j\in Z_j$ indicates a state in subspace $j$.
Superscripts are associated with the degree of a function or dimension of a space, for example $f^d$  may refer to a polynomial function, $f$, of maximum degree $d$.
Note one important exception: for the set of polynomials $\R_d[x]$, $d$ appears as a subscript, to avoid confusion with the superscript on $\R$ commonly used to indicate dimension.
\section{Dynamic Models}
\label{sec:dynamic_models}

This section introduces the dynamic models used to describe the robot and to plan trajectories in a receding-horizon fashion.
First, we present a high-fidelity model and simplified trajectory-producing model (Section \ref{subsec:models}).
Second, we present several states that must be included in the models to enable RTD (Section \ref{subsec:states}).
Third, we discuss the low-level controller used to track parameterized trajectories (Section \ref{subsec:low_level_controller}).
Finally, we present a model of tracking error required to enable safe planning (Section \ref{subsec:error}).

The proposed RTD method controls a robot described by a high-fidelity model using a low-level feedback controller to track parameterized trajectories generated by a lower dimensional trajectory-producing model.
Each trajectory is of duration $T > 0$, and planning is performed in a receding horizon fashion, where a new trajectory is chosen every $\tau\plan$ seconds (see Assumption \ref{ass:tau_plan}).
Planning over the low-dimensional space of trajectory parameters enables our trajectory planner to operate in real-time.
Safety is achieved by bounding trajectory tracking error, and by formulating a reachability-based constraint for obstacle avoidance.

Note, casual readers can gain an overview of this section from the previous paragraph and the Example blocks.

\subsection{High-Fidelity and Trajectory-Producing Models}\label{subsec:models}

Let the \defemph{high-fidelity model} have time $t \in [0,T]$, state $\z\hi(t) \in Z\hi$ at a particular time, feedback controller $u: [0,T]\times Z\hi \to U$, and dynamics described by:
\begin{equation}
\label{eq:high-fidelity_model}
\dot{\z}\hi(t) = f\hi(t,\z\hi(t),u(t,\z\hi(t))),
\end{equation}
where $f\hi:[0,T] \times Z\hi \times U \to \R^{n_{Z\hi}}$, $T > 0$, $ Z\hi \subset \R^{n_{Z\hi}}$ and $U \subset \R^{n_U}$.
We call $T$ the \defemph{planning time horizon}.

Since planning directly with a high-fidelity model in real time is challenging, we use a simpler \defemph{trajectory-producing model} to generate plans at runtime.
We write this model as:
\begin{equation}
\label{eq:traj-producing_model}
\begin{bmatrix}\dot{\z}(t) \\ \dot{k}(t) \end{bmatrix} = \begin{bmatrix}  f(t,\z(t),k(t)) \\ 0 \end{bmatrix}
\end{equation}
where $f: [0,T]\times Z \times K \to \R^{n_Z}$.
A trajectory produced by this model in the space $Z$ is called a \defemph{desired trajectory}.
The trajectory states $\z$ belong to a subspace $Z$ of $Z\hi$, where $\dim(Z) \leq \dim(Z\hi)$; the states (i.e., dimensions) mutual to $Z$ and $Z\hi$ are called \defemph{shared states}.
The parameters $k$ are drawn from a set $K \subset \R^{n_K}$, and are fixed over the planning time horizon $[0,T]$, as we describe in Section \ref{sec:trajectory_optimization}, and as is written in \eqref{eq:traj-producing_model} where $\dot{k}(t) = 0$.
Note that, to lighten notation, we drop the state and input arguments in the dynamics when they are are clear from context.
For example, we may write $f\hi(t,\z\hi,u)$ instead of $f\hi(t,\z\hi(t),u(t,\z\hi(t)))$.

Notice that the dynamic models above are defined over a compact time interval $[0,T]$.
This means that every plan generated by RTD is of duration $T > 0$; we set time to $0$ at the beginning of each planned trajectory without loss of generality.
This compact time horizon imposes a limit on the amount of time that RTD can spend planning in any receding-horizon planning iteration.
We formalize this with the following assumption.

\begin{assum}\label{ass:tau_plan}
In each receding-horizon planning iteration, the robot has a constant maximum allowed amount of time, denoted $\tau\plan \in (0,T)$, within which to find a new plan.
The planning time is fixed offline, then enforced at runtime.
\end{assum}

\noindent Though we do not prove that the trajectory planning time of the proposed method is bounded, we do enforce a time limit of $\tau\plan$ on online computation, after which it is terminated.
As stated, Assumption \ref{ass:tau_plan} does not prescribe what the robot should do after $\tau\plan$ has passed, or how this planning time relates to the robot hardware.
We address these concerns in Section \ref{sec:pers_feas}.
For now, stating the existence of $\tau\plan$ is sufficient to proceed.

Next, we place assumptions on the dynamics to make computation of tracking error and reachable sets tractable.

\begin{assum}\label{ass:dyn_are_lipschitz_cont}
The dynamics $f\hi$ from \eqref{eq:high-fidelity_model} are Lipschitz continuous in $t$, $\z\hi$, and $u$.
The dynamics $f$ from \eqref{eq:traj-producing_model} are Lipschitz continuous in $t$, $\z$, and $k$.
Since planning occurs in a receding-horizon fashion, a new trajectory parameter $k$ can be chosen every $\tau\plan$ seconds, i.e. the desired trajectory can vary discontinously from one planning iteration to the next.
\end{assum}

\begin{assum}\label{ass:state_and_control_sets_are_compact}
The sets $U$, $Z\hi$, $Z$, and $K$ are compact.
The robot's set of \defemph{initial conditions} are represented as a compact set $Z\hio \subset Z\hi$ for the high-fidelity model, and $Z_0 \subset Z$ in the shared states of the trajectory-producing model.% with nonzero volume in the $xy$-subspace of $Z$.
\end{assum}

\subsection{Required States}\label{subsec:states}

Now, we point out several states that must be in the spaces $Z$ and $Z\hi$ for RTD.
In this work we focus on ground applications where the robot's pose and environment can be represented in 2-D, i.e. the space $\R^2$ with coordinates denoted $x$ and $y$.
\begin{defn}\label{def:X_and_X_0}
Let $X$ denote the \defemph{$xy$-subspace of $Z$} with $\dim(X) = 2$.
We also refer to $X$ as the \defemph{spatial coordinates} of the robot's body.
Let $X_0$ denote the projection of $Z_0$ into the $xy$-subspace.
We call $X_0$ the robot's \defemph{footprint at time $0$}.
\end{defn}

\noindent All of the points on the robot's body lie in the state space $X$, with initial condition set $X_0$ at time $t = 0$.
Therefore, the high-fidelity dynamics in \eqref{eq:high-fidelity_model} must include the dynamics of every point on the robot's body.
However, as per \citet{elbanhawi2014sampling}, dynamic models typically only describe the position of a single point on the robot (typically the center of mass), and the dynamics of the rest of the robot's body are written relative to this point, because the robot is treated as a rigid body.
To perform safe trajectory planning, it is insufficient to ensure that just a single point on the robot avoids collision with obstacles; therefore, we consider the dynamics of the robot's entire body, leading to the following assumption.

\begin{assum}\label{ass:rigid_body_dynamics_ctr_of_mass}
Since the robot is traveling in the plane, we assume that the trajectory-producing state $\z$ includes a pair of coordinates $[x_c,y_c]^\top \in X$ that describe the position of the center of mass of the robot.
We further assume that the robot is a rigid body.
Let $\theta$ be the robot's heading.
The motion of every point on the robot's body are given by the states $[x,y]^\top \in X$ with the following dynamics:
\begin{align}
\begin{split}\label{eq:rigid_body_dynamics}
    \dot{x} &= \dot{x}_c - \dot{\theta}\cdot(y - y_c) \\
    \dot{y} &= \dot{y}_c + \dot{\theta}\cdot(x - x_c).
\end{split}
\end{align}
\end{assum}

\noindent The equations of rigid body motion in \eqref{eq:rigid_body_dynamics} are available in any introductory dynamics course, e.g., \citet[Lecture 7]{dynamics_MIT_OCW}.

\begin{rem}\label{rem:points_on_rigid_body_have_same_dyn_as_ctr_of_mass}
Assumption \ref{ass:rigid_body_dynamics_ctr_of_mass} requires the trajectory-producing dynamics \eqref{eq:traj-producing_model} to describe the motion of the robot's entire body, not just its center of mass.
This means that \eqref{eq:traj-producing_model} includes a pair of states for $[x_c,y_c]^\top$, and a pair of states for $[x,y]^\top$ with dynamics \eqref{eq:rigid_body_dynamics}.
So, the dimension of the trajectory-producing state space $Z$ is at least 4.
However, there are two cases where the dynamics of every point $[x,y]^\top$ can be treated as identical to $[x_c,y_c]^\top$ for the purpose of obstacle avoidance.
The first case is when the robot has a circular footprint, so rotating the robot's body does not change the subset of $X$ that the robot occupies.
The second case is when the robot's footprint does not have any yaw motion, in which case $\dot{\theta} = 0$ in \eqref{eq:rigid_body_dynamics}.
\end{rem}
\noindent Note that treating all points on the robot's body as $[x_c,y_c]^\top$ is useful because, the higher the dimension of the dynamics \eqref{eq:traj-producing_model}, the more difficult it is to compute trajectory plans in real time \citep{herbert2017fastrack,kousik2017safe,karaman2011sampling,howard2007roughterrainplanning}.
Section \ref{sec:obstacle_representation} (see Definition \ref{def:R_t_translation_and_rotation_family}) provides more detail on the motion of the robot's footprint through time.

Next, we define the robot's speed and yaw rate.

\begin{assum}\label{ass:max_speed_and_yaw_rate}
We assume that the robot has a speed coordinate $v$ in its high-fidelity model state $\z\hi$.
The robot is limited to a scalar max speed (its rate of travel in the subspace $X$), denoted $v_\regtext{max}$.
If the robot has a yaw (i.e., heading) state, its time derivative is limited to a scalar maximum, denoted $\omega_\regtext{max}$.
\end{assum}

\noindent Recall that by Assumption \ref{ass:rigid_body_dynamics_ctr_of_mass}, the robot has 2-D spatial coordinates of its center of mass $[x_c,y_c]^\top$, with dynamics in the high-fidelity model $f\hi$ from \eqref{eq:high-fidelity_model}.
If $f\hi$ has no speed state, then we append the coordinate $v = \sqrt{\dot{x}_c^2 + \dot{y}_c^2}$ to the state $\z\hi$, which now evolves in the space $Z\hi\times[0,\vmax]$, which preserves the compactness of the state space in Assumption \ref{ass:state_and_control_sets_are_compact}.

We now present examples of the high-fidelity model \eqref{eq:high-fidelity_model} and trajectory-producing model that satisfy the requirements above.

\begin{ex}\label{ex:segway}
Consider the Segway, depicted in Figure \ref{subfig:segway_time_lapse}.
This type of differential-drive robot can be described by a high-fidelity model as follows.

Let $\z\hi = [x_c, y_c, \theta, \omega, v]^\top$ be the states, where $x_c$ and $y_c$ describe the robot's center of mass as in Assumption \ref{ass:rigid_body_dynamics_ctr_of_mass}.
Heading is $\theta$, yaw rate is $\omega$, and speed is $v$, ensuring we satisfy Assumption \ref{ass:max_speed_and_yaw_rate}.
The dynamics $f\hi$ are:
\begin{align}\label{eq:high-fidelity_segway}
    \frac{d}{dt}\begin{bmatrix}
    x_c \\ 
    y_c \\ 
    \theta \\
    \omega \\
    v \end{bmatrix}
    \quad=\quad
    \begin{bmatrix} v\cos\theta \\ 
    v\sin\theta \\ 
    \omega \\
    \regtext{sat}_\gamma\Big(\beta_\gamma\cdot\big(u_1 - \omega\big)\Big) \\
    \regtext{sat}_\alpha\Big(\beta_\alpha\cdot\big(u_2 - v\big)\Big)
    \end{bmatrix},
\end{align}
where the control input is $u = [u_1, u_2]^\top \in U \subset \R^2$, $\mathrm{sat}_{\gm}$ (resp $\mathrm{sat}_{\al}$) saturates the yaw (resp. longitudinal) acceleration input to keep it in an interval $[\underline{\gm},\overline{\gm}]$ (resp. $[\underline{\al},\overline{\al}]$), and $\beta_\gamma,\ \beta_\alpha > 0$ are constants found from system identification.
In this case, the robot has a circular footprint, so all points on the robot can be described by the center of mass dynamics as per Remark \ref{rem:points_on_rigid_body_have_same_dyn_as_ctr_of_mass}.
\end{ex}

\noindent See Section \ref{subsec:segway_application} for the parameter values used for the Segway hardware in Figure \ref{subfig:segway_time_lapse} and simulation in Section \ref{sec:simulation_results}.

\begin{ex}\label{ex:segway_traj_producing_model}
We produce desired trajectories for the Segway as Dubins paths parameterized by a desired yaw-rate, $k_1$, and a desired speed, $k_2$.
Note, these trajectory parameters obey the max yaw rate and speed in Assumption \ref{ass:max_speed_and_yaw_rate}.
Let $\z = [x,y]^\top$.
The trajectory-producing model $f$ is:
\begin{align}\label{eq:traj_prod_segway}
    \frac{d}{dt}\begin{bmatrix}x \\ y \end{bmatrix} =
    \begin{bmatrix} k_2 - k_1\cdot(y - y_{c,0}) \\ k_1\cdot(x - x_{c,0}) \end{bmatrix},
\end{align}
where, at the beginning of each planning iteration, $x_{c,0}$ and $y_{c,0}$ are the initial position of the center of mass in a global reference frame that is rotated so the positive $x$-direction points in the robot's longitudinal direction of travel.
Therefore, the initial heading is $\theta(0) = 0$, and $\theta(t) = k_1t$; since the heading is only a function of time, it does not need to be included in the trajectory-producing model.
So, in this case, $Z = X$.
Recall that the trajectory parameters $[k_1,k_2]^\top \in K$ are constant over the planning time horizon $[0,T]$ as per \eqref{eq:traj-producing_model}, so their dynamics are omitted from \eqref{eq:traj_prod_segway}.
\end{ex}

Next, we discuss the tracking controller used to drive the high-fidelity model to the desired trajectories.

\subsection{Low-Level Controllers}\label{subsec:low_level_controller}
Recall the planning hierarchy noted in the introduction.
A high-level planner generates coarse plans, a trajectory planner transforms them into dynamically-feasible plans, and finally, a low-level controller tracks them.
RTD is a trajectory planner, and gives the user freedom to design their own low-level controller to track parameterized trajectories.
Here, we state the general form of these controllers, then provide an example for the Segway.

Given $k\in K$, we call the low-level controller a \defemph{feedback controller for $k$},
\begin{align}
u_k:[0,T]\times Z\hi\to U.\label{eq:fdbk_controller_uk}
\end{align}
Note that one could use entirely feedforward control, resulting in $u_k: [0,T] \to U$; we state the feedback controller as a more general case.

When controlled by $u_k$, we say that the high-fidelity model \defemph{tracks $k$} as a shorthand to mean that the high-fidelity model tracks the trajectory parameterized by $k$.
As mentioned above, RTD is agnostic to the type of feedback used (e.g., PID, LQR, MPC).

Note that designing such a controller is simplified by our use of a trajectory producing model defined over a compact time interval and compact parameter space.
That is, the user need not design a controller to track any possible trajectory, only the ones that are parameterized.
We find in practice that PD or PID control performs satisfactorily with low tracking error, as we show in Sections \ref{sec:simulation_results} and \ref{sec:hardware_demo}.

Consider the following example feedback controller for the Segway.
\begin{ex}\label{ex:segway_feedback_controller}
Recall the high-fidelity model in Example \ref{ex:segway}, with inputs $u_1$ (yaw acceleration) and $u_2$ (longitudinal acceleration).
Recall that the trajectory-producing model in Example \ref{ex:segway_traj_producing_model} has a yaw rate parameter $k_1$, and a longitudinal speed parameter $k_2$.
Define $u_k = [u_1, u_2]^\top$.
For the Segway, we use a PD controller to drive the high-fidelity model towards the parameterized trajectories:
\begin{align}\label{eq:segway_pd_controller}
    u_k(t,\z\hi(t)) = \begin{bmatrix}
        \beta_\theta(k_1t - \theta(t)) + \beta_\omega(k_1 - \omega(t)) + \beta_y e_y(t) \\
        \beta_v(k_2 - v(t)) + \beta_x e_x(t)
    \end{bmatrix},
\end{align}
where the position error terms are given by
\begin{align}
    \begin{bmatrix}
        e_x(t) \\ e_y(t)
    \end{bmatrix} = 
    \begin{bmatrix}
        \cos(\theta(t)) & \sin(\theta(t)) \\
        -\sin(\theta(t)) & \cos(\theta(t))
    \end{bmatrix}\begin{bmatrix}
        x(t) - x_c(t) \\
        y(t) - y_c(t)
    \end{bmatrix},
\end{align}
where $[x,y]^\top$ (resp. $[x_c,y_c]^\top$) are the position states of the trajectory-producing (resp. high-fidelity) model.
The scalars $\beta_\theta,\ \beta_\omega,\ \beta_v,\ \beta_x,$ and $\beta_y$ are non-negative control gains.
We report the particular values used in Section \ref{sec:application}.
\end{ex}

Recall that RTD accounts for tracking error between the high-fidelity model and trajectory-producing model when planning trajectories at runtime.
To understand the tracking error, we now introduce projection operators to directly relate the high-fidelity and trajectory-producing models.

\subsection{Projection Operators}\label{subsec:projection_operators}

The previous discussion introduces a variety of subspaces of the robot's state space $Z\hi$.
To better understand the relationship between these various subspaces, we define projection operators, adapted from \citet[Section III A, (15,16,18)]{chen2016exact}:

\begin{defn}\label{def:projection_operators}
The \defemph{projection operator} $\proj_{Z_i}: \P(Z) \to \P({Z_i})$ maps sets from the higher-dimensional space $Z$ to a lower-dimensional subspace $Z_i$.
For a set containing a single point, $\z\in Z$, $\proj_{Z_i}$ is defined as:
\begin{flalign}
        {\proj}_{Z_i}(\z)=\z_i,\label{eq: projection full to sub traj}
\end{flalign}
where $\z_i$ contains the components of $\z$ that lie in subspace $Z_i$.
For a set, $S \subseteq Z$, $\proj_{Z_i}$ is defined as:
\begin{flalign}
        {\proj}_{Z_i}(S) = \Big\{\z_i\in Z_i\ :\ \exists\ \z\in S\ \text{s.t.}\ {\proj}_{Z_i}(\z)=\z_i\Big\}. 
\end{flalign}
Similar to $\proj_{Z_i}$, let the operator $\proj_Z: \P({Z\hi}) \to \P(Z)$ project points or sets from the high-fidelity model state space into the lower-dimensional, trajectory-producing space.

We also define $\proj\inv: \P(Z_i) \to \P(Z)$ to be the \defemph{back-projection operator} from a subset $S_i \subseteq Z_i$ to the full space $Z$ is defined as:
\begin{flalign}
        {\proj}^{-1}(S_i) = \Big\{\z\in Z\ :\ \exists\ \z_i\in S_i\ \text{s.t.}\ {\proj}_{Z_i}(\z)=\z_i\Big\}.\label{eq:back_proj}
\end{flalign}
\end{defn}
\noindent 
Note that the projection operator is continuous \citep[Theorem 18.2(b)]{Munkres2000}, and that $S_i$ can be a subset of a subspace $Z_i$.
As an example of Definition \ref{def:projection_operators}, the operator $\idx: \P(Z) \to \P(X)$ projects points or sets into the $xy$-subspace, so $X = \idx(Z)$.
If $\z \in Z$, then $\idx$ maps $\z$ to $[\z_x,\z_y]^\top \in X$ where $\z_x$ and $\z_y$ are the $x$ and $y$ components of $\z$ respectively.

To simplify exposition, we abuse notation and also use $\idx$ to project directly from $Z\hi$ to the $xy$-subspace $X$, as opposed to composing $\idx$ with $\proj_Z$, when the intent is clear from context.
We also occasionally pass system dynamics to the projection operators to select the dynamics in a subspace.
For example, we may write $\proj_Z(f\hi(\cdot))$ to mean the high-fidelity model's dynamics in the shared states $Z$, even though the range of $f\hi$ does not return elements of $\P(Z\hi)$; this is a minor abuse of notation because $\dim(Z\hi) = \dim(f\hi(\cdot))$.

Next, we use these projection operators to examine the tracking error.
Note, we also use them in Section \ref{sec:system_decomp} to perform reachability analysis in multiple subspaces of a high-dimensional trajectory-producing model.

\subsection{Bounding Sources of Error}\label{subsec:error}

With the high-fidelity model, trajectory-producing model, and tracking controller established, we now address the robot's ability to track trajectories.

Our approach requires that that we can quantify and bound all error observed when tracking the parameterized trajectories.
There are two sources of error.
The first comes from model uncertainty between the robot and high-fidelity model which is used to estimate its future state for the next planning iteration.
This is expressed as \defemph{state estimation error}.
The second comes from the fact that the robot cannot perfectly track desired trajectories from \eqref{eq:traj-producing_model}; we call this \defemph{tracking error}.

We place bounds on the state estimation error and tracking error as follows.

\begin{assum}\label{ass:predict}
Let $k \in K$ be arbitrary and $u_k$ the corresponding feedback controller as in \eqref{eq:fdbk_controller_uk}.
Suppose the robot is at time $t$, with estimated state $\z\hio \in Z\hio$, and recall that $T$ is the planning time horizon.
The robot's \defemph{future state prediction} $\z_\regtext{pred}$ at any $t' \in [t, t+\tau\plan]$ is given by forward-integrating the high-fidelity model to get the trajectory $\z\hio: [0,\tau\plan] \to Z\hi$:
\begin{align}\begin{split}\label{eq:predicted_trajectory}
    &\z_\regtext{pred}(t';\z\hio,k) = \z\hio~+ \\ 
    &+ \int_0^{t'-t}\!\!f\hi(\tau,\z_\regtext{pred}(\tau+t),u_k(\tau,\z _\regtext{pred}(\tau+t))) d\tau,
\end{split}\end{align}
where time is shifted to be in the correct domain, $[0,T]$, for the dynamics $f\hi$ and controller $u_k$.
We assume the robot has a \defemph{state estimator} such that the state estimation error in the robot's spatial coordinates $x$ and $y$ is bounded for every $t' \in [t, t+\tau\plan]$.
In other words, at the start of every planning iteration, there exist $\vep_x, \vep_y \geq 0$ such that the position of the actual robot is within $\vep_x$ (resp. $\vep_y$) of its estimated position in the $x$ (resp. $y$) coordinate.
Note that this is trivially satisfied by picking large $\vep_x$ and $\vep_y$.
\end{assum}

\noindent In practice, since $\tau\plan$ is small, $\vep_x$ and $\vep_y$ are small (e.g., on the order of centimeters when $\tau\plan = 0.5$ s).
Next, we place a bound on the trajectory tracking error:

\begin{assum}\label{ass:tracking_error_fcn_g}
For each $i \in \{1,\ldots,n_Z\}$, there exists a bounded function $g_i: [0,T] \times Z\times K \to \R$ such that:
\begin{align}\label{eq:error_func}
    \max_{\z\hi \in A_{\z}} \, \left|\z\hii(t;\z\hio,k) - \z_i(t;\z_0,k)\right| \leq \int_0^t g_i(\tau,\z,k)d\tau,
\end{align}
for all $\z \in Z$, $t \in [0,T]$, and $k \in K$ where $A_{\z} := \{\z\hi \in Z\hi \mid \proj_Z(\z\hi) = \z\}$ is the set in which the high-fidelity model matches the trajectory-producing model in all the shared states.
We call $g = [g_1,\ldots,g_{n_Z}]^\top$ the \defemph{tracking error function}.
As with $f$ and $f\hi$ in Assumption \ref{ass:dyn_are_lipschitz_cont}, we assume that $g$ is Lipschitz continuous in $t$, $\z$, and $k$.
\end{assum}

\noindent The existence of the tracking error function means that the error in the shared states subspace $Z$ is bounded while the robot is tracking any desired trajectory given by \eqref{eq:traj-producing_model}.
The tracking error function can be determined empirically by simulating the high-fidelity model or by applying Sums-Of-Squares (SOS) optimization techniques \citep{lasserre2009moments}.
The construction of such a function is not the focus of this paper; however we make the following remark about its existence, and point the reader to Example \ref{ex:segway_traj_tracking_model} and Figure \ref{fig:g_fit_segway_error}:

\begin{rem}\label{rem:trackin_error_bound_existance}
Since the dynamics of the trajectory producing model and high-fidelity model are Lipschitz continuous, and we are considering compact time, state and parameter spaces, it is reasonable to assume that the tracking error can be bounded.
In particular, if one is not confident that their robot can track trajectories closely, one can augment the tracking error function $g$ with a large positive function.
In subsequent sections, a larger $g$ results in a larger Forward Reachable Set in the offline RTD computation, which results in more conservative trajectory optimization at runtime.
Recall, the goal of RTD at runtime is to choose trajectory parameters that do not cause the robot to reach obstacles.
Then, in a given planning iteration, a larger FRS means the robot reaches more of the state space for each trajectory parameter; so, a larger portion of the trajectory parameters would correspond to the robot reaching obstacles, meaning the planner would have to be more conservative.
Figure \ref{fig:g_fit_rover_error} shows an example of conservative $g$ functions for the Rover (see \eqref{eq:high-fidelity_rover} and Example \ref{ex:traj-producing_rover} for the high-fidelity and trajectory-producing models) planning lane changes at high speed and yawrate.

\end{rem}
\begin{figure}
    \centering
    \begin{subfigure}[t]{0.47\textwidth}
        \centering
        \includegraphics[width=0.9\textwidth]{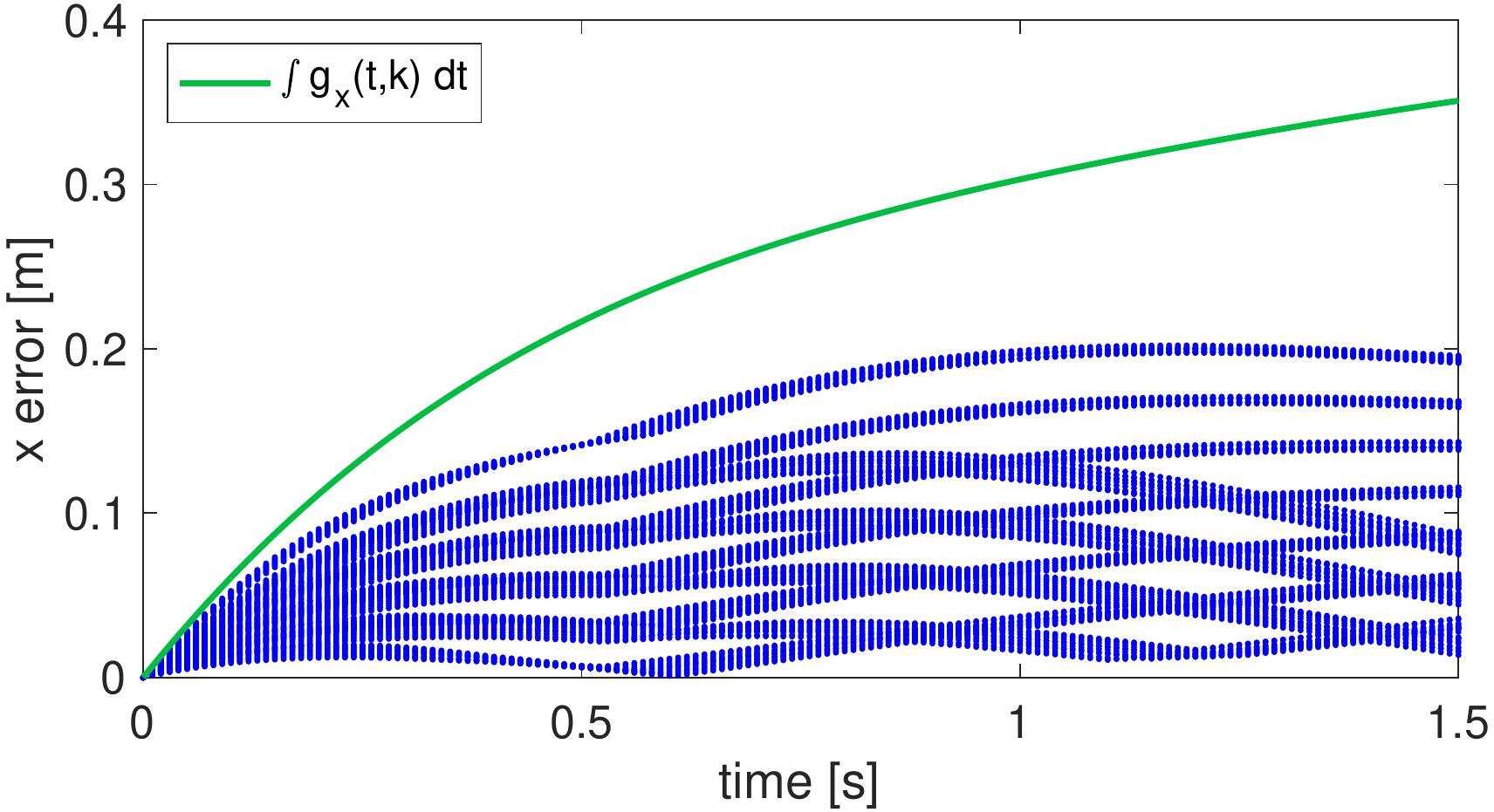}
        \caption{\centering}
        \label{subfig:x_error_rover_hi_speed}
    \end{subfigure}
    \begin{subfigure}[t]{0.47\textwidth}
        \centering
        \includegraphics[width=0.9\textwidth]{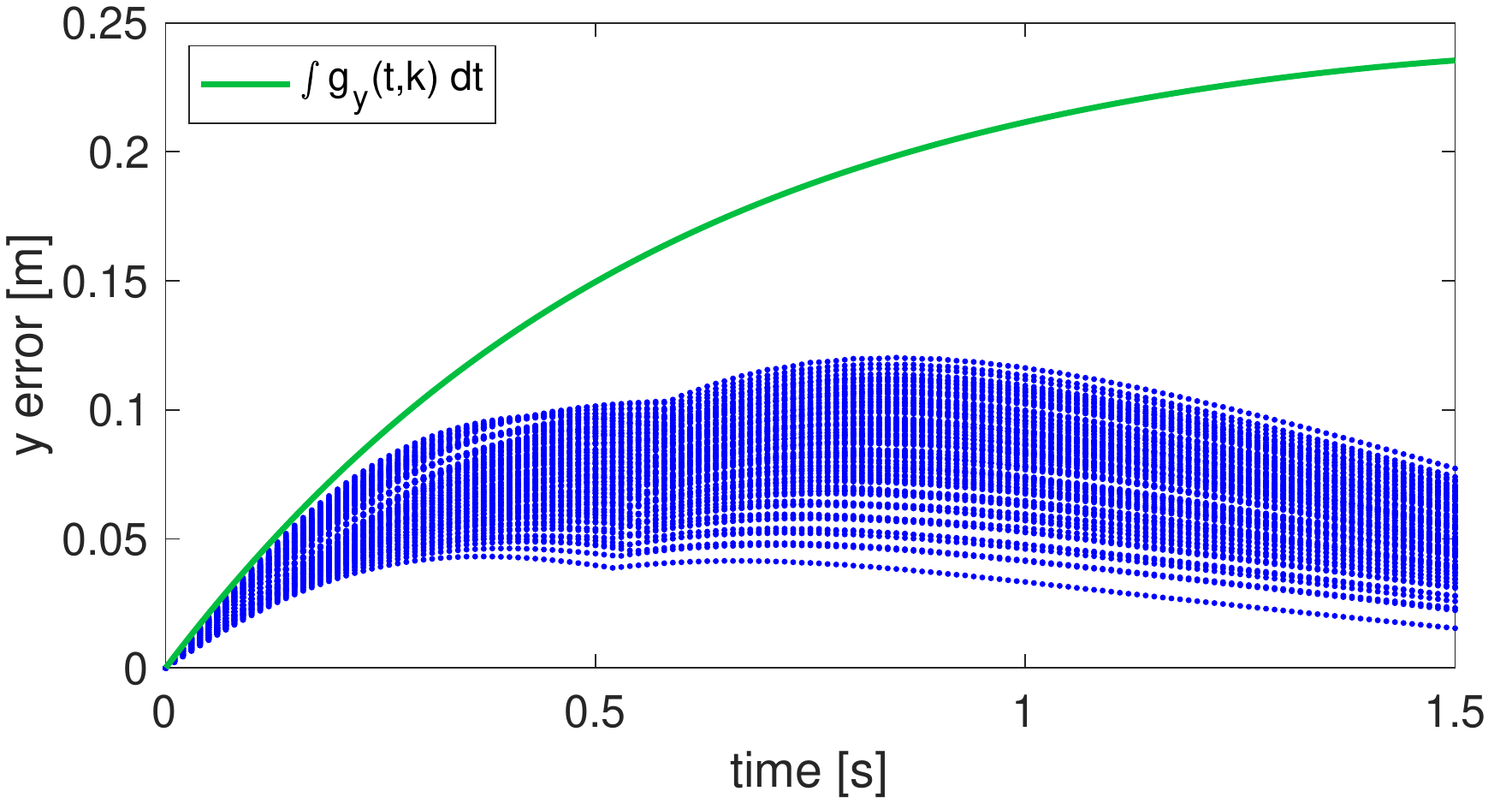}
        \caption{\centering}
        \label{subfig:y_error_rover_hi_speed}
    \end{subfigure}
    \caption{Error in $x$ (Figure \ref{subfig:x_error_rover_hi_speed}) and $y$ (Figure \ref{subfig:y_error_rover_hi_speed}) of the Rover's high-fidelity model \eqref{eq:high-fidelity_rover} when tracking a lane change maneuver generated as in Example \ref{ex:traj-producing_rover} in Section \ref{sec:system_decomp}.
    The parameterized trajectory has a velocity of 1.94 m/s and initial yaw rate of 0.95 rad/s.
    Error is the expression $\left|\z\hii(t;\z\hio,k) - \z_i(t;\z_0,k)\right|$  in Assumption \ref{ass:tracking_error_fcn_g}, where $i$ selects the $x$ and $y$ components of $\z\hi$ and $\z$.
    The blue dashed lines are example error trajectories created by sampling possible initial conditions.
    The green solid lines represent the functions $g_x$ and $g_y$, which conservatively bound all of the error trajectories.
    }
    \label{fig:g_fit_rover_error}
\end{figure}

\noindent 
We also note that we can construct a $g$ function that satisfies Assumption \ref{ass:tracking_error_fcn_g} (i.e bounds the tracking error) for the actual robot.
This can be done by checking to see if $g$ satisfying \eqref{eq:error_func} is conservative, and/or augmenting it with a positive function as described in Remark \ref{rem:trackin_error_bound_existance}.
Additionally, if the robot's footprint, $X_0$, is not a circle, error in the rigid body dynamics between the center of mass and other points on the footprint can also be bounded by $g$, by observing that the yawrate in \eqref{eq:rigid_body_dynamics} is bounded via Assumption \ref{ass:max_speed_and_yaw_rate}.

We now relate these types of error to understand how to ensure safety of the actual robot.

\begin{rem}\label{rem:error_relationship}
By Assumptions \ref{ass:predict} and \ref{ass:tracking_error_fcn_g}, while executing a trajectory parameterized by $k\in K$, every point on the actual robot's body lies within a box of size $\vep_x \times \vep_y$ of the same point on the robot's body described by using the high-fidelity model.
Therefore we can expand obstacles by $\pm\vep_x$ and $\pm\vep_y$ in the $x$- and $y$-directions respectively, to account for the gap between the high fidelity model and our actual robot. Note that the buffers $\pm\vep_x$ and $\pm\vep_y$ pertain to the hardware demonstrations; the simulations in Section \ref{sec:simulation_results} use the high-fidelity model to describe the robots motion.
\end{rem}

\noindent By Remark \ref{rem:error_relationship}, to make the actual robot safe, we ensure that the high-fidelity model is safe while planning with the trajectory-producing model \eqref{eq:traj-producing_model}.
We explicitly define the ``expanded obstacle'' in Section \ref{sec:pers_feas}.
Expanding an obstacle alone does not ensure safety, because we still have to represent the obstacle in a way that can be used to plan trajectories that avoid it at runtime; we address this in Section \ref{sec:obstacle_representation}.
Moreover, we still have to compensate for the tracking error that results from planning trajectories for the high-fidelity model \eqref{eq:high-fidelity_model} with the trajectory-producing model \eqref{eq:traj-producing_model}.
To relate these two models, we introduce a ``trajectory-tracking model'' as follows.

Let $L_d := L^1([0,T], [-1,1]^{n_Z})$ denote the space of absolutely integrable functions from $[0,T]$ to $[-1,1]^{n_Z}$.
We include $g$ in the trajectory-producing dynamics to create the \defemph{trajectory-tracking model} with dynamics:
\begin{align}
\dot{\z}(t,\z(t),k,d) = f(t,\z(t),k) + g(t,\z(t),k)\circ d(t) \label{eq:traj_tracking_model}
\end{align}
where we have reused the state $\z$ of the trajectory-producing model to emphasize that the trajectory-tracking model trajectories evolve in the state space $Z$.
Here, $d \in L_d$, so $d(t) \in [-1,1]^{n_Z}$ almost everywhere $ t \in [0,T]$.
Recall that $\circ$ denotes the Hadamard product.
Note, $d$ can be chosen to describe worst-case error behavior.
Similarly, we can use $d$ to make the trajectory-producing model ``match'' the high-fidelity model in the shared states:
\begin{lem}\label{lem:traj_prod_matches_hi_fid_model}
Suppose $\z\hio \in Z\hio$, $k \in K$, and $\z_0 = \proj_Z(\z\hio)$.
Then there exists $d \in L_d$ such that the high-fidelity model and the trajectory-tracking model agree on the shared state space $Z$, i.e.,
\begin{align}
\proj_Z\left(\z\hi(t;\z\hio,k)\right) = \z(t;\z_0,k)
   \label{eq:hi_and_tracking_trajs_match_using_d}
\end{align}
for all $t \in [0,T]$, where $\z\hi$ (resp. $\z$) is a trajectory produced by \eqref{eq:high-fidelity_model} (resp. \eqref{eq:traj_tracking_model}) with initial condition $\z\hio$ (resp. $\z_0$).
\end{lem}

\begin{proof}
From Assumption \ref{ass:tracking_error_fcn_g} and \eqref{eq:error_func}, recall that $g$ bounds the maximum possible absolute error in each shared state for all $t \in [0,T]$.
Therefore, almost everywhere $t \in [0,T]$, we can pick $d(t) \in [-1,1]^{n_Z}$ such that
\begin{align}
    \proj_Z\left(\z\hi(\cdot)\right) - \left(\int_0^t f(\tau,\cdot)d\tau + \z_0\right) = \int_0^t g(\tau,\cdot)\circ d(\tau)d\tau,\label{eq:projZ_fhi_minus_f_equals_g}
\end{align}
where arguments to $\z\hi$, $f$, and $g$ are dropped for readability.
Rearrange \eqref{eq:projZ_fhi_minus_f_equals_g} and use \eqref{eq:traj_tracking_model} to fulfill \eqref{eq:hi_and_tracking_trajs_match_using_d}.
\end{proof}

We conclude this section with an example tracking error function for the Segway.
\begin{ex}\label{ex:segway_traj_tracking_model}
The Segway's PD controller in Example \ref{ex:segway_feedback_controller} cannot perfectly drive the high-fidelity model (Example \ref{ex:segway}) to the Dubins paths of the trajectory-producing model (Example \ref{ex:segway_traj_producing_model}).
We represent the tracking error between the high-fidelity and trajectory-producing models as a function $g: [0,T] \to \R^2$ given by
\begin{align}
    g(t) = \begin{bmatrix}
    g_x(t) \\ g_y(t)
    \end{bmatrix},
\end{align}
where $g_x \in \R_4[t]$ and $g_y \in \R_3[t]$.
That is, the tracking error functions are time-varying polynomials of degree $4$ for $x$ and degree $3$ for $y$.
\end{ex}

\noindent See Figure \ref{fig:g_fit_segway_error} for example tracking error functions $g_x$ and $g_y$.

\begin{figure}
    \centering
    \begin{subfigure}[t]{0.47\textwidth}
        \centering
        \includegraphics[width=0.9\textwidth]{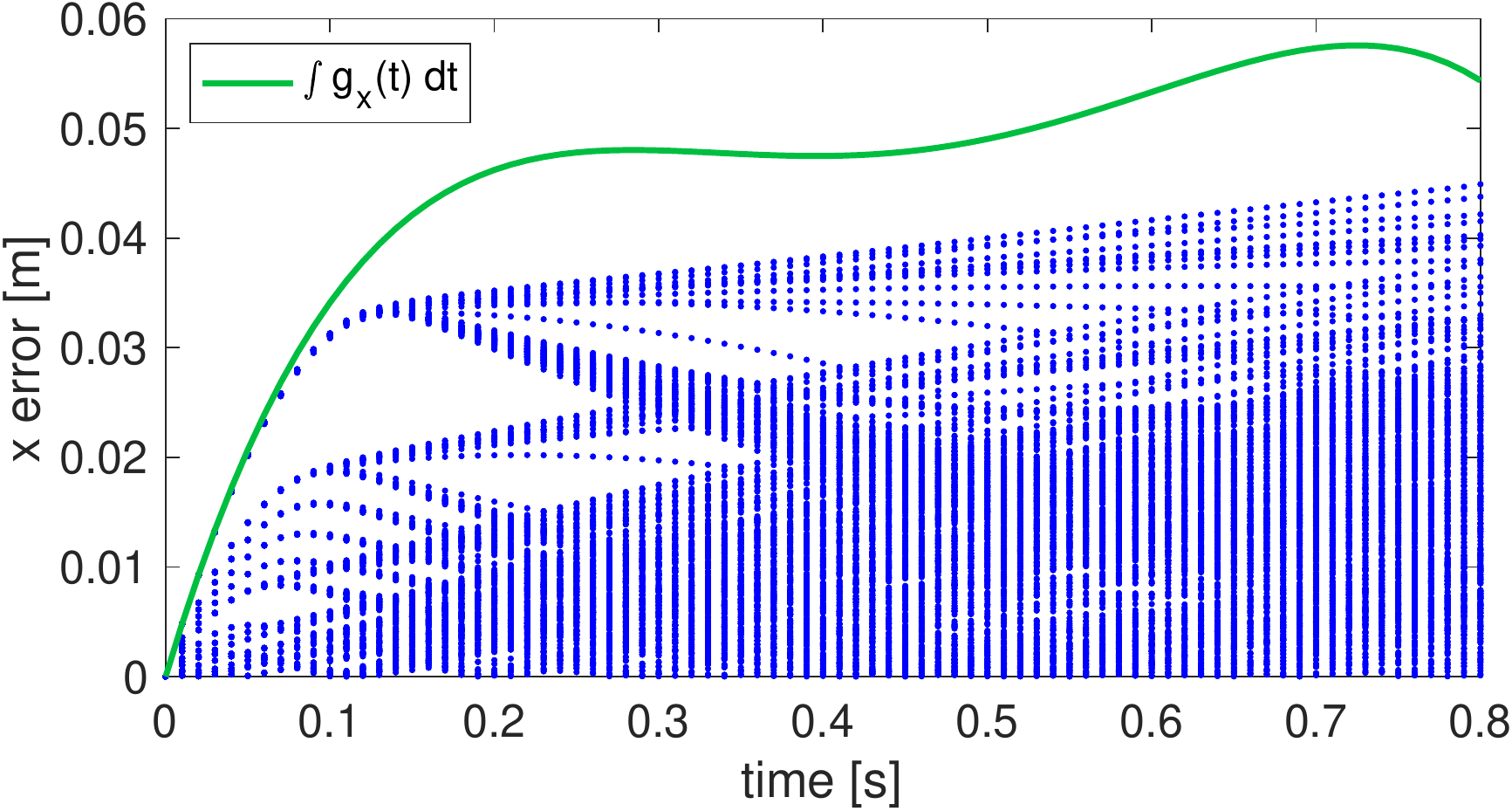}
        \caption{\centering}
        \label{subfig:x_error_segway_hi_speed}
    \end{subfigure}
    \begin{subfigure}[t]{0.47\textwidth}
        \centering
        \includegraphics[width=0.9\textwidth]{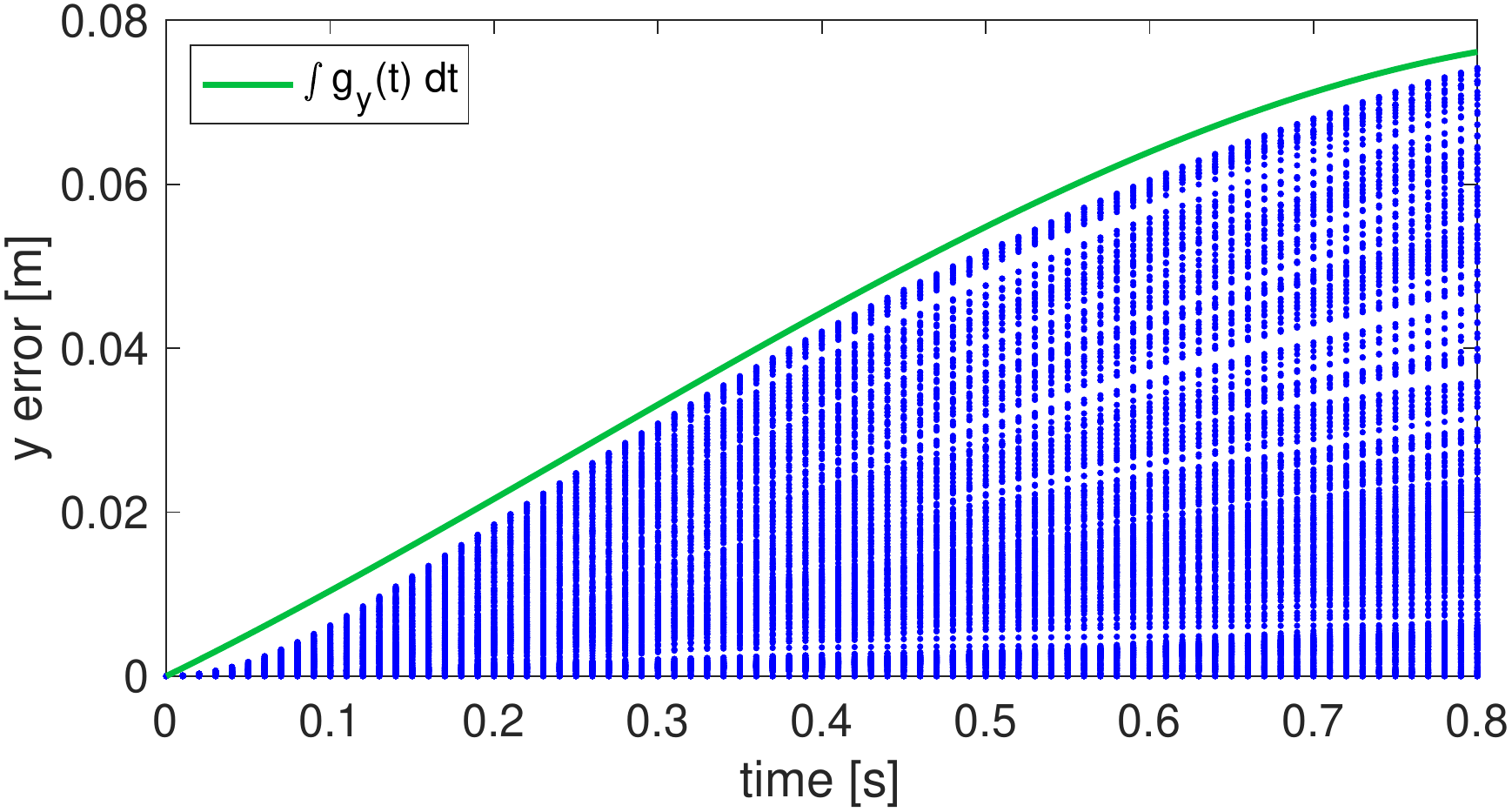}
        \caption{\centering}
        \label{subfig:y_error_segway_hi_speed}
    \end{subfigure}
    \caption{
    Error in $x$ (Figure \ref{subfig:x_error_segway_hi_speed}) and $y$ (Figure \ref{subfig:y_error_segway_hi_speed}) of the Segway's high-fidelity model \eqref{eq:high-fidelity_segway} when tracking Dubins paths generated as in Example \ref{ex:segway_traj_tracking_model}.
    The robot has a maximum yaw rate of $1$ rad/s and a maximum speed of 1.5 m/s.
    The blue dashed lines are example error trajectories created by sampling possible initial conditions.
    The  solid lines represent the functions $g_x$ and $g_y$, which bound all of the error trajectories, as described in Assumption \ref{ass:tracking_error_fcn_g}.}
    \label{fig:g_fit_segway_error}
\end{figure}

Next, in Section \ref{sec:FRSmethod}, we use the trajectory-tracking model to precompute a Forward Reachable Set (FRS) of the robot.
We use the FRS at runtime to identify unsafe trajectories at runtime in Section \ref{sec:trajectory_optimization}.
\section{Forward Reachable Set Computation}
\label{sec:FRSmethod}

The Reachability-based Trajectory Design (RTD) method for provably safe planning has two steps:
\begin{enumerate}
    \item Precompute a \defemph{Forward Reachable Set} (FRS) that captures all possible trajectories and associated parameters over a time interval $[0,T]$.
    \item Perform \defemph{trajectory optimization} online, with a user-specified cost function, to select trajectory parameters $k\in K$ that are \defemph{safe}, meaning if the robot follows a trajectory parameterized by $k$, it will not collide with an obstacle.
\end{enumerate}
The method is illustrated in Figure \ref{fig:FRS_cartoon}.
This section describes the FRS precomputation step, first introduced in \citep{kousik2017safe}.
We discuss how to choose the time horizon $T$ to ensure persistent feasibility in Section \ref{sec:pers_feas}.

Note that \citet{kousik2017safe} introduced an additional step called \defemph{set intersection}, which is performed during online operation before the trajectory optimization step.
Set intersection uses SOS programming to intersect obstacles, represented as semi-algebraic sets, with the FRS, resulting in a polynomial approximation of the safe set of trajectory parameters.
A brief discussion of set intersection is included in Appendix \ref{app:set_intersection}.
This intersection procedure is found to be too slow for real-time trajectory planning with 1-D obstacles (see Section V in \cite{kousik2017safe} for more details), which motivates the discrete obstacle representation presented in Section \ref{sec:obstacle_representation}.

This section proceeds as follows.
Section \ref{subsec:FRS_computation} formalizes the FRS and poses an infinite-dimensional program over continuous functions to compute it.
Section \ref{subsec:FRS_implementation} presents a Sums-Of-Squares (SOS) program to conservatively approximate solutions to the infinite-dimensional program.
Section \ref{subsec:FRS_memory_usage} discusses the memory required to implement the SOS program, which motivates the system decomposition approach in Section \ref{sec:system_decomp}.
The casual reader can examine \eqref{eq:frs_def} and Lemma \ref{lem:w_geq_1_on_frs} to understand the primary results of this section.

\begin{figure}
    \centering
    \includegraphics[width=0.95\columnwidth]{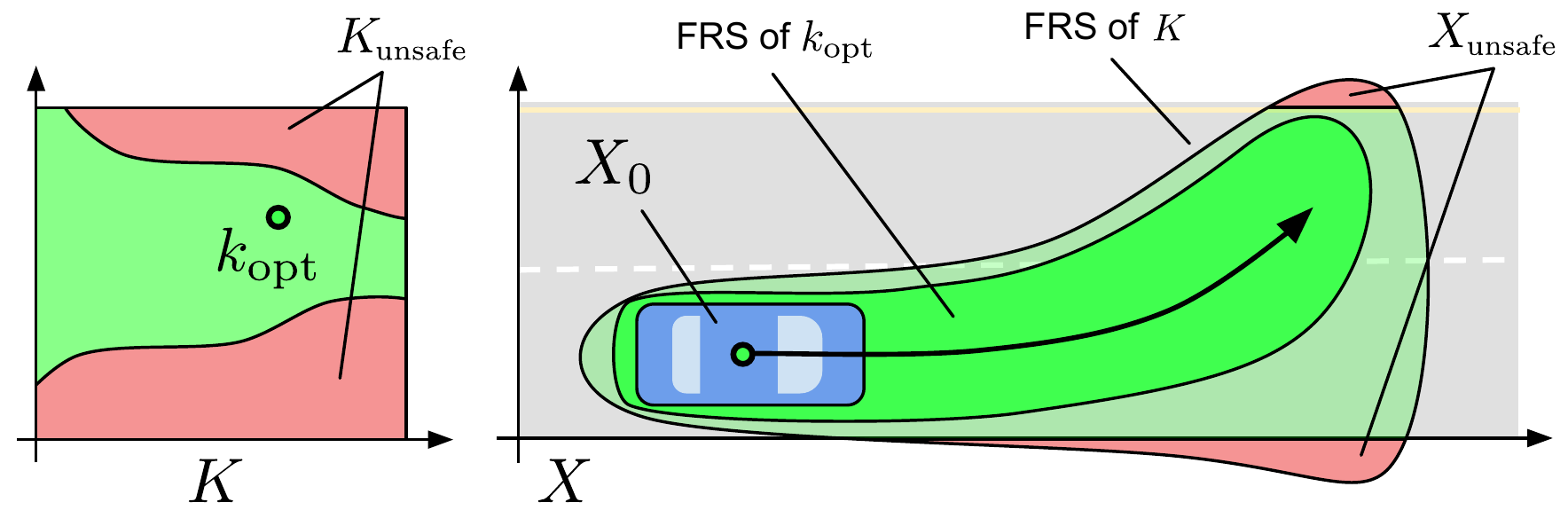}
    \caption{An illustration of the general approach to trajectory planning with RTD.
    One receding-horizon planning iteration is shown.
    The autonomous robot's trajectory parameter space $K$ is on the left, and the $xy$-subspace $X$ of its state space is on the right.
    The bell-shaped contour in $X$ shows the total extent of the forward reachable set (FRS), corresponding to the entirety of $K$.
    In $X$, areas are labeled as unsafe, corresponding to the labeled sets of trajectory parameters on the left.
    A safe parameter $k_\regtext{opt}$ is selected in $K$, and the corresponding trajectory (the arrow) and corresponding subset of the FRS (the contour around the arrow) in $X$ are shown on the right.
    }
    \label{fig:FRS_cartoon}
\end{figure}

\subsection{Problem Formulation}\label{subsec:FRS_computation}

The FRS contains positions in the $xy$-subspace $X$ that are reachable by a robot described by the high-fidelity model \eqref{eq:high-fidelity_model} while tracking trajectories produced by the trajectory-producing model \eqref{eq:traj-producing_model}, despite tracking error, over a time horizon $T$.
Formally, we define the FRS:
\begin{equation}\label{eq:frs_def}
\begin{split}
\X\frs = \Big\{&(x,k) \in\ X \times  K\; |\; \exists~ \z_0 \in Z_0, \tau \in [0,T],\\
&\mathrm{and}~d \in L_d~\mathrm{s.t.}~x = \idx(\tilde{\z}(\tau)),\\
&\mathrm{where}~\dot{\tilde{\z}}(t) = f(t,\tilde{\z}(t),k)+ g(t,\tilde{\z}(t),k)\circ d(t)\\
&\mathrm{a.e.}~t \in [0,T]~\mathrm{and}~\tilde{\z}(0) = \z_0 \Big\}.
\end{split}
\end{equation}
Recall that $L_d := L^1([0,T], [-1,1]^{n_Z})$ denotes the space of absolutely integrable functions from $[0,T]$ to $[-1,1]^{n_Z}$.

To understand the gap between the high-fidelity model and the trajectory-producing model, we rely upon a pair of linear operators, $\Lf, \Lg: AC\big([0,T] \times Z \times K \big) \to C\big( [0,T] \times Z \times K \big)$ which act on a test function $v$ as follows:
\begin{align}
\Lf v(t,\z,k) &= \frac{\partial v}{\partial t}(t,\z,k) + \sum_{i = 1}^{n} \frac{\partial v}{\partial \z_i}(t,\z,k) f_i(t,\z,k)\\
\Lg v(t,\z,k)  &= \sum_{i = 1}^{n} \frac{\partial v}{\partial \z_i}(t,\z,k) g_{i}(t,\z).
\end{align}
With these operators, we can compute the FRS by solving the following linear program, adapted from \citet[Section 3.3, Program $(D)$]{majumdar2014convex}.
The program has been altered for forward reachability, for uncertainty propagation in place of control synthesis, and to restrict the decision variable $w$ to $X$.
% \begin{flalign}
% 		& & \underset{v,w,q}{\text{inf}} \hspace*{0.25cm} & \int_{X \times K} w(x,k) ~ d\lambda_{X \times K} &&&& (D) \nonumber \\
% 		& & \text{s.t.} \hspace*{0.25cm} & \Lf v(t,\z,k) + q(t,\z,k) \leq 0, && ~\forall~(t,\z,k) \in [0,T] \times Z \times K && (D1) \nonumber\\
%         & & & \Lg v(t,\z,k) + q(t,\z,k) \geq 0, && ~\forall~(t,\z,k) \in [0,T] \times Z \times K  && (D2) \nonumber \\
%         & & & -\Lg v(t,\z,k) + q(t,\z,k) \geq 0, && ~\forall~(t,\z,k) \in [0,T] \times Z \times K  && (D3) \nonumber \\
%         & & & q(t,\z,k) \geq 0, && ~\forall~(t,\z,k) \in [0,T] \times Z \times K  && (D4) \nonumber \\
%         & & & -v(0,\z,k) \geq 0, && ~\forall~(\z,k) \in Z_0 \times K  && (D5) \nonumber \\
%         & & & w(x,k) \geq 0,&& ~\forall~(x,k) \in X \times K && (D6) \nonumber \\
%         & & & w(x,k) + v(t,\z,k) - 1 \geq 0, && ~\forall~(t,\z,k) \in [0,T] \times Z \times K && (D7) \nonumber
% \end{flalign}
\begin{flalign}
		& & \underset{v,w,q}{\text{inf}} \hspace*{0.25cm} & \int_{X \times K} w(x,k) ~ d\lambda_{X \times K} && \tag{$D$}\label{prog:sos_frs_inf_dim}\\
		& & \text{s.t.} \hspace*{0.25cm} & \Lf v(t,\z,k) + q(t,\z,k) \leq 0, && (D1) \nonumber\\
        & & & \Lg v(t,\z,k) + q(t,\z,k) \geq 0, && (D2) \nonumber \\
        & & & -\Lg v(t,\z,k) + q(t,\z,k) \geq 0, && (D3) \nonumber \\
        & & & q(t,\z,k) \geq 0, && (D4) \nonumber \\
        & & & -v(0,\z,k) \geq 0, && (D5) \nonumber \\
        & & & w(x,k) \geq 0, && (D6) \nonumber \\
        & & & w(x,k) + v(t,\z,k) - 1 \geq 0, && (D7) \nonumber
\end{flalign}
where $x = \idx(\z)$.
Constraints $(D1), (D2), (D3), (D4)$, and $(D7)$ apply for all $(t,\z,k) \in [0,T] \times Z \times K$.
Constraint $(D5)$ applies for all $(\z,k) \in Z_0 \times K$.
Constraint $(D6)$ applies for all $(x,k) \in X \times K$.
The given data for the problem are $f,\ g,\ Z,\ Z_0,\ K,\ \text{and } T$. The infimum is taken over $(v,w,q)\in C^1([0,T]\times Z\times K)\times C(X\times K)\times C([0,T]\times Z\times K)$. 
Theorem 3 of \citet{majumdar2014convex} shows that feasible solutions to \eqref{prog:sos_frs_inf_dim} conservatively approximate $\X\frs$.
We adapt this result in the following lemma.

\begin{lem}\label{lem:v_is_negative}
If $(v,w,q)$ satisfies the constraints in \eqref{prog:sos_frs_inf_dim}, then $v$ is non-positive and decreasing along trajectories of the trajectory-tracking system \eqref{eq:traj_tracking_model}.
In other words, let $\z\in Z$ and $x=\proj_X(\z)$; then $(x,k) \in \X\frs$ implies that $\exists~t \in [0,T]$ such that $v(t,\z,k) \leq 0$.
\end{lem}
\noindent The proof is available in Appendix \ref{app:reachability_proofs}.

The result of Lemma \ref{lem:v_is_negative} and constraint $(D7)$ is that the 1-superlevel set of any feasible $w$ contains $\X\frs$ \citep[Lemma 11]{kousik2017safe}.
In fact, the solution to this infinite dimensional linear program allows one to compute $\X\frs$:
\begin{lem}\citep[Theorem 4]{majumdar2014convex} \label{lem:w_geq_1_on_frs}
Let $(v,w,q)$ be a feasible solution to \eqref{prog:sos_frs_inf_dim}.
The 1-superlevel set of $w$ contains $\X\frs$.
Furthermore, there is a sequence of feasible solutions to \eqref{prog:sos_frs_inf_dim} whose second component $w$ converges from above to an indicator function on $\X\frs$ in the $L^1$-norm and almost uniformly.
\end{lem}

\subsection{Implementation} \label{subsec:FRS_implementation}
We now implement \eqref{prog:sos_frs_inf_dim} using SOS programming as in our prior work \citep{kousik2017safe}.
To do so, we require the following assumptions:
\begin{assum}\label{ass:f_and_g_are_poly}
The functions $f$ and $g$ are polynomials of finite degree in $\R[t,\z,k]$.
\end{assum}
\noindent Note that, to fulfill Assumption \ref{ass:tracking_error_fcn_g}, if $f$ is Taylor-expanded to be a polynomial, then $g$ must bound the error introduced by the Taylor-expansion.

\begin{assum} \label{ass:set}
The sets $K,\ Z,$ and $Z_0$ have semi-algebraic representations:
\begin{align}
K &= \left\{ k \in \R^{n_K} \mid h_{K_i}(k) \geq 0, ~\forall\ i = 1,...,n_{K} \right\} \label{eq:ass_2_1}\\
Z &= \left\{ \z \in \R^{n_Z} \mid h_{Z_i}(\z) \geq 0, ~\forall\ i = 1,...,n_{Z} \right\} \label{eq:ass_2_3} \\
Z_0 &= \left\{ \z \in Z \mid h_{0_i}(\z) \geq 0, ~\forall\ i = 1,...,n_{0} \right\}  \label{eq:ass_2_2}
\end{align}
where $h_{K_i} \in \R[k]$ and $h_{0_i}, h_{Z_i} \in \R[\z]$.
Since $X$ and $X_0$ are projected (as in Definition \ref{def:projection_operators}, using $\idx$) from semi-algebraic sets, they can also be represented semi-algebraically:
\begin{align}
    X &= \{p \in \R^2 \mid h_x(p) \geq 0,\ h_y(p) \geq 0\} \\
    X_0 &= \{p \in X \mid h_{x_0}(p) \geq 0,\ h_{y_0}(p) \geq 0\}.
\end{align}
\end{assum}

Assumption \ref{ass:set} is not prohibitive since common boxes and ellipses, and even non-convex sets, have semi-algebraic representations (see, e.g., \citet{majumdar2014convex}).
Typically, $Z$, $X_0$, and $Z_0$ are box- or ellipse-shaped.
The parameter space $K$ can be represented by a box or ellipse; more complex restrictions of the parameters can be enforced in the online optimization program described in Algorithm \ref{alg:trajopt}.
Also note that, for the SOS program posed next, we require that there exists $N \in \N$ such that, for any $q = (t,\z,\z_0,k) \in [0,T]\times Z\times Z_0 \times K$, the value of $N - \norm{q}_2^2 \geq 0$ \citep[Theorem 2.15]{lasserre2009moments}.
This is trivially satisfied since $[0,T]$, $Z,\ Z_0,$ and $K$ are compact by Assumption \ref{ass:state_and_control_sets_are_compact}.

\subsubsection{Computing the FRS}\label{subsubsec:compute_the_FRS}
To solve \eqref{prog:sos_frs_inf_dim}, we follow the implementation from Section 4 of \cite{kousik2017safe}.
We construct a sequence of convex SOS programs indexed by $l \in \N$ by relaxing the continuous function in \eqref{prog:sos_frs_inf_dim} to polynomial functions with degree truncated to $2l$. 
The inequality constraints in \eqref{prog:sos_frs_inf_dim} then transform into SOS constraints, so \eqref{prog:sos_frs_inf_dim} becomes a Semi-Definite Program (SDP) \citep{parrilo2000structured}.
To formulate this problem, let $h_T = t(T - t)$, and $H_T = \{h_T\}$.
Recalling the definitions in Assumption \ref{ass:set}, for $Z$, $Z_0$, and $K$, collect the polynomials that represent them in the sets $H_Z = \left\{h_{Z_1},\ldots,h_{Z_{n_{Z}}}\right\}$, $H_{Z_0} = \left\{h_{0_1},\ldots,h_{0_{n_0}}\right\}$, $H_K = \left\{h_{K_1},\ldots,h_{K_{n_K}}\right\}$, and $H_X = \left\{h_x, h_y\right\}$.

Let $Q_{2l}(H_T,H_Z,H_K) \subset \R_{2l}[t,\z,k]$ be the set of polynomials $p \in \R_{2l}[t,\z,k]$ (i.e., of total degree less than or equal to $2l$) expressible as:
\begin{equation}\label{eq: general sos}
  p = s_0 + s_1 h_T + \sum_{i=1}^{n_{Z}} s_{i+2} h_{Z_i} + \sum_{i=1}^{n_{K}} s_{i+{n_{Z}}+2} h_{K_i},
\end{equation}
for some polynomials $\{s_i\}_{i=0}^{n_{Z}+n_{K}+1} \subset \R_{2l}[t,\z,k]$ that are SOS of other polynomials.
Note that every such polynomial is non-negative on $[0,T] \times Z \times K$ \citep[Theorem 2.14]{lasserre2009moments}.
Similarly, define $Q_{2l}(H_{Z_0},H_K) \subset \R_{2l}[\z,k]$, and $Q_{2l}(H_X,H_K) \subset \R_{2l}[x,k]$, where $x = \idx(\z)$.

Employing this notation, the $l$\ts{th}-order relaxed SOS programming representation of \eqref{prog:sos_frs_inf_dim}, denoted $(D^l)$, is defined as follows:
\begin{align} 
		& & \underset{v^l,w^l,q^l}{\text{inf}} \hspace*{0.25cm} & y_{X \times K}^T \textrm{vec}(w^l) &&&& (D^l) \nonumber \\
		& & \text{s.t.} \hspace*{0.25cm} & -\Lf v^l - q^l &&\in Q_{2d_f}(H_T,H_Z,H_K) && (D^l1)\nonumber\\
        & & & \Lg v^l+ q^l &&\in Q_{2d_g}(H_T,H_Z,H_K) && (D^l2)\nonumber \\
        & & & -\Lg v^l + q^l&&\in Q_{2d_g}(H_T,H_Z,H_K)  && (D^l3)\nonumber \\
        & & & q^l &&\in Q_{2l}(H_T,H_Z,H_K)  && (D^l4)\nonumber \\
        & & & -v^l(0,\;\cdot\,) &&\in Q_{2l}(H_{Z_0},H_K) && (D^l5)\nonumber \\
        & & & w^l &&\in Q_{2l}(H_X,H_K) && (D^l6)\nonumber \\
        & & & w^l + v^l - 1 &&\in Q_{2l}(H_T,H_Z,H_K), && (D^l7)\nonumber 
\end{align}
where the infimum is taken over the vector of polynomials $(v^l,w^l,q^l) \in \R_{2l}[t,\z,k] \times \R_{2l}[x,k] \times \R_{2l}[t,\z,k]$, with $x = \idx(\z)$.
The vector $y_{X\times K}$ contains moments associated with the Lebesgue measure $\lambda_{X\times K}$, so $\int_{X\times K} w(x,k)d\lambda_{X\times K} = y_{X\times K}^T \mathrm{vec}(w)$ for $w \in \R_{2l[x,k]}$ \citep{majumdar2014convex}.
The numbers $d_f$ and $d_g$ are the smallest integers such that $2d_f$ and $2d_g$ are respectively greater than the total degree of $\Lf v^l$ and $\Lg v^l$.
To implement $(D^l)$, we consider the dual program, which is an SDP \citep{lasserre2009moments}.

\begin{rem}\label{rem:w_geq_1_is_in_FRS}
It can be shown that Lemma \ref{lem:v_is_negative} holds for functions that satisfy the constraints of $(D^l)$ \citep[Theorem 6]{majumdar2014convex}.
Additionally, one can apply the last constraint in $(D^l)$ to prove that the 1-superlevel set of any feasible $w^l$ is an outer approximation to $\X\frs$ \citep[Theorem 7]{majumdar2014convex}.
Furthermore, one can prove that $w^l$ converges from above to an indicator function on $\X\frs$ in the $L^1$-norm \citep[Theorem 6]{majumdar2014convex}.
\end{rem}

\subsection{Sums-of-Squares Memory Usage}\label{subsec:FRS_memory_usage}

We now inspect the memory required to implement \eqref{prog:sos_frs_inf_dim}.
For higher-dimensional systems, the large memory requirement motivates the system decomposition approach in Section \ref{sec:system_decomp}.

In this paper, the FRS is computed with Spotless \citep{tobenkin2013spotless}, a MATLAB-based SOS toolbox.
Spotless transforms the SOS optimization program into an SDP, which is solved with MOSEK \citep{mosek2010mosek}.
As the degree $l$ increases, the approximation of the FRS becomes a provably less conservative outer approximation of $\X\frs$ \citep[Theorem 7]{majumdar2014convex}.

However, solving $(D^l)$  is memory intensive, as the monomials of each polynomial are free variables; and a polynomial of degree $2l$ and dimension $n$ has $\binom{2l+n}{n}$ monomials.
The memory required by $(D^l)$ grows as $\mc{O}((n+1)^l)$ for fixed $l$ and $\mc{O}(l^{n+1})$ for fixed $n$ \citep[Section 4.2]{majumdar2014convex}.
Furthermore each free variable is stored as a 64-bit double, and MOSEK computes the Hessian of each SOS constraint \citep[Section 11.4]{mosek2010mosek}, which is proportional to the number of free variables squared (see, e.g., \citet[Chapter 14]{nocedal2006numericaloptimization}.

To estimate the amount of free variables generated for Program $(D^l)$, one can sum up the monomials in each decision variable polynomial.
These consist of the polynomials $v^l$, $w^l$, $q^l$ which are degree $2l$, and the $s$ polynomials for each semialgebraic set, defined in \eqref{eq: general sos}, whose degree are specified in $(D^l)$.

The Segway trajectory-tracking model (in Examples \ref{ex:segway_traj_producing_model} and \ref{ex:segway_traj_tracking_model}) has dimension $5$.
Solving $(D^5)$ for this system requires approximately $1.4\times 10^5$ free variables.
The highest dimension system for which we have computed reachable sets for is the car-like rover, shown in Figure \ref{subfig:rover_time_lapse}, which is described in the following section.
It has a $7$-dimensional state space model.
The program $(D^3)$ for this system requires approximately $1.1\times 10^5$ free variables, and MOSEK used 504 GB of RAM (i.e., memory).
We were unable to solve $(D^4)$, which has approximately $3.8\times 10^5$ free variables, on a computer with 3.5 TB of RAM.

This drastic increase in memory required as dimension increases motivates the system decomposition approach in Section \ref{sec:system_decomp}, where we can solve separate, lower dimension FRS computations, then combine them into the full dimension system with a separate SOS program.
\section{System Decomposition}
\label{sec:system_decomp}

Motivated by the memory issues presented in Section \ref{subsec:FRS_memory_usage}, this section presents a method to apply the FRS computation to higher-dimensional systems.
This section is broken into four parts.
First, we introduce the Rover robot as an example system for which the method presented in Section \ref{sec:FRSmethod} is intractable (Section \ref{subsec:decomp_rover_example}).
Second, we present a system decomposition that makes computing an FRS tractable for each subsystem of the decomposed system (Section \ref{subsec:self_contained_subsystems}).
Third, we present how to reconstruct an FRS of the full system from FRS's computed for the subsystems (Section \ref{subsec:FRS_reconstruction}).
Fourth, we explain how to implement the FRS reconstruction using a SOS program (Section \ref{subsec:FRS_reconstruction_implementation}).

The casual reader can examine Example \ref{ex:traj-producing_rover}, Definition \ref{def: scs}, Example \ref{ex:rover_self-contains_subsystems}, and Theorem \ref{thm:compute_FRS} to understand the primary results from this section.

As a further motivation, consider Example \ref{ex:segway_traj_producing_model}, where the trajectory-producing model creates arcs with constant speed and yaw rate.
In some applications, it is beneficial to plan with more complicated trajectories.
For example a passenger vehicle on a road would plan lane change, lane keeping, and lane return maneuvers, requiring a higher-dimensional model than the one that produces arcs.
As noted in Section \ref{subsec:FRS_memory_usage}, for fixed relaxation degree $l$, increasing the dimension $n$ of the trajectory-producing model increases memory usage of $(D^l)$ as $\mathcal{O}((n+1)^l)$ \citep[Section 4.2]{majumdar2014convex}.

This section adapts a general method from \citet{chen2016exact} and \citet{chen2016journal} for computing backwards reachable sets by system decomposition.
We adapt the method to forward reachability, illustrate how to apply SOS programming, and analyze the memory savings that result from using system decomposition.
This type of decomposition applies when the robot's dynamic model can be split into subsystems of lower dimension.
For example, the Segway and Rover model presented in this work, or the quadrotor model presented in \cite{chen2016exact}. 
Note that the recovery of the exact forward reachable set is not always possible with this approach. 
However, the resulting reachable set is guaranteed to be an overapproximation; hence is useful for the presented application of collision checking.
We focus on the case with two subsystems, though the approach generalizes to any finite number.
As per \citet{chen2016exact} and \citet{chen2016journal}, after separating the system, we compute a reachable set for each subsystem, then reconstruct the reachable set for the full system by intersecting the subsystem reachable sets.

\subsection{An Example System}\label{subsec:decomp_rover_example}

We begin with a practical example of trajectory-producing dynamics with a system dimension that makes computing the FRS intractable as discussed in Section \ref{subsec:FRS_memory_usage}.

\begin{ex}\label{ex:traj-producing_rover}
Recall the Rover from Figure \ref{subfig:rover_time_lapse}.
This robot uses the following bicycle model as the trajectory-producing model \eqref{eq:traj-producing_model}, with states $\z = [x_c,y_c,\theta]^\top \in Z\subset \R^3$ where $x_c$ and $y_c$ track the center of mass as in Remark \ref{rem:points_on_rigid_body_have_same_dyn_as_ctr_of_mass}.
% The trajectory-producing dynamics are:
\begin{flalign}
   \label{eq:traj-producing_rover}
     \frac{d}{dt}\begin{bmatrix}x_c(t)\\y_c(t)\\\theta(t)\\
        \end{bmatrix}
        &=\begin{bmatrix}k_3\cos (\theta(t)) -l_r\omega(t,k) \sin (\theta(t))\\k_3\sin (\theta(t)) +l_r\omega(t,k) \cos (\theta(t))\\
       \omega(t,k) \end{bmatrix}\\
       \omega(t,k)&=-2\frac{T_hk_1 - k_2}{T_h^2}t+k_1\label{eq: lane change yaw rate}
\end{flalign}
where $\omega$ is yaw rate, $k_3$ is longitudinal speed, and $l_r$ is the distance from the robot's rear-wheel to center of mass.
The trajectory parameters, $k \in K \subset \R^{3}$, produce lane change, lane keeping, and lane return maneuvers for an autonomous car driving on a straight road.
$T_h$ is the time required to complete a lane change; $k_1$ determines the initial yaw rate of the trajectory; and $k_2$ is the final heading of the trajectory.

To understand this parameterization, integrate \eqref{eq: lane change yaw rate} over time with initial condition $\theta(0) = 0$ to get the robot's heading:
\begin{flalign}\label{eq: lane change heading}
\theta(t)=-\frac{T_h\,k_1 - k_2}{T_h^2}t^2+k_1t
\end{flalign}
Notice that $k_1$ determines the final lateral displacement of the robot, and setting $k_2$ to the difference between the road and robot heading will create trajectories that align the robot with the road.
Sample maneuvers generated by \eqref{eq:traj-producing_rover} and \eqref{eq: lane change yaw rate} are depicted in Figure \ref{fig:lane_change_parameters}.
This parameterization captures lane change, lane keeping, and lane return maneuvers.
The total dimension of \eqref{eq:traj-producing_rover}, including time, is $n = n_Z + n_K + 1 = 7$, which is intractable for the FRS computation as in Section \ref{subsec:FRS_implementation}.
However, the full system \eqref{eq:traj-producing_rover} is separable into ``self-contained subsystems,'' which we use in this section to compute the FRS of the full system.
\end{ex}

\noindent
The Rover's high-fidelity model (as in \eqref{eq:high-fidelity_model}) and controller for tracking the trajectory-producing model in Example \ref{ex:traj-producing_rover} are presented in Section \ref{subsec:rover_application}.

\begin{figure}
    \centering
   \begin{subfigure}[t]{0.45\textwidth}
        \centering
        \includegraphics[width=\textwidth]{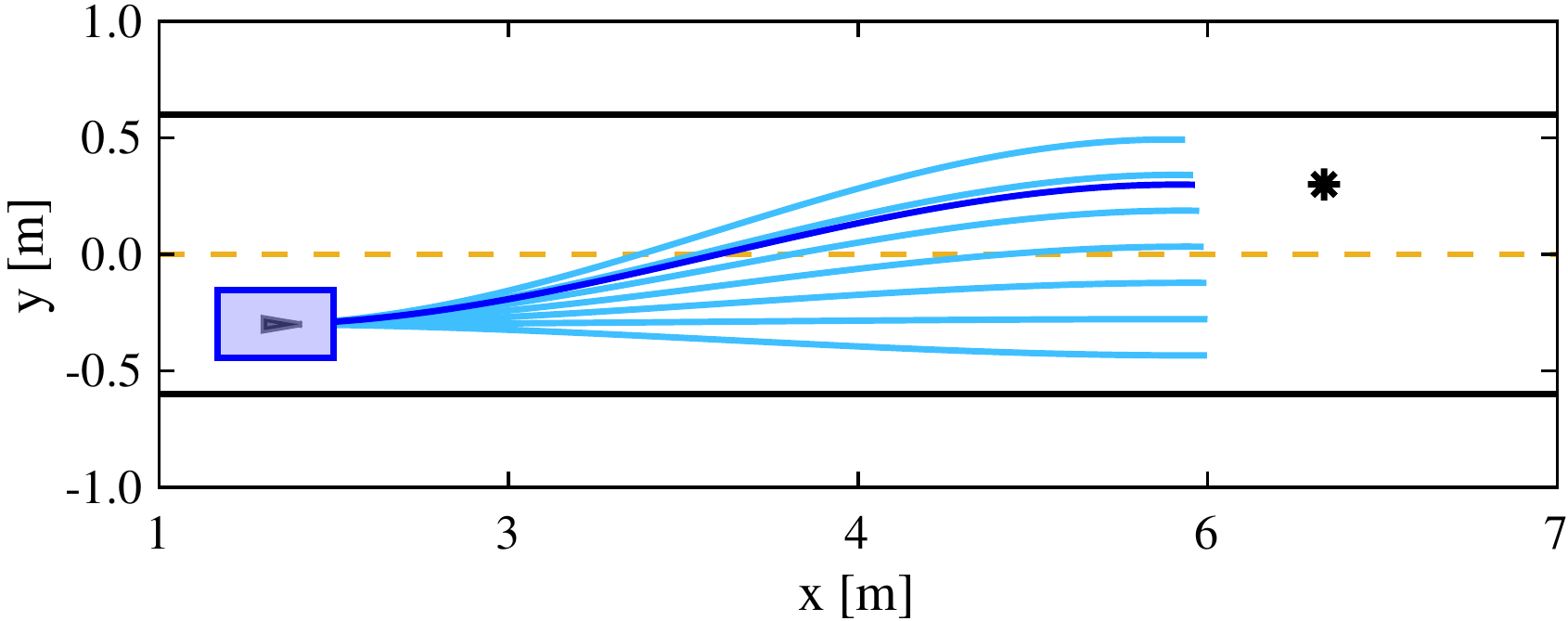}
        \caption{\centering}
        \label{subfig: rover_lane_change_a}
    \end{subfigure}
    \hfill
    \begin{subfigure}[t]{0.45\textwidth}
        \centering
        \includegraphics[width=\textwidth]{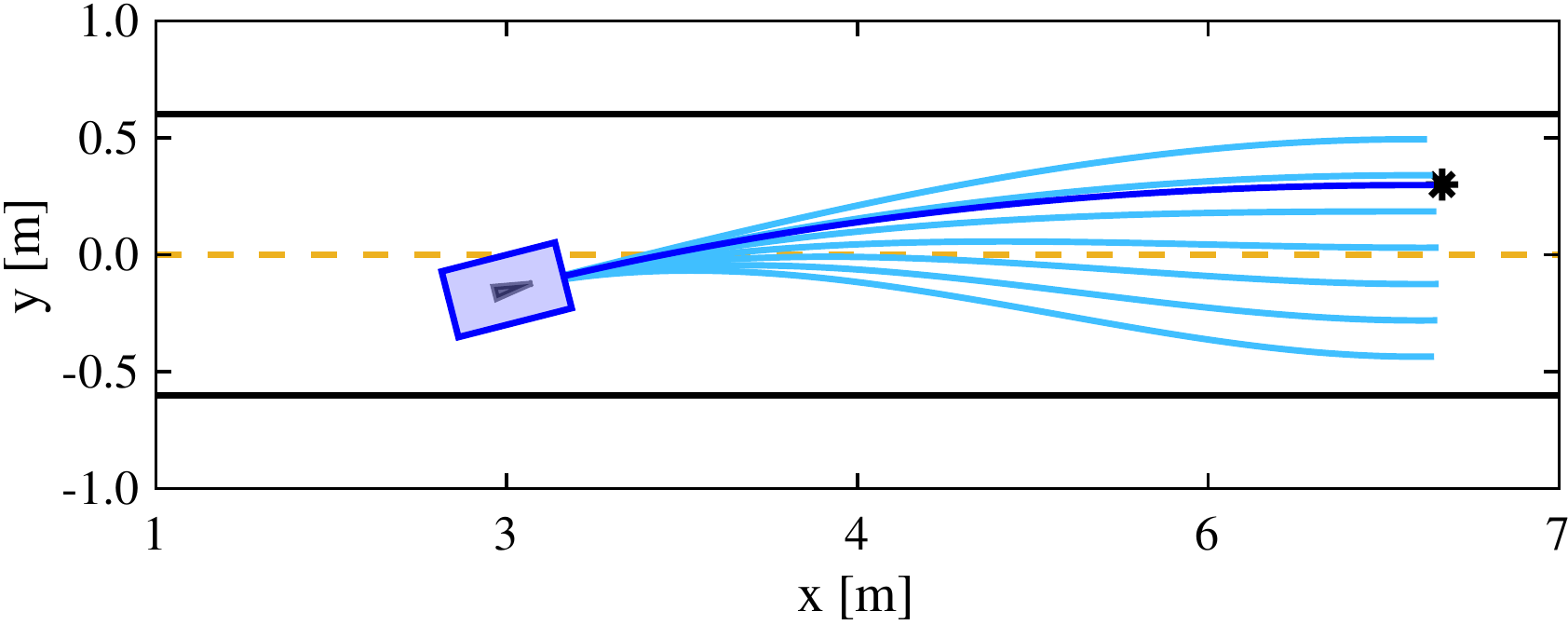}
        \caption{\centering}
        \label{subfig: rover_lane_change_b}
    \end{subfigure}
    \caption{Sample of lane change trajectories generated by \eqref{eq:traj-producing_rover} with the control law in \eqref{eq: lane change yaw rate}.
    The rectangle containing a triangle ``pointer'' represents the Rover and its initial heading. %Initial conditions of $\z_0=(-0.5,-0.3,0.0)$ and $\z_0=(0.4,-0.2,0.3)$ are respectively shown in Subfigures (a) and (b).
    Initial headings of $0.0$ and 0.25 are shown in subfigures (a) and (b), respectively.
    In subfigure (a), the Rover is driving straight in its lane and the sample trajectories consists of lane keeping and lane change maneuvers.
    In subfigure (b), the Rover has begun a lane change, and the sample trajectories consist of lane return maneuvers, and trajectories that complete a lane change.
    The parameters used are $T_h= 2$ s and $k_3= 2$ m/s, with $k_2 = 0.0$ in subfigure (a) and $k_2 = -0.25$ in (b).
    The light trajectories are generated with a sample of values of $k_1$ and plotted over a time horizon of 2 s. 
    The dark trajectory is the optimal trajectory to reach a desired waypoint, shown as an asterisk.}
    \label{fig:lane_change_parameters}
\end{figure}

\subsection{Self-Contained Subsystems}
\label{subsec:self_contained_subsystems}

This section describes how to decompose the trajectory-producing model \eqref{eq:traj-producing_model} into two subsystems.
We follow the methodology introduced by \citet[Section III A]{chen2016journal}, adapting the notation and dynamics to system \eqref{eq:traj-producing_model}, which we refer to as the \defemph{full system}.
Let the state, $\z\in Z$, be partitioned as $\z=(z_1,z_2,z_s)$ with $z_1\in \R^{n_1}$, $z_2\in \R^{n_2}$, $z_s\in \R^{n_s}$, $n_1,n_2>0$, $n_s\geq 0$, and $n_1+n_2+n_s=n_Z$.
Note, we use the notation $(z_1,z_2,z_s)$ as opposed to $[z_1^\top,z_2^\top,z_s^\top]^\top$ in this section for readability.
The states $z_1$ and $z_2$ belong to subsystems 1 and 2, respectively, and the states in $z_s$ belong to both subsystems. 
Therefore, the dynamics can be written:
\begin{flalign}\label{eq: general subsystems}
\begin{split}
    \dot{z}_1(t)&=f_1(t,z_1(t),z_2(t),z_s(t),k)\\
    \dot{z}_2(t)&=f_2(t,z_1(t),z_2(t),z_s(t),k)\\
  \dot{z}_s(t)&=f_s(t,z_1(t),z_2(t),z_s(t),k)\\
    \dot{k}(t)&=0.
    \end{split}
\end{flalign}

Next define the subsystem states and spaces $\z_1=(z_1,z_s)\in Z_1\subset \R^{n_1+n_s}$ and $\z_2=(z_2,z_s)\in Z_2\subset \R^{n_2+n_s}$.
The subspaces $Z_1$ and $Z_2$ are compact and have semi-algebraic representations.
Just as in Assumption \ref{ass:set}, the initial conditions and state space of subsystem $i$ are defined as:
\begin{flalign}
  Z_{0,i} &= \left\{ \z_i \in \R^{n_i+n_s}:\ h_{0_{j}}({\proj}\inv(\z_i)) \geq 0\ \forall\,j = 1,..,n_0 \right\}\label{eq:x_subsys_initial_set}\\
  Z_{i} &= \left\{ \z_i \in \R^{n_i+n_s}:\ h_{Z_{j}}({\proj}\inv(\z_i))\geq 0\ \forall\,j = 1,...,n_Z \right\}\label{eq:x_subsys_state_space}
\end{flalign}
for $i=1,2$.
Recall that $\proj\inv$ is the back-projection operator from any subspace $Z_i$ into $Z$ as in \eqref{eq:back_proj}.
These definitions lead to the following lemma, which confirms that the projection and back-projection operators work ``as expected'' in mapping between $Z$ and $Z_i$:

\begin{lem}\citep[Section IV, Lemma 1]{chen2016exact} \label{lem:projection_chen}
Let $\z\in Z$, $\z_i={\proj}_{Z_i}(\z)$, and $S_i \subseteq Z_i$ for some subsystem, $i$. 
Then $\z_i\in S_i \iff \z \in {\proj}^{-1}(S_i)$.
\end{lem}

Next, we restate the definition of a self-contained subsystem:

\begin{defn}\label{def: scs} \citep[Definition 5]{chen2016journal} Consider the following special case of \eqref{eq: general subsystems}:
\begin{flalign}\label{eq:full_sys_SCS}
\begin{split}
    \dot{z}_1(t)&=f_1(t,z_1(t),z_s(t),k)\\
    \dot{z}_2(t)&=f_2(t,z_2(t),z_s(t),k)\\
    \dot{z}_s(t)&=f_s(t,z_s(t),k)\\
    \dot{k}(t)&=0.
    \end{split}
\end{flalign}
We call each of the subsystems with states defined as $\z_i = (z_i,z_s)$, for $i = 1,\ 2$, a \defemph{self-contained subsystem} (SCS).
We call \eqref{eq:full_sys_SCS} the \defemph{full system}.
\end{defn}
The SCS's in \eqref{eq:full_sys_SCS} show that the evolution of each subsystem depends only on the subsystem states: $\dot{\z_i}$ depends only on $\z_i=(z_i,z_s,k)$.
Notice that the trajectory parameters $k$ can appear in both SCS's. 
Given some initial condition $\z_0 \in Z_0$, let, $\z:[0,T]\to Z$ be a trajectory of the full system \eqref{eq:full_sys_SCS}.
% Then $\z$ satisfies \eqref{eq:traj-producing_model} for all $t\in[0,T]$.
Similarly, if $\z_i:[0,T]\to Z_i$ is the trajectory of subsystem $i$, then $\z_i$ satisfies the following subsystem dynamics for all $t\in[0,T]$:
\begin{flalign}
\begin{split}
    \dot{\z_i}(t)&=\begin{bmatrix}f_i(t,z_i(t),z_s(t),k)\\
    f_s(t,z_s(t),k)\end{bmatrix}\\
    \dot{k}(t)&=0.
    \end{split}
\end{flalign}
Trajectories of the full system are related to the trajectories of the subsystem via the projection operator, $\proj_{Z_i}(\z(t)) = \z_i(t)$, from \eqref{eq: projection full to sub traj} \citep[Equation (12)]{chen2016journal}.

To account for tracking error, each error function ($g_i$ from Assumption \ref{ass:tracking_error_fcn_g}) must be defined independently for each subsystem, so that subsystems 1 and 2 are still SCS.
The error function is added to \eqref{eq:full_sys_SCS} as defined below: 
\begin{flalign}\label{eq:full_sys_SCS with disturbances}
\begin{split}
    \dot{z_1}(t)&=f_1(t,z_1(t),z_s(t),k)+g_1(t,z_1(t),z_s(t),k)\circ d(t)\\
    \dot{z_2}(t)&=f_2(t,z_2(t),z_s(t),k)+g_2(t,z_2(t),z_s(t),k)\circ d(t)\\
    \dot{z_s}(t)&=f_s(t,z_s(t),k)+g_s(t,z_s(t),k)\circ d(t)
    \end{split}
\end{flalign}
In the remainder of this section, we assume subsystems 1 and 2, with states $\z_1=(z_1,z_s)$ and $\z_2=(z_2,z_s)$, have dynamics defined in \eqref{eq:full_sys_SCS with disturbances}.
Subsystems 1 and 2 are SCS's, and Lemma \ref{lem:projection_chen} and \eqref{eq: projection full to sub traj} hold.

\begin{ex}\label{ex:rover_self-contains_subsystems}
Recall the Rover's trajectory-producing model in \eqref{eq:traj-producing_rover}.
Solving $(D^l)$ for this model is memory intensive since the total dimension is $n = 7$, as discussed at the end of Section \ref{subsec:FRS_memory_usage}.
However, we can decompose this system into two separate SCS's:
\begin{flalign}
    \dot{\z_1}(t)\begin{split}\label{eq:rover_SCS_x} &= \begin{bmatrix}k_3\cos (\theta(t)) -l_r\omega(t,k) \sin (\theta(t))\\
       \omega(t,k) \\\end{bmatrix}\\%+g_1(t,\z_1(t),k)~d(t)\\
    \end{split} \\
   \dot{\z_2}(t) \begin{split}\label{eq:rover_SCS_y} &= \begin{bmatrix}k_3\sin (\theta(t)) +l_r\omega(t,k) \cos (\theta(t))\\
       \omega(t,k) \\\end{bmatrix}\\%+g_2(t,\z_2(t),k)~d(t)\\
    \end{split}
\end{flalign}
where $\z_1 = [x,\theta]^\top \in Z_1$ and $\z_2 = [y,\theta]^\top \in Z_2$.
We produce the trajectory-tracking model \eqref{eq:traj_tracking_model} for each SCS, by including error functions $g_1$ and $g_2$ as in Assumption \ref{ass:tracking_error_fcn_g}; more details are provided in Section \ref{subsec:rover_application}.
\end{ex}

\subsection{FRS Reconstruction}\label{subsec:FRS_reconstruction}

Since subsystems 1 and 2 are SCS's, an FRS can be found for each using $(D)$.
Denote the two applications of $(D)$ as $(D_1)$ and $(D_2)$, respectively.
This section formulates an optimization program that overapproximates the intersection of the back-projections of the subsystems, thus overapproximating the FRS.
We refer to this program as \defemph{reconstruction}.
First define the intersection of the back-projections as:
\begin{align}\label{eq:v_backproj_intersection}
    \begin{split}
    \mathcal{V}=\{(\z,k)\ |\ v_1(t,\z_1,k)\leq 0,\ v_2(t,\z_2,k)\leq 0,&\\
    t\in[0, T],\ \z \in Z,\ k\in K\},&
    \end{split}
\end{align}
which uses Lemma \ref{lem:v_is_negative}; since each $v_i$ is negative on trajectories of subsystem $i$, the intersection of the back-projections is the set where both $v_1$ and $v_2$ are negative.

Next, let $(v_1,w_1,q_1)$ (resp. $(v_2,w_2,q_2)$) be a feasible solution to $(D_1)$ (resp. $(D_2)$).
An outer approximation of $\X\frs$ can be reconstructed with the following optimization program:

\begin{align} %\label{lp:construct w}
	 \underset{w_r}{\text{inf}} \hspace*{0.25cm} & \int_{X \times K} w_r(x,k) ~ d\lambda_{X \times K} &&(R)\nonumber\\%\label{cons:wr_geq_1_on_v}
        & w_r(x,k)\geq 1,\ \forall (x,k)\in \mathcal{V} &&(R1)\nonumber\\%\label{cons:w_r_geq_0}
       & w_r(x,k) \geq 0,\ \forall (x,k)\in X \times K &&(R2)\nonumber,
\end{align}
where $x = \idx(\z)$.
% $(R2)$ applies for all $(x,k) \in X \times K$.
% Note that the notion of intersecting the back-projections of the subsystem FRS's is implied by the fact that $v_1$ and $v_2$ are non-positive along trajectories of the full system.
Figure \ref{fig:system_decomp} shows the intersection of back-projections of the subsystem FRS's for the Rover.
Now, we prove that the solution to $(R)$ contains the FRS.

\begin{thm}\label{thm:compute_FRS}
Let $w_r$ be a feasible solution to $(R)$.
Then $\X\frs$ is a subset of the $1$-superlevel set of $w_r$.
\end{thm}

\noindent The proof is available in Appendix \ref{app:reachability_proofs}.

\subsection{Implementation}\label{subsec:FRS_reconstruction_implementation}

In this section, we implement a relaxation of $(R)$ with SOS polynomials.
We show that the system decomposition method reduces the upper bounds on memory usage.

Suppose $l \in \N$, and suppose $(v_1^l,w_1^l,q_1^l)$ and $(v_2^l,w_2^l,q_2^l)$ are feasible solutions to $(D_1^l)$ and $(D_2^l)$, which are $(D^l)$ applied to Subsystems 1 and 2 respectively.
Recall the sets $H_T, H_Z, H_K$, and $H_X$ from Section \ref{subsec:FRS_computation}, which contain the polynomials defining the sets $[0,T],\ Z,\ K$, and $X$ respectively.
Let $\alpha \in \N$ and $\alpha\geq l$.
We pose the following SDP to reconstruct the FRS:

\begin{align} %\label{lp:construct_w_sos}
		\underset{w_r^\alpha}{\text{inf}} \hspace*{0.25cm} & y_{X\times K}^\top\text{vec}(w_r^\alpha) && &&(R^\alpha)\nonumber\\
        & w_r^\alpha -1 &&\in Q_{2\alpha}(-v_1,-v_2,H_T,H_Z,H_K) &&(R1^\alpha)\nonumber\\
         & w_r^\alpha &&\in Q_{2\alpha}(H_X,H_K)&&(R2^\alpha)\nonumber
\end{align}
where $x = \idx(\z)$ and $w_r^\alpha \in \R_{2\alpha}[x,k]$.
As in $(D^l)$, the vector $y_{X\times K}$ contains moments associated with the Lebesgue measure $\lambda_{X\times K}$, so $\int_{X\times K} w_r^\al(x,k)d\lambda_{X\times K} = y_{X\times K}^T \mathrm{vec}(w_r^\al)$ for $w \in \R_{2\al[x,k]}$ \citep{majumdar2014convex}.

The proposed system decomposition approach reduces memory usage since solving $(D^l)$ for each subsystem reduces the problem dimension. 
In particular, the reconstruction program $(R^\alpha)$, only has two SOS constraints of degree $2\alpha$; hence, it has a less stringent memory requirement than $(D^l)$.
For the rover example, solving $(D^4)$ for subsystems \eqref{eq:rover_SCS_x} and \eqref{eq:rover_SCS_y} each requires approximately $1.5\times 10^5$ free variables and used $473$ GB of RAM.
The reconstruction program $(R^5)$ requires approximately $5.5\times 10^4$ free variables and used 227 GB of RAM.
In contrast, recall from Section \ref{subsec:FRS_memory_usage} that the RAM required for the full system was 504 GB.
Figure \ref{fig:compare system decomp to normal} compares the FRS computed with the decomposed and reconstruction programs to an FRS computed for the full system \eqref{eq:traj-producing_rover} by solving $(D^3)$.

With this section complete, we can compute a conservative approximation of the FRS for a wide class of mobile ground robots; in other words, we have completed the offline portion of RTD.
Note, by Theorem \ref{thm:compute_FRS}, the 1-super level sets of $w_r$ (and $w_r^\alpha$) contain $\X\frs$; therefore, subsequent theorems and lemmas pertaining to $w$ (and $w^l$) also hold for $w_r$ (and $w_r^\alpha$).
Next, in Sections \ref{sec:pers_feas}, we discuss conditions that must be met to ensure safety and persistent feasibility.
Then, in Sections \ref{sec:obstacle_representation} and \ref{sec:trajectory_optimization}, we address the online trajectory optimization portion of RTD.

\begin{figure*}
    \centering
    \includegraphics[width=\textwidth]{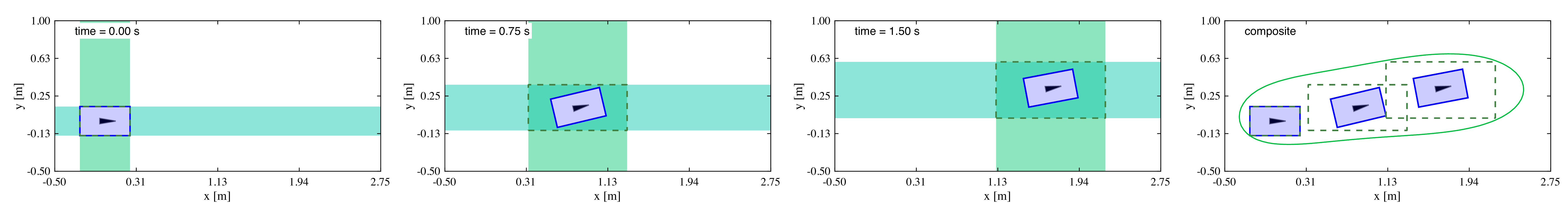}
    \caption{Example of the system decomposition and reconstruction for the FRS of the Rover's trajectory-producing system \eqref{eq:traj-producing_rover}.
    The robot is the rectangle with a triangle indicating its heading.
    The FRS and robot at 0.0, 0.75, and 1.5 s following a trajectory with parameters $k = (0.5 \ \text{rad/s},\ 0.0\ \text{rad},\ 1.1\ \text{m/s})$ are depicted from left to right.
    The vertical and horizontal bars show back-projections of the 0 sub-level sets of $v_i^4$ by $(D_i^4)$ for $i=1,2$.
    The dashed rectangle indicates the intersection of the back projections.
    The far right figure shows the intersections at each time, along with the 1-level set of $w_r^5$ as a solid line.}
    \label{fig:system_decomp}
\end{figure*}
\begin{figure}
    \centering
    \begin{subfigure}[t]{0.95\columnwidth}%{0.45\textwidth}
        \centering
        \includegraphics[width=\textwidth]{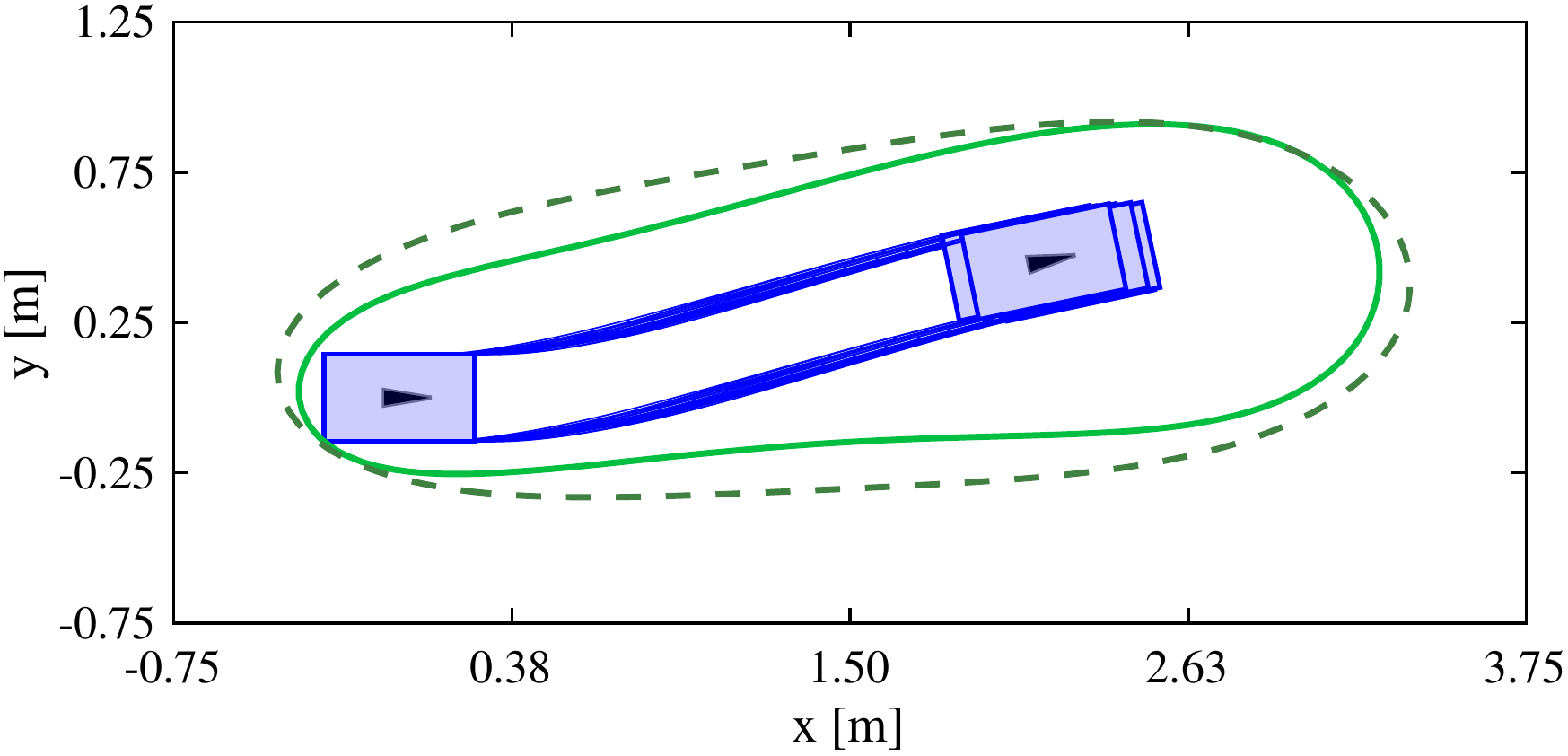}
        \caption{\centering}
        \label{subfig: rover_comparison_a}
    \end{subfigure}
    \hfill
    \begin{subfigure}[t]{0.95\columnwidth}%{0.45\textwidth}
        \centering
        \includegraphics[width=\textwidth]{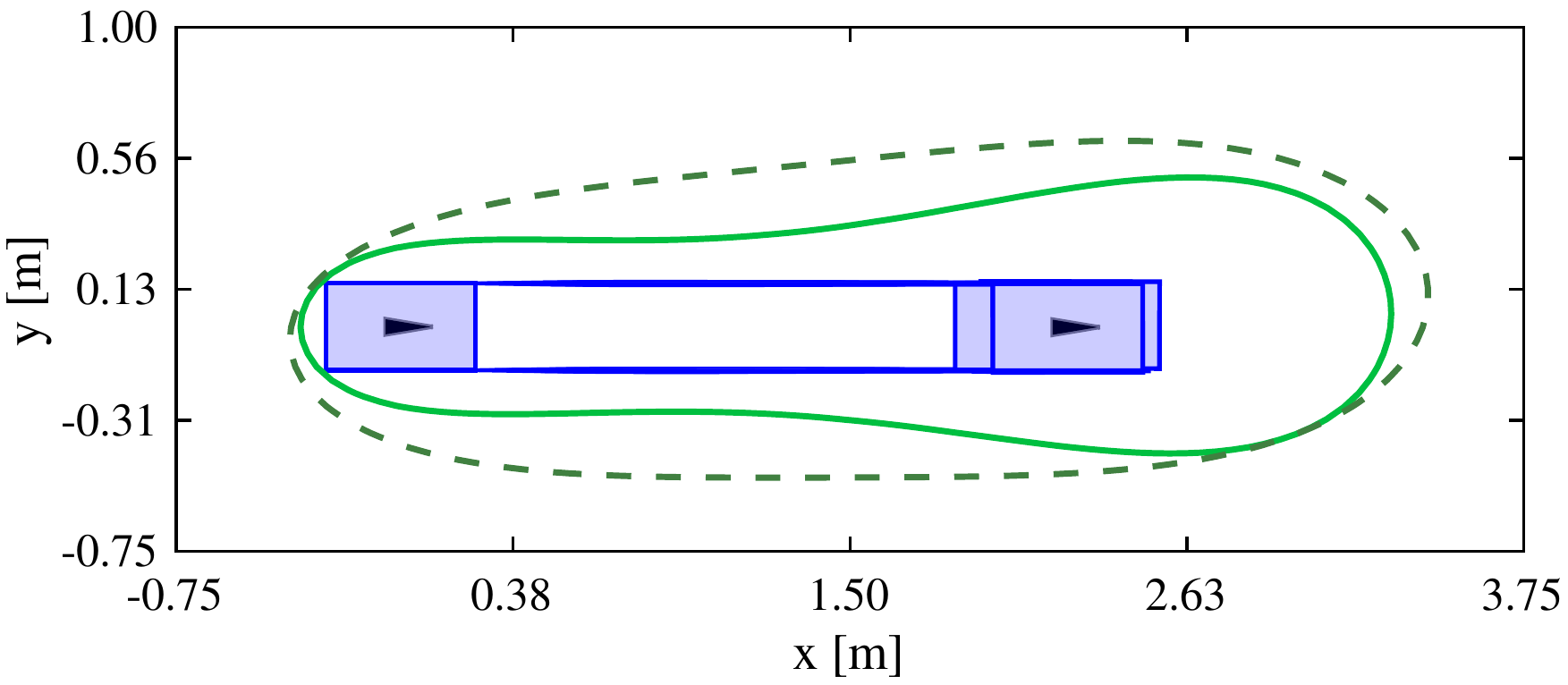}
        \caption{\centering}
        \label{subfig: rover_comparison_b}
    \end{subfigure}
    \caption{Comparison of reach sets computed for lane change trajectories produced by \eqref{eq:traj-producing_rover}. 
    The dark, dashed contours represent the 1-level set of $w^3$ computed for the full system. 
    The light contours represent the 1-level set of $w_r^5$, computed using the system decomposition and reconstruction methods.
    The reach sets are computed with a time horizon of $1.5$ s.
    Notice that the FRS computed with system decomposition is almost entirely contained within the FRS that does not use system decomposition; so, system decomposition reduces conservatism by enabling the computation of a higher-degree FRS.
    Example trajectories are generated by simulating the high-fidelity model described in \eqref{eq:high-fidelity_rover} for initial velocities and wheelangles between 0.8 and 1.5 m/s and -0.05 and 0.05 rad. Subfigure (a) shows the trajectory parameter $k$ = (0.5 rad/s, 0.0 rad, 2.0 m/s). Subfigure (b) shows the trajectory parameter $k$ = (0.0 rad/s, 0.0 rad, 1.6m/s).
    }
    \label{fig:compare system decomp to normal}
\end{figure}

\section{Conditions for Safety and Persistent Feasibility}\label{sec:pers_feas}

In this section, we state properties of a robot's environment, sensors, computation speed, and braking behavior that are required to ensure that planning with RTD is safe and persistently feasible.
We lower bound the planning time horizon in Remark \ref{rem:T_geq_tau_plan_plus_tau_stop}.
The main result in this section is Theorem \ref{thm:D_sense}, which determines a lower bound on sensor performance required to guarantee safety and persistent feasibility.

In Section \ref{subsec:obstacles_and_sensors}, we prescribe how obstacles must be sensed and processed at runtime.
In Section \ref{subsec:safety}, we formally define safety using the FRS.
In Section \ref{subsec:pers_feas} we provide conditions for persistent feasibility using the FRS and the robot's ability to brake to a stop.

The casual reader can understand the primary results from this section by reading Definition \ref{def:safety_non-intersection_condition}, Assumption \ref{ass:brake_in_pi_X}, and Theorem \ref{thm:D_sense}.

\subsection{Obstacles and Sensors}\label{subsec:obstacles_and_sensors}

\begin{defn}\label{def:obs}
An \defemph{obstacle} is a compact, connected subset of $X$ that must be avoided by the robot, and is assumed to be static with respect to time.
At any instance in time, there is a finite, maximum number of obstacles $n_{\text{obs}}$ within the robot's sensor horizon.
\end{defn}
\noindent Without loss of generality, we assume that the number of obstacles within the robot's sensor horizon at every instance in time is equal to $n_{\text{obs}}$. 
If there are fewer than $n_{\text{obs}}$ obstacles, then we treat the remaining obstacles as the empty set.
We now define how the robot senses obstacles.

\begin{assum}\label{ass:sense}
The robot has a finite \defemph{sensor horizon} $D_\regtext{sense}$, which is a radius around the robot within which all obstacles are observed, meaning that the robot has access to the size, shape, and location each such obstacle.
Occlusions and unexplored areas outside the sensor horizon are treated as static obstacles at each planning instance.
During operation, obstacles appear from outside the robot's sensor horizon and are sensed as soon as they are within the horizon; obstacles do not spontaneously appear within the sensor horizon.
\end{assum}

When running RTD on hardware, we also require that the following obstacle processing step happens before trajectory planning.

\begin{assum}\label{ass:obs_error_buffer}
By Assumption \ref{ass:predict}, the robot's current state estimate is bounded by $\vep_x$ in the $x$-coordinate and $\vep_y$ in the $y$-coordinate.
We assume that any sensed obstacle $X\sense \subset X$ is expanded by $\pm\vep_x$ (resp. $\pm\vep_y$) in the $x$ (resp. $y$) direction before being passed to the trajectory planner, i.e.
\begin{align}
    X\obs = X\sense \oplus \left\{[-\vep_x,\vep_x] \times [-\vep_y,\vep_y]\right\}\label{eq:Xobs_expanded}
\end{align}
is the set passed to the trajectory planner, where $\oplus$ indicates the Minkowski sum, i.e., $A \oplus B = \{a+b~|~a \in A, b \in B\}$.
\end{assum}

\noindent As noted in Remark \ref{rem:error_relationship}, this buffer pertains to the hardware demos, where the gap between the high fidelity model and actual robot must be accounted for.

In addition to describing how obstacles are perceived, we place assumptions on the timing allotted for planning.
Recall Assumption \ref{ass:tau_plan}, which establishes the planning time limit $\tau\plan$.
Here, we elaborate upon this assumption.

\begin{assum}\label{ass:tau_trajopt_and_tau_process}
The time required to process sensor data has a finite upper bound, $\tau_\regtext{process}$.
There is also a maximum allowed execution time for trajectory planning, $\tau_\regtext{trajopt}$.
We require $\tau\plan \geq \tau_\regtext{process} + \tau_\regtext{trajopt}$.
\end{assum}

\noindent In practice, most modern obstacle detectors have a bounded processing time for camera, lidar, or radar data \citep{johnson2016driving,liu2016ssd}.
Though we do not prove that the trajectory planning time of the proposed method is bounded, we do enforce a time limit of $\tau\plan$ on online computation, after which it is terminated.

\begin{figure}[t]
    \centering
    \includegraphics[width=0.9\columnwidth]{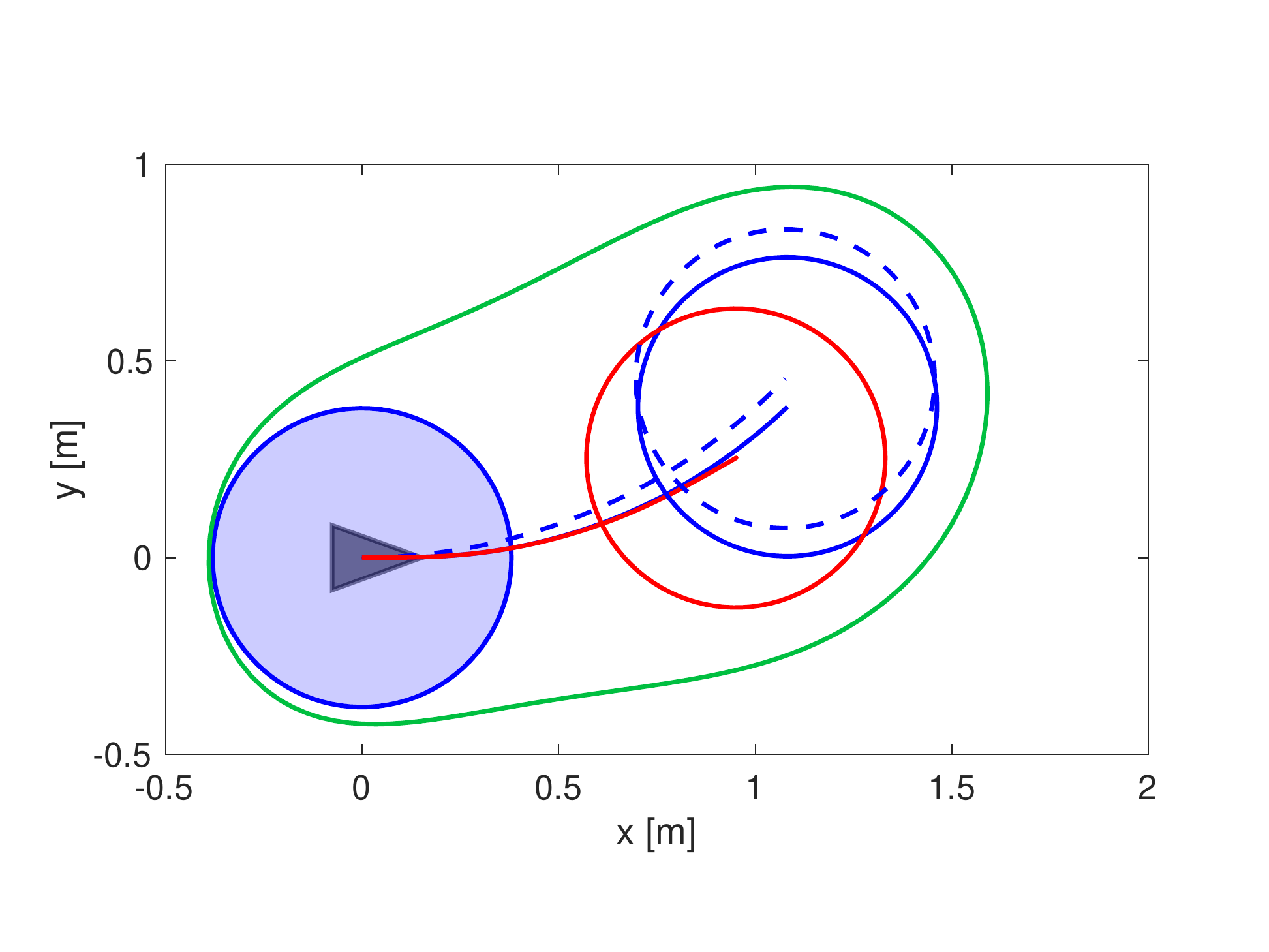}
    \caption{
    The Segway robot, as in Example \ref{ex:segway}, tracking trajectories planned in the $xy$-subspace $X$ using the trajectory-producing model Example \ref{ex:segway_traj_producing_model}.
    The robot begins with $[x_0,y_0]^\top = [0,0]^\top$ and the initial heading $\theta(0) = 0$ rad pointing to the ``right.''
    The robot has a circular footprint with radius $0.38$ m, and the initial state in is $[x_c,y_c,\theta,\omega,v]^\top = [0\ \regtext{m}, 0\ \regtext{m}, 0\ \regtext{rad}, 0.0\ \regtext{rad/s}, 1.5\ \regtext{m/s}]^\top$, plotted in $X$ as the solid circle on the left.
    The desired yaw rate $\omega\des$ (corresponding to $k_1$ in Example \ref{ex:segway_traj_producing_model}), is $1.0$ rad/s; the desired speed ($k_2$ in Example \ref{ex:segway_traj_producing_model}) is $v\des = 1.5$ m/s.
    The desired trajectory, with time horizon $T = 0.8$ s, is shown in dashed blue, with the robot's footprint plotted at the end.
    The high-fidelity model trajectory, and corresponding footprint at time $\tfin$ (using the tracking controller $u_k$ in Example \ref{ex:segway_feedback_controller}) is shown in solid blue.
    At $t = 0.5$ s, the robot begins using the braking controller from Example \ref{ex:segway_braking}.
    The robot is plotted with a solid line where it comes to a stop.
    The trajectory of the robot using $u_k$ for $t \in [0.5,1]$ s (as opposed to switching to $u\brk$) is shown with a dotted line.
    The contour is the FRS spatial projection for each $\omega\des$ (using the map $\pi_X^l$ from \eqref{eq:pi_X_on_points} with $l=5$).
    Notice that the braking trajectory stays within the FRS, as required by Assumption \ref{ass:brake_in_pi_X}.
    }
    \label{fig:segway_braking}
\end{figure}

\subsection{Safety}\label{subsec:safety}

Here, we address what it means for RTD to be safe.
Recall that RTD uses a receding-horizon strategy wherein it selects a new parameter $k$ at each planning iteration.
We create a ``non-intersection condition'' for safety in each planning iteration by stating how a safe subset of the FRS must not intersect with any obstacles.

Before proceeding, note that we compute an FRS for trajectories of the trajectory-tracking model \eqref{eq:traj_tracking_model}, as shown by \citet{kousik2017safe} and in Section \ref{sec:FRSmethod} of this paper.
We can think of the FRS as a map that associates trajectory parameters in $K$ with reachable points in $X$.
Suppose that the tuple $(v,w,q)$ is an optimal solution to $(D)$ from Section \ref{subsec:FRS_computation}.
Then, by Lemma \ref{lem:w_geq_1_on_frs}, $w: X\times K \to \{0,1\}$ is an indicator function on $\X\frs$ as follows.
Define the \defemph{FRS spatial projection map} $\pi_X: K \to \P(X)$ for which
\begin{align}
    \pi_X(k) = \left\{x \in X~|~w(x,k) = 1 \right\},\label{eq:pi_X_on_points}
\end{align}
which returns the set in $X$ of all points that are reachable by the robot's trajectory-tracking model \eqref{eq:traj_tracking_model} at any time in $[0,T]$.
This map lets us define safety in each plan:

\begin{defn}\label{def:safety_non-intersection_condition}
Suppose that $X\obs \subset X$ is a set of obstacles as in Definition \ref{def:obs}.
Then, at each planning iteration, we must pick a $k$ for which the FRS does not intersect any obstacles, i.e., $\pi_X(k) \cap X\obs = \emptyset$.
This \defemph{non-intersection condition} defines safety when planning with RTD.
\end{defn}

Now, recall that we cannot solve $(D)$ exactly; but, suppose that the tuple $(v^l,w^l,q^l)$ is an optimal solution to $(D^l)$ from Section \ref{subsec:FRS_computation} for some fixed degree $l \in \N$.
As per Remark \ref{rem:w_geq_1_is_in_FRS}, 1-superlevel sets of $w^l: X\times K \to \R$ contain $\X\frs$.
The map $\pi_X^l: K \to \P(X)$ is defined as:
\begin{align}
    \pi_X^l(k) = \left\{x \in X~|~w^l(x,k) \geq 1 \right\}.\label{eq:pi_X_l_on_points}
\end{align}
By Remark \ref{rem:w_geq_1_is_in_FRS}, for any $k \in K$, $\pi_X^l(k) \supseteq \pi_X(k)$, i.e., this map over-approximates the FRS.
Therefore, $k \in K$ is safe for the high-fidelity model to track if $\pi_X^l(k) \cap X\obs = \emptyset$.

\begin{defn}\label{def:safe_plan}
Any $k$ for which $\pi_X^l(k) \cap X\obs = \emptyset$ is called a \defemph{safe plan}.
\end{defn}

\noindent See Algorithm \ref{alg:trajopt} in Section \ref{sec:trajectory_optimization} for how we use the non-intersection condition from Definition \ref{def:safe_plan} online.

\subsection{Persistent Feasibility}\label{subsec:pers_feas}

We now prescribe how the robot must brake (Section \ref{subsubsec:pers_feas_braking}), how long its planning time horizon must be (Section \ref{subsubsec:pers_feas_time_horizon}), and how long its sensor horizon must be (Section \ref{subsubsec:pers_feas_sensor_horizon}, to ensure persistent feasibility.

Recall that the trajectory-producing model \eqref{eq:traj-producing_model} does not incorporate a braking maneuver.
But, in any receding-horizon planning iteration, if the robot cannot find a new trajectory plan, it must be able to safely brake to a stop.
Suppose that, at the beginning of a planning iteration, the robot is currently tracking a safe trajectory found in the previous planning iteration.
Further suppose that the robot is unable to identify a new safe trajectory in the current iteration or any subsequent iterations.
Then, the trajectory from the previous iteration must bring the robot to a safe stop.
However, the parameterized trajectories in this paper do not necessarily explicitly include braking.
To compensate for this, we use the fact that the FRS allows us to identify a subset of the state space within which the robot is collision-free (using $\pi_X$ from \eqref{eq:pi_X_l_on_points}).
Consequently, if the robot can stop within this safe subset, we say it can ``brake within the FRS," which enables us to guarantee safety and persistent feasibility.

Next, we formalize braking within the FRS, and provide conditions on the robot's braking behavior, the planning time horizon $T$, and the robot's sensor horizon to make it possible to brake within the FRS.
For readability, we define these ideas here, then provide more detail in Appendix \ref{app:pers_feas}.

\subsubsection{Braking Within the FRS}\label{subsubsec:pers_feas_braking}

We now restate the above reasoning for braking within the FRS more formally.
First, recall that each trajectory parameterized by $k$ is of duration $T$.
Suppose that $X\obs$ is an obstacle and, at time $0$, we have a safe plan given by $k_0 \in K$, as in Definition \ref{def:safe_plan}.
We only know that $k_0$ is safe for $t \in [0,T]$, but persistent feasibility requires us to ensure safety for all $t \geq 0$.
While tracking $k_0$, the robot must simultaneously plan its next trajectory, i.e., find some $k_1$ for which $\pi_X^l(k_1) \cap X\obs = \emptyset$.
By Assumption \ref{ass:tau_plan}, the robot has a duration of $\tau\plan < T$ to find $k_1$.
If a safe $k_1$ is not found by $\tau\plan$, the only way for the robot to be safe for all $t > \tau\plan$ is to brake to a stop.
Since $\pi_X^l(k_0) \cap X\obs = \emptyset$, we require that the robot brakes within the set $\pi_X^l(k_0)$, i.e., it brakes ``within the FRS.''
This section describes conditions to ensure that it is possible for the robot to brake safely starting at $\tau\plan$ of the current planning iteration.
We begin by stating how the robot brakes:

\begin{assum}\label{ass:brake_ctrl_and_traj}
At time $0$, let the robot, with high-fidelity dynamics \eqref{eq:high-fidelity_model}, be at an arbitrary initial condition $\z\hio$ and about to track an arbitrary $k \in K$.
We assume that there exists a finite \defemph{braking time} given by $\tau\brk: Z\hio\times K \to \R_{\geq 0}$, within which the robot can come to a stop using a \defemph{braking controller},
\begin{align}
    u\brk: [0,\infty)\times Z\hi \to U.\label{eq:ubrk_braking_controller}
\end{align}
If the braking controller is applied at $\tau\plan$ then for any $\hat{\tau} > \tau\plan +\tau\brk(\z\hio,k)$ the robot will be stopped:
\begin{align}
    \idv\bigg(f\hi\Big(\hat{\tau},\z\brk(\hat{\tau}),u\brk\big(\hat{\tau},\z\brk(\hat{\tau})\big)\Big)\bigg) = 0,
\end{align}
where $V$ is the subspace of the state space $Z\hi$ corresponding to the yaw rate and speed states as described in Assumption \ref{ass:max_speed_and_yaw_rate}.
Here, $\z\brk: [0,\infty) \to Z\hi$ is a trajectory of the high-fidelity model \eqref{eq:high-fidelity_model} produced when the braking controller $u\brk$ is used (the time domain of the high-fidelity model is extended to infinity to allow this).
\end{assum}

Now we formally specify braking within the FRS.

\begin{assum}\label{ass:brake_in_pi_X}
Consider an arbitrary initial condition $\z\hio$ at time $0$, and arbitrary $k \in K$.
Suppose the robot, described by the high-fidelity dynamics \eqref{eq:high-fidelity_model}, tracks $k$ for $t \in [0,\tau\plan)$, then applies the braking controller \eqref{eq:ubrk_braking_controller} for $t \geq \tau\plan$.
We assume that, at every $t \in [\tau\plan, \tau\plan + \tau\brk(\z\hio,k)]$, the spatial component of the robot's braking trajectory $\z\brk: [0,T] \to Z\hi$ lies within the set of points reachable by trajectory-tracking model:
\begin{align}
    \idx\left(\z\brk(t)\right) \in \pi_X(k).
\end{align}
\end{assum}

Note that, since $\z\hio$ is arbitrary, Assumption \ref{ass:brake_in_pi_X} requires that every point on the robot's body stays within $\pi_X(k) \subset X$ when braking, as per Assumption \ref{ass:rigid_body_dynamics_ctr_of_mass}.
There are several ways to satisfy Assumption \ref{ass:brake_in_pi_X}.
One way is to directly include braking maneuvers in the parameterized trajectories.
However, this increases the complexity of the offline reachability analysis by either increasing the degree or dimension of the trajectory-producing model, or introducing time-switching dynamics.
To avoid the complexity increase, in this paper, we instead choose the planning time horizon $T$ of the parameterized trajectories to be long enough that the robot can brake within the FRS as discussed above.

To proceed, we first present an example of a braking controller $u\brk$ as in \eqref{eq:ubrk_braking_controller} for the Segway.
Then we discuss a way to choose $T$ to ensure this controller can bring the robot to a stop within the FRS, to satisfy Assumption \ref{ass:brake_in_pi_X}.
We present more detail on choosing $T$ in Appendix \ref{app:pers_feas}.

To design a braking controller, first suppose we know $\tau\brk$ (e.g., from physical properties of the robot such as maximum acceleration).
Then we can proceed as in the following example for the Segway.

\begin{ex}\label{ex:segway_braking}
Consider again the Segway's high-fidelity model from Example \ref{ex:segway}.
Now, we use $\tau\brk$ to design a braking controller $u\brk$ as in \eqref{eq:ubrk_braking_controller}.
Suppose that the robot is applying its low-level controller $u_k$ from Example \ref{ex:segway_feedback_controller} to track a trajectory $k \in K$, over the time span $t \in [0,\tau\plan]$.
Let $\z\hi: [0,\tau\plan] \to Z\hi$ denote the trajectory of the high-fidelity model using $u_k$, and let $\tilde{\tau} = \tau\brk(\z\hi(\tau\plan),k)$.
At time $\tau\plan$, we switch to $u\brk$, given by
\begin{align}\label{eq:segway_braking_controller}
    u\brk(t,\z\hi(t)) = \begin{cases}\left(\frac{\tilde{\tau} - t}{\tilde{\tau} - \tau\plan}\right)^4u_k(t,\z\hi(t)) & t \in [\tau\plan,\tilde{\tau}) \\
        0 & t \geq \tilde{\tau},
    \end{cases}
\end{align}
where $u_k$ is as in Example \ref{ex:segway_traj_tracking_model}.
Recall that $u_k$, and therefore $u\brk$, produce two control inputs.
The first input is a commanded yaw rate, and the second a commanded speed (see \eqref{eq:high-fidelity_segway}).
Therefore, \eqref{eq:segway_braking_controller} reduces the commanded yaw rate and speed quartically to zero, but still uses feedback around the desired position and heading to cause the Segway to brake ``along'' the trajectory parameterized by $k$.
\end{ex}
\noindent An example braking trajectory for the Segway is shown in Figure \ref{fig:segway_braking}.

Next, we relate the planning time horizon to braking.

\subsubsection{Planning Time Horizon}\label{subsubsec:pers_feas_time_horizon}

Now we relate the braking maneuvers above to the planning time horizon $T$, to ensure that the FRS is computed so that robot can satisfy Assumption \ref{ass:brake_in_pi_X}.
Recall that, for any $k \in K$, the robot is able to generate a feedback controller $u_k$ as in \eqref{eq:fdbk_controller_uk}.
So, for any $\z\hio \in Z\hio$ and $k \in K$, the robot's \defemph{non-braking trajectory}, given by the high-fidelity model, \eqref{eq:high-fidelity_model} can be written as $\z\hi: [0,T] \to Z\hi$ for which
\begin{align}
    \z\hi\left(t;\z\hio,k\right) = \z\hio + \int_0^T f\hi\left(t,\z\hi(t),u_k(t,\z\hi(t))\right)dt. \label{eq:non-braking_traj}
\end{align}
To pick $T$, first recall that the robot's braking time $\tau\brk$ is finite for every initial condition and trajectory parameter.
So, there exists a maximum \defemph{stopping time} $\tau\stp$, given by:
\begin{align}
    \tau\stp = \max_{\z\hio \in Z\hio,\ k \in K} \tau\brk(\z\hio,k), \label{eq:tau_stop}
\end{align}
where the maximum is achieved because $Z\hio$ and $K$ are compact by Assumption \ref{ass:state_and_control_sets_are_compact}.

\begin{rem}\label{rem:T_geq_tau_plan_plus_tau_stop}
By Theorem 10 of \citet{kousik2017safe}, $\tau\stp$ as in \eqref{eq:tau_stop} can be used to lower-bound the planning time horizon $T$.
If $T$ is larger than $\tau\plan + \tau\stp$, and the robot at time $0$ has a safe plan of length $\tau\plan + \tau\stp$, then the robot always has enough time to brake if it cannot find a new safe trajectory within $\tau\plan$.
In other words, if $T \geq \tau\plan + \tau\stp$, the robot will travel farther along a non-braking trajectory (as in \eqref{eq:non-braking_traj}) than along a braking trajectory from the same initial condition.
Note that $\tau\stp$ may be large, leading to a large (and therefore conservative) FRS.
To combat this conservatism, we can pick $T < \tau\plan + \tau\stp$ empirically, by ensuring that $T$ is long enough such that, for any $k \in K$, the robot travels at least as far as its braking distance along a non-braking trajectory of duration $T$.
\end{rem}

\noindent We illustrate how to pick $T < \tau\plan + \tau\stp$ by continuing the previous Segway example.

\begin{ex}\label{ex:segway_min_time_horizon}
Consider computing an FRS for the Segway from Example \ref{ex:segway} with $\vmax = 1.25$ m/s and $\tau\plan = 0.5$ s.
On the hardware, we find empirically that stopping from $1.25$ m/s requires $\tau\stp \approx 1.5$ s, over a stopping distance of no more than $0.625$ m.
If we set $T = \tau\plan + \tau\stp$, the robot's non-braking trajectories would be up to $T\cdot\vmax = 2.5$ m long.
However, suppose we choose a number $\tau_v = (0.625~\regtext{m}) / (1.25~\regtext{m/s}) = 0.5$ s, and then set $T = \tau\plan + \tau_v = 1$ s.
Then, non-braking trajectories would be at most $1.25$ m long, which leaves enough distance in the FRS for the robot to stop if it begins braking after traveling for $\tau\plan\cdot\vmax = 0.625$ m.
\end{ex}
\noindent A detailed discussion of $\tau_v$ is in Appendix \ref{app:pers_feas}.
Next, to conclude this section, we specify a minimum sensor horizon required to ensure it is possible for the robot to achieve safety and persistent feasibility.

\subsubsection{Minimum Sensor Horizon}\label{subsubsec:pers_feas_sensor_horizon}
Now, to ensure that the robot is safe for all $t$, i.e., that the trajectory planning procedure is persistently feasible, we need to enforce a lower bound on the robot's sensor horizon $D\sense$ (from Assumption \ref{ass:sense}).
This is to ensure that the robot senses obstacles that are far enough away that it can plan a safe non-braking trajectory of duration $T$ every $\tau\plan$ seconds.
The following result is a modified version of Theorem 10 from \citet{kousik2017safe}.

\begin{thm}\label{thm:D_sense}
Let $X\obs \subset X$ be a set of obstacles as in Definition \ref{def:obs}.
Let $v_\regtext{max}$ be the robot's maximum speed as in Assumption \ref{ass:max_speed_and_yaw_rate}.
Let $\tau\plan$ be the planning time as in Assumption \ref{ass:tau_plan}.
Suppose that $T$ is large enough that Assumption \ref{ass:brake_in_pi_X} holds; so, for any $\z\hio \in Z\hio$ and any $k \in K$, the spatial component of the robot's braking trajectory lies within $\pi_X(k)$.
At time $0$, suppose that the robot has a safe plan $k_0 \in K$ (as in Definition \ref{def:safe_plan}).
Recall that $\vep_x$ and $\vep_y$ are the robot's maximum state estimation error in the $x$ and $y$ coordinates of $X$ as in Assumption \ref{ass:predict}, and let $\vep = \sqrt{\vep_x^2 + \vep_y^2}$.
Suppose the sensor horizon $D\sense$ obeys Assumption \ref{ass:sense} and satisfies
\begin{align}
D\sense & \geq (T + \tau\plan)\cdot\vmax + 2\vep.\label{eq:min_sensor_horizon}
\end{align}
Then, the robot can find either find a new safe plan every $\tau\plan$ seconds, or can brake safely if no new safe plan is found.
\end{thm}
The proof is in Appendix \ref{app:pers_feas}.

\noindent Theorem \ref{thm:D_sense} guarantees that the proposed RTD method is safe and persistently feasible, concluding this section.

In the next section, we address how to represent obstacles so that the online planning algorithm, i.e., picking a new trajectory parameter $k \in K\safe$ at every iteration, can be done in real time.
\section{Obstacle Representation}\label{sec:obstacle_representation}

This section presents a method of representing the robot's environment as a finite, discrete set, to enable real-time performance of RTD.
Note, the approach developed in this section is general, and can be applied to robots with arbitrary convex, compact footprints in arbitrary environments.
The casual reader can understand the primary results of this section by reading Section \ref{subsubsec:b_and_r_motivation} and Section \ref{subsec:proving_X_p_works}.

The main result of this section is Theorem \ref{thm:X_p}, which proves that the proposed obstacle representation can be used to represent safe plans.
By safe plans, we mean the set of safe trajectory parameters at each receding-horizon planning iteration:

\begin{defn}\label{def:K_safe}
Given an obstacle $X\obs \subset X$, let $K_\text{safe}$ denote the set of \defemph{safe trajectory parameters}.
No point on the robot's body, described by the high dimensional model \eqref{eq:high-fidelity_model}, can collide with the obstacle when tracking a trajectory parameterized by any $k \in K_\text{safe}$.
\end{defn}

\noindent Theorem \ref{thm:X_p} proves that the proposed obstacle representation enables inner approximating $K\safe$.
Note, in our prior work, we approximated $K\safe$ by solving an SDP \citep{kousik2017safe}.
We show in Appendix \ref{app:set_intersection} that solving this SDP is too slow for real-time planning, whereas the proposed obstacle representation is not.

We proceed as follows.
In Section \ref{subsec:FRS_projections}, we explain how to use the FRS computed in Sections \ref{sec:FRSmethod} and \ref{sec:system_decomp} to identify safe trajectory parameters in a single planning iteration.
In Section \ref{subsec:robot_obstacle_geometry_motivation}, we identify several geometric quantities used to construct our novel obstacle representation for arbitrary convex robot footprints.
Finally, in Section \ref{subsec:proving_X_p_works}, we explain how to construct the obstacle representation, and prove that it enables identifying safe trajectory parameters.
All of the proofs for this section are in Appendix \ref{app:obs_rep}.

\subsection{FRS Projections}\label{subsec:FRS_projections}
To relate obstacles to unsafe trajectories, we use the representation of the FRS from Section \ref{sec:FRSmethod} to project a point on an obstacle in $X$ to the corresponding set of parameters in $K$ for which the robot would reach that point on that obstacle.

Recall that obstacles are sets $X\obs \subset X$, where $X$ is the $xy$-subspace of the trajectory producing model's state space $Z$.
Also recall the FRS spatial projection map $\pi_X: K \to \P(X)$ in \eqref{eq:pi_X_on_points}, which maps a set of trajectory parameter $k \in K$ to all points of $X$ that are reachable within the time horizon $[0,T]$ by the robot's trajectory-tracking model \eqref{eq:traj_tracking_model}.
We define a related map $\pi_K$ that maps a subset $X'$ of $X$ to the set $\pi_K(X') \subset K$ for which any trajectory tracking some $k \in \pi_K(X')$ travels through at least one point in $X'$.

Suppose that the tuple $(v,w,q)$ is an optimal solution to Program $(D)$ from Section \ref{subsec:FRS_computation}.
Then by Lemma \ref{lem:w_geq_1_on_frs}, $w: X\times K \to \{0,1\}$ can is an indicator function on $\X\frs$.
Define the set-valued map $\pi_K: \P(X) \to \P(K)$ as
\begin{align}
    &\pi_K(X') = \{ k \in K~\mid \exists~x \in X'~\mathrm{s.t.}~w(x,k) = 1 \} \label{eq:pi_K}.
\end{align}

\noindent We call $\pi_K$ the \defemph{FRS parameter projection map}.
If $X' \subset X$, we say that $\pi_K(X')$ are the parameters corresponding to $X'$.
We use the word ``projection'' for these operators to relate them to the projection operators $\proj_{Z_i}$ in Definition \ref{def:projection_operators}.
Recall that $\proj_{Z_i}$ returns points in a subspace $Z_i$ of the state space $Z\hi$ that are identified by an identity relationship.
Similarly, $\pi_X$ and $\pi_K$ return points in a subspace of the reachable set $\X\frs$ that are identified by the indicator function $w$.
The following lemma demonstrates the utility of $\pi_K$.

\begin{lem}\label{lem:when_param_cant_cause_crash}
Consider an arbitrary point $p \in X \setminus X_0$.
Let $k \in \pi_K(p)^C$.
At $t = 0$, let the robot, described by the high-fidelity model \eqref{eq:high-fidelity_model}, be at the state $\z_{\text{hi},0} \in Z\hi$.
Suppose the robot tracks the trajectory parameterized by $k$, producing the high-fidelity model trajectory $\z\hi: [0,T] \to Z\hi$.
Then, no point on the robot's body ever reaches $p$.
More precisely, there does not exist any pair $(t,\z\hio) \in [0,T]\times Z\hio$ such that $p = \idx(\z\hi(t))$.
\end{lem}

\noindent The proof is in Appendix \ref{app:obs_rep}.
See Figure \ref{fig:point_obs} for an illustration of Lemma \ref{lem:when_param_cant_cause_crash}.
This lemma lets us find parameters in $K$ for which the robot avoids points in $X$.
So, by representing obstacles with points in $X$, we can find obstacle-avoiding trajectories, motivating the next discussion.

\subsection{Robot and Obstacle Geometry}\label{subsec:robot_obstacle_geometry_motivation}

Suppose $X\obs$ represents one or more obstacles in $X$.
The overall purpose of Section \ref{sec:obstacle_representation} is to find a finite set of discrete points $X_p \subset X$ to represent $X\obs$ such that the trajectory parameters corresponding to $X_p$ are a conservative approximation of those corresponding to $X\obs$, i.e. $\pi_K(X_p) \supseteq \pi_K(X\obs) = K_\text{safe}^C$.
Then, as in Lemma \ref{lem:when_param_cant_cause_crash}, if the robot cannot collide with any of the points in $X_p$, it cannot collide with the obstacle $X\obs$.
This is illustrated in Figure \ref{fig:obs_with_buffer}.
We call $X_p$ the \defemph{discretized obstacle}.

The motivation behind discretizing the obstacle in this manner is that $\pi_K(X_p)$ can be implemented as a list of point constraints at runtime for the path planning optimization program in Section \ref{sec:trajectory_optimization}; in practice, this allows the online trajectory planner to run in real time.

In Section \ref{subsec:robot_obstacle_geometry_motivation}, we find four \defemph{geometric quantities}, $\rbar$, $\bbar$, $r$, and $a > 0$ that are determined by the geometry of the robot and by a user-specified buffer distance $b > 0$.
Then, in Section \ref{subsec:proving_X_p_works}, we use these quantities and buffer to construct $X_p$.

The remainder of Section \ref{subsec:robot_obstacle_geometry_motivation} proceeds as follows.
Section \ref{subsubsec:generality_of_obs_rep} places assumptions on the robot and obstacle geometry to express the generality of the proposed method.
Section \ref{subsubsec:geometric_objects} introduces several geometric objects used throughout the section.
In Section \ref{subsubsec:b_and_r_motivation}, we introduce the buffer $b$ and the geometric quantities $\bbar$, $r$, $\rbar$, and $a$, which are used to produce the discretized obstacle representation.
Finally, in Section \ref{subsubsec:geom_trans_family}, we present a geometric expression for the robot's dynamics.
Next, in Section \ref{subsec:finding_geometric_quantities}, we find the geometric quantities $\rbar$, $\bbar$, $r$, and $a$.

To build intuition for these geometric quantities, the reader can skip to Figure \ref{fig:rbar_b_r_and_a_examples_rect_and_circ} at the end of Section \ref{subsec:finding_geometric_quantities}, which shows each quantity for rectangular and circular robot footprints.

\subsubsection{Generality of Proposed Method}\label{subsubsec:generality_of_obs_rep}
Before proceeding, we introduce assumptions on the shape of the robot and obstacles.
This is to clarify the generality of the proposed obstacle representation.

We use the following general robot representation:

\begin{assum}\label{ass:X0_cpt_cvx}
The robot's footprint $X_0 \subset X$ is compact and convex with nonzero volume.
\end{assum}
\noindent Footprints fulfilling this assumption, such as circles and rectangles, are common for ground robots (consider the Segway and Rover in Figure \ref{fig:hardware_time_lapse}).
If the robot's footprint is not convex, it can be contained within a convex hull or rectangular bounding box \citep{rectangle_bound_curve}.
We emphasize that the method in this section applies to \textit{arbitrary} convex robot footprints, not just the circle and rectangle examples for the Segway and Rover.

We use the following general obstacle representation:

\begin{assum}\label{ass:obs_are_polygons}
Each obstacle $X\obsi \subseteq X\obs$ is a closed polygon with a finite number of vertices and edges.
\end{assum}
\noindent Note that these polygons are not necessarily convex.
This assumption holds for common obstacle representations such as occupancy grids or line segments fit to planar point clouds.
If an obstacle is not a closed polygon within the sensor horizon (such as a long wall), it can be closed by intersection with the sensor horizon $D\sense$ (as in Assumption \ref{ass:sense}), which can be over-approximated by a regular polygon (the intersection is a closed set \citep[Theorem 17.1]{Munkres2000}).
Note that $X\obs$ may contain one or more obstacles; the definitions and proofs in this section still hold if $X\obs$ is a union of polygons, which is itself a (potentially disjoint) polygon \citep{minkowski_sum_fogel}.
Therefore, we refer to $X\obs$ as the singular obstacle for ease of exposition.

Next, we define several geometric objects used throughout the remainder of the section.

\subsubsection{Geometric Objects}\label{subsubsec:geometric_objects}
Before defining how to construct the discretized obstacle representation, we define several geometric objects used throughout the remainder of the section.
Examples of these objects are shown in Figure \ref{fig:passing_through}.

\begin{defn}\label{def:line_segment_I}
Let $I \subset \R^2$ be a \defemph{line segment}, also called an \defemph{interval} when it lies on either the $x$- or $y$-axis.
Let $E_I = \{e_1,e_2\} \subset I$ denote the \defemph{endpoints of $I$}, such that $I$ can be written as $I = \{e_1 + s\cdot(e_2 - e_1)~\mid~s \in [0,1]\}$.
The \defemph{length} of $I$ is $\norm{e_1 - e_2}_2$.
% Note that a line segment can have length $0$ if $e_1 = e_2$.
Suppose $I$ has a pair of distinct endpoints $\{e_1,e_2\}$, and we create the set $\ell_I = \{e_1 + s\cdot(e_2 - e_1)~\mid~s \in \R\} \subset \R^2$, i.e. a line that passes through $e_1$ and $e_2$.
We call $\ell_I$ \defemph{the line defined by $I$}.
\end{defn}
\noindent Note that a line segment can have a length of $0$ if $e_1 = e_2$.
We also define a specific type of line segment called a chord:
\begin{defn}\label{def:chord}
Let $A \subset \R^2$ be a set with a boundary and $a_1, a_2 \in \bd A$.
The line segment $\kp = \{a_1 + s\cdot(a_2 - a_1)~\mid~s \in [0,1]\}$ is a \defemph{chord} of $A$.
\end{defn}
\noindent Note that $\kp$ need not be a subset of $A$, e.g., if $A$ is not convex.
Finally, we define an arc and its circle:
\begin{defn}\label{def:circle_and_arc}
A \defemph{circle} $C \subset \R^2$ of radius $R \geq 0$ with center $p \in \R^2$ is the set $\left\{p' \in \R^2~\mid~\norm{p' - p}_2 = R\right\}$.
An \defemph{arc} $A \subset \R^2$ is any connected, closed, strict subset of a circle; this means that any arc has two \defemph{endpoints} $a, b \in \R^2$.
\end{defn}
\noindent Note, that given two arc endpoints $a, b$ and a radius $R$, we can produce an arc $A$ as follows: find $\theta_1 = 2\tan\inv\left(\frac{a_y - p_y}{a_x - p_x + R}\right)$ and $\theta_2 = 2\tan\inv\left(\frac{b_y - p_y}{b_x - p_x + R}\right)$.
If $\theta_1 < \theta_2$, set $\Theta = [\theta_1, \theta_2] \subset \R$ or $\Theta = [\theta_2, \theta_1 + 2\pi]$ (to choose the direction of the arc), and similarly if $\theta_2 < \theta_1$.
Then $A = \left\{q + R\cdot[\cos\theta, \sin\theta]^\top~\mid~\theta \in \Theta\right\} \subset \R^2$.

\subsubsection{Buffer and Point Spacing Motivation}\label{subsubsec:b_and_r_motivation}

\begin{figure}\centering
    \begin{subfigure}[t]{0.48\textwidth}
        \centering
        \includegraphics[width=\columnwidth]{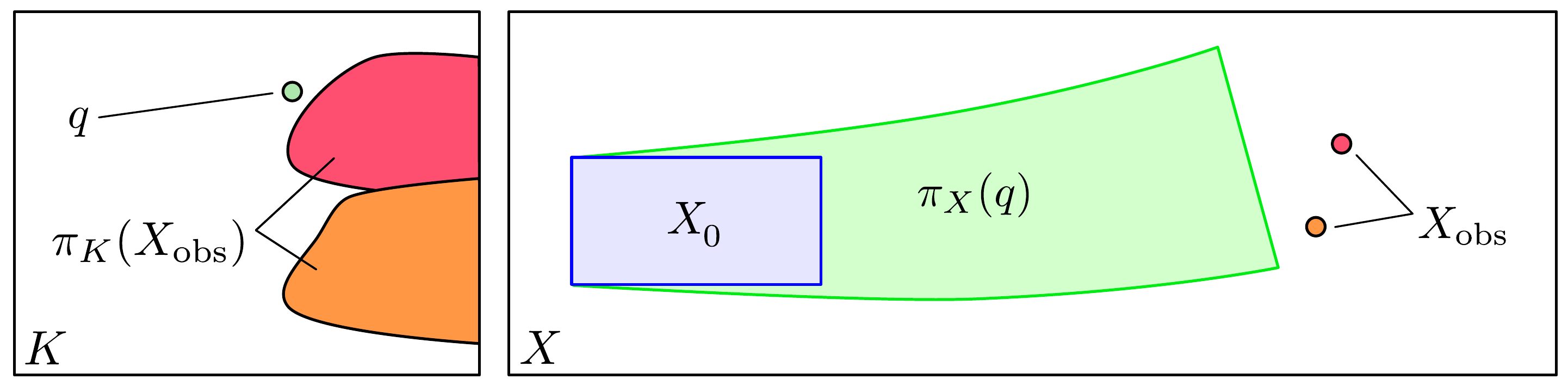}
        \caption{\centering}
        \label{fig:point_obs}
    \end{subfigure}
    \hspace{0.25cm}
    \begin{subfigure}[t]{0.48\textwidth}
        \centering
        \includegraphics[width=\columnwidth]{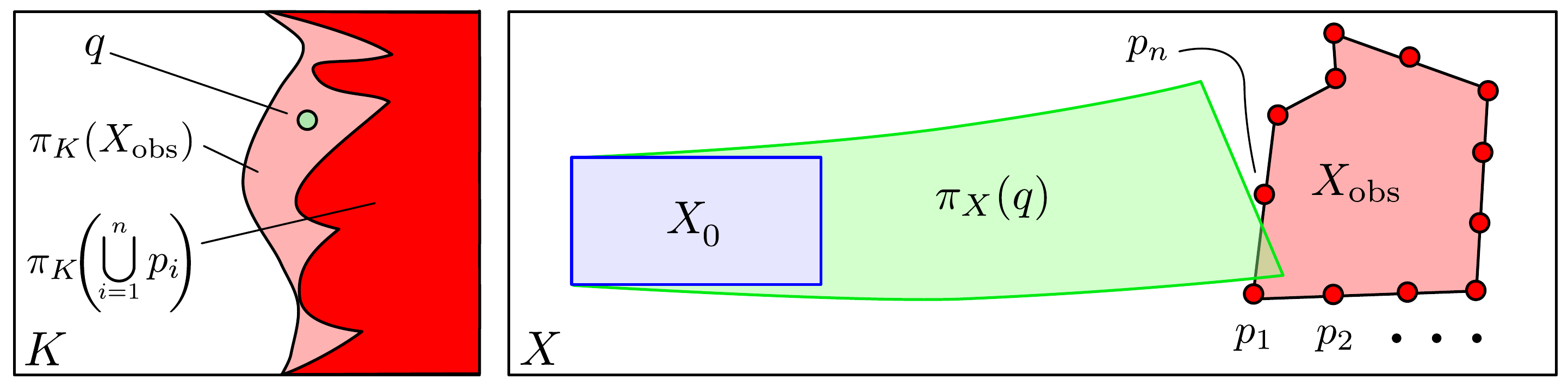}
        \caption{\centering}
        \label{fig:obs_no_buffer}
    \end{subfigure}
    \begin{subfigure}[t]{0.48\textwidth}
        \centering
        \includegraphics[width=\columnwidth]{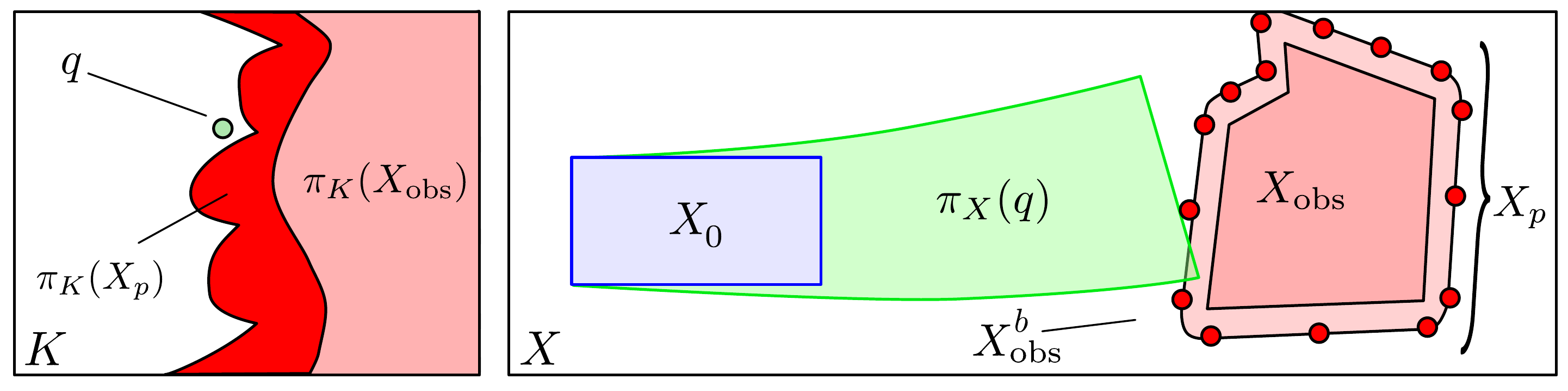}
        \caption{\centering}
        \label{fig:obs_with_buffer}
    \end{subfigure}
    \caption{
    Motivation and method for buffering and discretizing obstacles.
    The robot has footprint $X_0$ in the $xy$-subspace $X$ on the right, and the trajectory parameter space $K$ is on the left.
    In Figure \ref{fig:point_obs}, the obstacle $X\obs$ consists of two points, to illustrate the map $\pi_K$, which maps each point to a subset of $K$ containing all trajectory parameters that could cause the robot to reach either point; since $q \in \pi_K(X\obs)^C$, by Lemma \ref{lem:when_param_cant_cause_crash}, the robot cannot collide with either obstacle point.
    Figure \ref{fig:obs_no_buffer} shows an arbitrary polygonal obstacle (as in Assumption \ref{ass:obs_are_polygons}) with a set of discrete points $\{p_1,\cdots,p_n\}$ sampled from its boundary.
    These points are mapped to the subset of the parameter space $K$ labeled $\pi_K\left(\bigcup_{i = 1}^n p_i\right)$.
    A parameter $q$ is chosen outside of the parameters corresponding to these points, but still lies within the projection of the actual obstacle $\pi_K(X\obs)$, and therefore may cause a collision as illustrated by the set $\pi_X(q)$.
    Figure \ref{fig:obs_with_buffer} shows the same obstacle, but buffered.
    The boundary of the buffered obstacle is sampled to produce the discrete, finite set $X_p$.
    The trajectory parameters corresponding to $X_p$ are a superset of the unsafe parameters $\pi_K(X\obs)$, so the robot cannot collide with the obstacle despite the FRS spatial projection $\pi_X(q)$ penetrating between two of the points of $X_p$.
    }
\end{figure}

Recall that the goal of this entire section is to construct the discretized obstacle, $X_p \subset X$.
To that end, consider constructing $X_p$ from points on the boundary of $X\obs$, as illustrated in Figure \ref{fig:obs_no_buffer}.
Since the high-fidelity model of the robot \eqref{eq:high-fidelity_model} produces continuous trajectories in the subspace $X$ (see Assumption \ref{ass:dyn_are_lipschitz_cont}), the robot cannot collide with an obstacle without passing through the obstacle's boundary.

However, constructing $X_p$ with a finite number of points on $\partial X\obs$ may be insufficient to prevent collisions.
To see why, consider a candidate discretized obstacle $X_p = \{p_1, p_2, \cdots, p_n\} \subset \partial X\obs$, with  $n \in \N$.
Then any $k \in \pi_K(X_p)$ may cause the robot to reach one or more $p_i \in X_p$.
Suppose $q \in \pi_K(X_p)^C$.
There is no guarantee that $\pi_X(q) \cap X\obs = \emptyset$, i.e. that $q$ would not cause a collision with the obstacle, because the robot may be able to travel between adjacent points in $X_p$ as shown in Figure \ref{fig:obs_no_buffer}.
To address this issue, we buffer the obstacle, then select points from its boundary with a maximum point spacing allowed between the points.

The purpose of this section, then, is to rigorously define the \defemph{buffer} and \defemph{point spacing} to enable constructing $X_p$.
First, we define the buffer:

\begin{defn}\label{def:buffer}
Let $b > 0$ be a distance, called a \defemph{buffer}.
The \defemph{buffered obstacle}, $X\obs^b \supset X\obs$, is a compact subset of $X$ such that the maximum Euclidean distance between $X\obs$ and $X\obs^b$ is $b$:
\begin{align}
    X\obs^b = \left\{p \in X\ \mid\ \exists~p' \in X\obs~\text{\textnormal{s.t.}}~\norm{p - p'}_2 \leq b\right\}. \label{eq:X_obs_b_buffered_obstacle}
\end{align}
\end{defn}

\noindent Buffering an obstacle by $b$ reduces the amount of free space available for the robot to navigate through.
To address this, we find an upper bound $\bbar$ on $b$ in Section \ref{subsubsec:bbar}.

Note that our use of buffering in RTD is different from typical trajectory planning methods.
Trajectory planners that only consider the dynamics of the center of mass of the robot require obstacles to be buffered to compensate for the robot's footprint \citep{elbanhawi2014sampling,kuwata2009rrt}.
However, in RTD, the footprint is already accounted for in the set $X_0$; buffering is only necessary to construct the discretized obstacle representation.

Having established the buffer $b$ and its upper bound $\bbar$, we can now define the point spacing.
To do so, we first inspect the geometry of the buffered obstacle $X\obs^b$, because the point spacing is used to discretize the boundary of $X\obs^b$.
The following lemma describes the geometry of the buffered obstacle.

\begin{lem}\label{lem:buffered_obs_arcs_and_lines}
The boundary of the buffered obstacle, consists of a finite set of line segments $L$ and a finite set of arcs $A$ of radius $b$.
More precisely, let $n_L \in \N$ (resp. $n_A \in \N$) denote the number of line segments (resp. arcs).
Let $L_i \in L$ (resp. $A_i \in A$) denote the $i$\ts{th} line segment (resp. arc).
Note that each $L_i$ and $A_i$ is a subset of $X$.
Then the boundary of the buffered obstacle can be written as the union of all of the lines and arcs:
\begin{align}
    \bd X\obs^b\quad=\quad\left(\bigcup_{i = 1}^{n_L} L_i\right)~\cup~\left(\bigcup_{i = 1}^{n_A} A_i\right).
\end{align}
\end{lem}

Now, consider a discretized obstacle $X_p$ that is generated by selecting a set of points from $\bd X\obs^b$ such that the points are spaced by a distance $r > 0$ along the line segments and by a distance $a > 0$ along the arcs, as illustrated in Figure \ref{fig:obs_with_buffer}.
\begin{defn}\label{def:point_and_arc_spacing}
We call $r > 0$ the \defemph{point spacing} and $a > 0$ the \defemph{arc point spacing}.
\end{defn}
\noindent We prove in Section \ref{subsec:finding_geometric_quantities} that, by selecting $r$ and $a$ as function of the buffer $b$, the robot cannot pass completely between any pair of points in $X_p$ and collide with an obstacle.

Similar to the upper bound $\bbar$ on the buffer, we find an upper bound $\rbar$ for $r$.
Recall that $\bbar$ limits the buffer $b$, to prevent obstacles from taking up too much free space.
On the other hand, $\rbar$ makes sure that the point spacing is small enough that the discretized obstacle can be used to ensure safety; that is, the points in the $X_p$ must be close enough that our robot cannot pass between them.
We use $r$ itself as an upper bound of $a$.

Now we have motivated the geometric quantities $b$, $\bbar$, $r$, $a$, and $\rbar$.
However, we still have not specified how to actually find these quantities.
To do so, we first need a geometric representation of the robot's dynamics, presented next.

\subsubsection{Geometric Representation of the Dynamics}\label{subsubsec:geom_trans_family}

To understand how to relate the motion of the robot's body to the discretized obstacle representation, we now provide a geometric expression for the robot's trajectories.

Notice that, along any trajectory of the high-fidelity model, we can treat the robot's body as the footprint $X_0$ subject to a planar translation and rotation (about the robot's center of mass).
This leads to the following definition.

\begin{defn}\label{def:R_t_translation_and_rotation_family}
We define a \defemph{transformation} $R_t \in \SE(2)$.
Each $R_t$ is given by a rotation angle $\theta_t \in [0,2\pi)$ and a translation vector $s_t \in \R^2$, so $R_t$ transforms a point $p \in \R^2$ as
\begin{align}
    R_tp = \begin{bmatrix}\cos\theta_t &-\sin\theta_t \\ \sin\theta_t &\cos\theta_t \end{bmatrix}(p - c) + s_t + c,
\end{align}
where $c \in \R^2$ is the center of rotation (which we typically consider as the geometric center of $X_0$ when applying $R_t$ to the robot).
The subscript indicates that the transformation is indexed by time $t \in [0,T]$.
We define a \defemph{transformation family} $\{R_t\ |\ t \in [0,T]\}$ of planar translations and rotations that is continuous with respect to $t$.
\end{defn}

\noindent To simplify exposition, we leave out ``$t \in [0,T]$'', and instead write $\{R_t\}$, when the time index is clear from context.
Note that the continuity of $\{R_t\}$ is important because we use transformation families to express the motion of the robot's body through space geometrically.

Though we are examining the motion of the robot's body, Definition \ref{def:R_t_translation_and_rotation_family} allows us to consider arbitrary rotations and translations of the robot's footprint independent of trajectories of the high-fidelity model.
This is important because the discretized obstacle should not depend on the high-fidelity model, only on the geometry of the robot's body.
To this end, we define the application of an arbitrary $R_t$ to the entire set $X_0 \subset \R^2$ as:
\begin{align}
    R_tX_0 = \{R_tp~\mid~p \in X_0\}.
\end{align}

To ensure that any $\{R_t\}$ is well-defined in the robot's spatial coordinates $X$, and to simplify exposition, we make the following assumption.

\begin{assum}\label{ass:X_subset_R2_contains_origin_xy-axes}
Recall that, in Definition \ref{def:X_and_X_0}, $X$ is called the $xy$-subspace of the robot's state space $Z\hi$, so $X \subset \R^2$.
We assume that the $X$ contains the origin.
\end{assum}

Next, in Section \ref{subsec:finding_geometric_quantities}, we use the geometric objects from Section \ref{subsubsec:geometric_objects} along with transformation families find the geometric quantities $\rbar$, $\bbar$, $r$, and $a$.

\subsection{Finding the Geometric Quantities}\label{subsec:finding_geometric_quantities}

In this section, we describe how to compute the geometric quantities described in Section \ref{subsec:robot_obstacle_geometry_motivation}.
The arguments presented in this section describe a procedure to compute those quantities for arbitrary convex, compact robot footprints.
The more casual reader can skip to the Section \ref{subsubsec:example_geometric_quantities}, which includes examples of these quantities for rectangular and circular footprints.

This section proceeds as follows.
First, in Section \ref{subsubsec:rbar} we find the maximum point spacing $\rbar$.
Second, in Section \ref{subsubsec:bbar} we upper bound the buffer distance with $\bbar$, which we call the maximum penetration distance.
Third, in Section  \ref{subsubsec:find_r}, given a choice of buffer $b \in (0,\bbar)$, we find the point spacing $r$.
Fourth, in Section \ref{subsubsec:find_a}, we find the arc point spacing $a$.
Finally, in Section \ref{subsubsec:example_geometric_quantities}, we give examples of these quantities.

\subsubsection{Bounding the Point Spacing}\label{subsubsec:rbar}

We now seek to understand how close together points must be in the discrete obstacle representation.
We do this by upper bounding the point spacing with the geometric quantity $\rbar$.
We find $\rbar$ first because the other quantities, $\bbar$, $\rbar$, $r$, and $a$ all depend on $\rbar$.

This discussion builds on Theorem 1 from \citet{width_of_a_chair}.
To build intuition, imagine a wall in $X$ with a gap that is large enough for the robot to pass through without touching the wall.
If we keep shrinking this gap, eventually the robot is unable to pass through.
In this subsection, informally, we find the largest gap that the robot cannot pass all the way through.
We use the size of the gap as the upper bound $\rbar$ on the spacings $r$ and $a$ when constructing $X_p$.
Imagine that the buffered obstacle's boundary is treated as the wall.
If the wall is sampled so that points are closer than $\rbar$ apart, this is akin to a gap of width at most $\rbar$ between each pair of points.

To proceed, we first formally define the notion of passing the robot's footprint through a line segment.
Then, we find the size of the ``largest gap'' discussed above.

To define ``passing through'' a gap, represented a line segment $I$, we first establish a half-plane $P_I$ that is ``defined'' by $I$; we use $P_I$ as a region that the robot begins in, so that, to pass through $I$, the robot must leave the half-plane $P_I$.
To create this half-plane, consider the function $\dels: \R^2\times\R^2\times\R^2 \to \R$ for which
\begin{align}\begin{split}
    \dels(e_1,e_2,p) =~&\frac{1}{\norm{e_2 - e_1}_2}\Big((e_{2y}-e_{1y})p_x ~- \\
    &- (e_{2x}-e_{1x})p_y - e_{2y}e_{1y} - e_{2y}e_{1x}\Big),\label{eq:delta_distance_func}
\end{split}\end{align}
where the subscript $x$ or $y$ denotes the corresponding coordinate of a point in $\R^2$.
If $I$ has distinct endpoints $\{e_1,e_2\}$, then $\dels(e_1,e_2,p)$ is the perpendicular distance from the point $p$ to the line defined by $I$.
The sign of $\dels(e_1,e_2,p)$ is positive if $p$ lies to the ``left'' of the line defined by $I$, relative to the ``forward'' direction from $e_1$ to $e_2$.
The function $\dels$ is illustrated in Figure \ref{subfig:signed_delta}.
We use $\dels$ to define a half-plane in $\R^2$ as follows:

\begin{defn}\label{def:halfplane_P_I}
Let $c \in X_0$ denote the center of mass $[x_{c,0},y_{c,0}]^\top$ of the robot's footprint at time $0$, as in Assumption \ref{ass:rigid_body_dynamics_ctr_of_mass}.
Let $I$ be a line segment as in Definition \ref{def:line_segment_I} with two distinct endpoints $E_I = \{e_1,e_2\}$.
Then $P_I \subset \R^2$ denotes the closed \defemph{half-plane defined by $I$}; this half-plane is determined by the line defined by $I$ and by $c$ as:
\begin{align}
    P_I = \left\{p \in X~\mid~\mathrm{sign}(\dels(e_1,e_2,p)) = \mathrm{sign}(\dels(e_1,e_2,c))\right\},\label{eq:P_I_halfplane_defined_by_I}
\end{align}
where $\mathrm{sign}(a) = 1$ for $a \geq 0$ and $-1$ otherwise.
Now suppose that $I$ is a line segment of length $0$, i.e. $e_1 = e_2$, so we cannot directly define $P_I$ as in \eqref{eq:P_I_halfplane_defined_by_I}.
Suppose that $e_1 \neq c$.
So, we can pick a point $e'$ for which $(e' - e_1)\cdot(c - e_1) = 0$ where $\cdot$ denotes the standard inner product on $\R^2$, so the line segment from $e_1$ to $c$ is perpendicular to the line segment from $e_1$ to $e'$.
Then, $P_I$ is given by \eqref{eq:P_I_halfplane_defined_by_I}, but using $e'$ in place of $e_2$.
\end{defn}
\noindent In the case where $e_1 = e_2 = c$, $P_I$ is undefined.
See Figures \ref{fig:passing_through} and \ref{subfig:signed_delta} for illustrations of the different cases of $P_I$.
Notice that, except when $e_1 = e_2 = c$, $P_1$ is always a closed half-plane, even if $c$ lies on the line defined by $I$.
The utility of $P_I$ is that, if the line defined by $I$ does not intersect $X_0$, then $X_0 \subset P_I$, i.e. $P_I$ contains $X_0$.
So, we can use $P_I$ as a region that the robot starts in at time $0$.

\begin{figure*}
    \centering
    \begin{subfigure}[t]{0.24\textwidth}
        \centering
        \includegraphics[width=\columnwidth]{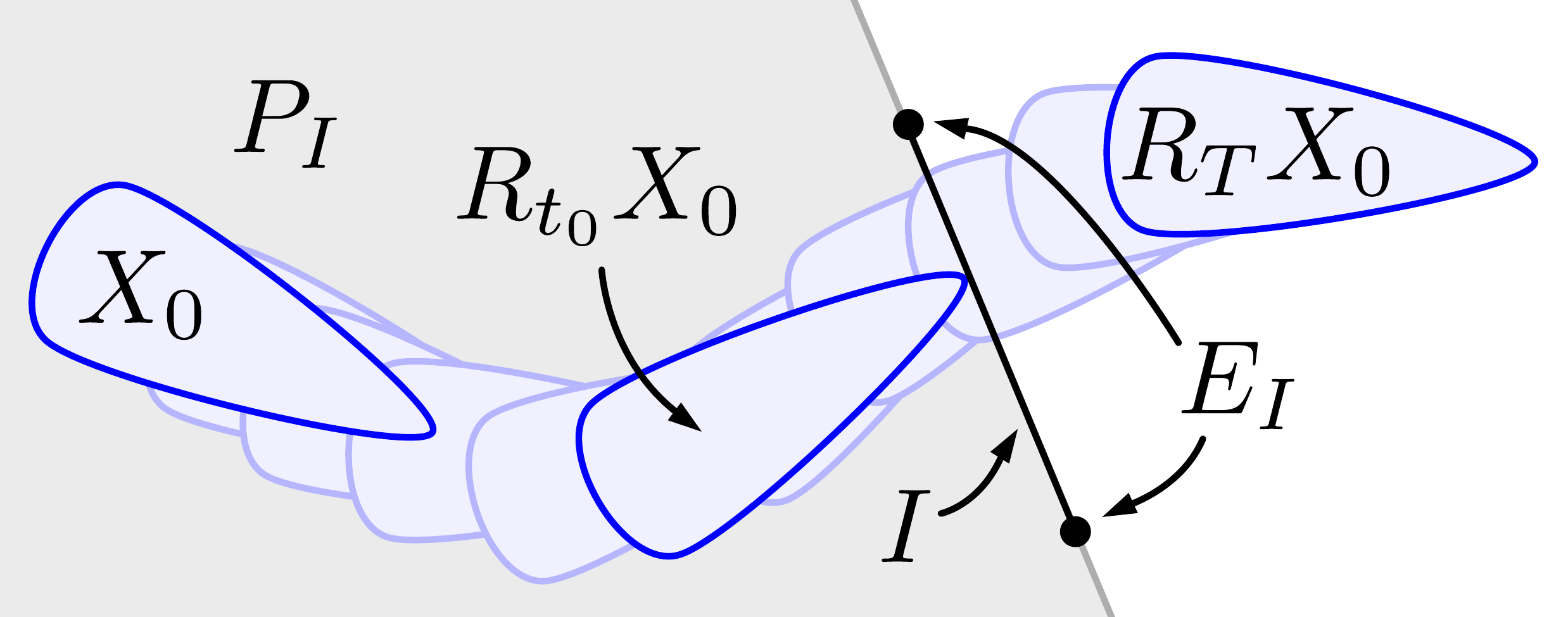}
        \caption{\centering}
        \label{subfig:pass_through_fully}
    \end{subfigure}
    \begin{subfigure}[t]{0.24\textwidth}
        \centering
        \includegraphics[width=\columnwidth]{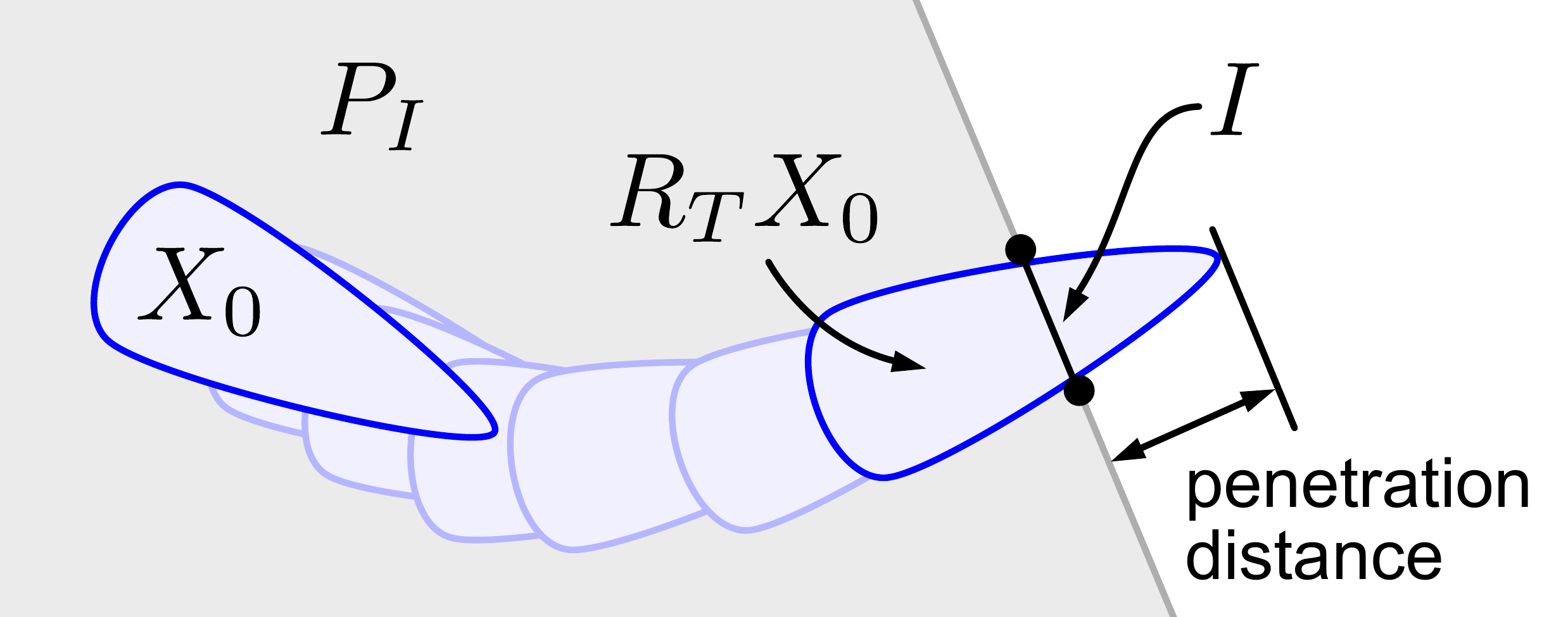}
        \caption{\centering}
        \label{subfig:penetrate}
    \end{subfigure}
    \begin{subfigure}[t]{0.24\textwidth}
        \centering
        \includegraphics[width=\columnwidth]{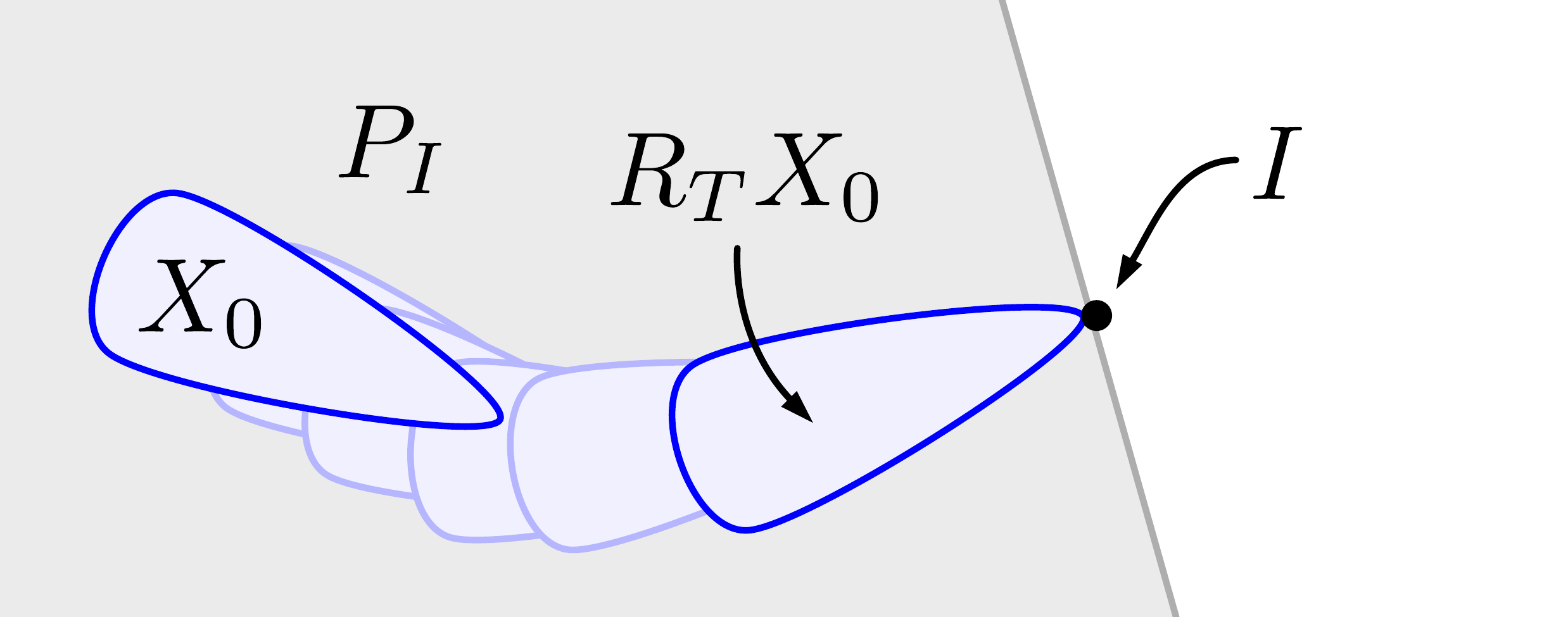}
        \caption{\centering}
        \label{subfig:pass_through_point}
    \end{subfigure}
    \begin{subfigure}[t]{0.24\textwidth}
        \centering
        \includegraphics[width=\columnwidth]{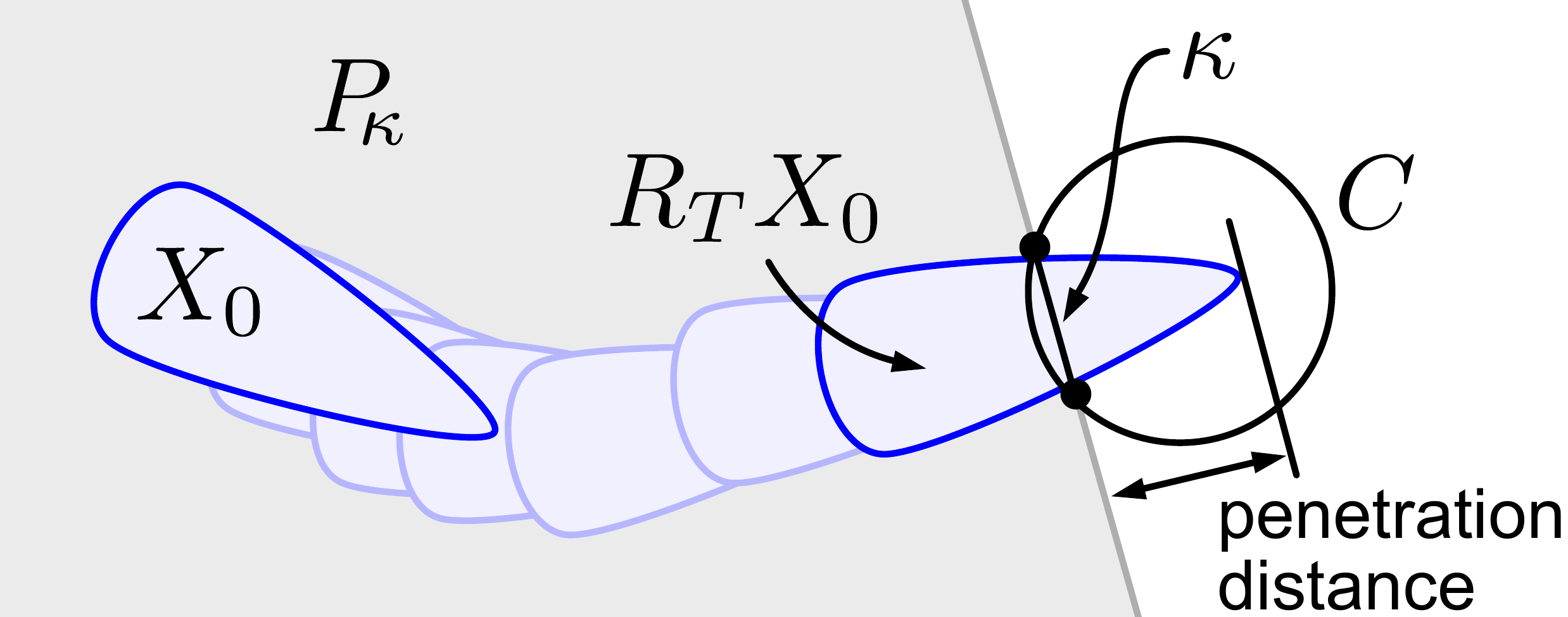}
        \caption{\centering}
        \label{subfig:pass_through_circle}
    \end{subfigure}
    \caption{
    Passing through (as in Definition \ref{def:pass_through}), penetrating (as in Definition \ref{def:penetrate}), and penetrating into a circle (as in Definition \ref{def:pass_and_penetrate_circles_and_arcs}).
    In each subfigure, a family $\{R_t\}_{t \in [0,T]}$ of continuous rotations and translations attempts to pass the convex, compact set $X_0$ through the line segment $I$ with endpoints $E_I$.
    At $t = 0$, $X_0$ lies in the half-plane $P_I$, defined by $I$ as in Definition \ref{def:halfplane_P_I}.
    Each figure contains $X_0$ at its initial position $R_0X_0$ and final position $R_TX_0$ indicated by a dark outline.
    The lighter outlines between these positions show examples of $X_0$ being translated and rotated as $\{R_t\}$ is applied.
    In Figure \ref{subfig:pass_through_fully}, $X_0$ is able to pass fully through $I$; the index $t_0 \in [0,T]$ where $X_0$ first touches $I$ is also shown with a dark outline.
    In Figure \ref{subfig:penetrate}, $X_0$ is unable to pass fully through $I$, but penetrates through $I$ by some distance into $P_I^C$.
    In Figure \ref{subfig:pass_through_point}, the line segment $I$ has length $0$, so $X_0$ cannot pass through it, but instead stops as soon as it touches $I$, and achieves $0$ penetration distance through $I$.
    Note that, in this case, $P_I$ is defined by a line perpendicular to the line segment from $I$ to the center of mass of $X_0$, as per Definition \ref{def:halfplane_P_I}.
    In Figure \ref{subfig:pass_through_circle}, the circle $C$ has a chord $\kappa$, and $X_0$ penetrates into $C$ through $\kappa$ by the penetration distance shown.
    The half-plane defined by $\kappa$ is denoted $P_\kp$.}
    \label{fig:passing_through}
\end{figure*}

\begin{defn}\label{def:pass_through}
Let $I \subset X\setminus X_0$ be a line segment with endpoints $E_I$ as in Definition \ref{def:line_segment_I}, and $P_I$ be the half-plane defined by $I$ as in Definition \ref{def:halfplane_P_I}.
Suppose that the robot lies fully within $P_I$ at time $0$, i.e. $X_0 \subset \mathrm{int}(P_I)$.
Let $\{R_t\}$ be a transformation family as in Definition \ref{def:R_t_translation_and_rotation_family}.
Let $t_0, t_1$ be indices in $(0,T]$ such that $R_tX_0$ intersects the ``middle'' of $I$, i.e. $R_tX_0 \cap (I\setminus E_I) \neq \emptyset$, for all $t \in [t_0, t_1]$.
Furthermore, suppose that $R_tX_0 \subset P_I$ for all $t \in [0,t_0)$, and that no $R_tX_0$ can intersect the endpoints $E_I$ (i.e. $R_tX_0 \cap E_I = \emptyset$) except at $t = T$.
We say that such a transformation family attempts to \defemph{pass $X_0$ through $I$}.
If $X_0$ is able to leave $P_I$ while passing through $I$, i.e. $R_TX_0 \subset P_I^C$, then $X_0$ is said to \defemph{pass fully through $I$}.
\end{defn}

\noindent See Figure \ref{fig:passing_through} for an illustration of passing through and passing fully through.
The motion of the robot at each $t$ is represented by each set $R_tX_0$.

Notice that, if $X_0$ must pass through $I$, it is not allowed to go ``around'' $I$ when passing through.
Furthermore, over the time horizon $[t_0,t_1]$ in Definition \ref{def:pass_through}, the set made by the intersection $R_tX_0 \cap I$ is a chord (as in Definition \ref{def:chord}) of $R_tX_0$ \citep[Theorem 1]{width_of_a_chair}.
We now state a property of $X_0$ used to bound the size of the aforementioned ``gap in a wall'' in Lemma \ref{lem:largest_gap_rbar} below.

\begin{defn}\label{def:thickness_and_width}
Given a unit vector in $\R^2$ at an angle $\theta$, the \defemph{thickness} of $X_0$ along this unit vector is the distance between the two unique lines that are tangent to $X_0$ and perpendicular to the vector.
The \defemph{width} of $X_0$ is defined as the minimum thickness of $X_0$ when searching over all $\theta \in [0,2\pi)$, and the \defemph{diameter} of $X_0$ is, similarly, the maximum thickness \citep[Section 1]{width_of_a_chair}.
\end{defn}

\begin{figure}
    \centering
    \begin{subfigure}[t]{0.49\columnwidth}
        \centering
        \includegraphics[width=0.75\columnwidth]{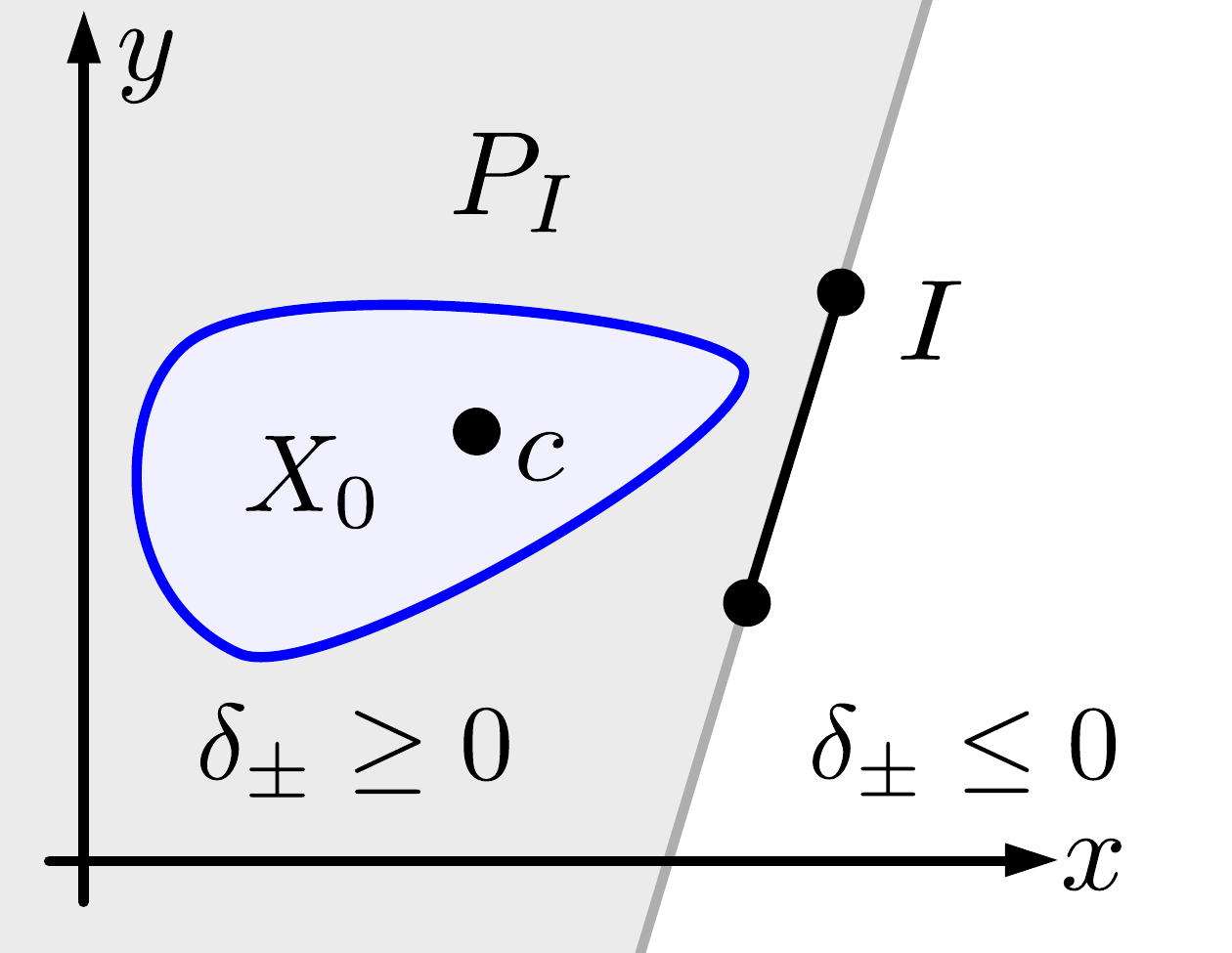}
        \caption{\centering}
        \label{subfig:signed_delta}
    \end{subfigure}
    \begin{subfigure}[t]{0.49\columnwidth}
        \centering
        \includegraphics[width=0.75\columnwidth]{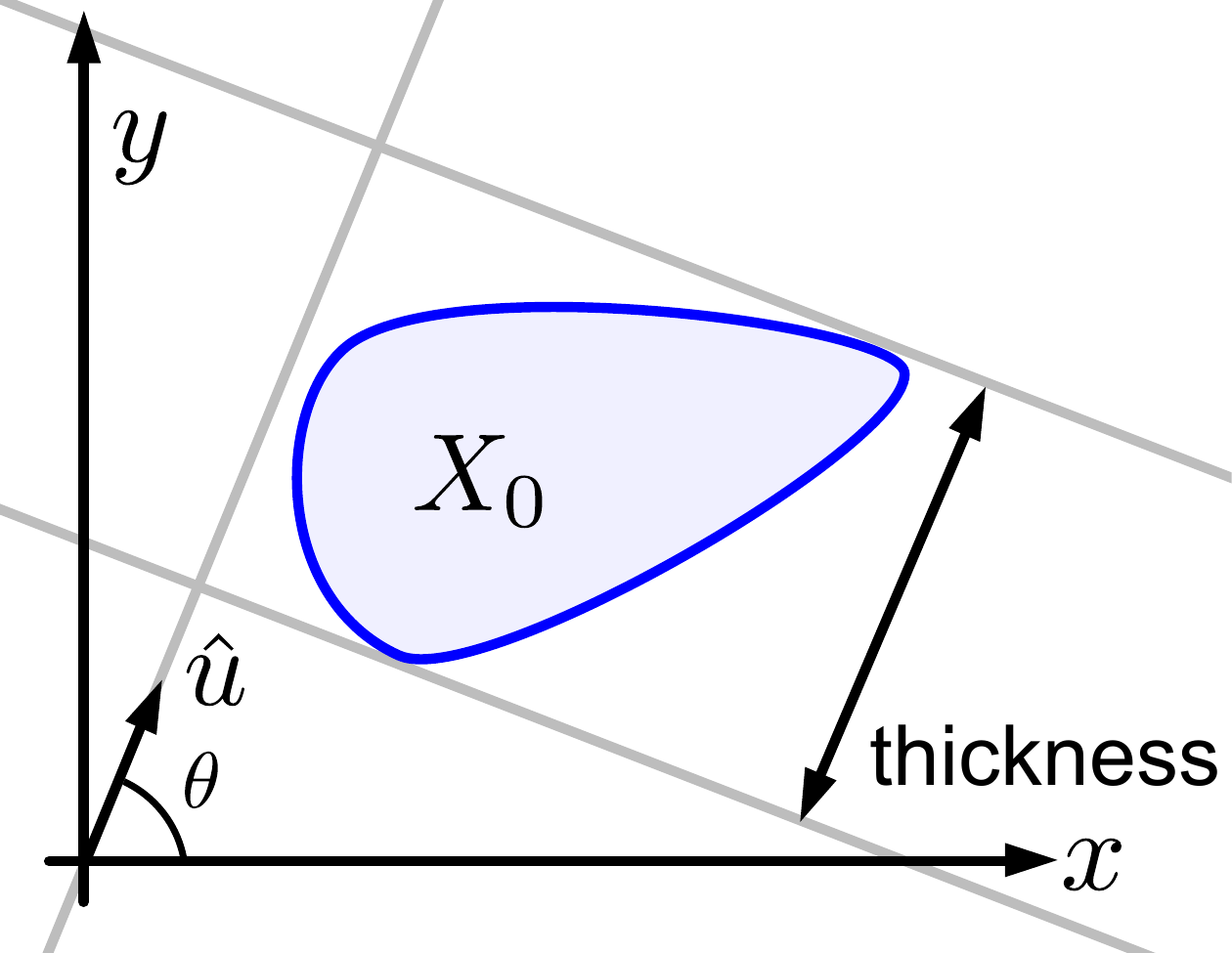}
        \caption{\centering}
        \label{subfig:thickness}
    \end{subfigure}
    \caption{An arbitrary, compact, convex set $X_0$ lies in the plane.
    In Figure \ref{subfig:signed_delta}, the line segment $I$ defines the closed half-plane $P_I$ (the filled grey area) using the function $\delta_\pm$ from \eqref{eq:delta_distance_func}.
    If the endpoints of $I$ are labeled $e_1$ and $e_2$, then the set $P_I$ contains all points $p \in \R^2$ for which the sign of $\delta_\pm(e_1,e_2,p)$ is the same as the sign of $\delta_\pm(e_1,e_2,c)$, where $c$ is the center of mass of $X_0$.
    In Figure \ref{subfig:thickness}, a unit vector $\hat{u}$ is fixed to the origin with angle $\theta$.
    The thickness of $X_0$ is given by the distance between the two unique lines that are tangent to $X_0$ and perpendicular to $\hat{u}$.}
\end{figure}

\noindent See Figure \ref{subfig:thickness} for an illustration of thickness.
Note that the width is nonzero and finite because $X_0$ is compact and has nonzero volume by Assumption \ref{ass:X0_cpt_cvx}.

\begin{lem}\label{lem:largest_gap_rbar}
\citep[Theorem 1]{width_of_a_chair} Let $I \subset (X\setminus X_0)$ be a line segment with endpoints $E_I$ and length $L > 0$ (as in Definition \ref{def:line_segment_I}).
Let $X_0$ be the robot's footprint at time $0$ (as in Definition \ref{def:X_and_X_0}), with width $W > 0$ (as in Definition \ref{def:thickness_and_width}).
Then $X_0$ can pass through $I$ (as in Definition \ref{def:pass_through}) if and only if $W < L$.
\end{lem}

From this lemma, the robot's width defines the smallest gap that the robot can pass through.
Therefore, we define $\rbar$ as the robot's width:

\begin{defn}\label{def:rbar}
The quantity $\rbar$ denotes the \defemph{maximum point spacing}, which is equal to the width of the robot footprint $X_0$ as in Definition \ref{def:thickness_and_width}.
\end{defn}

\noindent The maximum point spacing relates to the points in the discretized obstacle $X_p$ as follows.
As illustrated in Figure \ref{fig:obs_with_buffer}, the discretized obstacle $X_p$ is constructed by first buffering an obstacle $X\obs$ by a distance $b$ (see Definition \ref{def:buffer}), then sampling the boundary of $X\obs^b$ such that the distance between consecutive sampled points is strictly less than $\rbar$.

\begin{defn}\label{def:adjacent_points}
We refer to consecutive sampled points as \defemph{adjacent points} of the discretized obstacle $X_p$.
\end{defn}

\noindent We address the notion of adjacent points in more detail in Section \ref{subsubsec:construct_disc_obs_X_p}.
Suppose that we attempt to pass $X_0$ through the gap between two adjacent points of $X_p$, and do not allow $X_0$ to overlap with either of the points while passing through.
Since each pair of adjacent points of $X_p$ are strictly closer than $\rbar$ to each other, we know by Lemma \ref{lem:largest_gap_rbar} that the robot can never pass fully through the gap.
Consequently, finding $\rbar$ correctly is critical, leading to the following remark:

\begin{rem}\label{rem:rbar_exact}
The quantity $\rbar$ must either be found exactly or underapproximated to ensure safety.
If $\rbar$ is overapproximated, then the robot described by $X_0$ may be able to pass between a pair of points spaced slightly less than $\rbar$ apart, as per Lemma \ref{lem:largest_gap_rbar}.
Methods exist to exactly compute the width of arbitrary compact convex sets.
For example, the algorithm by \citet{rectangle_bound_curve} finds the smallest bounding rectangle of the set; then the length of the rectangle's shorter leg is the set's width.
A geometric procedure to find the width using the rotation angle $\theta$ and the thickness (as in Definition \ref{def:thickness_and_width}) is presented by \citet[Section 1]{width_of_a_chair}.
\end{rem}

Next, we use $\rbar$ to bound the buffer with the quantity $\bbar$.

\subsubsection{Bounding the Buffer}\label{subsubsec:bbar}

As in Section \ref{subsubsec:rbar}, imagine a wall with gap of width $\rbar$.
Lemma \ref{lem:largest_gap_rbar} proves that the robot cannot pass fully through this gap.
However, the robot can still penetrate through the gap by some distance before it gets stopped by the wall.
In this section, we find the farthest distance that the robot can penetrate through the gap.
We use this maximum penetration distance as an upper bound $\bbar$ on the obstacle buffer, so $b \in (0,\bbar)$.

Recall that our intention is to discretely sample the boundary of the buffered obstacle in \eqref{eq:X_obs_b_buffered_obstacle} to produce a set $X_p$, so the spacing between adjacent points of $X_p$ (as in Definition \ref{def:adjacent_points}) must be smaller than $\rbar$.
If the robot is not allowed to touch any points in $X_p$, it cannot penetrate farther than the distance $\bbar$ between any pair of adjacent points.
So, obstacles do not need to be buffered by a distance larger than $\bbar$.
The existence of $\bbar$ is proven in Lemma \ref{lem:max_penetration_bbar}.
To proceed, we first define the word ``penetrate'' precisely.

\begin{defn}\label{def:penetrate}
Let $I \subset (X\setminus X_0)$ be a line segment as in Definition \ref{def:line_segment_I}.
Let $P_I$ be the half-plane defined by $I$ as in Definition \ref{def:halfplane_P_I}, and suppose $X_0 \subset P_I$ strictly.
Let $\{R_t\}$ be a transformation family that attempts to pass $X_0$ through $I$ by Definition \ref{def:pass_through}.
Suppose $X_0$ cannot pass fully through $I$, and that $R_TX_0 \cap P_I^C$ is nonempty, so there is some portion of $X_0$ that does pass through $I$.
Consider all line segments perpendicular to $I$ with one endpoint on $I$ and the other at a point in $R_TX_0$ in $P_I^C$.
We call the maximum length of any of these line segments the \defemph{penetration distance of $X_0$ through $I$}.
The set $R_TX_0$ \defemph{penetrates} $I$ by this distance, as in Figure \ref{subfig:penetrate}.
If $I$ is of length $0$, then the penetration distance of $X_0$ through $I$ is always $0$, as in Figure \ref{subfig:pass_through_point}.
\end{defn}

\begin{lem}\label{lem:max_penetration_bbar}
Let $X_0$ be the robot's footprint at time 0 (as in Definition \ref{def:X_and_X_0}), with width ${\rbar}$ (as in Definition \ref{def:rbar}).
Let $\irbar \subset (X\setminus X_0)$ be a line segment of length ${\rbar}$ (as in Definition \ref{def:line_segment_I}).
Then there exists a \defemph{maximum penetration distance} $\bbar$ (as in Definition \ref{def:penetrate}) that can be achieved by passing $X_0$ through $\irbar$ (as in Definition \ref{def:pass_through}).
\end{lem}

\begin{figure}
    \centering
    \includegraphics[width=0.6\columnwidth]{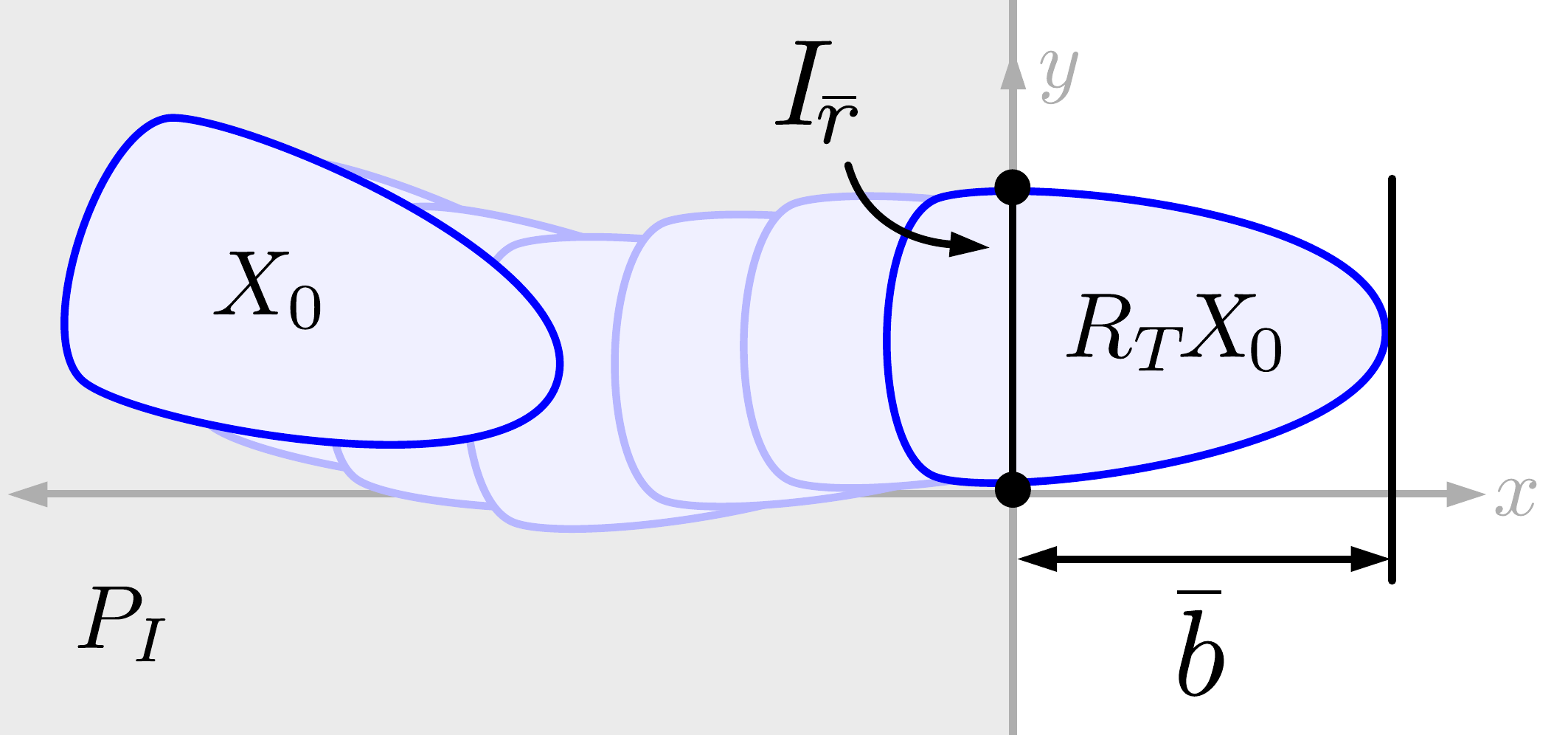}
    \caption{An arbitrary compact, convex set $X_0$ of width $\rbar$ penetrates a line segment $\irbar$ by the distance $\bbar$ when a transformation family $\{R_t\}_{t \in [0,T]}$ is applied to pass $X_0$ through $\irbar$ as in Definition \ref{def:pass_through}.
    Since $\irbar$ is of length $\rbar$, $X_0$ cannot pass fully through by Lemma \ref{lem:largest_gap_rbar}.
    At the initial index $t = 0$ and the final index $t = T$, the sets $R_0X_0$ and $R_TX_0$ are shown with dark outlines.
    A sampling of intermediate indices $t \in (0,T)$ are shown with light outlines.}
    \label{fig:max_penetration_bbar}
\end{figure}

\noindent To relate Lemma \ref{lem:max_penetration_bbar} to the robot, consider the following.
If we buffered an obstacle $X\obs$ by the amount $\bbar$, and spaced points along the boundary of $X\obs$ by a distance less than $\rbar$, then the \textit{farthest} that the robot could pass between any pair of adjacent points \textit{without touching either point} is strictly less than $\bbar$.
Therefore, the robot could not collide with the obstacle without touching one of the points.
In other words, if the robot avoids every such point, then the robot avoids the obstacle.
Consequently, finding $\bbar$ correctly is important, which we emphasize with the following remark:

\begin{rem}\label{rem:bbar_inexact}
To find the point spacing $r$, we must ensure that $b$ is in fact less than the maximum penetration distance. 
As a result, it is critical to underapproximate $\bbar$.
\end{rem}

\noindent A geometric method for finding $\bbar$ for an arbitrary convex robot footprint is presented in Lemma \ref{lem:rotate_then_translate_to_penetrate} in Appendix \ref{app:obs_rep}.

Next, we find the point spacing $r$.

\subsubsection{Finding the Point Spacing}\label{subsubsec:find_r}

Let $\rbar$ be as in Definition \ref{def:rbar} and $\bbar$ as in Lemma \ref{lem:max_penetration_bbar}.
We choose $b \in (0,\bbar)$, then use $b$ to find the point spacing $r$ (Definition \ref{def:point_and_arc_spacing}).
We prove that $r$ exists with the following lemma.

\begin{lem}\label{lem:find_r}
Let $X_0 \subset \R^2$ be the robot's footprint at time $0$ (as in Definition \ref{def:X_and_X_0}), with width $\rbar$ (as in Definition \ref{def:rbar}).
Let $\bbar$ be the maximum penetration depth corresponding to $X_0$ (as in Lemma \ref{lem:max_penetration_bbar}).
Pick $b \in (0,\bbar)$.
Then there exists $r \in (0,\rbar]$ such that, if $I_r$ is a line segment of length $r$ (as in Definition \ref{def:line_segment_I}), and if $\{R_t\}$ is any transformation family that attempts to pass $X_0$ through $I_r$ (as in Definition \ref{def:pass_through}), then the penetration distance of $X_0$ through $I_r$ (as in Definition \ref{def:penetrate}) is less than or equal to $b$.
\end{lem}

Lemma \ref{lem:find_r} states the existence of the point spacing $r$.
A method compute $r$ for arbitrary compact, convex robot footprints is presented in Appendix \ref{app:obs_rep} in the proof of Lemma \ref{lem:find_r}.
We find $r$ analytically for rectangular and circular footprints in Examples \ref{ex:rectangle_X0} and \ref{ex:circular_X0} below.

We use $r$ as follows.
Suppose our robot has a footprint $X_0$ as in Definition \ref{def:X_and_X_0}, with width $\rbar$ as in Definition \ref{def:rbar}, and associated maximum penetration distance $\bbar$ as in Lemma \ref{lem:max_penetration_bbar}.
Pick $b \in (0,\bbar)$.
Suppose $X\obs \subset X$ is an obstacle as in Definition \ref{def:obs}, and let it consist of polygons as in Assumption \ref{ass:obs_are_polygons}.
Construct $X\obs^b$, the buffered obstacle, with \eqref{eq:X_obs_b_buffered_obstacle}.
Recall by Lemma \ref{lem:buffered_obs_arcs_and_lines} that the boundary of the buffered obstacle consists of line segments and arcs.
Then, $r$ lets us construct the portion of the discretized obstacle $X_p$ that corresponds to the line segments in $\bd X\obs^b$.
In particular, suppose we sample each line segment of $\bd X\obs^b$ such that adjacent points (as in Definition \ref{def:adjacent_points}) are no farther than $r$ apart.
Then, by Lemma \ref{lem:find_r}, if $I_r$ is a line segment between two of these adjacent points, the robot can penetrate no further than $b$ through $I_r$ (as in Definition \ref{def:penetrate}).
In other words, the robot cannot reach $X\obs$ by going ``between'' the adjacent points of the line segments.

However, we have not yet explained how to sample the arcs of $\bd X\obs^b$.
We do so next, by finding the arc point spacing $a$.

\subsubsection{Finding the Arc Point Spacing}\label{subsubsec:find_a}

Note that we cannot necessarily use $r$ as the point spacing distance when sampling the arcs of $\bd X\obs^b$.
To understand why, informally, imagine $X_0$ penetrating into a circle $C$ of radius $b \in (0,\bbar)$ instead of a line segment of length $\rbar$ as in Lemma \ref{lem:find_r}.
Suppose that $X_0$ stops when it touches the center of the circle.
For the sake of argument, suppose that the boundary $\bd X_0$ (which exists because $X_0$ is compact by Assumption \ref{ass:X0_cpt_cvx}) intersects $C$ in exactly two points; then, in the intersection of $X_0$ with $C$, there is an arc of radius $b$ between these two points.
If the length of this arc were equal to $r$, for an arbitrary convex $X_0$, then we could sample ``along'' each arc by the distance $r$.
But this is not true in general; one can check that it is false if $X_0$ is circular, as in Example \ref{ex:circular_X0}.
Therefore, we need a different point spacing for the arcs, which is the \defemph{arc point spacing} $a$ as in Definition \ref{def:point_and_arc_spacing}.

Before finding the arc point spacing $a$, we extend the concepts of passing through and penetrating from line segments to circles and arcs:
\begin{defn}\label{def:pass_and_penetrate_circles_and_arcs}
Let $C \subset \R^2$ be a circle of radius $R$ with center $p$ as in Definition \ref{def:circle_and_arc}.
Let $X_0$ be the robot's footprint at time $0$ as in Definition \ref{def:X_and_X_0}.
Let $\kp$ be a chord of $C$ as in Definition \ref{def:chord}.
Then \defemph{passing $X_0$ into $C$ through $\kp$} is defined as passing $X_0$ through the chord $\kp$ as in Definition \ref{def:pass_through}.
If the length of $\kp$ is less than the width of $X_0$, then, by Lemma \ref{lem:largest_gap_rbar}, $X_0$ cannot pass fully through $\kp$, but does penetrate the chord up to some distance as in Definition \ref{def:penetrate}.
Let $P_{\!\kp}$ be the closed half-plane defined by $\kp$ as in Definition \ref{def:halfplane_P_I}.
The \defemph{penetration of $X_0$ into $C$ through $\kp$} is the maximum Euclidean distance from any point in $X_0\cap C$ to a point in $X_0 \cap P_{\!\kp}^C$.
\end{defn}

\noindent This definition is illustrated in Figure \ref{subfig:pass_through_circle}.
We prove that $a$ exists with the following lemma.

\begin{lem}\label{lem:find_a}
Let $X_0$ be the robot's footprint at time $0$ (as in Definition \ref{def:X_and_X_0}), with width $\rbar$ (as in Definition \ref{def:rbar}).
Let $\bbar$ be the maximum penetration distance corresponding to $X_0$ (as in Lemma \ref{lem:max_penetration_bbar}).
Pick $b \in (0,\bbar)$, and let $C \subset (X\setminus X_0)$ be a circle of radius $b$ centered at a point $p \in X$ (as in Definition \ref{def:circle_and_arc}).
Then there exists a number $a \in (0,\rbar)$ such that, if $\kp_a$ is any chord of $C$ of length $a$ (as in Definition \ref{def:chord}), then the penetration of $X_0$ into $C$ through $\kp$ (as in Definition \ref{def:pass_and_penetrate_circles_and_arcs}) is no larger than $b$.
\end{lem}

\noindent Lemma \ref{lem:find_a} provides the arc point spacing $a \in (0,\rbar)$.
The proof in Appendix \ref{app:obs_rep} explains a method for finding $a$ for arbitrary compact, convex robot footprints.
We can in fact prove a tighter bound, that $a$ is always shorter than the point spacing $r$ from Lemma \ref{lem:find_r}; in other words, $a \in (0,r)$.
This claim is straightforward to prove using the same techniques from the proof of Lemma \ref{lem:find_a}.

As with $r$, we find $a$ analytically for rectangular and circular footprints in Examples \ref{ex:rectangle_X0} and \ref{ex:circular_X0}, presented next.

\subsubsection{Example Geometric Quantities}\label{subsubsec:example_geometric_quantities}

Now, we have completed finding the geometric quantities $\rbar$ (Lemma \ref{lem:largest_gap_rbar}), $\bbar$ (Lemma \ref{lem:max_penetration_bbar}), $r$ (Lemma \ref{lem:find_r}), and $a$ (Lemma \ref{lem:find_a}) that were desired at the outset of Section \ref{subsec:robot_obstacle_geometry_motivation}.
To wrap up, we note what to do if $r$ and $a$ cannot be found exactly, then present two examples of $\rbar$, $\bbar$, $r$, and $a$ for rectangular and circular robot footprints.

\begin{rem}\label{rem:r_and_a_hard_to_find}
Since $r$ and $a$ are point spacings for discretizing obstacles, they must be underapproximated if they cannot be found exactly, by logic similar to that of Remark \ref{rem:rbar_exact}.
Otherwise, the robot may be able to penetrate farther than the distance $b$ between them.

For some convex, compact robot footprints, such as the rectangle and circle in Examples \ref{ex:rectangle_X0} and \ref{ex:circular_X0}, $r$ and $a$ can be found analytically.
For arbitrary convex, compact footprints, one can use the procedure proposed in Appendix \ref{app:obs_rep} (Lemma \ref{lem:rotate_then_translate_to_penetrate}) for finding $\bbar$, but limit the penetration distance to $b$; then one can find $r$ as in Lemma \ref{lem:find_r} with \eqref{eq:delv_find_r_vertical_span} and find $a$ as in Lemma \ref{lem:find_a} by placing a circle $C_b$ and finding the chord $\kp_a$.
\end{rem}

\begin{ex}\label{ex:rectangle_X0}
(Rectangular footprint).
Suppose $X_0$ is a rectangle of length $L$ and width $W$, with $L > W$.
Then $\rbar = W$ and $\bbar = W/2$.
Given $b \in (0,\bbar)$, $r = 2b$ and $a = 2b\sin(\pi/4)$.
A visual proof is in Figure \ref{subfig:ex_rectangle}.
\end{ex}

\begin{ex}\label{ex:circular_X0}
(Circular footprint).
Suppose $X_0$ is a circle of radius $R$.
Then $\rbar = 2R$ and $\bbar = R$.
Pick $b \in (0,\bbar)$.
Define the positive angles $\theta_1 = \cos\inv\left(\frac{R - b}{R}\right)$ and $\theta_2 = \cos\inv\left(\frac{b}{2R}\right)$.
Then we find $r = 2R\sin\theta_1$ and $a = 2b\sin\theta_2$.
A visual proof is in Figure \ref{subfig:ex_circle}.
\end{ex}

\begin{figure}
    \centering
    \begin{subfigure}[t]{0.45\columnwidth}
        \centering
        \includegraphics[width=\columnwidth]{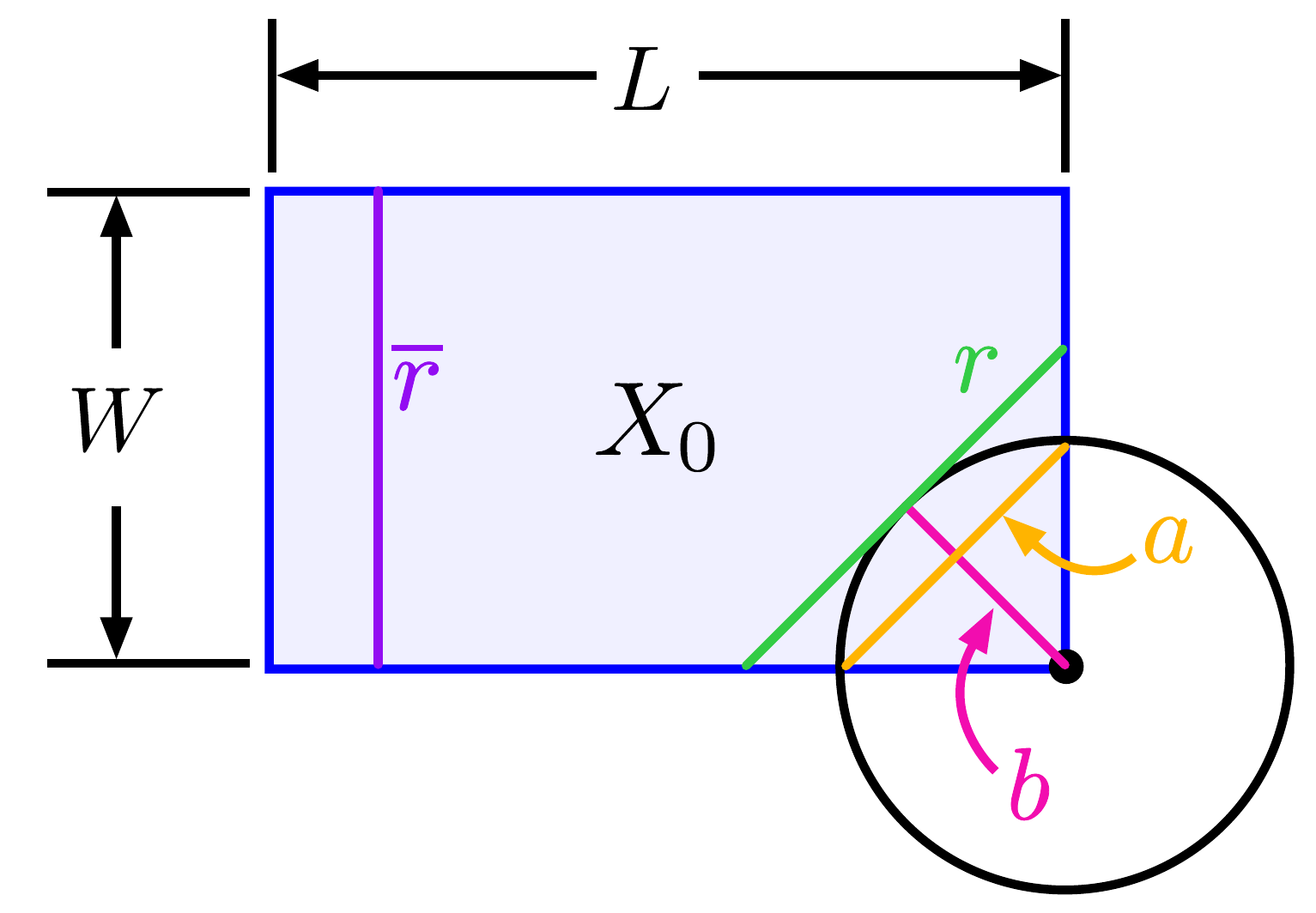}
        \caption{\centering}
        \label{subfig:ex_rectangle}
    \end{subfigure}
    \begin{subfigure}[t]{0.45\columnwidth}
        \centering
        \includegraphics[width=\columnwidth]{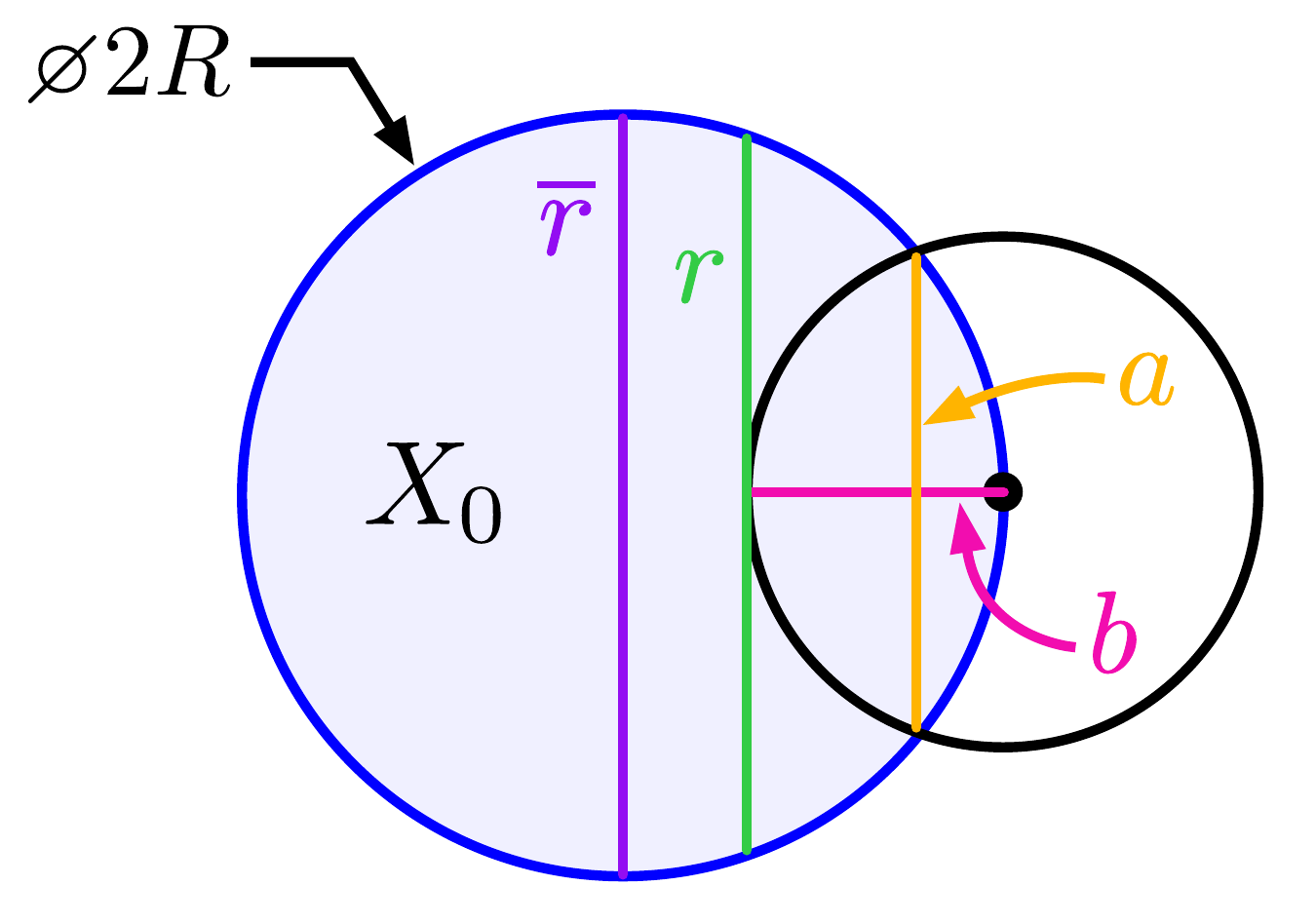}
        \caption{\centering}
        \label{subfig:ex_circle}
    \end{subfigure}
    \caption{
    An illustration of the numbers $\rbar$, $b$, $r$, and $a$ for rectangular and circular robot footprints (see Examples \ref{ex:rectangle_X0} and \ref{ex:circular_X0}).
    The left subfigure shows a rectangular footprint, with length $L$ and width $W$.
    The right subfigure shows a circular robot footprint with diameter $2R$.
    % For both footprints, the maximum point spacing $\rbar$ (as in Definition \ref{def:rbar}) is in purple.
    The maximum penetration distance $\bbar$ (as in Lemma \ref{lem:max_penetration_bbar}) is omitted for clarity.
    % The buffer $b$ is in magenta.
    % The point spacing $r$ is in mint green.
    % The arc point spacing $a$ is in orange.
    The circle centered on $\bd X_0$ corresponds to the circle $C_b$ used in \eqref{prog:find_a} to find $a$ (see Lemma \ref{lem:find_a}).}
    \label{fig:rbar_b_r_and_a_examples_rect_and_circ}
\end{figure}

This completes finding the geometric quantities $\rbar$, $\bbar$, $r$, and $a$.
Next, we use these quantities to construct the discretized obstacle, and prove that this enables identifying the set of safe trajectory parameters $K\safe$.

\subsection{Preserving Safety with Discretized Obstacles}\label{subsec:proving_X_p_works}

Now we present an algorithm to take a buffered obstacle and discretize its boundary, producing the discretized obstacle $X_p$.
We prove in Theorem \ref{thm:X_p} that, if the robot cannot collide with any point in $X_p$, then it also cannot collide with the obstacle.
Finally, we discuss sources of conservatism in the discretization approach.

\subsubsection{Constructing the Discretized Obstacle}\label{subsubsec:construct_disc_obs_X_p}

To proceed, we first get the buffered obstacle, then establish three useful functions for constructing the discretized obstacle with Algorithm \ref{alg:construct_X_p}.
We then prove that the discretized obstacle can be used to identify safe trajectory parameters with Theorem \ref{thm:X_p}.

To get the buffered obstacle, let $X\obs$ consist of polygons (as in Assumption \ref{ass:obs_are_polygons}).
Suppose $X_0$ is the robot's footprint at time $0$ (as in Definition \ref{def:X_and_X_0}), which is compact and convex with nonzero volume (as in Assumption \ref{ass:X0_cpt_cvx})
Suppose that $\rbar$ is found for $X_0$ as in Definition \ref{def:rbar} and $\bbar$ as in Lemma \ref{lem:max_penetration_bbar}.
Select $b \in (0,\bbar)$, then find $r$ with \eqref{prog:find_r} and $a$ with \eqref{prog:find_a}.
Buffer the obstacle to produce $X\obs^b$ as in \eqref{eq:X_obs_b_buffered_obstacle}.
Now, we can discretize $\bd X\obs^b$.

The first two functions extract the lines and arcs from the boundary of the buffered obstacle.
Then, by  Lemma \ref{lem:buffered_obs_arcs_and_lines}, we can rewrite $\bd X\obs^b = L\cup A$ where $L$ is a finite set of closed line segments (as in Definition \ref{def:line_segment_I}) and $A$ is a finite set of closed arcs (as in Definition \ref{def:circle_and_arc}).
Let $n_L$ be the number of line segments and $n_A$ be the number of arcs.
For $i = 1,\cdots, n_L$, let $L_i \subset L$ denote the $i$\ts{th} segment, and similarly $A_i \subset A$ for the $i$\ts{th} arc.
Then the function \defemph{\texttt{extractLines}} takes in the buffered obstacle $X\obs^b$ and returns the set $L$ of all line segments on $\bd X\obs^b$.
Similarly, the function \defemph{\texttt{extractArcs}} takes in $X\obs^b$ and returns the set $A$ of all circular arcs on $\bd X\obs^b$.

We now define a third function, \defemph{\texttt{sample}}$: \P(\R^2)\times \R \to \P(\R^2)$, to discretize the line segments and arcs.
Suppose $S \subset \R^2$ is a connected curve with exactly two endpoints and no self-intersections; note we are conflating a curve with its image.
Let $s > 0$ be a distance.
Then $P = \texttt{sample}(S,s)$ is a set containing the endpoints of $S$.
Furthermore, if the total arclength along $S$ is greater than $s$, then $P$ also contains a finite number of points spaced along $S$ such that, for any point in $P$, there exists at least one other point that is no farther away than the arclength $s$ along $S$.
Note that the line segments in $L$ and the arcs in $A$ can be parameterized; then the \texttt{sample} function can be implemented using interpolation of a parameterized curve.

\begin{algorithm}
\small
\begin{algorithmic}[1]
\State {\bf Require:} $X\obs^b \subset X$, $r\in \R_{\geq 0}$,\ $a \in \R_{\geq 0}$
\State $L \leftarrow \texttt{extractLines}\left(X\obs^b\right),\ A \leftarrow \texttt{extractArcs}\left(X\obs^b\right)$
\State $X_p \leftarrow \emptyset$
\State{\bf For each:}  $i \in \{1,\ldots,n_L\}$
    \State\hspace{.2in} $X_p \leftarrow X_p \cup \texttt{sample}(L_i,r)$
\State{\bf end}
\State{\bf For each:}  $j \in \{1,\ldots,n_A\}$
    \State\hspace{.2in}$X_p \leftarrow X_p \cup \texttt{sample}(A_j,a)$
\State{\bf end}
\State{\bf Return} $X_p$
\end{algorithmic}
\caption{\small Construct Discretized Obstacle (\texttt{discretizeObs})}
\label{alg:construct_X_p}
\end{algorithm}

Suppose that $X_p$ is constructed from a buffered obstacle $X\obs^b$ using Algorithm \ref{alg:construct_X_p}.
Then $X_p$ contains the endpoints of each line segment or arc of $\bd X\obs^b$, since it is constructed using \texttt{sample}.
In addition, for each line segment of $\bd X\obs^b$, $X_p$ contains additional points spaced along the line segment such that each point is within the distance $r$ (in the 2-norm) from at least one other point.
Similarly, for each arc of $\bd X\obs^b$, $X_p$ contains points spaced along the arc such that each point is within the arclength $a$ of at least one other point; this implies that distance between any pair of adjacent points along each arc is no more than $a$.
Finally, note that $|X_p|$ is finite, because there are a finite number of polygons in $X\obs$ (see Assumption \ref{ass:obs_are_polygons}), each polygon has a finite number of edges, and $r, a > 0$.

\subsubsection{Proving Safety}
Now, we formalize the notion that $X_p$ represents the obstacles $X\obs$ without affecting the guarantee of safety.
Recall that the purpose of constructing $X_p$ is to map the obstacles $X\obs$ into the parameter space $K$ via the map $\pi_K: X \to K$ as in \eqref{eq:pi_K}.
To ensure safety, the set $\pi_K(X_p)$ must contain all possible unsafe trajectory parameters $\pi_K(X\obs)$, which is the complement of the set $K\safe$, leading to the following theorem:

\begin{thm}\label{thm:X_p}
Let $X_0$ be the robot's footprint at time 0 as in Definition \ref{def:X_and_X_0}, with width $\rbar$ as in Definition \ref{def:rbar}.
Let $X\obs \subset (X\setminus X_0)$ be a set of obstacles as in Definition \ref{def:obs}.
Suppose that the maximum penetration depth $\bbar$ is found for $X_0$ as in Lemma \ref{lem:max_penetration_bbar}.
Pick $b \in (0,\bbar)$, and find the point spacing $r$ with \eqref{prog:find_r} and the arc point spacing $a$ with \eqref{prog:find_a}.
Construct the discretized obstacle $X_p$ in Algorithm \ref{alg:construct_X_p}.
Then, the set of all unsafe trajectory parameters corresponding to $X\obs$ is a subset of the trajectory parameters corresponding to $X_p$, i.e. $\pi_K(X_p) \supseteq \pi_K(X\obs)$.
\end{thm}

Theorem \ref{thm:X_p} provides the main result of this section: $\pi_K(X_p)^C \subseteq K\safe$.
In other words, we can use $X_p$ to inner approximate $K\safe$.
Next, we discuss the conservatism of the proposed obstacle representation.

\subsubsection{Conservatism of the Discretized Obstacle in Practice}\label{subsubsec:obs_rep_conservatism}

Our approach has two sources of conservatism.
The first is that we compute overapproximations to the robot's FRS.
The second is that we must buffer obstacles to discretize them.

First, we discuss the FRS overapproximation from Section \ref{sec:FRSmethod}.
Note that, in Theorem \ref{thm:X_p}, the FRS parameter projection map $\pi_K$ is defined by an exact solution to $(D)$.
However, as per Section \ref{subsec:FRS_implementation}, we can only compute solutions to the relaxed problem $(D^l)$.
This does not affect the safety guarantee of the discretized obstacle, as we note in the following remark.

\begin{rem}\label{rem:w_l_for_pi_X}
Suppose $(v^l,w^l,q^l)$ is a solution to $(D^l)$ for some $l \in \N$.
Define the map $\pi_K^l: \P(X) \to \P(K)$ by
\begin{align}
    % \pi_X^l(K') &= \{x \in X\ \mid~\exists~k \in K'~\mathrm{s.t.}~w^l(x,k) \geq 1\}, \label{eq:pi_X_l} \\
    \pi_K^l(X') &= \Big\{k \in K\ \mid~\exists~x \in X'~\mathrm{s.t.}~w^l(x,k) \geq 1 \Big\} \label{eq:pi_K_l}.
\end{align}
By Lemma \ref{rem:w_geq_1_is_in_FRS}, $\X\frs$ is contained in the 1-superlevel set of $w^l$.
This means that, for any $X' \subset X$, $\pi_K^l$ overapproximates $\pi_K$, i.e. $\pi_K^l(X') \supseteq \pi_K(X')$.
Therefore, if $k \in K\safe$ corresponds to some $X\obs$, then $\pi_K^l(X\obs)^C \subseteq \pi_K(X\obs)^C = K\safe$.
In other words, Theorem \ref{thm:X_p} still holds for $\pi^l$.
\end{rem}

Second, to conclude this section, we discuss the effect of choosing different buffers $b \in (0,\bbar)$.
If the buffer is small, then $r$ and $a$ must be small, according to Lemmas \eqref{lem:find_r} and \eqref{lem:find_a}.
If $r$ and $a$ are small, by Algorithm \ref{alg:construct_X_p}, the points of $X_p$ are spaced closer together, increasing $|X_p|$.
Each point in $X_p$ is mapped to a nonlinear constraint at runtime, so increasing $|X_p|$ may increase the execution time of the trajectory optimization.
If the buffer is large, then $r$ and $a$ can be larger, but the trajectory optimization may become more conservative, because buffering obstacles reduces the total free space available for the robot to move through.
A large buffer may also not reduce the trajectory optimization execution time, because increasing $b$ allows for increasing $r$ and $a$ only up to $\rbar$.
So, $b$ should be chosen as small as possible to reduce conservatism, but large enough to enable fast execution times.
We emphasize that \textit{every} choice of $b \in (0,\bbar)$ can be used for safe online planning by Theorem \ref{thm:X_p}, so it is possible to find  $b$ empirically without risking safety.

In summary, this section has presented a method for constructing a discrete, finite obstacle representation.
We have shown in Theorem \ref{thm:X_p} that this representation enables identifying the safe set of trajectory parameters in any planning iteration.
Next, we use the proposed obstacle representation for online planning.
\section{Online Trajectory Optimization}\label{sec:trajectory_optimization}

This section describes the real-time trajectory planning part of RTD, presented in Algorithm \ref{alg:trajopt}.
Recall that, given an arbitrary $k \in K$, the robot generates a feedback controller $u_k$ \eqref{eq:fdbk_controller_uk} that tracks $k$.
Given a user specified cost function $J:K\rightarrow \R$, and an initial robot state $\z_{\text{hi},0}\in X\hi$, Algorithm \ref{alg:trajopt} performs trajectory optimization, i.e., picks a new $k$ at each planning iteration, in a real-time, provably safe manner.

\subsection{Online Planning Algorithm Overview}

Algorithm \ref{alg:trajopt} proceeds as follows.
It begins with a user-specified initial trajectory parameter $k_0$, and predicts the state of the high-fidelity model under the control input $u_{k_0}$, beginning from $\z\hio$ at $t=\tau\pln$ using \eqref{eq:predicted_trajectory} from Assumption \ref{ass:predict}; we assume that $\proj_X(\z\hio) = [x_c,y_c]^\top$, i.e. the robot's state estimate of its position tracks the center of mass as opposed to the robot's entire body.
At each iteration, the algorithm begins by applying the control input $u_{{k_j}^*}$ to the high fidelity model (Line \ref{lin:control}), bringing the robot's state to $\z_{\text{hi},j+1}$.
While this control input is being applied, the algorithm senses surrounding obstacles (Line \ref{lin:obs}), builds a representation for them (Line \ref{lin:union}), applies the \texttt{buffer} to produce the buffered obstacle as in \eqref{eq:X_obs_b_buffered_obstacle} (Line \ref{lin:buffer}), then represents the obstacles as a set of discrete points $X_p$ by applying Algorithm \ref{alg:construct_X_p} (Line \ref{lin:make_Xp}).
In addition, the control input to be applied in the next iteration is computed (Line \ref{lin:opt}).
We describe the behavior of the function \texttt{OptK} in greater detail below; in brief, it uses the obstacle constraints and user-specified cost function to generate an optimal control input $u_{k_{j+1}^*}$ that is applied during the next iteration.
Next, we predict the future position of the high-fidelity model under the optimal control input $u_{k_{j+1}^*}$ beginning from $\z_{\text{hi},j+1}$ using \eqref{eq:predicted_trajectory} as in Assumption \ref{ass:predict} (Line \ref{lin:predict}).

Note, at the same time that Lines \ref{lin:obs}-\ref{lin:predict} are executing, the control input $u_{{k_j}^*}$ is applied to the high fidelity model (Line \ref{lin:control}), bringing the robot's state to $\z_{\text{hi},j+1}$.
As a result, the function \texttt{OptK} takes in the state of the robot after the application of control input $u_{{k_j}^*}$ at $t_{j+1}$, user-specified cost function $J$, $w^l$, obstacle points, and planning time.
% As a result, \texttt{OptK} is given the robot's state at the next iteration.
\texttt{OptK} either finds a new $k_{j+1}^*$ and returns the associated feedback controller $u_{k_{j+1}^*}$ as in \eqref{eq:fdbk_controller_uk}; or it returns the braking controller $u_\mathrm{brake}$ associated with the previous iteration's controller $u_{k_j^*}$ (as in Assumption \ref{ass:brake_ctrl_and_traj}).
If the braking controller is applied, the robot stops within the FRS spatial projection corresponding to $k_j^*$ given by $\pi_X^l(k_j^*) \subset X$ from \eqref{eq:pi_X_on_points}, which is possible by Assumption \ref{ass:brake_in_pi_X}.
Note that uncertainty in state estimation is accounted for by preprocessing obstacles as in Assumption \ref{ass:obs_error_buffer}.

Algorithm \ref{alg:trajopt} can be used for real-time planning due to Algorithm \ref{alg:construct_X_p}, written here as the \texttt{discretizeObs} function (Line \ref{lin:make_Xp}); and due to the enforcement of $\tau_\text{trajopt}$ in the \texttt{OptK} function (Line \ref{lin:opt}).
Recall from Assumption \ref{ass:tau_plan} that the time to execute one planning iteration, $\tau\plan$, is equal to the time to sense and buffer obstacles and predict the next position of the robot, which is denoted by $\tau_\text{process}$, plus the time to run $\texttt{OptK}$, which is denoted by $\tau_\text{trajopt}$.

\begin{algorithm}%[ht]
\small
\begin{algorithmic}[1]
    \State {\bf Require:} $b, r, a, \tau\plan, \tau_\regtext{trajopt}, T \in \R_{\geq 0}$; $w^l \in \R_{2l}[x,k]$; $k_0 \in K$; $\z_{\text{hi},0}\in Z\hi$; and $J: K \to \R$.
    \State {\bf Initialize:} $j = 0$, $t_j = 0$, $k^*_j = k_0$, $\z_{\text{hi},j+1} = \z\hi(t_j + \tau\plan; \z_{\text{hi},0},k_{j}^*)$.
    \State{\bf Loop:}\;
        
        \State\hspace{0.2in}{\bf Apply} $u_{k^*_j}$ to the robot for $[t_j,t_j+\tau\pln)$. \label{lin:control}
    
        /* Lines \ref{lin:obs}--\ref{lin:predict} execute simultaneously with Line \ref{lin:control} */
    
        \State\hspace{0.2in}{\bf Sense} obstacles $\{X\obsi\}_{i=1}^{n\obsj}$.\label{lin:obs}
        
        \State\hspace{0.2in}$X\obs \leftarrow \bigcup_{i=1}^{n\obsj} X\obsi$.\label{lin:union}
        
        \State\hspace{0.2in}$X\obs^{b} \leftarrow \texttt{buffer}(X\obs,b)$.\label{lin:buffer}
        
        \State\hspace{0.2in}$X_p \leftarrow \texttt{discretizeObs}\left(X\obs^b,r,a\right)$.\label{lin:make_Xp}
        
        \State\hspace{0.2in}{\bf Compute} $u_{k^*_{j+1}} \leftarrow \texttt{OptK}\left(\z_{\text{hi},j+1},J, w^l,X_p, \tau_\text{trajopt} \right)$. \label{lin:opt}
        
        \State\hspace{0.2in} $\z_{\text{hi},j+2} \leftarrow$ $\z\hi(2\tau\pln+t_j; \z_{\text{hi},j+1},k_{j+1}^*)$.\label{lin:predict}
        
        \State\hspace{0.2in} $t_{j+1} \leftarrow t_j + \tau\pln$ and $j \leftarrow j + 1$.
        
    \State{\bf End}
\end{algorithmic}
\caption{\small RTD Online Planning}
\label{alg:trajopt}
\end{algorithm}

Next, we comment on the formulation of  \texttt{OptK}. 

\subsection{Trajectory Optimization Formulation}

From Theorem \ref{thm:X_p} and Remark \ref{rem:w_l_for_pi_X}, we know that $X_p$ safely represents the obstacle (i.e. $\pi_K^l(X_p) \subseteq K\safe$).
In other words, if $(v^l,w^l,q^l)$ is a feasible solution to $(D^l)$ (see Section \ref{subsec:FRS_computation}), we can evaluate $w^l$ on the discretized obstacle $X_p$ to conservatively approximate the corresponding set of unsafe trajectory parameters.
As a result, given an arbitrary cost function $J:K \to \R$, the optimization program \texttt{OptK} takes the following form:
\begin{flalign}\label{prog:OptK}
\begin{split}
\underset{k}{\text{min}}&\ J(k)\\
\text{s.t.}&\ w^l(p,k) < 1 \quad ~\forall~p\in X_p.
\end{split}
\end{flalign}
This formulation has a finite list of constraints and a low-dimensional decision variable $k$, which allows \texttt{OptK} to typically terminate within the time limit $\tau_\mathrm{trajopt}$.

\subsection{Proving Safety}
To conclude this section, we confirm that Algorithm \ref{alg:trajopt} is provably safe with the following remark.

\begin{rem}\label{rem:trajopt_alg_is_safe}
From Theorem \ref{thm:D_sense}, we know that the robot can either find a safe plan of duration $T$ (as in Definition \ref{def:safe_plan}) or brake for all $t \geq 0$ as long as it has a safe plan at $t = 0$.
So, in Algorithm \ref{alg:trajopt}, Theorem \ref{thm:D_sense} proves that Lines \ref{lin:control} and \ref{lin:obs} are safe by lower-bounding $D\sense$, and assumes that Lines \ref{lin:union} -- \ref{lin:opt} are safe.
Then, Theorem \ref{thm:X_p} proves that Lines \ref{lin:union} -- \ref{lin:opt} are safe, by constructing $X_p$ such that $\pi_K^l(X_p) \subseteq K\safe$ (as in Definition \ref{def:K_safe}) at each iteration.
By Assumption \ref{ass:predict}, Line \ref{lin:predict} predicts the robot's position to within a box of size $\vep_x\times\vep_y$, and by Assumption \ref{ass:obs_error_buffer}, the obstacle $X\obs$ is expanded to compensate for this state prediction error.
So, since all of the lines inside the loop are safe, Algorithm \ref{alg:trajopt} is safe.
\end{rem}

\section{Application}\label{sec:application}

This section details the application of RTD to two robots: the Segway (Figure \ref{subfig:segway_time_lapse}), and the Rover (Figure \ref{subfig:rover_time_lapse}).
Both robots use Algorithm \ref{alg:trajopt} (as in Section \ref{sec:trajectory_optimization}) for online safe trajectory planning, demonstrated in simulation (Section \ref{sec:simulation_results}) and on hardware (Section \ref{sec:hardware_demo}).

\subsection{Segway}\label{subsec:segway_application}

The Segway is a differential-drive robot.
We apply RTD to the Segway to show that the proposed method can provide collision-free trajectory planning in unstructured, random environments.

The Segway has been used as a running example through this paper.
Example \ref{ex:segway} presents its high-fidelity model \eqref{eq:high-fidelity_segway}.
Example \ref{ex:segway_traj_producing_model} presents its trajectory-producing model.
Example \ref{ex:segway_feedback_controller} presents its tracking controller.
Example \ref{ex:segway_traj_tracking_model} presents its tracking error function.
Example \ref{ex:segway_braking} presents its braking controller.
Example \ref{ex:circular_X0} presents the geometric quantities needed to represent obstacles for the Segway.

Next, we describe the Segway's model parameters, the FRS computation, and obstacle representation geometric quantities.

\subsubsection{Model Parameters}
The robot has a circular footprint with a $0.38$ m radius.
As in Assumption \ref{ass:max_speed_and_yaw_rate}, it is limited to a maximum yaw rate $\omega = \pm 1$ rad/s and a maximum speed of $v = 1.5$ m/s in simulation and $v = 1.25$ m/s on the hardware.
The acceleration bounds are $[\underline{\gamma},\overline{\gamma}] = [-5.9,+5.9]$ rad/s$^2$, and $[\underline{\alpha},\overline{\alpha}] = [-3.75,3.75]$ m/s$^2$.
Given a current yaw rate, $\omega$, the commanded yaw rate, $\omega_\text{des}$, we require $|\omega - \omega\des| \leq 1$ rad/s in simulation, and $|\omega - \omega\des| \leq 0.5$ rad/s on the hardware.
Motion capture data is used to find the parameters $\beta_{\gamma} = 2.95$, $\beta_{\alpha} = 3.00$.

The control gains are as follows, for the tracking controller in Example \ref{ex:segway_traj_tracking_model}.
In simulation, $\beta_x = \beta_y = \beta_\omega = \beta_v = 20$, and $\beta_\theta = 10$.
On the hardware, $\beta_x = \beta_y = \beta_\theta = 0$, and $\beta_v = \beta_\omega = 10$; these hardware gains are estimated as model parameters for the Segway's built-in (black box) controller.

The Segway fulfills the assumptions on state estimation, tracking error, and braking as follows.
State estimation, as in Assumption \ref{ass:predict}, has no error in simulation, so $\vep_x = \vep_y = 0$.
On the hardware, we find that SLAM using Google Cartographer \citep{google_cartographer} with a planar lidar results in $\vep_x = \vep_y = 0.1$ m.
The tracking error function $g$ is constructed to satisfy Assumption \ref{ass:tracking_error_fcn_g} by fitting to simulated tracking error data as shown in Figure \ref{fig:g_fit_segway_error}.
The braking controller is as in Example \ref{ex:segway_braking}.
Assumption \ref{ass:brake_in_pi_X} requires that all braking trajectories lie within the FRS for persistent feasibility.
Designing and validating such a braking controller can be done with SOS programming, but is not the focus of this work.
Figure \ref{fig:segway_braking} illustrates that the braking controller satisfies Assumption \ref{ass:brake_in_pi_X}.

\subsubsection{FRS Computation}\label{subsec:segway_FRS_computation}

The FRS is computed by solving $(D^l)$ in Section \ref{sec:FRSmethod}, with $l = 5$ (from Section \ref{subsec:FRS_implementation}).

We find in practice that the tracking error is proportional to the initial speed, so computing multiple FRS's allows us to reduce conservatism.
At runtime, we select which FRS to use based on the Segway's estimated initial speed at the beginning of the current planning iteration.

For the simulations, we computed one FRS for each of the following initial speed ranges: 0--0.5 m/s, 0.5--1.0 m/s, and 1.0--1.5 m/s.
For the hardware, we computed FRS's for initial speed ranges of 0.0--0.5 m/s, and 0.5--1.25 m/s.

In simulation, the FRS is computed over a time horizon of $T = 0.6$ s for the 0.0--0.5 m/s FRS, and $T = 0.8$ s for the other two FRS's.
For the hardware, the Segway's FRS is computed over a time horizon $T = 1$ s, chosen as per Example \ref{ex:segway_min_time_horizon}.
For all of the Segway FRS's, we used $\tau\plan = 0.5$.

\subsubsection{Obstacle Representation}\label{subsubsec:segway_obstacle_representation}

We use the following geometric quantities (as introduced in Section \ref{sec:obstacle_representation}) to represent obstacles for the Segway.
The width of the Segway is $\rbar = 0.76$ m (Definition \ref{def:rbar}) and the maximum penetration distance is $\bbar = 0.38$ m (Lemma \ref{lem:max_penetration_bbar}).
In the simulations, we empirically chose a buffer size of $b = 0.001$ m.
This choice of $b$ results in a point spacing $r=0.055$ m and arc point spacing $a = 0.002$ m as per Definition \ref{def:point_and_arc_spacing} and Example \ref{ex:circular_X0}.
On the hardware, we use a buffer size of $b = 0.05$ m, so $r = 0.37$ m and $a = 0.10$ m.

Recall from above that we used $\tau\plan = 0.5$.
Our choice of buffer $b$ was the smallest buffer that allowed the runtime trajectory optimization to solve consistently within $\tau\plan$ (recall that the number of constraints for trajectory optimization increases as $b$ decreases).

\subsection{Rover}\label{subsec:rover_application}
The Rover is a front wheel steering, all-wheel drive platform, and demonstrates the utility of RTD in passenger robot applications.
The trajectory producing model is presented in Example \ref{ex:traj-producing_rover}.
We use the system decomposition technique discussed in Section \ref{sec:system_decomp} to compute the FRS's.
We now present the dynamic model of the Rover, the trajectory tracking model, the FRS computation, and the obstacle representation geometric quantities.

\subsubsection{Dynamic Models and Parameters}

The Rover has a rectangular footprint of length $0.5$ m and width $0.29$ m centered at the center of mass. The distance from the rear axle to the center of mass, $l_r$, is 0.0765 m.

The Rover's high-fidelity model has a state vector $\z\hi = [x, y, \theta, v_x, \delta]^\top$, where $v_x$ is longitudinal speed and $\delta$ is the angle of the front (steering) wheels relative to the Rover's longitudinal direction of travel.
The dynamics $f\hi$ as in \eqref{eq:high-fidelity_model} are:
\begin{flalign}
    \label{eq:high-fidelity_rover}
       \frac{d}{dt}\begin{bmatrix}x\\y\\\theta\\ v_x\\ \delta \\
        \end{bmatrix} &=
        \begin{bmatrix}v_x\cos (\theta) -\dot{\theta}\cdot(c_1+c_2 v_x^2) \sin (\theta)\\v_x\sin (\theta) +\dot{\theta}\cdot(c_1+c_2 v_x^2)  \cos (\theta)\\
        \frac{v_x}{c_3+c_4 v_x^2}\tan(\delta)\\
        c_5+c_6\cdot(v_x-u_1)+c_7\cdot(v_x-u_1)^2\\
        c_9\cdot(u_2-\delta)
       \end{bmatrix}.
\end{flalign}
This model utilizes steady-state assumptions for the lateral dynamics, but the constants $c_2$ and $c_4$ account for wheel slip \citep{rajamani2011vehicle}.
Motion capture data was used to fit the constants, $c$.
The steering wheel angle input, $u_1$, is bounded by $|\delta(t)| \leq 0.5$ rad for all $t$, and the speed input, $u_2$, is limited to 0 to 2 m/s.
The tracking controller, for both trajectories and braking maneuvers, is a proportional controller similar to the Segway's \eqref{eq:segway_pd_controller}; hence \eqref{eq:high-fidelity_rover} is Lipschitz continuous in $t, \z\hi,$ and $u$ as required by Assumption \ref{ass:dyn_are_lipschitz_cont}.
Example \ref{ex:traj-producing_rover} presents the Rover's trajectory-producing model \eqref{eq:traj-producing_rover}.
Recall that the trajectory-tracking model \eqref{eq:traj_tracking_model}, is the trajectory-producing model plus the tracking error functions (as in Assumption \ref{ass:tracking_error_fcn_g}).

For the Rover, the trajectory-tracking model dynamics $\dot{\z}_i: [0,T] \times K \to \R^2$ (as in \eqref{eq:traj_tracking_model}) for each SCS are given by:
\begin{flalign}
\label{eq:x subsystem}
       \dot{\z_1} &= \begin{bmatrix}k_3(1-\frac{\theta^2}{2}) -l_r\omega\cdot (\theta-\frac{\theta^3}{6})\\
       \omega \\\end{bmatrix}  + \begin{bmatrix} g_1\\ 0\end{bmatrix}d_1\\
\label{eq:y subsystem}
        \dot{\z_2} &= \begin{bmatrix}k_3(\theta-\frac{\theta^3}{6}) +l_r\omega\cdot(1-\frac{\theta^2}{2})\\
       \omega \\\end{bmatrix} + \begin{bmatrix} g_2\\ 0 \end{bmatrix}d_2
\end{flalign}
where: $g_x, g_y \in \R_3[t,k]$ are degree 3 polynomials that satisfy Assumption \ref{ass:tracking_error_fcn_g}; the yaw rate $\omega(t,k) = \frac{-k_2}{2}t+k_1(1-t)$ is given by \eqref{eq: lane change yaw rate} with $T_h = 2$ s; and $d_1, d_2: [0,T] \to [-1,1]$ are scalar-valued functions.

The Rover satisfies the assumptions on state estimation, tracking error, and braking.
For state estimation, as with the Segway, there is no error in simulation; on the hardware, $\vep_x = \vep_y = 0.1$ m.
The tracking error functions in \eqref{eq:x subsystem} and \eqref{eq:y subsystem} are fit to trajectory data as with the Segway.
The braking controller is verified empirically.

\subsubsection{FRS Computation}\label{subsubsec:rover_frs_comp}
 
For the Rover, we solve $(D_i^4)$ for the subsystems $i = 1, 2$ in \eqref{eq:rover_SCS_x} and \eqref{eq:rover_SCS_y}.
Then, $(R^5)$ reconstructs the full system FRS.
 
As with the Segway, we find that the tracking error for the Rover is reduced by computing multiple FRS's, each corresponding to a different range of initial conditions.
We computed 42 FRS's for the Rover in total.
Each FRS has one of three ranges of initial speeds: 0.0--0.75 m/s, 0.75--1.5 m/s, and 1.5--2.0 m/s; one of seven ranges of initial wheel angles evenly spaced between -0.5 and 0.5 rad; and either positive or negative headings.

The Rover selects an FRS at runtime based on its initial velocity, wheel angle, and heading at each planning iteration.
The time horizons are $T = 1.25$ s for the slowest FRS's and $T = 1.5$ s for the faster FRS's.
All FRS's use $\tau\plan = 0.5$ s in simulation.
On hardware, we use $\tau\plan = 0.375$ s and one FRS that is able to plan trajectories with velocities between 1 and 1.5 m/s due to the limited size of the physical testing area available.

The range of trajectory parameters for each FRS is determined as follows:
The final headings, $k_2$, are between 0 and 0.5 (resp. -0.5) rad for FRS's with negative (resp. positive) initial headings.
The initial yawrates, $k_1$, are between $\max(-1,-1+2k_2)$ and $\min(1,1+2k_2)$ rad/s.
The desired velocities, $k_3$, are set so the change between initial and commanded velocity is less than 1 m/s, and a minimum of 0.5 m/s for the slowest FRS.

\subsubsection{Obstacle Representation}\label{subsubsec:rover_frs_obs}

We use a buffer $b = 0.01$ m for the Rover, resulting in the point spacing $r = 0.02$ m and arc point spacing $a = 0.014$ m as per Example \ref{ex:rectangle_X0}.
\section{Simulation Results}\label{sec:simulation_results}

This section compares RTD against a Rapidly-exploring Random Tree (RRT) planner based on \citet{kuwata2009rrt,theta_star_rrt,unicycle_rrt}; and against the GPOPS-II Nonlinear Model-Predictive Control (NMPC) planner \citep{gpopsii}.
The contribution of this section is the comparison of RTD to RRT and NMPC, and the demonstration of safety of RTD over thousands of simulations.
Code used in the simulations is available at \url{https://github.com/skvaskov/RTD}.

Section \ref{subsec:simulation_environment} presents the timing and environments used for the Segway and Rover simulations.
Section \ref{subsec:sim_traj_planner_implementation} presents the RTD, RRT, and NMPC implementations used.
Section \ref{subsec:simulation_experiments} explains each simulation experiment we ran, what results we expected to see, and what results were found.
The experiments are presented in detail and discussed in Sections \ref{subsec:experiment_1}--\ref{subsec:experiment_3}.
The results are summarized in Table \ref{tab:segway_experiments_1_thru_3} for the Segway and Table \ref{tab:rover_experiments_1_thru_3} for the Rover.
Section \ref{subsec:sim_discussion} discusses the various results.

\subsection{Simulation Timing, Environments, and High-Level Planners}\label{subsec:simulation_environment}

We now discuss the timing parameters, environments, and high-level planners used for the simulations.

Recall the planning hierarchy introduced in Section \ref{sec:intro}.
RTD is a trajectory planner, in the middle level of the hierarchy; therefore, RTD's role is to plan trajectories that attempt to achieve a coarse path plan generated by a high-level planner.
In this work, the high-level planner generates intermediate waypoints, or desired locations, between the robot and the global goal.
We use these waypoints to generate the cost function for trajectory optimization in each planning iteration.

\subsubsection{Timing}
Recall that trajectory planning is performed in a receding-horizon fashion, where the robot computes a plan of duration $T$ s while executing a previously-determined plan.
The robot is also limited by a physical sensor horizon, $D\sense$.
The robot is given a finite amount of time, $\tau\plan$, within which it must find a plan, and it executes a duration $\tau\move \leq T$ of a given plan.
Note that in real-world applications and previous sections of this paper, $\tau\move$ is the same as $\tau\plan$.
We define $\tau\move$ separately in this section because we simulate the RTD, RRT, and NMPC planners with and without real-world timing limits to compare performance.
For the Segway and Rover, we use $\tau\move = 0.5$ s.
Also recall that, as per Assumption \ref{ass:tau_plan}, $\tau\plan = \tau_\regtext{process} + \tau_\regtext{trajopt}$, where $\tau_\regtext{trajopt}$ is the time limit enforced on the trajectory planner at each planning iteration.
For all planners and all simulations, we assume that $\tau_\regtext{process} = 0$.
Finally, recall by Assumption \ref{ass:predict} that the robot can predict its future state to within $\vep_x$ and $\vep_y$ in the $x$- and $y$-directions respectively.
For all planners and all simulations, since the robot is represented as the high-fidelity model \eqref{eq:high-fidelity_model}, there is no state estimation error, so $\vep_x = \vep_y = 0$.

\subsubsection{Segway Simulation Environment}
The simulated environment for the Segway is a $9\times 5$ m\ts{2} room, with the longer dimension oriented east-west.
The room is filled with $6$ to $15$ randomly-distributed box-shaped obstacles with a side length of $0.3$ m.
A random start location is chosen on the west side of the room and a random goal is chosen on the east side.
The simulated environment is similar to the hardware demo depicted in Figure \ref{subfig:segway_time_lapse}.
A trial is considered successful if the Segway reaches the goal without crashing (i.e., touching any obstacles).
Since obstacles are distributed randomly, it may be impossible to reach the goal in some trials; we address this by counting the number of crashes and number of goals reached separately.

\subsubsection{Segway High-Level Planner}
For the Segway's high-level planner, we use Dijkstra's algorithm on a graph representing a grid in the robot's $xy$-subspace $X$; this provides a coarse path and intermediate waypoints between the Segway and the global goal.
At each planning iteration, the cost function given to \texttt{OptK} (as in Section \ref{sec:trajectory_optimization} Algorithm \ref{alg:trajopt}) attempts to minimize the distance to the current waypoint.

\subsubsection{Rover Simulation Environment}
The simulated environment for the Rover is a larger version of the mock road depicted in Figure \ref{subfig:rover_time_lapse}, which mimics a highway environment.
The simulated road lies along the $x$-direction (oriented east-west) and is centered at $y = 0$.
It is 2.0 m wide (including the shoulder), with two $0.6$ m wide lanes centered at $y = 0.3$ m and $y = -0.3$ m.
The Rover plans trajectories with speeds up to 2 m/s.
In each trial, three randomly sized box-shaped obstacles of lengths 0.4--0.6 m and widths 0.2--0.3 m are placed in alternating lanes.
This obstacle arrangement is used to force the Rover to attempt two lane changes per trial; note that the RTD, RRT, and NMPC trajectory planners are all general implementations (as described in Section \ref{subsec:sim_traj_planner_implementation}), not specialized to this particular obstacle arrangement.
The obstacles have a random heading of $\pm$ 2 degrees relative to the road, and their centers are allowed to vary by $\pm$ 0.1 m from lane center in the $y$-dimension.
The spacing between the obstacles in the $x$-direction is given by a normal distribution with a mean of $4$ m and standard deviation of $0.6$ m. 
The Rover begins each trial centered in a random lane, with a velocity of $0$ m/s.
A trial is considered successful if the Rover crosses a line positioned 30 m after the third obstacle without intersecting any obstacle or road boundary (i.e. crashing).

\subsubsection{Rover High-Level Planner}
For high-level path planning, a desired waypoint is placed a set distance ahead of the robot and centered in the current lane.
If the waypoint is inside or behind an obstacle relative to the Rover, the waypoint is switched to the other lane.
It was found empirically that placing the waypoint 4 m ahead of the Rover at each planning iteration causes it to switch lanes soon enough the Rover is typically capable of performing a lane change; this 4 m ``lookahead distance'' was used for all three planners.

\subsection{Trajectory Planner Implementation}\label{subsec:sim_traj_planner_implementation}

RTD, RRT, and NMPC are all implemented in MATLAB on a 2.10 GHz computer with 1.5 TB of RAM.
Timeouts are enforced with MATLAB's \texttt{tic} and \texttt{toc} functions.
For all of the planners, if the Segway has braked to a stop without crashing, one planning iteration is spent rotating in place towards the current waypoint before replanning.
We now discuss implementation details for each planner.

\subsubsection{RTD Implementation Details}

Here, we discuss particular implementation details used for RTD in the simulations; see Section \ref{sec:application} for the general overview of how RTD is applied to the Segway and Rover robots.

We use the following cost functions for \texttt{OptK} in each planning iteration.
For the Segway, the cost function is the robot's Euclidean distance at time $T$ to the waypoint generated by the high-level planner.
For the Rover, the cost function is the Euclidean distance at time $T_h$ from the planned trajectory's endpoint to the waypoint, weighting error in $x$ vs. error in $y$ at a ratio of 1:2.
The final heading parameter, $k_2$ in \eqref{eq: lane change yaw rate}, is set to be the negative of the Rover's initial heading (saturated at $\pm$ 0.5 rad), so the Rover only optimizes over $k_1$ and $k_3$ in each iteration.
In other words, the Rover only optimizes over trajectories that will align the robot with the road.

Both the Segway and the Rover use MATLAB's \texttt{fmincon} generic nonlinear solver to implement the online trajectory optimization \texttt{OptK} (see Program \eqref{prog:OptK} and Algorithm \ref{alg:trajopt} Line \ref{lin:opt}).
For both robots, we use an optimality tolerance of $10^{-3}$.
Since \texttt{fmincon} is a generic gradient-based nonlinear solver, it requires an initial guess each time it is called (i.e., in each planning iteration).
For the Segway, the initial guess of $k \in K$ corresponds to zero yaw rate and maximum speed.
For the Rover, the initial guess is either the trajectory parameters from the previous planning iteration, or parameters corresponding to driving straight if the previous iteration converged to an infeasible result.

We use the following design choice to speed up \texttt{fmincon}.
Recall from Section \ref{sec:obstacle_representation} that we generate a discrete, finite representation of obstacles at each planning iteration.
Each discrete obstacle point becomes a nonlinear constraint for \texttt{OptK} (i.e., \texttt{fmincon}) as per \ref{prog:OptK} in Section \ref{sec:trajectory_optimization}.
Since \texttt{fmincon}'s solve time increases with the number of constraints, we reduce the number of constraints in each planning iteration by discarding points in $X_p$ that lie outside of the FRS for any trajectory parameter $k\in K$.
Note that, since no such points are reachable (because they lie outside of the FRS), this does not impact RTD's safety guarantees.

\subsubsection{RRT Implementation Details}
The Segway and Rover use similar RRT implementations, based on several papers \citep{kuwata2009rrt,theta_star_rrt,unicycle_rrt}, which describe a variety of heuristics for growing a tree of a robot's trajectories with nodes in the high-fidelity state space.
Both RRT implementations use the entire duration $\tau\plan$ to plan a trajectory at each planning iteration.

To account for the robot's footprint, obstacles are buffered by Minkowski sum with a polygonal outer approximation of a closed disk, with radius given by the desired buffer distance (see Experiment 1 in Section \ref{subsec:experiment_1} for how we empirically select the buffer).
This produces a representation of each buffered obstacle as a collection of half-planes.

For the Segway, the RRT planner begins by checking if the previously-planned trajectory is still feasible \citep{kuwata2009rrt}, meaning that none of its nodes lie inside any buffered obstacles.
If the past trajectory is feasible, the tree is initialized with the previous plan's nodes; if the past trajectory is infeasible, the tree is initialized from the robot's initial state.
For the Rover, which operates in a simpler environment, a new tree was initialized for every planning iteration.
New nodes of the tree are created by first choosing a random existing node, with the choice biased towards more recently-generated nodes.
From the randomly-chosen node, the high-fidelity robot model is forward-integrated under a random desired yaw rate (or wheel angle) and desired speed \citep{kuwata2009rrt,unicycle_rrt}.
Forward-integration of the high-fidelity model dynamics returns points in the robot's $xy$-subspace $X$ distributed in time by a \defemph{time discretization}.
A new node is discarded if any of these points lie inside any buffered obstacle, outside of the robot's environment (the room for the Segway and the road for the Rover), or outside the robot's sensor horizon.
In addition, for the Segway, recall (from Section \ref{subsec:simulation_environment}) that Dijkstra's algorithm is used for generating a high-level plan; nodes farther than 1.5 m from the high-level plan are discarded \citep{theta_star_rrt}.
For both the Segway and the Rover, the RRT attempts to plan a braking maneuver at each planning iteration.

Forward integration of the robot's high-fidelity model is required for dynamic feasibility of the RRT trajectory plans, given the complexity of the high-fidelity models of the Segway and Rover \citep{elbanhawi2014sampling}.
The \defemph{edge time}, or total duration of each forward-integration, along with the time discretization, are heuristic choices that affect the computation time and complexity of paths that the RRT can generate; these numbers were selected empirically for each system.
Our implementation makes use of MATLAB's symbolic and function generation toolboxes.
For the Segway, we generate an integration function that takes in an initial condition and returns a trajectory of the robot's high-fidelity model, forward-integrated with an RK4 method, for a predetermined edge time and step size.
We confirmed that calls to this function are as fast or faster than a C++ implementation by forward integrating each robot's high-fidelity model from random initial conditions using the ODEINT C++ library \citep{odeint}.
For the Rover, we used forward Euler integration, as we found it was able to navigate the environments safely (see Section \ref{subsec:experiment_1}).

Recall that, for both the Segway and Rover, a high-level planner generates intermediate waypoints as described in Section \ref{subsec:simulation_environment}.
When growing the RRT, samples are biased to turn towards waypoints as described by \citet{kuwata2009rrt}.
For the Segway, the RRT attempts to find a plan that minimizes distance to the waypoint.
For the Rover, we find that minimizing distance to the waypoint results in the RRT generating long paths with large changes in yaw rate because path smoothness is not included in the cost.
To combat this, we instead set the RRT's cost at each node as the cumulative distance from the root node, plus a penalty for lying close to obstacles \citep{kuwata2009rrt}, which was found empirically to reduce the number of crashes.
Once the RRT has grown for the duration $\tau\plan$, we choose the node with the lowest cost among these to produce the trajectory plan.
The Rover's RRT has an additional heuristic to encourage smoothness: when the waypoint is in the same lane as the rover, the standard deviation of sampled wheel angles is reduced.

\subsubsection{NMPC Implementation Details}
The Segway and Rover both use GPOPS-II for the nonlinear model predictive control planner \citep{gpopsii}.
GPOPS-II is an algorithm that approximates the trajectory planning problem as a polynomial optimization program.
This software uses internal heuristics to choose a finite number of collocation points, then evaluates the polynomial approximation of the robot's high-fidelity model and obstacle avoidance constraints at each of these points.
The accuracy of the solution and the run time of the algorithm is dependent on the tolerance of the polynomial approximation.

The cost function used at each planning iteration is to minimize distance between the last collocation point and the waypoint generated by the high-level planner.

We use the following constraints at each planning iteration.
Obstacles are represented as constraints on the $x$ and $y$ coordinates of the robot's center of mass at each collocation point.
Each obstacle is buffered using a Minkowski sum with a polygonal outer-approximation of a closed disk with radius given by a user-selected buffer distance (see Section \ref{subsec:experiment_1}).
This representation means that, to check for collision of a trajectory with an obstacle, a finite number of half-plane checks are performed per obstacle per collocation point.
We also use the maximum speed and yaw rate (Assumption \ref{ass:max_speed_and_yaw_rate}) as constraints.

We encode persistent feasibility for NMPC in the following manner: if no feasible trajectory can be found within $\tau\plan$, the robot continues executing the last feasible trajectory that NMPC found.
For the Segway, we include an additional constraint, where the end of any planned trajectory must have zero speed and yaw rate, to force NMPC to plan a braking maneuver.
For the Rover, we set the minimum time horizon to 1.5 s, (the braking time from 2 m/s); although potentially less robust than the Segway's constraint, we found it to be sufficient for the environment the Rover is tested in.

The decision variables for NMPC are the robot's state and control input at each collocation point.
For the Segway, the NMPC planner chooses a desired yaw rate and velocity as the control input at each collocation point, and plans with the robot's high-fidelity model \eqref{eq:high-fidelity_segway} from Example \ref{ex:segway}.
For the Rover, the NMPC planner chooses a desired wheel angle and velocity as the control input at each collocation point, and plans with the robot's high-fidelity model \eqref{eq:high-fidelity_rover}.

We initialize GPOPS-II at each planning iteration as follows.
The planner is given a coarse trajectory guess at the first planning iteration, and each subsequent iteration is seeded with last feasible trajectory.
The GPOPS-II parameters used are: 4--10 collocation points per phase and a mesh tolerance of $10^{-6}$.

\subsection{Experiment Overview}\label{subsec:simulation_experiments}

We use a series of experiments to explore the performance of RTD, RRT, and NMPC.
In each experiment, we either relax or enforce requirements of real-time planning and real sensor horizons.
Here, before describing each experiment in detail, we present an overview of their purpose and implementation here.
The results are summarized in Table \ref{tab:segway_experiments_1_thru_3} for the Segway and Table \ref{tab:rover_experiments_1_thru_3} for the Rover.

For each robot, we generate 1,000 random trials that fit the environments described in Section \ref{subsec:simulation_environment}.
Since these are randomly generated, it is not guaranteed that feasible (i.e. collision-free) paths exist from the start to the goal in every trial.
This is useful, because it requires planners to be safe even when the high-level planner can only find infeasible paths to the goal.

\subsubsection{Experiment 1 Overview}
Experiment 1 (Section \ref{subsec:experiment_1}) determines the distance used to buffer obstacles for RRT and NMPC, as discussed in Section \ref{subsec:sim_traj_planner_implementation}; this is because, to the best of our knowledge, it is unclear how large obstacle buffers need to be to ensure safety for these methods.
This experiment does not enforce real-time planning or limited sensor horizon requirements, to ensure that the planners have enough time and information to plan.
To relax the real-time requirement, $\tau\move$ is 0.5 s for both the Segway and Rover, but $\tau\plan$ is large to give each planner time to find a plan.
To relax the sensor-horizon, $D\sense$ is large enough that each planner has obstacle information about the entire scene from $t = 0$ s onwards.

\subsubsection{Experiment 2 Overview}
Experiment 2 (Section \ref{subsec:experiment_2}) explores the effect of enforcing real-time planning and a limited sensor horizon on the planners.
In this case, $\tau\plan = \tau\move$, meaning that each robot can only plan over the duration that it moves at each planning iteration.
The sensor horizon $D\sense$ is chosen to reflect each robot's hardware.
The buffer used for RRT and NMPC is the buffer that performed best in Experiment 1.
The buffer for RTD is chosen in $(0,\bbar)$ as described in Section \ref{sec:application} to ensure safety.

\subsubsection{Experiment 3 Overview}
Experiment 3 (Section \ref{subsec:experiment_3}) shows that RTD performs provably safe trajectory planning in real time when subject to the minimal sensor horizon given by Theorem \ref{thm:D_sense}.
The minimum sensor horizon is significantly smaller than the sensor horizon that each robot's hardware is capable of (e.g., the Segway's minimal sensor horizon is $1.9$ m, whereas its hardware has a sensor horizon of $4$ m), but RTD is still safe.
RRT and NMPC are not tested in this experiment.

\subsection{Experiment 1: Buffer Size for RRT and NMPC}\label{subsec:experiment_1}

\begin{table}%[t]
    \centering
\begin{tabular}{|l|l|r|r|r|r|}
\hline
\multicolumn{2}{|l|}{\multirow{2}{*}{Segway Exp. 1}} & \multicolumn{2}{c|}{RRT} & \multicolumn{2}{c|}{NMPC} \\
\multicolumn{2}{|l|}{} & \multicolumn{1}{c|}{Goals} & \multicolumn{1}{c|}{Crashes} & \multicolumn{1}{c|}{Goals} & \multicolumn{1}{c|}{Crashes} \\ \hline
\multirow{4}{*}{Buffer [m]} & 0.40 & 83.6 & 3.6 & 86.2 & 11.7 \\ \cline{2-6} 
 & 0.45 & 86.2 & 1.4 & \textbf{97.0} & 0.6 \\ \cline{2-6} 
 & 0.50 & 81.9 & 0.4 & 96.0 & 0.4 \\ \cline{2-6}
 & 0.65 & 71.9 & \textbf{0.0} & 83.5 & \textbf{0.0} \\ \hline
\end{tabular}
    \caption{Comparison of success and crash rates for varying buffer sizes for the Segway.
    A buffer size of 0.45 m provides the best balance of performance and safety for both RRT and NMPC.
   The Segway's braking distance of 0.625 m from 1.25 m/s means that the 0.65 m buffer prevents RRT and NMPC from crashing, but both methods become conservative with this buffer.}
     \label{tab:segway_experiment_1}
\end{table}

\begin{table}%[t]
    \centering
\begin{tabular}{|l|l|r|r|r|r|}
\hline
\multicolumn{2}{|l|}{\multirow{2}{*}{Rover Exp. 1}} & \multicolumn{2}{c|}{RRT} & \multicolumn{2}{c|}{NMPC} \\
\multicolumn{2}{|l|}{} & \multicolumn{1}{c|}{Goals} & \multicolumn{1}{c|}{Crashes} & \multicolumn{1}{c|}{Goals} & \multicolumn{1}{c|}{Crashes} \\ \hline
\multirow{3}{*}{Buffer} & 0.29,\ 0.26 & \textbf{99.8}& \textbf{0.0} & 99.6 & \textbf{0.0} \\ \cline{2-6} 
 & 0.34,\ 0.31 & 97.9 & \textbf{0.0} & 98.8 & \textbf{0.0} \\ \cline{2-6} 
 [m] & 0.39,\ 0.36 & 95.8 & \textbf{0.0} & 97.8 & 0.01 \\ \hline
\end{tabular}
    \caption{Comparison of success and crash rates for varying buffer sizes for the Rover.
    Buffers are listed given in m in the $(x, y)$ dimensions. A buffer size of (0.29,\ 0.26) maximizes performance without crashing.}
     \label{tab:rover_experiment_1}
\end{table}

\subsubsection{Goal}
The goal for Experiment 1 is to determine how to buffer obstacles for RRT and NMPC.
To the best of our knowledge, these planners do not prescribe a provably-safe buffer size.
To ensure that the buffer size is the only parameter that influences RRT and NMPC, this experiment relaxes the real-time and limited sensor horizon requirements, giving the planners enough time and information to find a plan in most planning iterations.

\subsubsection{Setup}
The parameters used for Experiment 1 are as follows.
For the Segway, $\tau\move = 0.5$ s, $\tau\plan = 10$ s, and $D\sense = 100$ m.
For the Rover, $\tau\move = 0.5$ s, $\tau\plan = 10$ s, and $D\sense = 30$ m.
Since $\tau\plan > \tau\move$, the real-time requirement is relaxed.
Since $D\sense$ is large, the limited sensor horizon requirement is relaxed.
For both the Segway and the Rover, obstacles are buffered by Minkowski sum with a polygonal outer approximation of a closed disk (see Section \ref{subsec:sim_traj_planner_implementation}).
For the Segway, since the robot's radius is 0.38 m, we test buffer sizes of 0.40, 0.45, and 0.50 m.
We also test a buffer size of 0.65 m; which accounts for the braking distance of the Segway.
For the Rover, obstacles are buffered in the $(x,y)$ dimensions by a rectangle encompassing rotations of up to 0.6 rad (0.29, 0.26 m).
Although a more complicated collision check could be used for the footprint, this type of buffering reduces computational complexity and is commonly used in driving applications \cite{mcnaughton2011motion}. 
Additional buffers of 0.0, 0.05, 0.10 are tested.

\subsubsection{Expected Results}
We expect the results of Experiment 1 to show that, as the buffer size is increased for both planners and both robots, the number of crashes reduces (because any plan that avoids a buffered obstacle places the robot farther away from the actual obstacle for a larger buffer size), and the number of goals reached reduces (because a larger buffer reduces the amount of free space available to each planner).
We expect no crashes for either planner with a buffer size of 0.65 m for the Segway.

\subsubsection{Results}
The results of Experiment 1 are summarized in Table \ref{tab:segway_experiments_1_thru_3} for the Segway and Table \ref{tab:rover_experiments_1_thru_3} for the Rover.
Recall that, since the trials are randomly generated, we do not expect every trial to have a collision-free path from start to goal.
On the Segway, RRT and NMPC fulfill the expectation that, as the buffer size increases, the number of goals and crashes both reduce; a buffer size of 0.45 m provies the best balance between goals and crashes.
On the Rover, we see that the buffer size of the expanded footprint plus 0.0 m has the best performance with no crashes.

Surprisingly, NMPC had a crash with the largest buffer size for the Rover.
In this instance, the solver was unable to find a feasible solution in one planning iteration because too much free space was removed due to the buffered obstacles.
This resulted in the robot colliding with the simulated environment boundary after while trying to emergency brake.

\subsubsection{Discussion}
We now discuss the results of Experiment 1.
Crashes occur for the RRT and NMPC planners for two reasons.
First, the smaller buffer sizes are potentially too small to compensate for both robots' inability to perfectly track a planned trajectory (recall that, in our implementation, RRT plans trajectories with an RK4 or forward Euler approximation of the high-fidelity model, and NMPC uses a polynomial approximation).
Second, if the trajectory planner is unable to find a feasible trajectory at a planning instance, the robot attempts to brake, but there is no guarantee that this is possible while staying safe for these methods.
We address the first cause in subsequent experiments by choosing an RRT and NMPC buffer size of 0.45 m for the Segway and $(0.29, 0.26)$ m for the Rover.
This choice is a balance between a high success rate and a low crash rate.
We do not address the second cause, because this introduces a heuristic that, to the best of our knowledge, is not provided in the literature for RRT or NMPC.

We also test RRT and NMPC on the Segway with buffer sizes of 0.65 m, to check that, if no feasible solution is found in a planning iteration, both methods should always be able to brake without crashing.
We find that the largest buffer size results in the most conservative performance, with 71.8\% of goals reached for RRT, and 83.5\% for NMPC.
As expected, both planners are always able to come to a stop without crashing.

% SEGWAY RTD EXPERIMENT 1
\begin{figure*}
    \centering
    \begin{subfigure}[t]{0.32\textwidth}
        \centering
        \includegraphics[width=\textwidth]{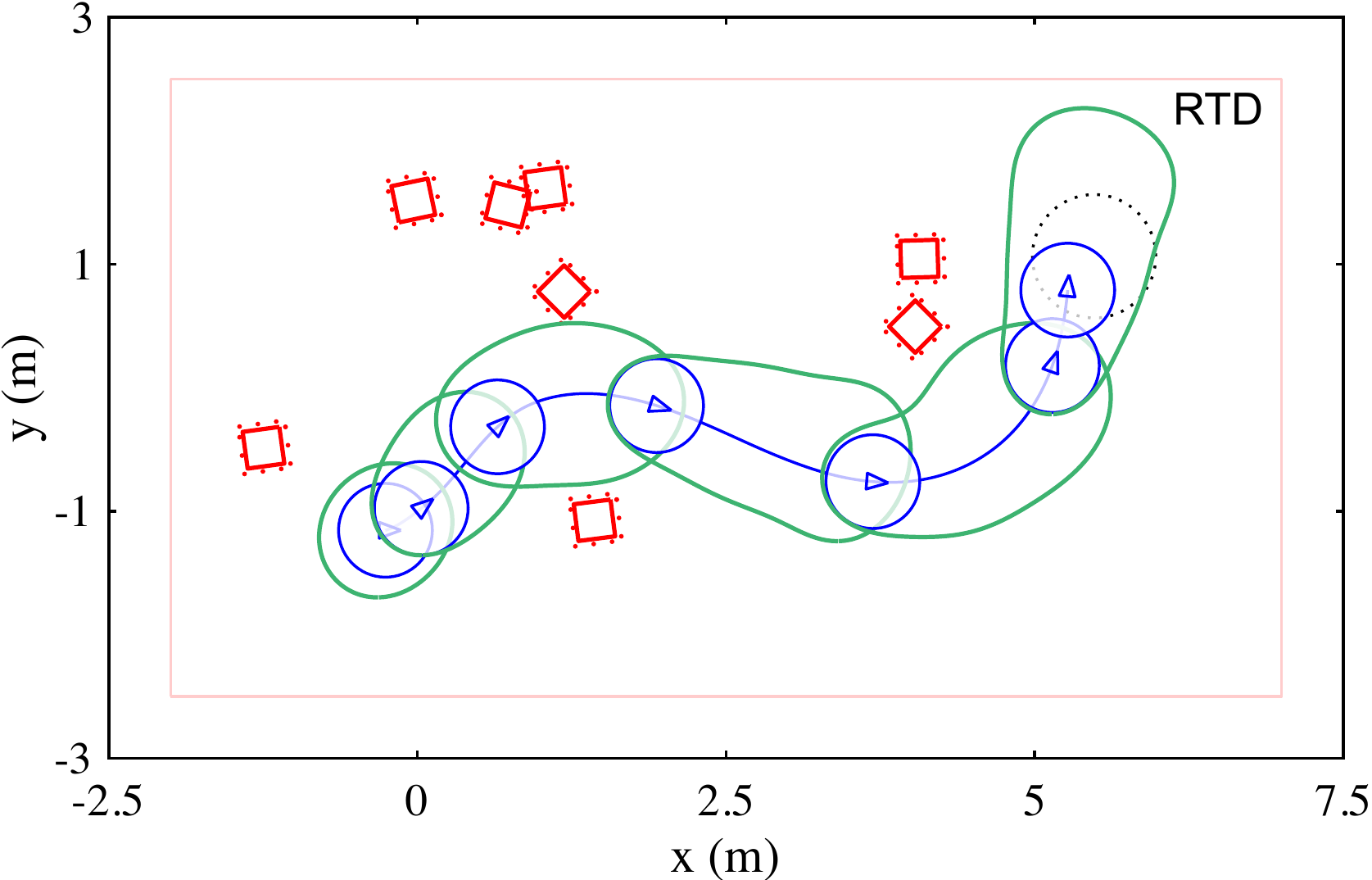}
        \caption{\centering}
        \label{subfig:segway_exp2_all_success_a}
    \end{subfigure}
    \begin{subfigure}[t]{0.32\textwidth}
        \centering
        \includegraphics[width=\textwidth]{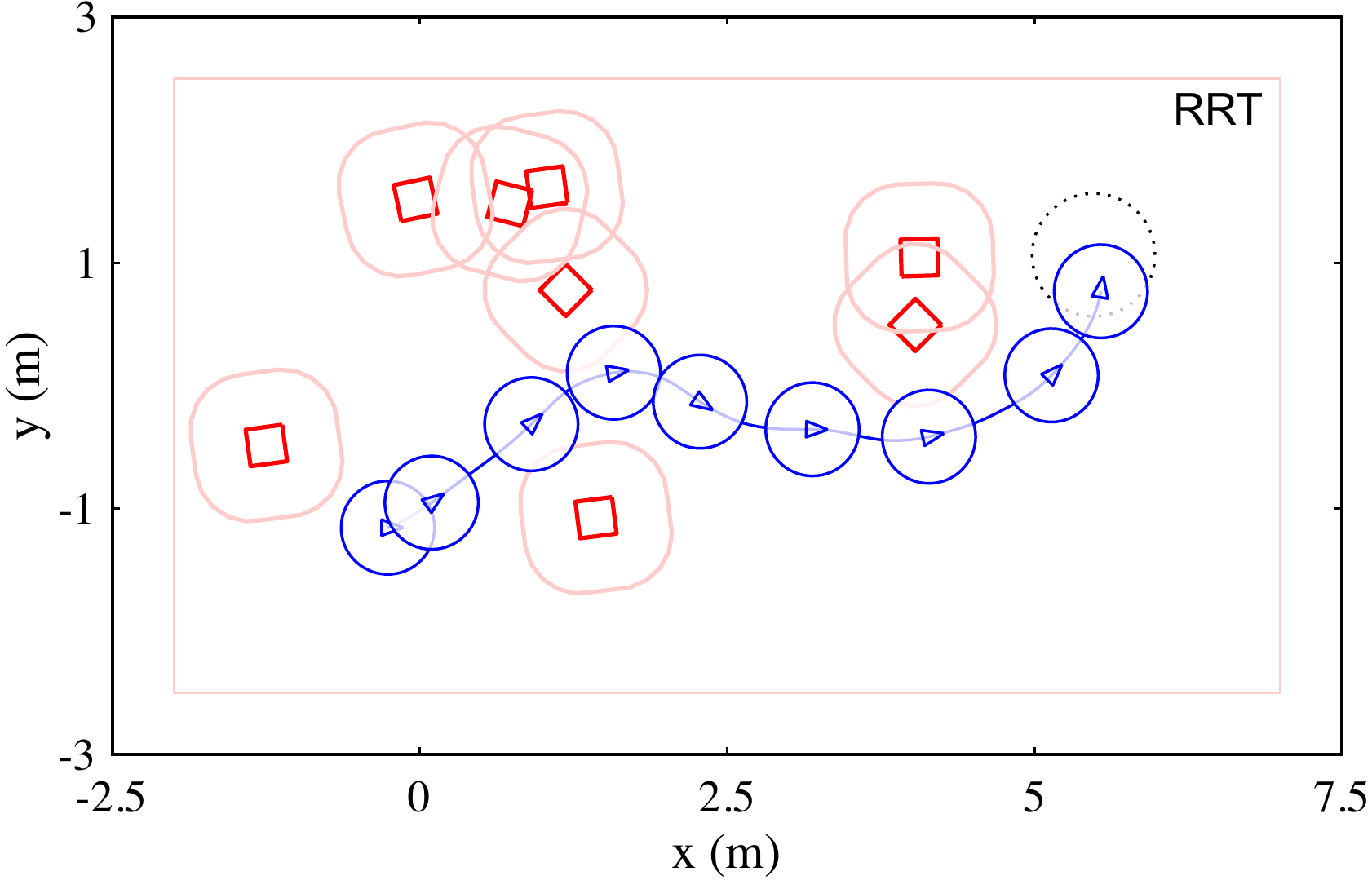}
        \caption{\centering}
        \label{subfig:segway_exp2_all_success_b}
    \end{subfigure}
    \begin{subfigure}[t]{0.32\textwidth}
        \centering
        \includegraphics[width=\textwidth]{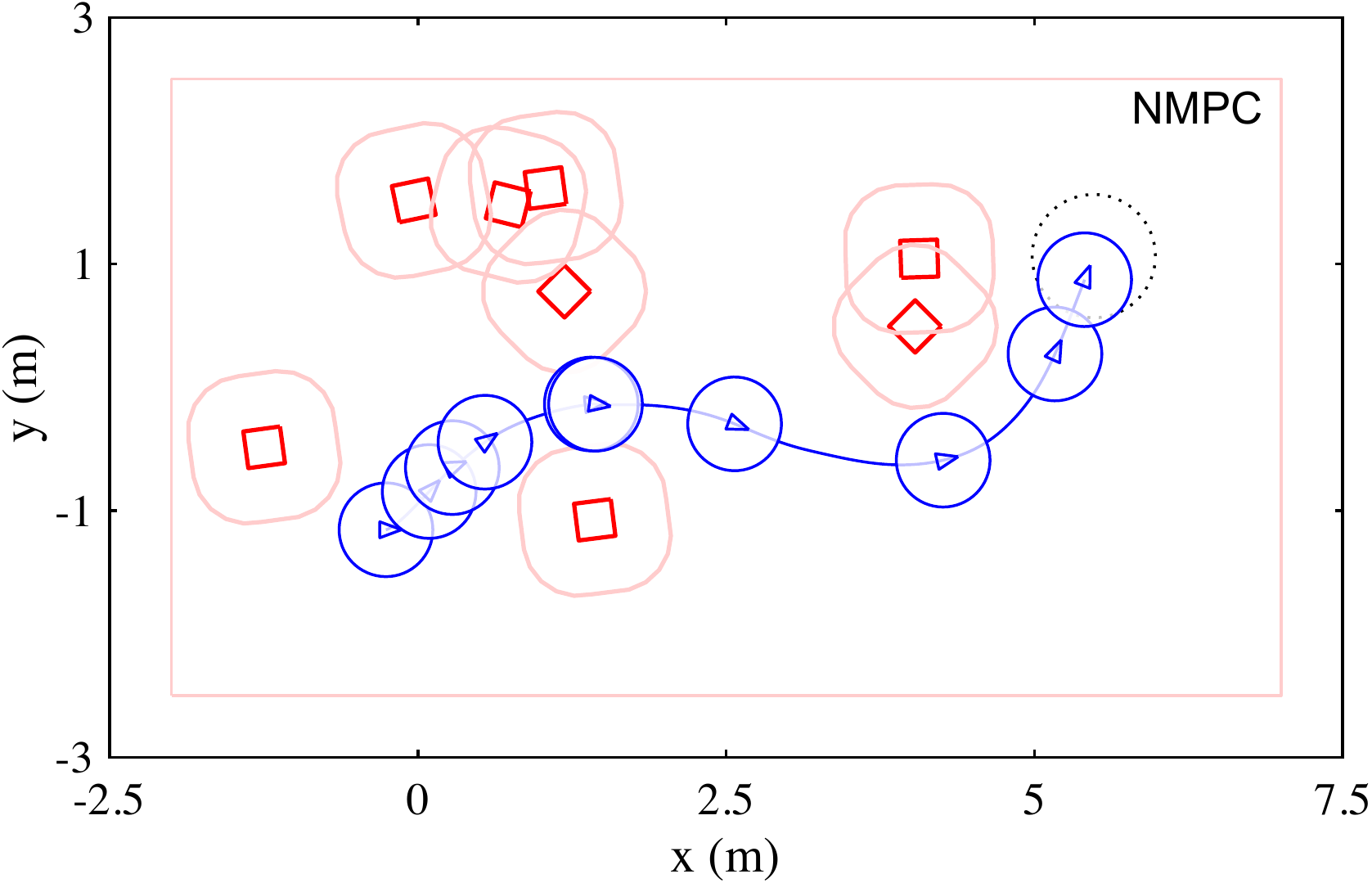}
        \caption{\centering}
        \label{subfig:segway_exp2_all_success_c}
    \end{subfigure}
    \begin{subfigure}[t]{0.32\textwidth}
        \centering
        \includegraphics[width=\textwidth]{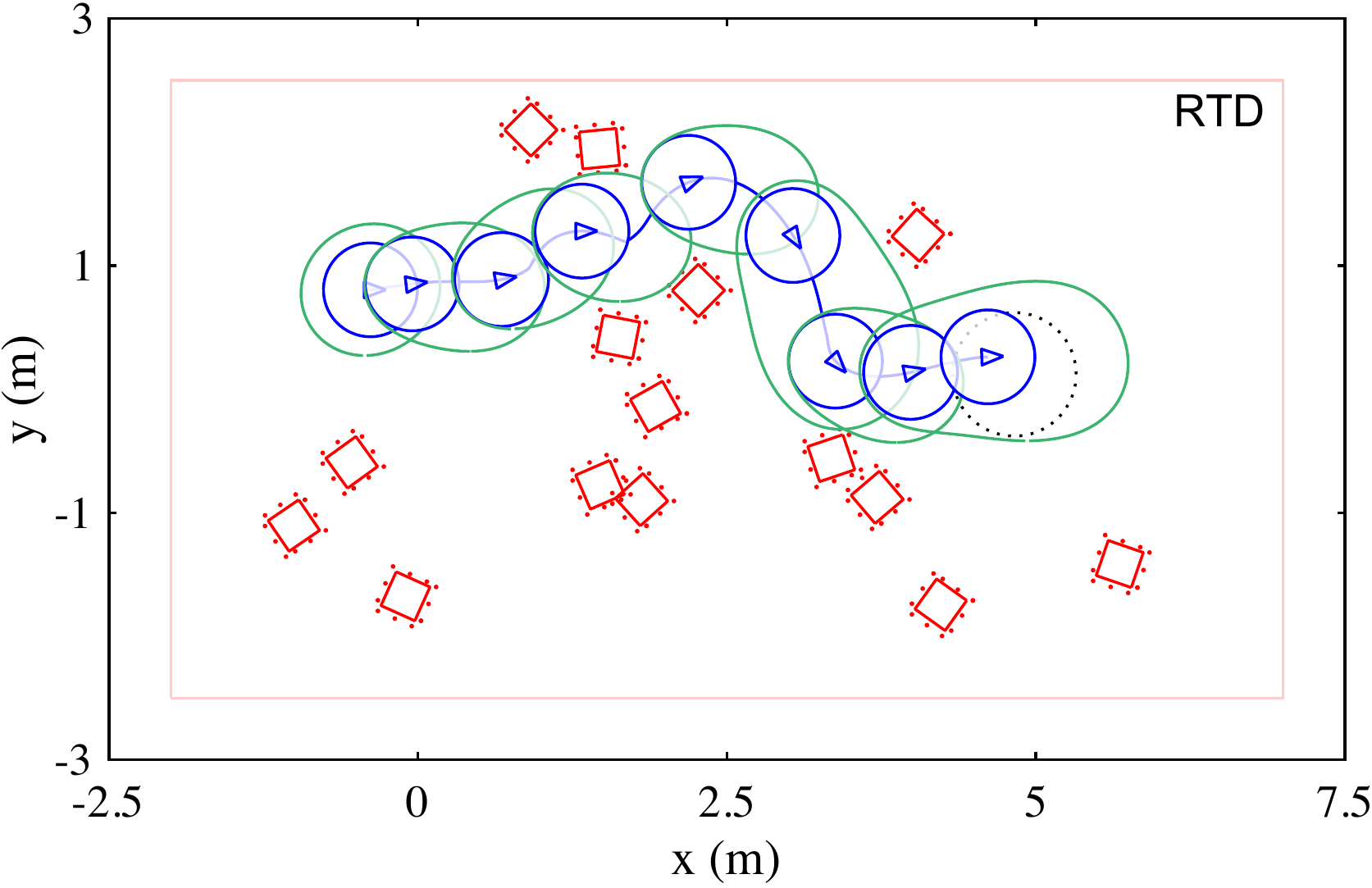}
        \caption{\centering}
        \label{subfig:segway_exp2_frs_success_b}
    \end{subfigure}
    \begin{subfigure}[t]{0.32\textwidth}
        \centering
        \includegraphics[width=\textwidth]{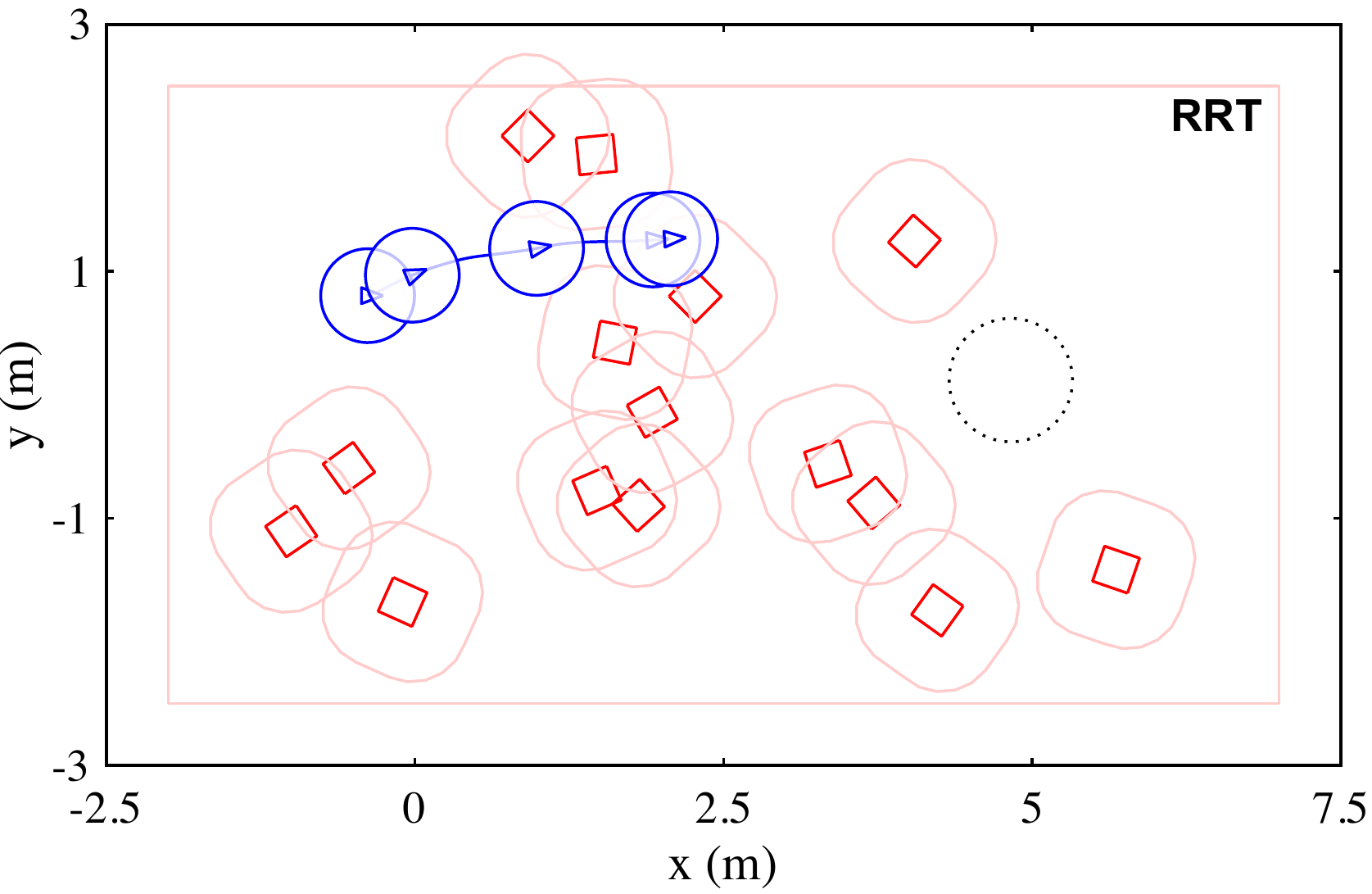}
        \caption{\centering}
        \label{subfig:segway_exp2_frs_success_e}
    \end{subfigure}
    \begin{subfigure}[t]{0.32\textwidth}
        \centering
        \includegraphics[width=\textwidth]{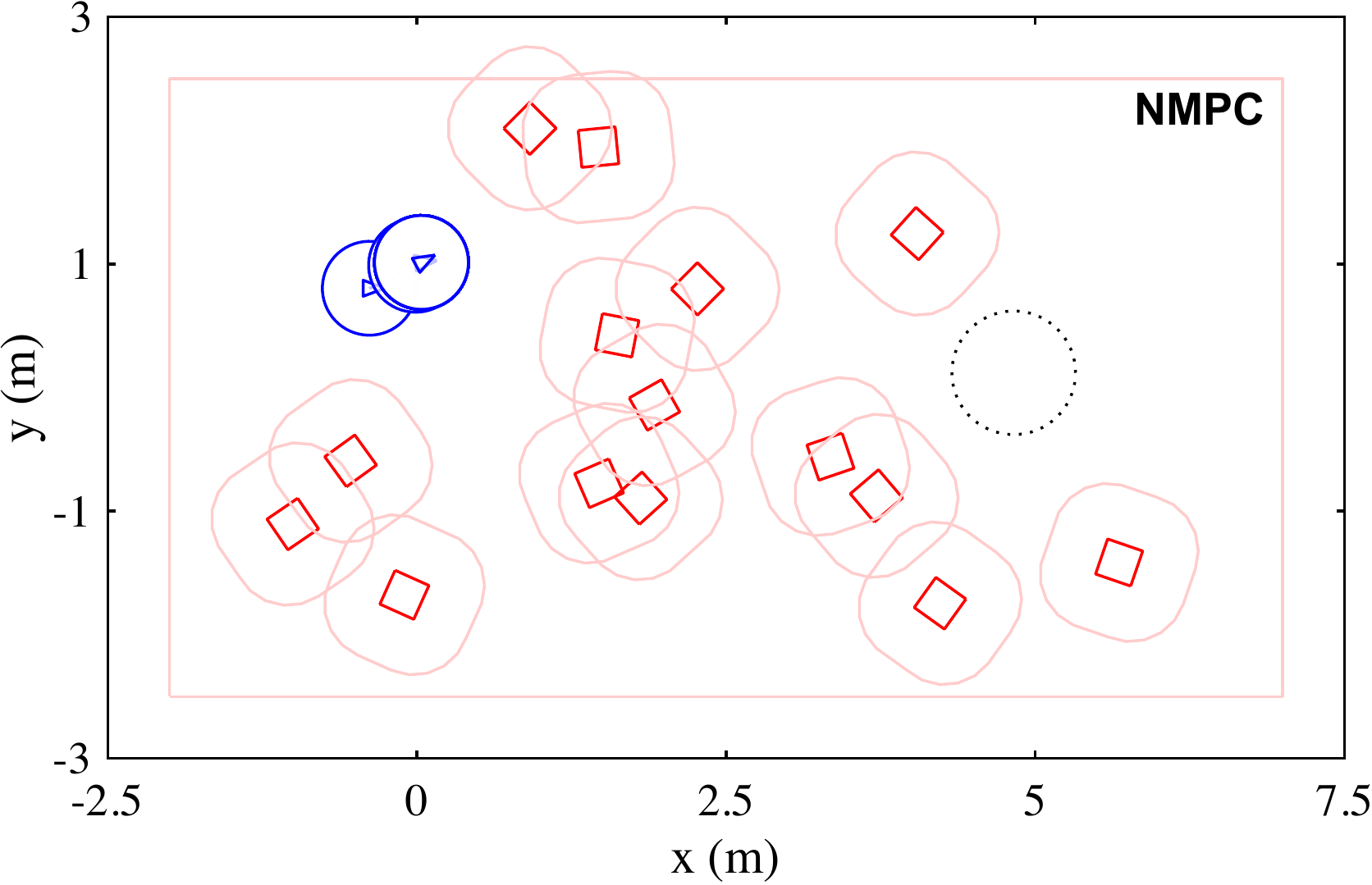}
        \caption{\centering}
        \label{subfig:segway_exp2_frs_success_f}
    \end{subfigure}
    \begin{subfigure}[t]{0.32\textwidth}
        \centering
        \includegraphics[width=\textwidth]{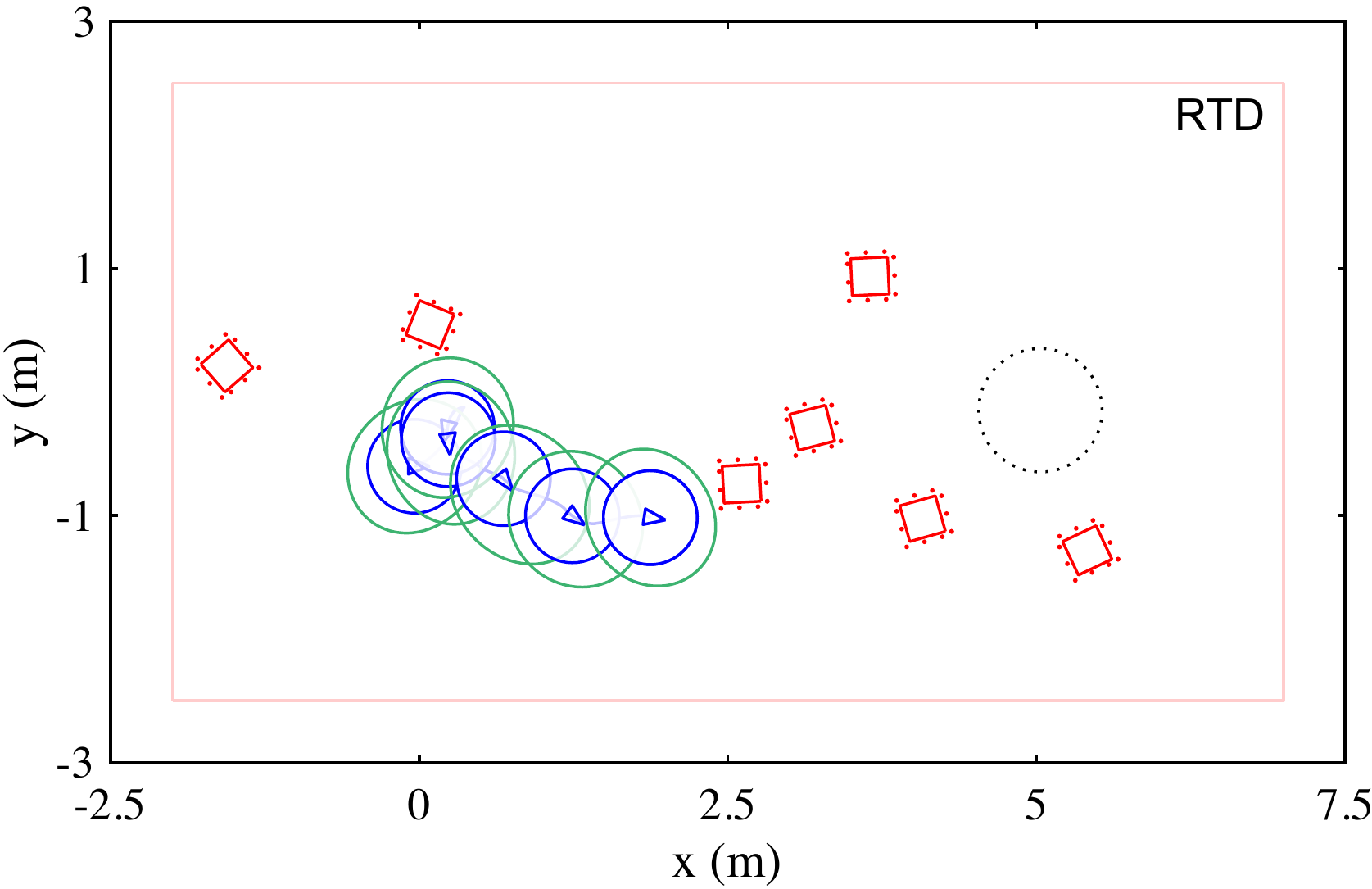}
        \caption{\centering}
        \label{subfig:segway_exp2_frs_fail_a}
    \end{subfigure}
    \begin{subfigure}[t]{0.32\textwidth}
        \centering
        \includegraphics[width=\textwidth]{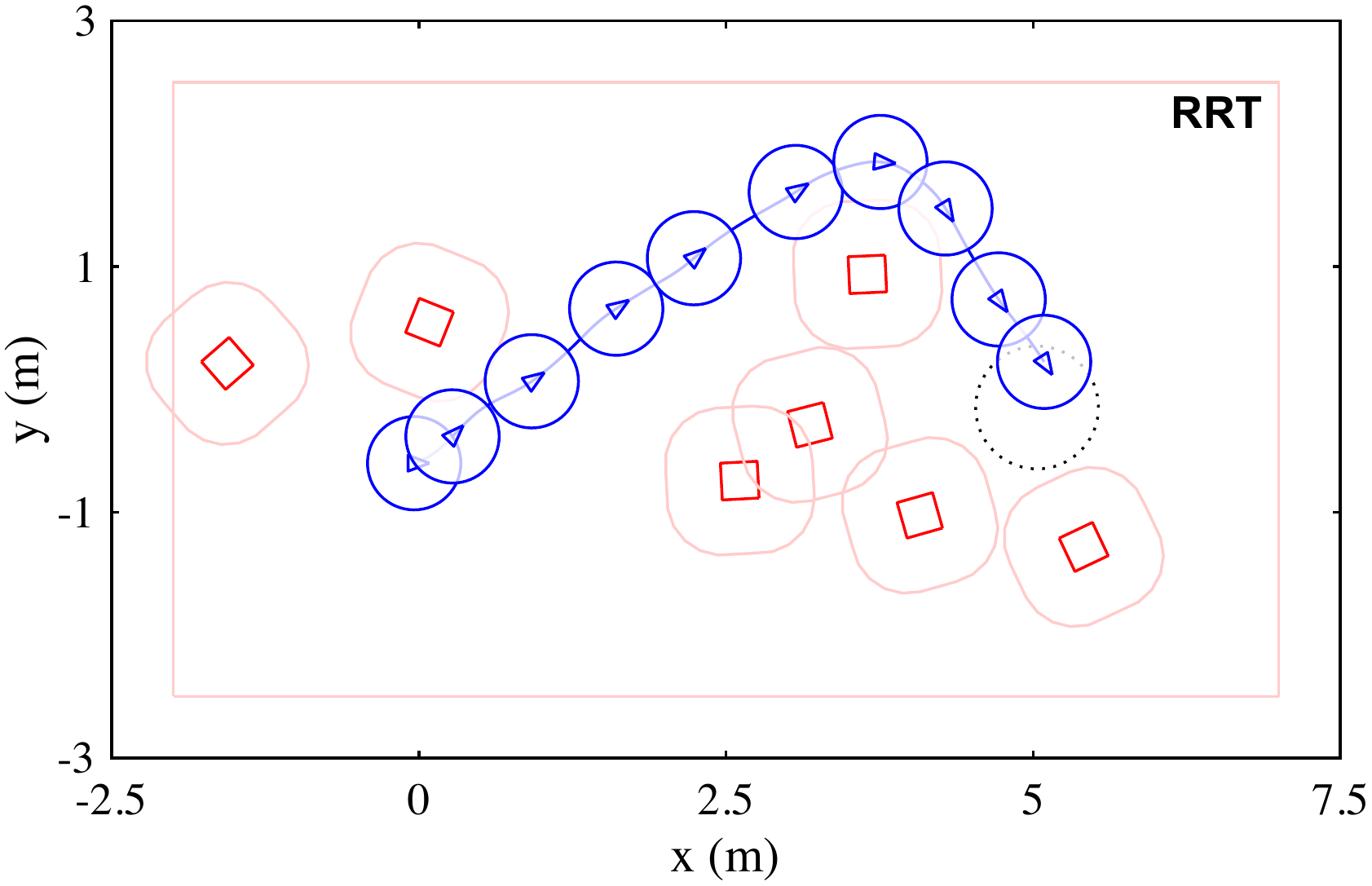}
        \caption{\centering}
        \label{subfig:segway_exp2_frs_fail_b}
    \end{subfigure}
    \begin{subfigure}[t]{0.32\textwidth}
        \centering
        \includegraphics[width=\textwidth]{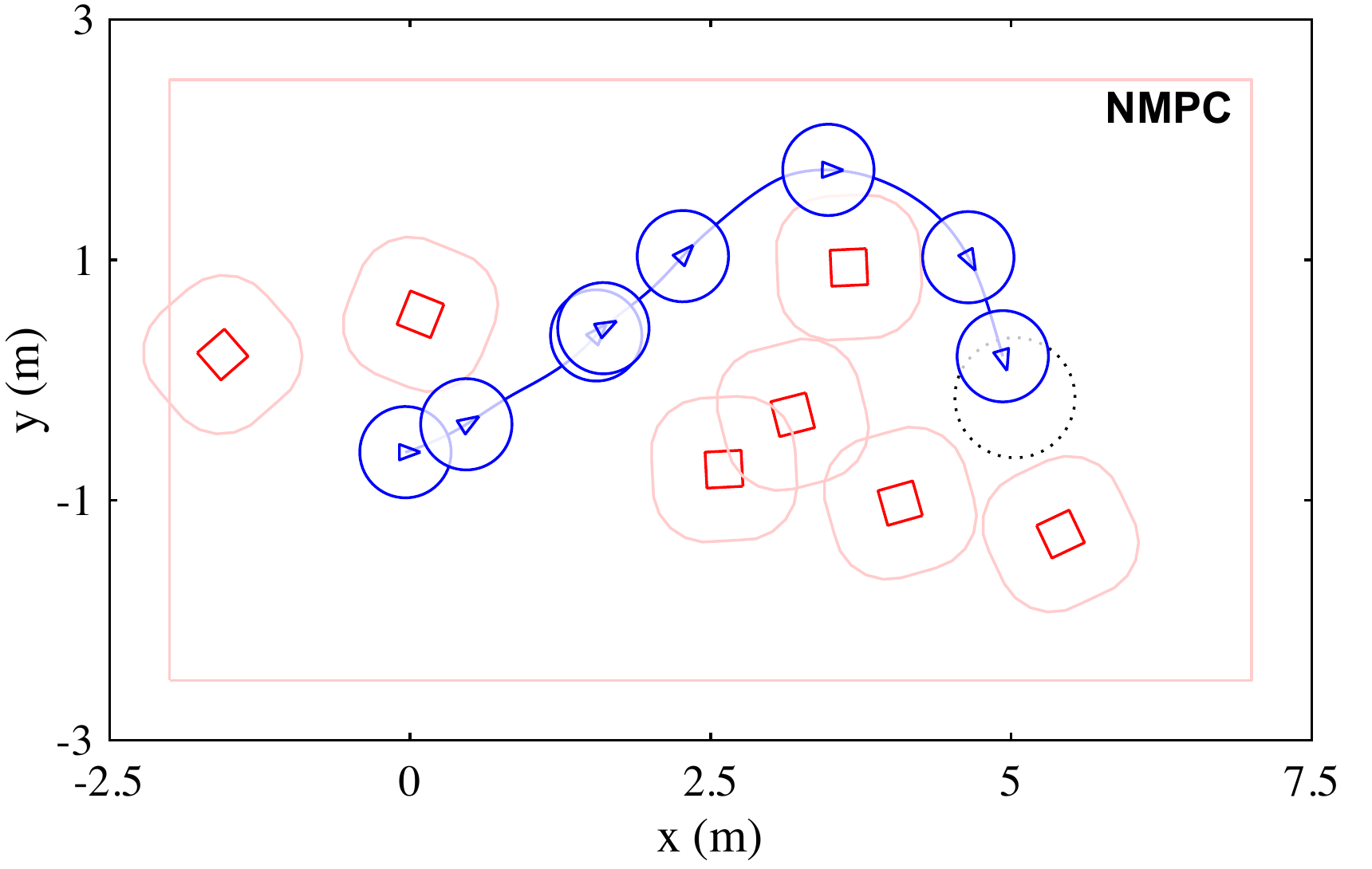}
        \caption{\centering}
        \label{subfig:segway_exp2_frs_fail_c}
    \end{subfigure}
    \caption{Sample environments from Experiments 1 (RRT and NMPC) and 2 (RTD) (Sections \ref{subsec:experiment_1} and \ref{subsec:experiment_2}) for the Segway, which starts on the west (left) side of the environment, with the goal plotted as a dotted circle on the east (right) side of the environment.
    Recall that RTD is not tested in Experiment 1, but the same environments are used in all three experiments, making this comparison possible.
    The Segway's pose is plotted as a solid circle every 1.5 s, or less frequently when the Segway is stopped or spinning in place.
    For RTD, contours of the FRS (i.e. the set $\pi_X^l(k_\regtext{opt})$ from \eqref{eq:pi_K_l}) are plotted.
    The actual (non-buffered) obstacles for all three planners are plotted as solid boxes.
    For RTD, the discretized obstacle is plotted as points around each box.
    For RRT and NMPC, the buffered obstacles are plotted as light lines around each box.
    Row 1 (Subfigures (a), (b), and (c)) shows an environment where all three planners are successful.
    Row 2 shows an environment where RTD is successful, but RRT and NMPC are not.
    Subfigure (d) shows RTD reaching the goal.
    Subfigure (e) shows RRT attempting to navigate a gap between several obstacles, where it is unable to find a new plan; it crashes when it tries to brake along its previously-planned trajectory.
    Subfigure (f) shows NMPC braking because it cannot compute a safe plan to navigate the same gap where RRT crashed; here, NMPC happens to brake safely and gets stuck because it cannot find a new plan fast enough.
    Row 3 shows an environment where RTD stops safely, but fails to reach the goal, whereas RRT and NMPC do reach the goal.
    Subfigure (g) shows that RTD initially turns north more sharply than RRT or NMPC, which forces it to brake safely; it then finds a safe path south, which causes the high-level planner to reroute it even farther south to where there is no feasible solution, causing RTD to get stuck because the southern route is considered feasible by the high-level planner.
    Subfigures (h) and (i) show RRT and NMPC reaching the goal because they do not turn north as sharply initially, so the high-level planner is able to route them north and around the obstacles.}
    \label{fig:segway_exp2_success}
\end{figure*}

\subsection{Experiment 2: Real-time Planning and Limited Sensor Horizon}\label{subsec:experiment_2}

\subsubsection{Goal}
The goal for Experiment 2 is to understand the performance of RTD, RRT, and NMPC when subject to real-time and limited sensor horizon requirements.
RTD is designed to satisfy these requirements while provably ensuring safety.
RRTs are typically capable of rapid planning, though not necessarily with arbitrary dynamics \citep{elbanhawi2014sampling,kuwata2009rrt}.
For NMPC, these requirements can cause wide variations in performance depending on how constraints are represented \citep{Frash2013_ACADO_MPC,Gao2014_robustMPC,gpopsii,howard2007roughterrainplanning,boss2008urbanchallenge}.

\subsubsection{Setup}
The parameters used for Experiment 2 are as follows.
For the Segway, $\tau\move = \tau\plan = 0.5$ s and $D\sense = 4.0$ m.
For the Rover, $\tau\move = \tau\plan = 0.5$ s, and $D\sense = 5$ m.
Since $\tau\move = \tau\plan$, the amount of time allowed for planning is the same as the amount of time that each robot executes from the previously-planned trajectory, meaning the real-time requirement is enforced.
Since $D\sense$ is smaller than the size of each robot's environment (see Section \ref{subsec:simulation_environment}), the limited sensor horizon requirement is enforced.
The RRT and NMPC buffer size is 0.45 m for the Segway, and (0.29, 0.26) m for the Rover. The buffer sizes used for RTD are given in Sections \ref{subsubsec:segway_obstacle_representation} and \ref{subsubsec:rover_frs_obs}

\subsubsection{Expected Results}
We expect the results of Experiment 2 to be as follows.
For both robots, we expect RTD to have a similar number of goals reached as RRT and NMPC, and we expect RRT and NMPC to reach the goal less often than in Experiment 1.
This is due to the limited sensor horizon, meaning the high-level planner no longer has access to the entire environment at time $0$, and therefore may make poor routing decisions.
As for crashes, RTD is designed with real-time performance as a requirement, and prescribes a minimum sensor horizon in Theorem \ref{thm:D_sense} that is less than $D\sense$ for both robots.
Therefore, we expect RTD to have no crashes.
We expect RRT and NMPC to have slightly more crashes than in Experiment 1, because the sensor horizon is shorter, and because the real-time requirement means that these two planners may be unable to find feasible plans as often, resulting in both planners braking more frequently.

\subsubsection{Results}
The results of Experiment 2 are summarized in Table \ref{tab:segway_experiments_1_thru_3} for the Segway and Table \ref{tab:rover_experiments_1_thru_3} for the Rover.
For the Segway, RTD reaches the goal more often than the other two planners do in Experiment 1 or in Experiment 2 (96.3\%); recall that the same environments are used in all three experiments, making this comparison possible.
RRT surprisingly reaches the goal less often in Experiment 2 than in Experiment 1 (78.2\% vs . 86.3\%); and NMPC is incapable of reaching the goal (0\% vs. 83.7\%).
RTD has no crashes, as expected;
RRT crashes less often (2.4\% vs. 3.6\%); and NMPC does not crash because it struggles to move the robot at all.
For the Rover, RTD reaches the goal 95.4\% of the time.
RRT reaches the goal slightly less often than in Experiment 1 (97.6\% vs. 99.8\%); and NMPC is incapable of reaching the goal (0\% vs. 99.6\%). 
RTD has no crashes;
RRT crashes once (0.01\%); and NMPC does not crash because it struggles to move the robot.

\subsubsection{Discussion}
We now discuss the results of Experiment 2.
For both the Segway and Rover, RTD's performance is as expected based on the theory in this paper: it is able to reach the goal, can plan in real time, and has no crashes.
The Segway's RRT has a reduction in crashes, which is surprising, but is likely because the real-time requirement means that RRT is less likely to find a feasible plan at every iteration, and must brake more often.
For the Segway's NMPC planner, we notice that GPOPS-II is able to find trajectories rapidly when the vehicle is not near obstacles; but, since the obstacles are randomly-placed and produce non-convex constraints, the solver struggles to solve quickly when near them, resulting in 0 goals and 0 crashes.
For the Rover, compared to Experiment 1, the RRT planner reaches the goal slightly less often, but still crashes, as expected due to the reduced planning time limit; unlike the Segway, the Rover cannot spin in place to potentially find a new plan after braking.
The Rover's NMPC planner suffers the same issues near the obstacle constraints as the Segway's NMPC planner.
It is worth noting that, for the Rover, we were able to generate heuristics for the RRT that exploited the structure of the environment, which enables the RRT to more goals than RTD.
However, we see in the random environments generated for the Segway that RTD is reaches more goals than RRT.

Figure \ref{fig:segway_exp2_success} demonstrates Experiments 1 and 2 for the Segway; the RRT and NMPC plots are from Experiment 1, and the RTD plots are from Experiment 2, since RTD is not run in Experiment 1.
The same randomly-generated environments are used in all experiments, so this direct comparison is possible.
The figure shows one environment where RTD, RRT, and NMPC all reach the goal without crashing; one environment where RTD reaches the goal, RRT crashes, and NMPC gets stuck; and one environment where RTD brakes safely whereas RRT and NMPC reach the goal.
In the second environment, RRT crashes because, while trying to navigate a gap between two obstacles, it is unable to find a feasible plan; it then attempts to brake along its previous trajectory, but touches an obstacle while doing so, because the Segway cannot necessarily brake exactly along its previous plan produced by RRT.
NMPC gets stuck trying to navigate this same gap where RRT crashes, because the gap is a non-convex region with enough obstacle constraints that the NMPC planner computes slowly.
Unlike RRT, NMPC brakes much earlier, which happens to be safe, but then is unable to find a plan to navigate the gap.
In the third environment, RTD gets stuck because, early on, it finds a different path from RRT and NMPC; this new path causes the high-level planner to reroute RTD towards a region where the high-level planner believes that the route is feasible, but RTD determines that it is not, resulting in RTD braking safely.
This demonstrates that, even if the high-level planner makes infeasible decisions, RTD is safe.

Figure \ref{fig:rover_exp_3} demonstrates Experiment 2 for the Rover with one environment where RTD succeeds, RRT crashes, and NMPC gets stuck; and one environment where all planners brake safely.
RRT crashes when it travels too close to an obstacle to find a feasible plan at the next planning iteration, causing it to try to brake, resulting in a crash.
In some environments, NMPC is able to find plans until the obstacles appear in its sensor horizon.

\begin{figure*}%[h]
    \centering
    \begin{subfigure}[t]{0.53\textwidth}
        \centering
        \includegraphics[width=1.0\textwidth]{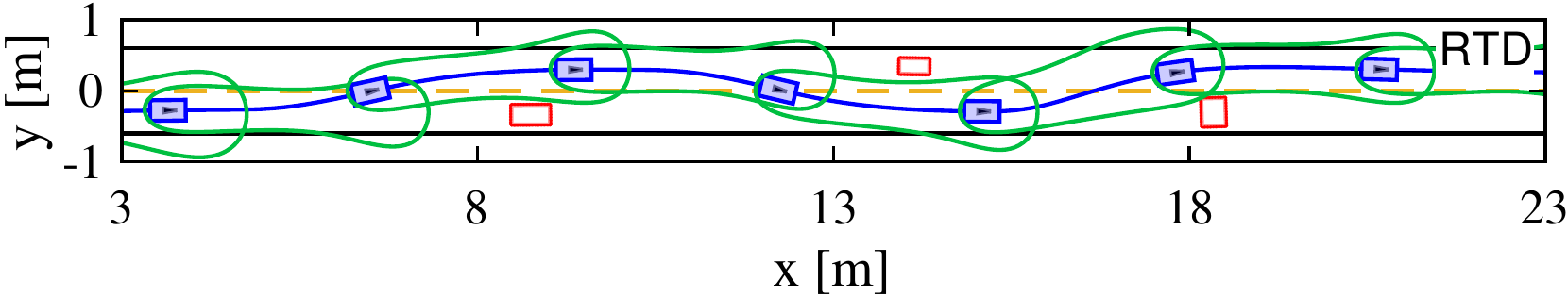}
        \caption{\centering}
        \label{subfig: rover_exp3_success_a}
    \end{subfigure}
    \hfill
    \begin{subfigure}[t]{0.45\textwidth}
        \centering
        \includegraphics[width=1.0\textwidth]{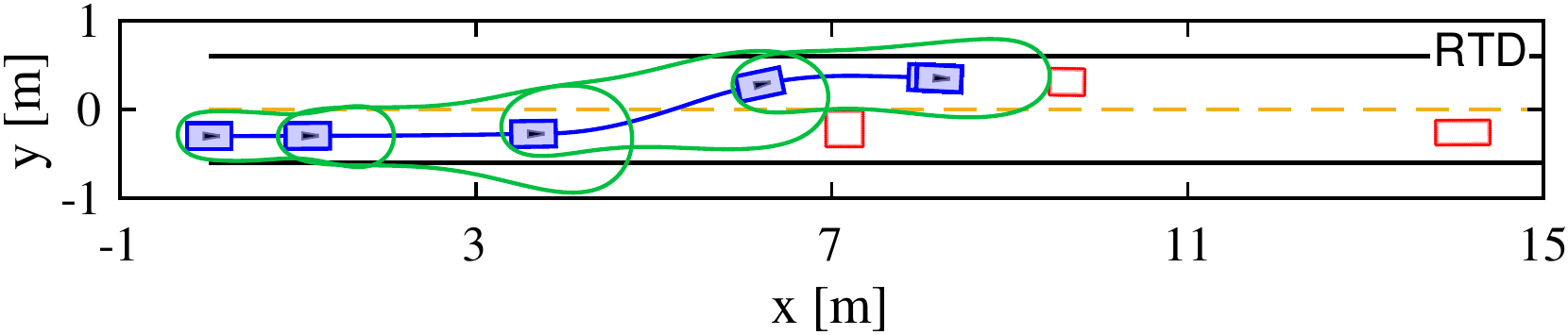}
        \caption{\centering}
        \label{subfig: rover_exp3_crash_a}
    \end{subfigure}
    \begin{subfigure}[t]{0.53\textwidth}
        \centering
        \includegraphics[width=1.0\textwidth]{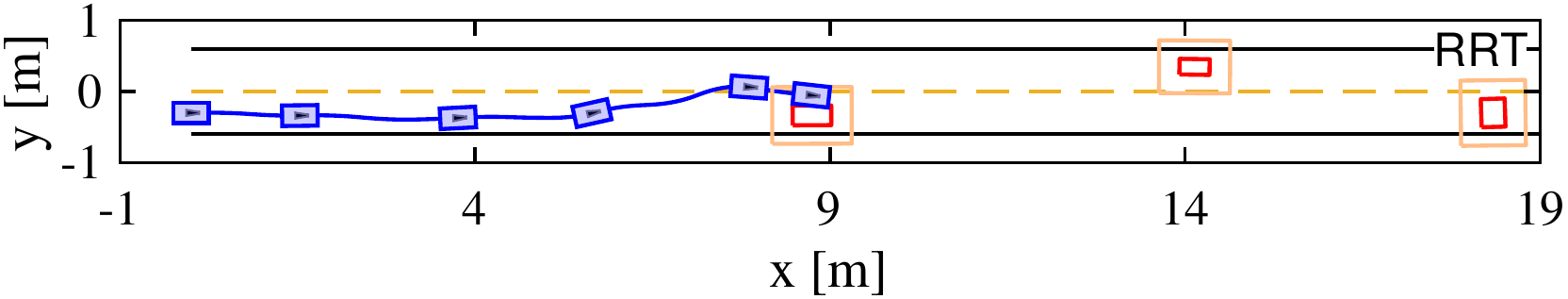}
        \caption{\centering}
        \label{subfig: rover_exp3_success_b}
    \end{subfigure}
    \hfill
    \begin{subfigure}[t]{0.45\textwidth}
        \centering
        \includegraphics[width=1.0\textwidth]{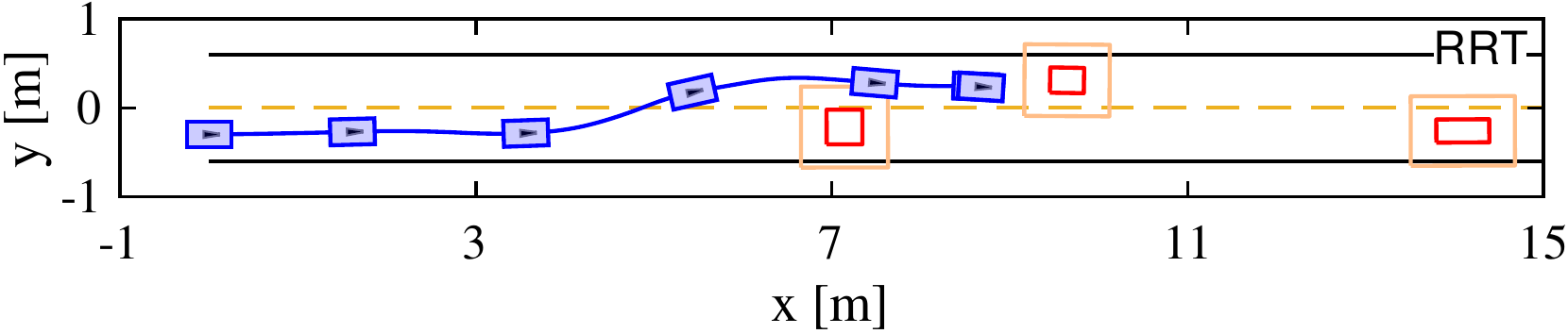}
        \caption{\centering}
        \label{subfig: rover_exp3_crash_b}
    \end{subfigure}
    \begin{subfigure}[t]{0.53\textwidth}
        \centering
        \includegraphics[width=1.0\textwidth]{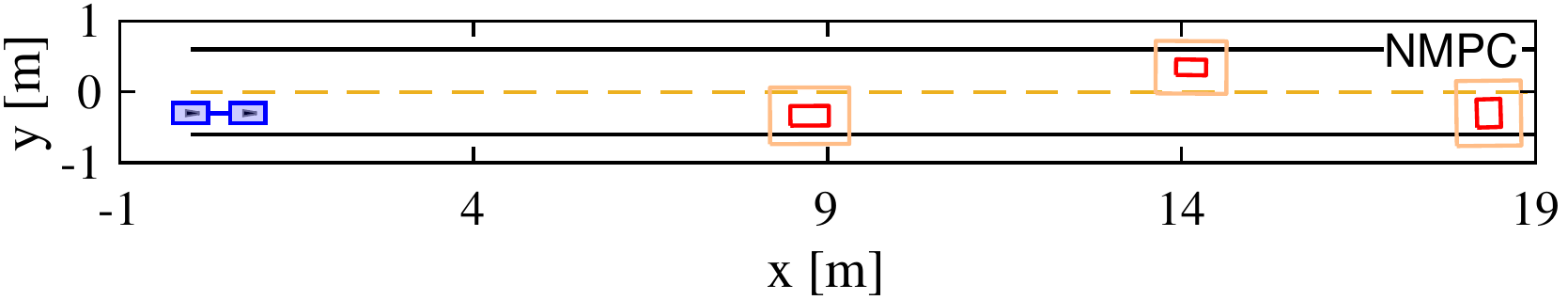}
        \caption{\centering}
        \label{subfig: rover_exp3_success_c}
    \end{subfigure}
    \hfill
    \begin{subfigure}[t]{0.45\textwidth}
        \centering
        \includegraphics[width=1.0\textwidth]{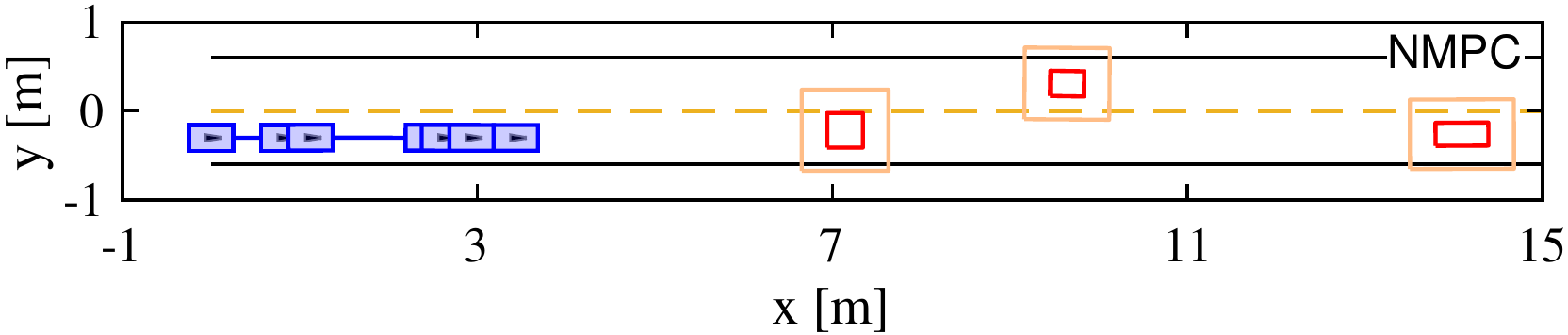}
        \caption{\centering}
        \label{subfig: rover_exp3_crash_c}
    \end{subfigure}
    
    \caption{Two sample environments from Experiment 2 for the Rover.
    The Rover's trajectory, starting from the far left, is a solid line, and its pose at several sample time instances is plotted with solid rectangles.
    Obstacles are plotted as red boxes.
    Buffered obstacles for RRT and NMPC are plotted with light solid lines.
    Subfigures (a) and (b) show RTD avoiding the obstacles.
    The subset of the FRS associated with the optimal parameter every $1.5$ s is plotted as a contour.
    Subfigures (c) and (d) show the RRT method.
    In Subfigure (c), RRT is unable to safely track its planned trajectory around the first obstacle.
    In Subfigure (d), RRT is able to come to a stop before the second obstacle.
    Subfigures (e) and (f) show NMPC, which stops due to enforcement of real-time planning limits.}
    \label{fig:rover_exp_3}
\end{figure*}

\subsection{Experiment 3: Real Planning Time and Minimal Sensor Horizon}\label{subsec:experiment_3}

\subsubsection{Goal}
The goal for Experiment 3 is to confirm that RTD performs safe, real-time trajectory planning even when the sensor horizon is the minimum possible as per Theorem \ref{thm:D_sense}.
This is useful because, to be practical, RTD must be able to tolerate environments where a robot's sensors are only effective in a small area.

\subsubsection{Setup}
The parameters used for Experiment 3 are as follows.
For the Segway, $\tau\move = \tau\plan = 0.5$ s and $D\sense = 1.9$ m.
For the Rover, $\tau\move = \tau\plan = 0.5$ s and $D\sense = 4$ m.
Since $\tau\move = \tau\plan$, the real-time planning requirement is enforced, as in Experiment 2.
The sensor horizon $D\sense$ is given by Theorem \ref{thm:D_sense}.
Buffer sizes for both robots are the same as in Experiment 2.

\subsubsection{Expected Results}
We expect the results of Experiment 3 to show that RTD has zero crashes for either robot.
We expect the number of goals reached to be less than those in Experiment 2, because a smaller sensor horizon means that the high-level planner for both robots has less information when making routing decisions.
So, there may be more environments where the high-level planners cause both robots to brake safely without reaching the goal.

\subsubsection{Results}
The results of Experiment 3 confirm the expectation.
Both robots have 0 crashes.
The Segway reaches the goal 96.2\% of the time, versus 96.3\% in Experiment 2.
The Rover reaches the goal 95.2\% of the time, versus 95.4\% in Experiment 2.

\subsubsection{Discussion}
We now discuss the results of Experiment 3.
Neither robot has any crashes with the minimal sensor horizon.
Furthermore RTD maintains performance in terms of goals reached; this is likely because the sensing requirement presented in Theorem \ref{thm:D_sense} assumes the robot is traveling at its maximum speed.
This means, intuitively, that a smaller sensor horizon is sufficient at lower speeds.

\subsection{Overall Simulation Discussion}\label{subsec:sim_discussion}

The experiments show that RTD is successful in reaching the desired goal comparably often to RRT and NMPC for both the Segway and Rover.
Importantly, RTD has 0 crashes in all of the simulations.

RRT crashes because its paths may take it near obstacles, where it is difficult to build a dense tree since most nodes are infeasible.
When this happens, RRT attempts to brake, but there is no guarantee that this can be done safely.
Interestingly, for the Segway, reducing the allowed planning time $\tau\plan$ reduces the crash rate.
This is because RRT cannot find a feasible plan as frequently with the lower planning time, so it brakes more often, and begins braking when further away from obstacles.

NMPC crashes because, when the robot is near an obstacle, there are a large number of non-convex constraints in the resulting optimization program, so finding a feasible solution within the planning time $\tau\plan$ is difficult.
If no plan is found, the robot attempts to continue executing its last feasible plan (which includes a braking maneuver), but the algorithm has no guarantee that doing so is safe.

We found that increasing the buffer size (Experiment 1) reduces the number of crashes for both RRT and NMPC for the Segway, as expected.
The tradeoff for buffer size is that a larger buffer reduces the free space available for the robot to move through, reducing how often each robot reaches the goal.
Importantly, crashes occur for both planners even when they are not required to plan in real-time or with a limited sensor horizon.
In other words, RRT and NMPC are shown to not be persistently feasible, confirming that persistent feasibility must be considered for robots operating with receding horizon trajectory planners.

RTD is sometimes unable to reach the goal, but still always brakes safely.
Note that RRT and NMPC on both the Segway and Rover platforms are sometimes also unable to reach the goal.
For the Segway, stopping safely before reaching the goal occurs when RTD plans a path too close to an obstacle, in which case the online optimization \texttt{OptK} struggles to find a non-stopped solution even after spinning the Segway in place.
This may be remedied by changing the high-level planner to penalize obstacles more, or by changing the cost function in the online optimization.
The Rover stops without reaching the goal when the reachable set is too large to make a lane change through a tight gap between two obstacles.
This may be due to the fact that the decomposition technique used to compute the FRS's is conservative when the footprint rotates.
This could be remedied by using a simpler trajectory parameterization, like the Segway's, in low-speed, tight scenarios.

For the Rover's environments we notice that RRT and NMPC have excellent performance in Experiment 1.
We found that the sparse (compared to the Segway), structured, and static environment, eases the development of heuristics for both the waypoint and trajectory planners.
The benefits of RTD are greater in the random environments generated for the Segway.

Overall, the simulation experiments confirm that RTD is safe and persistently feasible.
See Figures \ref{fig:segway_exp2_success} and \ref{fig:rover_exp_3} for examples of RTD performing trajectory planning for the Segway and Rover platforms.

Next, we discuss the RTD hardware demonstration.

\begin{table*}[t]
    \centering
    \begin{tabular}{|c|c|c|l|r|r|}
    \hline
    \multicolumn{6}{|l|}{Segway Simulation Results} \\
    \hline
    Experiment & $\tau\plan$ [s] & $D\sense$ [m] & Planner & Goals [\%] & Crashes [\%] \\
    \hline\hline
    \multirow{2}{*}{1} & \multirow{2}{*}{10.0} & \multirow{2}{*}{100}
        &  RRT & 86.2 & 3.6  \\
        \cline{4-6}
        &&& NMPC & \textbf{97.0} & \textbf{0.6}  \\
    \hline\hline
    \multirow{3}{*}{2} & \multirow{3}{*}{0.5} & \multirow{3}{*}{4.0} &  RTD & \textbf{96.3} & \textbf{0.0}   \\
        \cline{4-6}
        &&& RRT & 78.2 & 2.4  \\
        \cline{4-6}
        &&& NMPC & 0.0 & \textbf{0.0}  \\
    \hline\hline
    3 & 0.5 & 1.5 & RTD & 96.2 & \textbf{0.0}   \\
    \hline
    \end{tabular}
    \caption{Simulation results of Experiments 1--3 for the Segway.
    RTD is the only method that never experiences crashes, as expected; it also reaches the goal more frequently than RRT or NMPC.
    NMPC reaches the goal more often than RTD and RRT, with fewer crashes than RRT, but is unable to plan in real time (Experiment 2).
    In Experiment 3, RTD is capable of planning safely when given the smallest possible sensor horizon allowed for persistent feasibility by Theorem \ref{thm:D_sense}.}
     \label{tab:segway_experiments_1_thru_3}
\end{table*}

\begin{table*}[t]
    \centering
    \begin{tabular}{|c|c|c|l|r|r|}
    \hline
    \multicolumn{6}{|l|}{Rover Simulation Results} \\
    \hline
    Experiment & $\tau\plan$ [s] & $D\sense$ [m] & Planner & Goals [\%] & Crashes [\%] \\
    \hline\hline
    \multirow{2}{*}{1} & \multirow{2}{*}{10.0} & \multirow{2}{*}{$30$} & RRT & \textbf{99.8} & \textbf{0.0}   \\
        \cline{4-6}
        & & & NMPC & 99.6 & \textbf{0.0}  \\
    \hline\hline
    \multirow{3}{*}{2} & \multirow{3}{*}{0.5} & \multirow{3}{*}{5.0} & RTD & 95.4 & \textbf{0.0}   \\
        \cline{4-6}
        & & & RRT & \textbf{97.6} & 0.1  \\
        \cline{4-6}
        & & & NMPC & 0.0 & \textbf{0.0}  \\
    \hline\hline
    3 & 0.5 & 4.0 & RTD & 95.2 & \textbf{0.0}   \\
    \hline
    \end{tabular}
    \caption{
    Simulation results of Experiments 1--3 for the Rover.
    RTD is the only method that can both reach the goal and never crash when real time planning is enforced.
    In the Rover's road-like environment, RRT has excellent performance, but crashed in 1 out of 1000 trials when the real-time planning limit was enforced.
    In Experiment 3, RTD is capable of planning safely when given the smallest possible sensor horizon allowed for persistent feasibility by Theorem \ref{thm:D_sense}.}
     \label{tab:rover_experiments_1_thru_3}
\end{table*}

\section{Hardware Demonstration}\label{sec:hardware_demo}

This section details the application of RTD to the Segway (Figure \ref{subfig:segway_time_lapse}) and Rover (Figure \ref{subfig:rover_time_lapse}) hardware platforms.
Section \ref{sec:simulation_results} confirms that RTD is safe and persistently feasible, so it is able to plan safe trajectories in real time.
The hardware demonstrations affirm this point.
Videos of the robots are available at \url{https://youtu.be/FJns7YpdMXQ} for the Segway and \url{https://youtu.be/bgDEAi_Ewfw} for the Rover.

\subsection{Segway}

\subsubsection{Hardware Setup}
The first hardware demo uses the Segway Robotics Mobility Platform shown in Figure \ref{subfig:segway_time_lapse}.
Sensing is performed with a Hokuyo UTM-30LX planar lidar; in practice, we found this sensor to be accurate up to $D\sense = 4.0$ m away (recall that the Segway runs indoors, so the effective sensor horizon is small).
The robot is controlled by a 4.0 GHz laptop with 64 GB of memory, running MATLAB and the Robot Operating System (ROS).
Google Cartographer is used for localization and mapping \citep{google_cartographer}.
All computation is run onboard. 
Since SLAM and state estimation requires $\tau_\mathrm{process} = 0.2$ s per iteration (as in Assumption \ref{ass:tau_plan}), we enforce $\tau_\mathrm{trajopt} = 0.3$ s when calling \texttt{OptK} (as in Algorithm \ref{alg:trajopt} from Section \ref{sec:trajectory_optimization}).
We find in practice that the state estimation error is never more than $0.1$ m in the global $xy$-coordinate frame while the Segway tracks any parameterized trajectory, so we set $\vep_x = \vep_y = 0.1$ m as in Assumption \ref{ass:predict}.
The FRS is computed for the Segway as described in Section \ref{subsec:segway_application}.

\subsubsection{Demonstration}
The Segway is run on a $4\times8$ m\ts{2} tile floor with $30$ cm cubical obstacles randomly distributed just before run time.
The Segway has no prior knowledge of the obstacles.
Two points are picked on opposite ends of the room and used as the start and goal points in an alternating fashion.

A supplementary video illustrates the performance of RTD.
Despite the randomly-placed obstacles, the Segway RMP platform is able to operate safely while consistently reaching its goal.
As in the simulation, the Segway uses a low speed and a high speed FRS (see Section \ref{subsec:segway_FRS_computation}).
In the handful of instances where the Segway brakes, the high-level planner generates waypoints that require passing through a gap that is too small for the high speed FRS; the Segway swaps to the low speed FRS after stopping, and is then able to navigate the gap.

\subsection{Rover}

\subsubsection{Hardware Setup}
The second hardware demo uses a Rover car-like robot based on a Traxxas RC platform.
The Rover is tested on a 7 m long mock road, which is a tiled surface, as shown in Figure \ref{subfig:rover_time_lapse}.
This setup resembles the simulation environment, but with a shorter road and smaller obstacles.
The Rover is equipped with a front-mounted Hokuyo UST-10LX planar lidar for sensing and localization; as the Rover runs indoors, we found this sensor to be accurate up to at least $D\sense = 3.5$ m away given occlusions and obstacle density.
An NVIDIA TX-1 computer on-board is used to run the sensor drivers, state estimator, feedback controller, and low-level motor controller.
The Rover uses ROS to communicate with an Intel Core i7 7820HK (2.90 GHz) CPU/64 GB RAM laptop over wifi.
The laptop is used for localization and mapping, to capture experiment data, and to run the function \texttt{OptK} from Algorithm \ref{alg:trajopt}.
We use $\vep_x = \vep_y = 0.1$ m for the state estimation error as in Assumption \ref{ass:obs_error_buffer}.
The FRS is computed for the Rover as described in Section \ref{subsec:rover_application}.

\subsubsection{Demonstration}
For each trial, the Rover is placed at one end of the mock road and instructed to drive to a goal at the other end at speeds of 1--1.5 m/s.
One to three obstacles are placed between the Rover and the goal.
The obstacles are $0.3\times 0.3\times 0.3$ m\ts{3} cardboard cubes.
The Rover is not given prior knowledge of the obstacles for each trial, and uses its planar lidar to detect them in real-time. The Rover has an enforced planning time limit of $\tau_\mathrm{trajopt}+\tau_\mathrm{process} = 0.375$ s.
Contrary to the Segway, the timeouts were enforced together.
This is because localization and map updates were provided smoothly at 20 Hz, so the algorithm did not need to pause and wait for an update as often as the Segway did.
Eight trials were run back-to-back and filmed in one take, as presented in the supplementary video.
Several types of scenarios are constructed to encourage the Rover to change lanes or force it to brake to a stop.
Eighteen trials were run in addition to the filmed trials, and resulted in zero crashes.
The Rover uses one FRS to plan at speeds between 1.0--1.5 m/s.
Due to the minimum speed, the Rover is occasionally unable to navigate tight gaps; this could be remedied by using a low speed FRS with a different trajectory parameterization.

\section{Conclusion}
\label{sec:conclusion}
This paper presents the Reachability-based Trajectory Design (RTD) method, which plans provably safe trajectories in real time for arbitrary ground mobile robots. 
Other state-of-the-art methods for planning rely on spatial or temporal discretization and rely on heuristics to manage tradeoffs between accuracy and run time to simultaneously enable safety and real-time performance.
With RTD, the robot plans using a continuous set of parameterized trajectories.
The Forward Reachable Set (FRS), computed offline, contains reachable positions of the robot, including tracking error, when tracking these trajectories over a fixed time horizon.
RTD specifies criteria for the robot's sensor horizon and stopping distance to ensure that the robot is persistently feasible, meaning it is always able to find a safe trajectory.

RTD plans trajectories using a receding-horizon strategy.
In each receding-horizon planning iteration, a nonlinear optimization program is solved to select optimal trajectory parameters.
The FRS is used to create a map that sends obstacles from the state space to the parameter space as nonlinear constraints for the online solver in a provably safe way, ensuring that any trajectory selected to satisfy the constraints cannot cause a collision.
This paper presents a provably safe method for representing arbitrary obstacles with a discrete set of points.
This representation allows the online optimization program to solve in real-time.
This paper also adapts a system decomposition technique for computing the FRS, extending the application of RTD to higher dimensional systems.

In this paper, RTD has been applied to two systems in both simulation and hardware: a Segway RMP robot navigating a room full of random obstacles, and a car-like Rover robot performing lane change maneuvers on a mock two lane road.
The Segway's parameterized trajectories are Dubins paths with varying velocities and yaw rates.
Using a high-level planner to produce a coarse route, the Segway is able to safely traverse the room despite random, unforeseen obstacle configurations.
The Rover's FRS computation uses the presented system decomposition method.
The Rover's parameterized trajectories are lane change maneuvers generated with a bicycle model.
A high-level planner tells the Rover to change lanes if an obstacle is sensed nearby in the same lane.
RTD then synthesizes either a safe lane change or safe braking maneuver.
Simulation results for both the Segway and Rover compare the performance and safety of RTD to Rapidly-exploring Random Tree (RRT) and Nonlinear Model-Predictive Control (NMPC) methods.
When real-time planning limits are enforced, RTD is able to outperform RRT and NMPC in terms of number of goals reached without causing any crashes.
Videos of the hardware demonstrations are available at \url{https://www.youtube.com/watch?v=FJns7YpdMXQ} for the Segway and \url{https://www.youtube.com/watch?v=bgDEAi_Ewfw} for the Rover.
Code used for the reachable set computation and simulation results is available at \url{https://github.com/skvaskov/RTD}

\bibliographystyle{plainnat}
\bibliography{references}

\begin{appendices}
\section{Reachability Analysis Proofs}\label{app:reachability_proofs}

In this appendix, we prove Lemma \ref{lem:v_is_negative} (Section \ref{sec:FRSmethod}) and Theorem \ref{thm:compute_FRS} (Section \ref{sec:system_decomp}).
We restate the lemma and theorem for ease of reading.

\begin{customlem}{\ref{lem:v_is_negative}}
If $(v,w,q)$ satisfies the constraints in \eqref{prog:sos_frs_inf_dim}, then $v$ is non-positive and decreasing along trajectories of the trajectory-tracking system \eqref{eq:traj_tracking_model}.
In other words, let $\z\in Z$ and $x=\proj_X(\z)$; then $(x,k) \in \X\frs$ implies that $\exists~t \in [0,T]$ such that $v(t,\z,k) \leq 0$.
\end{customlem}
\begin{proof}
Notice that $v(0,\z_0,k) \leq 0$ for all $\z_0 \in Z_0$ and $k \in K$ by $(D5)$.
So, for any $\tau \in [0,T]$, $k \in K$, and $d \in L_d$, we have:
\begin{align}\begin{split}
v(\tau,\z(\tau),k) =~&v(0,\z(0),k) + \smallint\limits_0^{\tau} \left(\Lf v(t,\z(t),k)\right)dt~+ \\
&+ \smallint\limits_0^{\tau} \left(\Lg v(t,\z(t),k)\circ d(t) \right) dt \label{eq:v_fund_thm_of_calc}\end{split} \\
\begin{split}\leq~&v(0,\z(0),k) + \smallint_0^{\tau} \left( \Lf v(t,\z(t),k)\right)dt~+ \\
&+ \smallint_0^{\tau} q(t,\z(t),k) dt \label{eq:v_lt_lfv_plus_q}\end{split} \\
\leq~&v(0,\z(0),k), \label{eq:v_lt_v0} &
\end{align}
where \eqref{eq:v_fund_thm_of_calc} follows from the Fundamental Theorem of Calculus; \eqref{eq:v_lt_lfv_plus_q} follows from $(D2)$ and $(D3)$; and \eqref{eq:v_lt_v0} follows from $(D1)$.
\end{proof}

\begin{customthm}{\ref{thm:compute_FRS}}
Let $w_r$ be a feasible solution to $(R)$.
Then $\X\frs$ is a subset of the $1$-superlevel set of $w_r$.
\end{customthm}
\begin{proof}
Let $\z_0 \in Z_0$, $k \in K$, and $d \in L_d$ be arbitrary such that $\z: [0,T] \to Z$ is a trajectory of the full system \eqref{eq:full_sys_SCS}.
Let $x=\idx(\z)$ and $X=\idx(Z)$.
Let $\z_1(t) = \proj_{Z_1}(\z(t))$ give the corresponding trajectory of subsystem 1, and similarly let $\z_2$ give the trajectory of subsystem 2.
By Lemma \ref{lem:projection_chen}, since the full system \eqref{eq:full_sys_SCS} is decomposable, $\z(t) \in \proj\inv(\z_1(t)) \cap \proj\inv(\z_2(t))$.
Recall that $(v_1,w_1,q_1)$ is a feasible solution to $(D_1)$, which denotes $(D)$ solved with the dynamics of subsystem 1.
By Lemma \ref{lem:v_is_negative}, $v_1(t,\z_1(t),k)$ is non-positive and decreasing along the trajectory $\z_1(t)$ for every $t \in [0,T]$, and similarly $v_2(t,\z_2(t),k) \leq 0$ for $\z_2(t)$.
The set $\mathcal{V}$ \eqref{eq:v_backproj_intersection} contains $(x,k)$ in $X\times K$, such that $v_1(t,\z_1,k)\leq 0$, $v_2(t,\z_2,k)\leq 0$ and $t\in [0,T]$.
Constraint $(R1)$ requires that $w_r(x,k)\geq 1$ if $(x,k)\in \mathcal{V}$.
Since $\z_0$, $k$, and $d$ were arbitrary, the proof is complete.
\end{proof}

\section{Conditions for Persistent Feasibility}\label{app:pers_feas}

In this appendix, we provide conditions to ensure that, at any speed, the robot's planned trajectory \eqref{eq:traj-producing_model} is spatially longer than the corresponding braking trajectory, i.e. the robot achieves a larger displacement in $X$ when not braking as opposed to braking.
This is because, if we know that a non-braking trajectory is safe over its entire distance, then Assumption \ref{ass:brake_in_pi_X}, that the robot can stop safely in its direction of travel, is plausible.
To do this, we first prove Theorem \ref{thm:D_sense}, which ensures that the robot is able to sense obstacles that could cause a collision during any plan.
We then state the minimum planning time horizon with Remark \ref{rem:planning_time_horizon_min}.

\subsection{Proof of Theorem \ref{thm:D_sense}}

\begin{customthm}{\ref{thm:D_sense}}
Let $X\obs \subset X$ be a set of obstacles as in Definition \ref{def:obs}.
Let $v_\regtext{max}$ be the robot's maximum speed as in Assumption \ref{ass:max_speed_and_yaw_rate}.
Let $\tau\plan$ be the planning time as in Assumption \ref{ass:tau_plan}.
Suppose that $T$ is large enough that Assumption \ref{ass:brake_in_pi_X} holds; so, for any $\z\hio \in Z\hio$ and any $k \in K$, the spatial component of the robot's braking trajectory lies within $\pi_X(k)$.
At time $0$, suppose that the robot has a safe plan $k_0 \in K$ (as in Definition \ref{def:safe_plan}).
Recall that $\vep_x$ and $\vep_y$ are the robot's maximum state estimation error in the $x$ and $y$ coordinates of $X$ as in Assumption \ref{ass:predict}, and let $\vep = \sqrt{\vep_x^2 + \vep_y^2}$.
Suppose the sensor horizon $D\sense$ obeys Assumption \ref{ass:sense} and satisfies
\begin{align}
D\sense & \geq (T + \tau\plan)\cdot\vmax + 2\vep.\tag{\ref{eq:min_sensor_horizon}}
\end{align}
Then, the robot can find either find a new safe plan every $\tau\plan$ seconds, or can brake safely if no new safe plan is found.
\end{customthm}
\begin{proof}
In this proof, we check that the robot can brake within any safe plan, and that it can sense obstacles far away enough to generate safe plans.

First, we check that the robot can begin braking safely at any time $t = j\cdot\tau\plan$ where $j \in \N$.
Recall that the robot replans over each time horizon $[j\cdot\tau\plan ,(j+1)\cdot\tau\plan]$, so it will either have a new safe plan $k_j$ or will begin braking at each $t = j\cdot\tau\plan$.
We know that the robot is safe over $t \in [0,T]$ by the premises, and it can brake safely (i.e., within $\pi_X(k_0)$ by Assumption \ref{ass:brake_in_pi_X}).
Similarly, if the robot has a safe plan of duration $T$ at $t = j\cdot\tau\plan$, then the robot can still brake safely if a new safe plan cannot be found before $(j+1)\cdot\tau\plan$.

Now, we check that the sensor horizon in \eqref{eq:min_sensor_horizon} is large enough for the robot to sense all possible obstacles that are reachable at each $t = j\cdot\tau\plan$ despite state estimation error.
Recall that the obstacles are static by Definition \ref{def:obs} and that, by Assumption \ref{ass:predict}, at any time $t$, the robot can predict its future position at $t + \tau\plan$ to within a box of size $\vep_x\times\vep_y$.
Also recall that, by Assumption \ref{ass:obs_error_buffer}, a sensed obstacle $X\sense$ is expanded as in \eqref{eq:Xobs_expanded} to the set $X\obs$ to compensate for state estimation error.
So, at each time $t = j\cdot\tau\plan$ the robot must plan with respect to all obstacles that are reachable within the time horizon $T$ from the robot's future position at $t = (j+1)\cdot\tau\plan$.
This means that a safe plan $k_{j+1}$ found over the time horizon $[j\cdot\tau\plan ,(j+1)\cdot\tau\plan]$ must avoid all obstacles within the distance $D_T = T\cdot\vmax + \vep$ of the robot's future position at $t = (j+1)\cdot\tau\plan$.
Notice that the maximum possible distance between the robot's position at $j\cdot\tau\plan $ and at $t = (j+1)\cdot\tau\plan$ is $D\plan = \tau\plan\cdot\vmax + \vep$.
Therefore, at each time $t = j\cdot\tau\plan $, the robot must sense all obstacles that are within the distance $D\plan + D_T = \tau\plan\cdot\vmax + T\cdot\vmax + 2\vep$.
Since $D\sense \geq (T + \tau\plan)\cdot\vmax + 2\vep$ and the robot senses all obstacles within $D\sense$ at $t = j\cdot\tau\plan$, we are done.
\end{proof}

\subsection{Choosing the Time Horizon}

To choose the time horizon $T$, we begin by defining the braking trajectory with controller $u\brk$ from \eqref{eq:ubrk_braking_controller}:
\begin{align}
    \z\brk(t;\z\hio,k) = \z\hio + \int_0^{\tau'} f\hi(\tau,\z\brk(\tau),u\brk(\tau))d\tau.\label{eq:braking_traj}
\end{align}
where $\tau' = \tau\plan + \tau\brk(\z\hio,k)$.
Recall that, by Assumption \ref{ass:brake_ctrl_and_traj}, for every $\z\hio$ and $k$, there exists a finite \defemph{braking distance} $D\brk: Z\hio\times K \to \R_{\geq 0}$ given by
\begin{align}
    D\brk(\z\hio,k) = \int_{\tau\plan}^{\tau'} \norm{\idx\left(f\hi(\tau,\z\brk,u\brk)\right)}_2 d\tau,\label{eq:D_brake}
\end{align}
where again $\tau' = \tau\plan + \tau\brk(\z\hio,k)$.
Equation \eqref{eq:D_brake} follows from the formula for length along a differentiable parametric curve \citep[Theorem 6.27]{rudin1976principles}.
Recall that the robot begins braking at $t = \tau\plan$ while tracking $k$ from initial condition $\z\hio$.
Since $Z\hio$ and $K$ are compact, there exists a \defemph{maximum braking distance}:
\begin{align}
    D\stp =  \max_{\z\hio,\ k}~D\brk(\z\hio,k). \label{eq:D_stop}
\end{align}
where where $\z\hio \in Z\hio$ and $k \in K$.
Recall from Assumption \ref{ass:max_speed_and_yaw_rate} that the robot's high-fidelity model \eqref{eq:high-fidelity_model} has a state that tracks its speed in the subspace $X$, and a max speed $\vmax$.
Let $D\brkmax: [0,\vmax] \to \R_{\geq 0}$ be the maximum braking distance at a particular speed:
\begin{align}
    D\brkmax(v) = \sup_{\z\hio,\ k} \left\{D\brk(\z\hio,k)~\mid~\proj_V(\z\hio) = v\right\},
\end{align}
where $\z\hio \in Z\hio$, $k \in K$, and $\proj_V$ returns the value of speed state in $\z\hio$.
The maximum is achieved for each $v$ because $\proj_V$ is continuous (see Definition \ref{def:projection_operators}), so the preimage $\proj_V\inv\left(\{v\}\right)$ is a closed subset of $Z\hio$, and therefore compact \citep[Theorem 26.2]{Munkres2000}.
We now relate $D\brkmax$ to the vehicle speed to formalize the idea that, as the robot travels faster, its maximum braking distance increases.

\begin{assum}\label{ass:brake_dist_grows_sublinearly}
The maximum braking distance at any speed is upper bounded by a linear function of speed:
\begin{align}
    D\brkmax(v) \quad\leq\quad C\stp \cdot v,\label{eq:braking_dist_less_than_linear_fcn}
\end{align}
where $C\stp \in \R_{\geq0}$.
\end{assum}

\noindent 
To see why $D\brkmax$ can be upper-bounded by a linear function of speed, consider the following example.
For automobiles, the maximum braking distance is proportional to the kinetic energy of the vehicle, which is proportional to the square of the vehicle's speed \citep{brakingdistance}.

Figure \ref{fig:brk_sublinearly} shows that this relationship holds for the high-fidelity models of the Segway and Rover robots described in Section \ref{sec:application}.

\begin{figure}
    \centering
    \begin{subfigure}[t]{0.48\textwidth}
        \centering
        \includegraphics[width=0.95\textwidth]{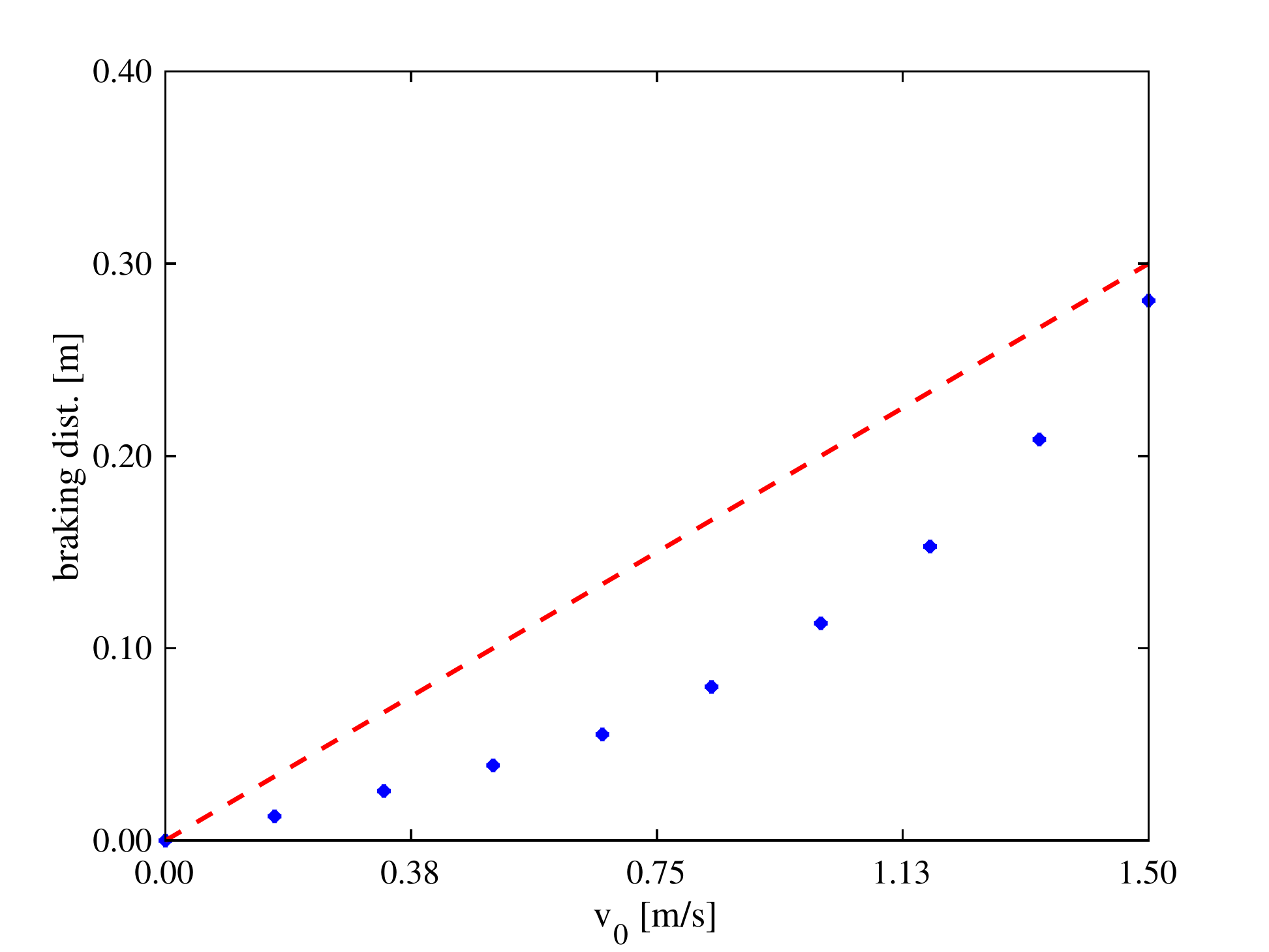}
        \caption{\centering}
        \label{subfig:segway_braking_tau_v}
    \end{subfigure}
    \begin{subfigure}[t]{0.48\textwidth}
        \centering
        \includegraphics[width=0.95\textwidth]{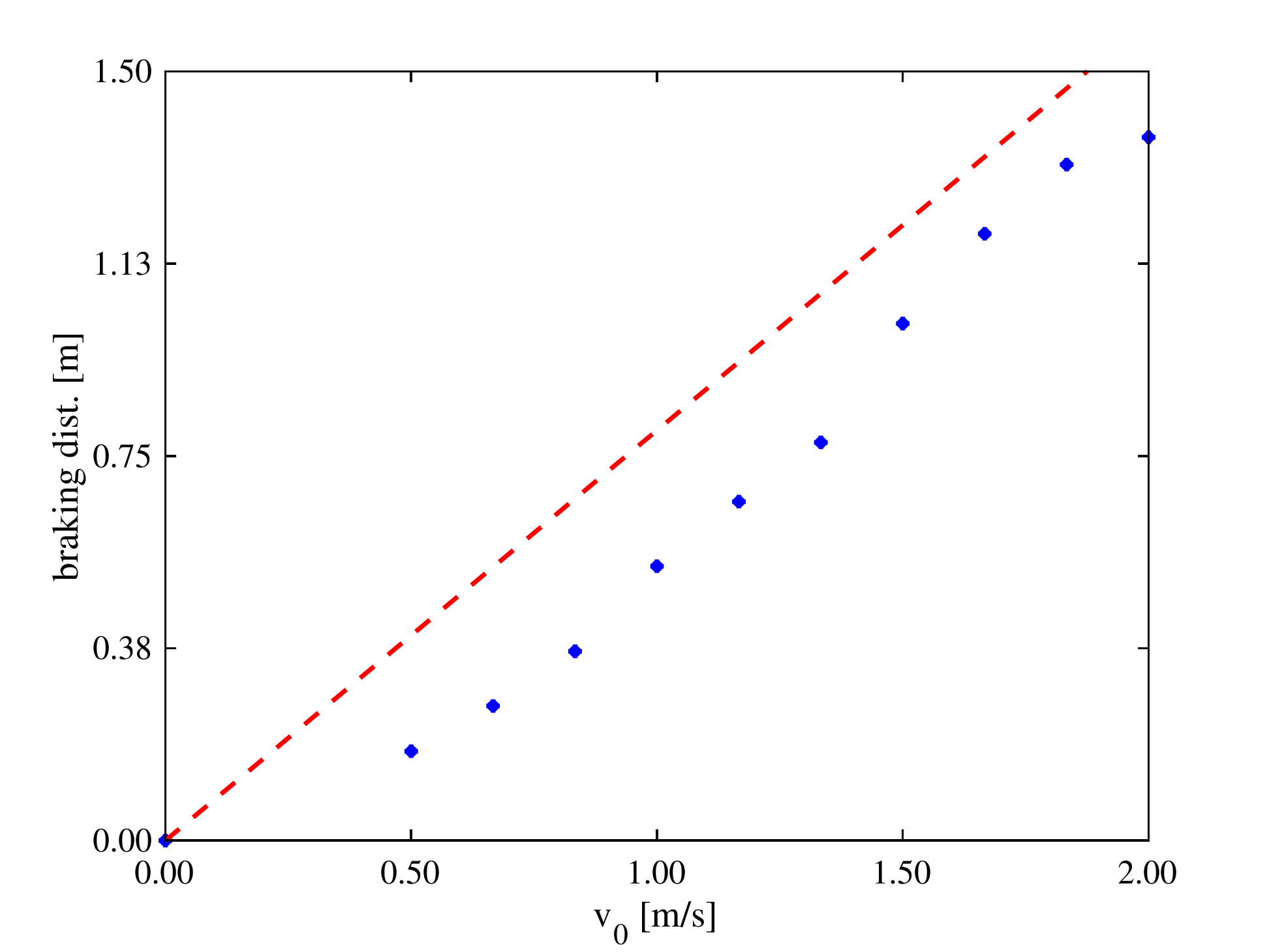}
        \caption{\centering}
        \label{subfig:rover_braking_tau_v}
    \end{subfigure}
    \caption{
    Braking distance of the high fidelity models for the Segway shown in Figure \ref{subfig:segway_braking_tau_v}, and the Rover in Figure \ref{subfig:rover_braking_tau_v} versus initial speed $v_0$ are plotted as blue asterisk's. Notice that both robots have a braking distance which is upper bounded by a linear function (red-dashed line) of speed as in Assumption \ref{ass:brake_dist_grows_sublinearly}.}
    \label{fig:brk_sublinearly}
\end{figure}

\begin{assum}\label{ass:existence_of_tau_v}
Let the trajectory producing model be as in \eqref{eq:traj-producing_model} and the braking trajectory be as in \eqref{eq:braking_traj}.
There exists a duration $\tau_v$ such that the distance traveled by integrating the trajectory producing model over the time interval $t\in [\tau\plan,\tau\plan + \tau_v]$ is greater than or equal to the distance traveled by the high-fidelity model when executing the braking trajectory over the time interval $t\in [\tau\plan,\tau\plan + \tau\brk(\z\hio,k)]$:

\begin{equation}\begin{split}
    &\int_{\tau\plan}^{\tau\plan + \tau\brk(\z\hio,k)} \norm{\idx\left( f\hi(\tau,\z\brk,u\brk)\right)}_2 d\tau \\
    &\leq \int_{\tau\plan}^{\tau\plan + \tau_v} \norm{\idx\left( f(\tau,\z,k)\right)}_2,\label{eq:brake_dist_less_than_nonbrake_dist_tv}
\end{split}\end{equation}
where we are again using the arclength formula as in \eqref{eq:D_brake}, and the arguments to $\z\brk,\ u\brk$, and $\z$ have been dropped for ease of notation.
\end{assum}

To see how this assumption can be easily satisfied, take the example of the Segway's model \eqref{eq:traj_prod_segway}; which plans trajectories with a constant speed.
By setting $\tau_v = C\stp$ from Assumption \ref{ass:brake_dist_grows_sublinearly}, Assumption \ref{ass:existence_of_tau_v} is satisfied.

\begin{rem}\label{rem:planning_time_horizon_min}
For an arbitrary $\z\hio \in Z\hio$ and $k \in K$, let the robot's non-braking trajectory be given by \eqref{eq:non-braking_traj}, and let $\tau_v$ be given by \eqref{eq:brake_dist_less_than_nonbrake_dist_tv}.
If the planning time horizon $T$ satisfies
\begin{align}
    T \geq \tau\plan + \tau_v,
\end{align}
then, for a particular choice of $k$, the total distance traveled by braking trajectories is less than the total distance traveled by the trajectory tracking model \eqref{eq:traj_tracking_model} at $t=T$ for a particular choice of $d\in L_d$.
To see why this is true, first note, by Lemma \ref{lem:traj_prod_matches_hi_fid_model}, there exists a choice of $d\in L_d$, such that the trajectory tracking model matches the non-braking trajectory \eqref{eq:non-braking_traj} from $t=0$ to $t=\tau\plan$. Furthermore, the non-braking and \eqref{eq:non-braking_traj} and braking \eqref{eq:braking_traj} trajectories are equivalent from $t=0$ to $t=\tau\plan$. 
Next, by Assumption \ref{ass:existence_of_tau_v}, the distance traveled from $t=\tau\plan$ to $t=T$ of the trajectory tracking model \eqref{eq:traj_tracking_model} with $d(t)\equiv 0$ is greater than the braking distance of the robot.
\end{rem}

\noindent Since trajectories produced by the trajectory tracking model \eqref{eq:traj_tracking_model} are contained in the FRS, selecting the time horizon to satisfy he inequality in Remark \ref{rem:planning_time_horizon_min} makes it possible to satisfy Assumption \ref{ass:brake_in_pi_X}; where the robot is required to brake within the FRS.

\section{Set Intersection}\label{app:set_intersection}

RTD performs trajectory planning by optimizing over $k \in K_\text{safe}$, as described in Section \ref{sec:trajectory_optimization}.
This requires determining $K_\text{safe}$ at run-time.
\citet{kousik2017safe} show that a \defemph{set intersection} procedure can be used to approximate $K_\text{safe}$ using an SDP, as mentioned in Section \ref{sec:FRSmethod}.
In this appendix, we demonstrate that set intersection is too slow for real-time trajectory planning.

\subsection{Set Intersection with SOS}
Suppose $X\obs \subset X$ is an obstacle represented as a semi-algebraic set, with the list of polynomials $H\obs = \{h_{i,\text{obs}}\}_{i=1}^{n\obs} \subset \R[x]$.
Then Program (19) from \citet{kousik2017safe} is used to find a polynomial $h \in \R[k]$ that is provably an inner approximation of $K\safe$.
Here, we restate Program (21) from \citet{kousik2017safe}, which implements set intersection using SOS programming.
Suppose that $w^l$ is a solution to $(D^l)$ from Section \ref{subsec:FRS_implementation}.
Recall the map $Q_{2l'}$ that gives a space of SOS polynomials, as defined in Section \ref{subsubsec:compute_the_FRS}.
Then, picking $l' \in \N$, we find the SOS polynomial $h \in Q_{2l'}(H_K) \subset \R[k]$ with the program
\begin{align} \label{lp:set_intersection}
		\underset{h}{\text{inf}} \hspace*{0.25cm} & y_K^\top\text{vec}(h) \\
        & 1 - w^l - h &&\in Q_{2l'}(H\obs,H_K) \\
          & h &&\in Q_{2l'}(H_K).
\end{align}
This program is translated into an SDP at runtime.
Notice that, given its size and simplicity, this program does not have the same memory usage problems as the FRS computation (see Section \ref{subsec:FRS_memory_usage}).
However, this program may run slowly depending upon the obstacle representation in the set $H\obs$, as we describe next.

\subsection{Inspecting Set Intersection Speed}

\begin{table*}
    \centering
    \begin{tabular}{|c|l|r|r|}
    \hline
    \multicolumn{4}{|l|}{Set Intersection vs. Obstacle Discretization} \\
    \hline
    Obstacle Shape & \multicolumn{1}{|c|}{Method} & Mean Time [ms] & Std. Dev [ms] \\
    \hline
    Box & Set Intersection \citep{kousik2017safe} & 17,800 & 1010 \\
    \hline
    Line & Set Intersection & 1050 & 73 \\
    \hline
    Box & Discretization (proposed) & 4 & 7 \\
    \hline
    Line & Discretization & 3 & 2 \\
    \hline
    \end{tabular}
    \caption{Timing results of the set intersection procedure from \citet{kousik2017safe} versus the proposed obstacle discretization procedure, both of which provably represent the set $K\safe$ of safe trajectory parameters to be used for online optimization.
    The proposed method is two orders of magnitude faster.}
     \label{tab:set_intersection_timing}
\end{table*}

To determine the speed of the set intersection SDP, we run \eqref{lp:set_intersection} 100 times with $H\obs$ representing a single 2-D, box-shaped obstacle at a random position, similar to what is used in the simulation results of Section \ref{sec:simulation_results}.
As with $(D^l)$ in Section \ref{subsec:FRS_implementation}, we implement this SDP using MATLAB's Spotless toolbox \citep{tobenkin2013spotless}, and solve the resulting conic program with MOSEK \citep{mosek2010mosek}.
Running \eqref{lp:set_intersection} on the box obstacles takes a mean solve time of 17.8 s.
For comparison, we also run Program \eqref{lp:set_intersection} 100 times with $H\obs$ representing a single, randomly-generated 1-D line-segment obstacle in each trial.
We test this type of obstacle because more complex obstacles can be constructed from line segments \citep{kousik2017safe}.
The polynomial $w^l$ is taken from the solution to $(D^l)$ for the FRS of the Segway dynamics from Example \ref{ex:segway} (see Section \ref{subsec:segway_application} for further details).

The results are as follows, and are summarized in Table \ref{tab:set_intersection_timing}.
Program \eqref{lp:set_intersection} solves in 1.05 s on average.
The set intersection timing results show that representing a polygonal obstacle with a collection of line segment obstacles is faster than representing the obstacle as a 2-D semi-algebraic set.
However, even a line segment representation would require approximately 4 s to solve \eqref{lp:set_intersection} for a single box, because the solve time increases linearly with the number of line segments \citep{kousik2017safe}.

Section \ref{sec:obstacle_representation} of this paper presents a discretized obstacle representation that eliminates the need for set intersection, and allows for the online trajectory optimization to run in real time.
For comparison with set intersection, we tested the proposed method (see Algorithm \ref{alg:construct_X_p} in Section \ref{subsec:proving_X_p_works}) to discretize each box and line obstacle from the test of \eqref{lp:set_intersection} described above.
We evaluated $w^l$ on the resulting discrete set of points to produce a list of nonlinear constraints that overapproximate $K\safe$ (as proven in Theorem \ref{thm:X_p} in Section \ref{subsec:proving_X_p_works}).
The proposed method is three orders of magnitude faster than set intersection, as reported in Table \ref{tab:set_intersection_timing}.
% OBSTACLE REPRESENTATION APPENDIX %
\section{Obstacle Representation}\label{app:obs_rep}

This appendix contains the proofs from Section \ref{sec:obstacle_representation}.
In addition, we state Lemmas \ref{lem:middle_chord_is_not_shortest} and \ref{lem:parallel_chords_are_shorter} that provide geometric tools for finding point spacings, and Lemma \ref{lem:rotate_then_translate_to_penetrate} that provides a method for constructing the penetration distance $\bbar$ for arbitrary convex robot footprints.

\subsection{Proofs from Section \ref{subsec:FRS_projections}}
\begin{customlem}{\ref{lem:when_param_cant_cause_crash}}
Consider an arbitrary point $p \in X \setminus X_0$.
Let $k \in \pi_K(p)^C$.
At $t = 0$, let the robot, described by the high-fidelity model \eqref{eq:high-fidelity_model}, be at the state $\z_{\text{hi},0} \in Z\hi$.
Suppose the robot tracks the trajectory parameterized by $k$, producing the high-fidelity model trajectory $\z\hi: [0,T] \to Z\hi$.
Then, no point on the robot's body ever reaches $p$.
More precisely, there does not exist any pair $(t,\z\hio) \in [0,T]\times Z\hio$ such that $p = \idx(\z\hi(t))$.
\end{customlem}
\begin{proof}
Suppose for the sake of contradiction that there exists some $t \in [0,T]$ and $\z\hio \in Z\hio$ for which $p = \idx(\z\hi(t))$.
By Lemma \ref{lem:traj_prod_matches_hi_fid_model}, there exists $d \in L_d$ such that the trajectory-tracking model \eqref{eq:traj_tracking_model} has a trajectory $\z: [0,T] \to Z$ for which $\proj_Z(\z\hi(t)) = \z(t)$ at $t$.
Then, $w(p,k) = 1$ by Lemma \ref{lem:w_geq_1_on_frs}.
But, by \eqref{eq:pi_K}, $k \in \pi_K(p)^C$ implies that $w(p,k) = 0$, which is a contradiction.
\end{proof}

\subsection{Proof from Section \ref{subsec:robot_obstacle_geometry_motivation}}

\begin{customlem}{\ref{lem:buffered_obs_arcs_and_lines}}
Let $n_L \in \N$ (resp. $n_A \in \N$) denote the number of line segments (resp. arcs).
Let $L_i \in L$ (resp. $A_i \in A$) denote the $i$\ts{th} line segment (resp. arc).
Note that each $L_i$ and $A_i$ is a subset of $X$.
Then the boundary of the buffered obstacle can be written as the union of all of the lines and arcs:
\begin{align}
    \bd X\obs^b\quad=\quad\left(\bigcup_{i = 1}^{n_L} L_i\right)~\cup~\left(\bigcup_{i = 1}^{n_A} A_i\right).
\end{align}
\end{customlem}
\begin{proof}
The following statements paraphrase Section 9.2 of \citet{minkowski_sum_fogel}, which shows that the set $X\obs^b$ is equivalent to the Minkowski sum of $X\obs$ with a closed disk of radius $b$.
The procedure of constructing $X\obs^b$, which we call buffering, is also called ``offsetting'' a polygon.
Offsetting a closed and bounded polygon by a distance $b$ produces a closed and bounded shape with a boundary that consists of line segments and circular arcs of radius $b$.
So, the sets $L$ and $A$ are finite because $X\obs$ is closed and bounded by Assumption \ref{ass:obs_are_polygons}.
\end{proof}

\subsection{Proof from Section \ref{subsubsec:rbar}}
\begin{customlem}{\ref{lem:largest_gap_rbar}}
\citep[Theorem 1]{width_of_a_chair} Let $I \subset (X\setminus X_0)$ be a line segment with endpoints $E_I$ and length $L > 0$ (as in Definition \ref{def:line_segment_I}).
Let $X_0$ be the robot's footprint at time $0$ (as in Definition \ref{def:X_and_X_0}), with width $W > 0$ (as in Definition \ref{def:thickness_and_width}).
Then $X_0$ can pass through $I$ (as in Definition \ref{def:pass_through}) if and only if $W < L$.
\end{customlem}
\begin{proof}
See \citet{width_of_a_chair} for a more detailed proof.
We only sketch out the intuition here.
Recall that $X_0$ is convex and compact with nonzero volume by Assumption \ref{ass:X0_cpt_cvx}.

Suppose a transformation family $\{R_t\}$ passes $X_0$ through $I$ as in Definition \ref{def:pass_through}.
Then there exists an interval of time $[t_0,t_1] \subset (0,T]$ for which $R_tX_0 \cap (I\setminus E_I)$ is nonempty for all $t \in [t_0,t_1]$; note that $t_1 > t_0$ because $X_0$ has nonzero volume.
The set $R_tX_0 \cap (I\setminus E_I)$ is a chord (as in Definition \ref{def:chord}) of $R_tX_0$ with length greater than or equal to the width $W$ as in Definition \ref{def:thickness_and_width}.
Since $X_0$ can pass fully through $I$, the endpoints $E_I$ never intersect any $R_tX_0$.
Therefore the length of the chord $R_tX_0 \cap I$ is always less than $L$, so $L > W$.

Now suppose $W < L$.
If $X_0$ has diameter $D$, then $X_0$ can fit completely inside a rectangle with short side length $W$ and long side length $D$ \citep[Theorem 3]{rectangle_bound_curve}.
This rectangle can be rotated so that its short side is parallel to $I$, then pass fully through $I$ by pure translation, i.e. with no further rotations.
Since $X_0$ fits inside the rectangle, $X_0$ can pass fully through $I$.
\end{proof}

\subsection{Proof from Section \ref{subsubsec:bbar}}

\begin{customlem}{\ref{lem:max_penetration_bbar}}
Let $X_0$ be the robot's footprint at time 0 (as in Definition \ref{def:X_and_X_0}), with width ${\rbar}$ (as in Definition \ref{def:rbar}).
Let $\irbar \subset (X\setminus X_0)$ be a line segment of length ${\rbar}$ (as in Definition \ref{def:line_segment_I}).
Then there exists a maximum penetration distance $\bbar$ (as in Definition \ref{def:penetrate}) that can be achieved by passing $X_0$ through $\irbar$ (as in Definition \ref{def:pass_through}).
\end{customlem}

\begin{proof}
This proof is illustrated in Figure \ref{fig:max_penetration_bbar}.
We sketch the intuition first.
To find $\bbar$, we use transformation families $\{R_t\}$ to pass $X_0$ through $\irbar$.
Recall that $X_0$ cannot pass fully through $\irbar$ by Lemma \ref{lem:largest_gap_rbar}.
Then, we measure the penetration distance corresponding to each transformation family to find a supremum.

Now we proceed rigorously.
Note that $X_0$ is compact and convex with nonzero volume as in Assumption \ref{ass:X0_cpt_cvx}.
Recall by Assumption \ref{ass:X_subset_R2_contains_origin_xy-axes} that the $xy$-subspace $X \subset \R^2$ contains the origin of $\R^2$.
To ease the exposition, suppose without loss of generality that $X_0$ lies entirely in the intersection of $X$ with the left half-plane of $\R^2$, and that $\irbar$ is fixed to the origin and oriented vertically in the upper half-plane, so $\irbar = \{0\}\times[0,{\rbar}]$.
In this case, the half-plane $\pirbar$ defined by $\irbar$ (as in Definition \ref{def:halfplane_P_I}) is the closed left half-plane.
This can be done without loss of generality because, when passing $X_0$ through $\irbar$ with a transformation family $\{R_t\}$ (as in Definition \ref{def:R_t_translation_and_rotation_family}), we only care about the relative position of $X_0$ to $\irbar$ at each $t \in [0,T]$.
If $X_0$ and $\irbar$ are oriented arbitrarily in $X$, we can first rotate and translate both $X_0$ and $\irbar$ with the same transformation to move the ``lower'' endpoint of $\irbar$ to the origin, then pass $X_0$ through $\irbar$, and finally undo the first rotation and translation to return $X_0$ and $\irbar$ to their original positions.

Let $\mc{R}_{\rbar}$ denote the set of all transformation families $\{R_t\}$ that attempt to pass $X_0$ through $\irbar$ as per Definition \ref{def:pass_through}.
By Lemma \ref{lem:largest_gap_rbar}, $X_0$ cannot pass fully through $\irbar$ because $\irbar$ is of length ${\rbar}$; but $X_0$ may penetrate $\irbar$ by some distance (as in Definition \ref{def:penetrate}), which depends upon the transformation family $\{R_t\}$.
We must show that, across all $\{R_t\} \in \mc{R}_{\rbar}$, there is a maximum penetration distance.

Consider an arbitrary $\{R_t\} \in \mc{R}_{\rbar}$.
Since $\irbar$ is collinear with the $y$-axis, we can find the penetration distance of $X_0$ through $\irbar$ corresponding to $\{R_t\}$ using a function $\delh: \P(\R^2) \to \R$, which returns the right-most point of a set $A \subset \R^2$:
\begin{align}
    \delh(A)~=~\sup_a\left\{\ a_x~\mid~{a \in A}\right\}, \label{eq:delh_find_horz_penetration_dist}
\end{align}
where $a_x$ is the $x$-component of the point $a$.
So,  given a particular $\{R_t\} \in \mc{R}_{\rbar}$, $\delh(R_TX_0)$ is the penetration distance of $X_0$ through $\irbar$ by Definition \ref{def:penetrate}.
Recall that $X_0$ is compact (i.e. closed and bounded in $X$) and that $X_0$ cannot pass fully through $\irbar$ by Lemma \ref{lem:largest_gap_rbar} (i.e. the horizontal displacement achieved by $R_TX_0$ is bounded).
Therefore, $\delh(R_TX_0)$ is upper bounded.

We have shown that the penetration distance is bounded for each family $\{R_t\} \in \mc{R}_{\rbar}$.
To prove the claim that there is a maximum penetration distance, we must show that the value of $\delh$ is upper bounded across all $\{R_t\} \in \mc{R}_{\rbar}$.
In other words, we want to know that the following supremum is finite:
\begin{flalign}
    \bbar\quad=\quad\underset{\{R_t\}}{\sup}\quad&\delh(R_TX_0)\label{prog:find_bbar_max_penetration} \\
    \text{s.t.}\quad&\{R_t\}\,\in\,\mc{R}_{\rbar}.
\end{flalign}
Recall from Definition \ref{def:thickness_and_width} that $X_0$ has a finite diameter $D$, which is the largest possible distance between two parallel lines that are tangent to $X_0$.
So, for any $\{R_t\} \in \mc{R}_{\rbar}$, if $\delh(R_TX_0) > D$, then $X_0$ has passed fully through $I$.
But this is impossible by Lemma \ref{lem:largest_gap_rbar}.
Since $\bbar \leq D$, \eqref{prog:find_bbar_max_penetration} is upper bounded.
\end{proof}

See Figure \ref{fig:max_penetration_bbar_suboptimal} for an illustration of a suboptimal solution to \eqref{prog:find_bbar_max_penetration}.

\begin{figure}
    \centering
    \includegraphics[width=0.6\columnwidth]{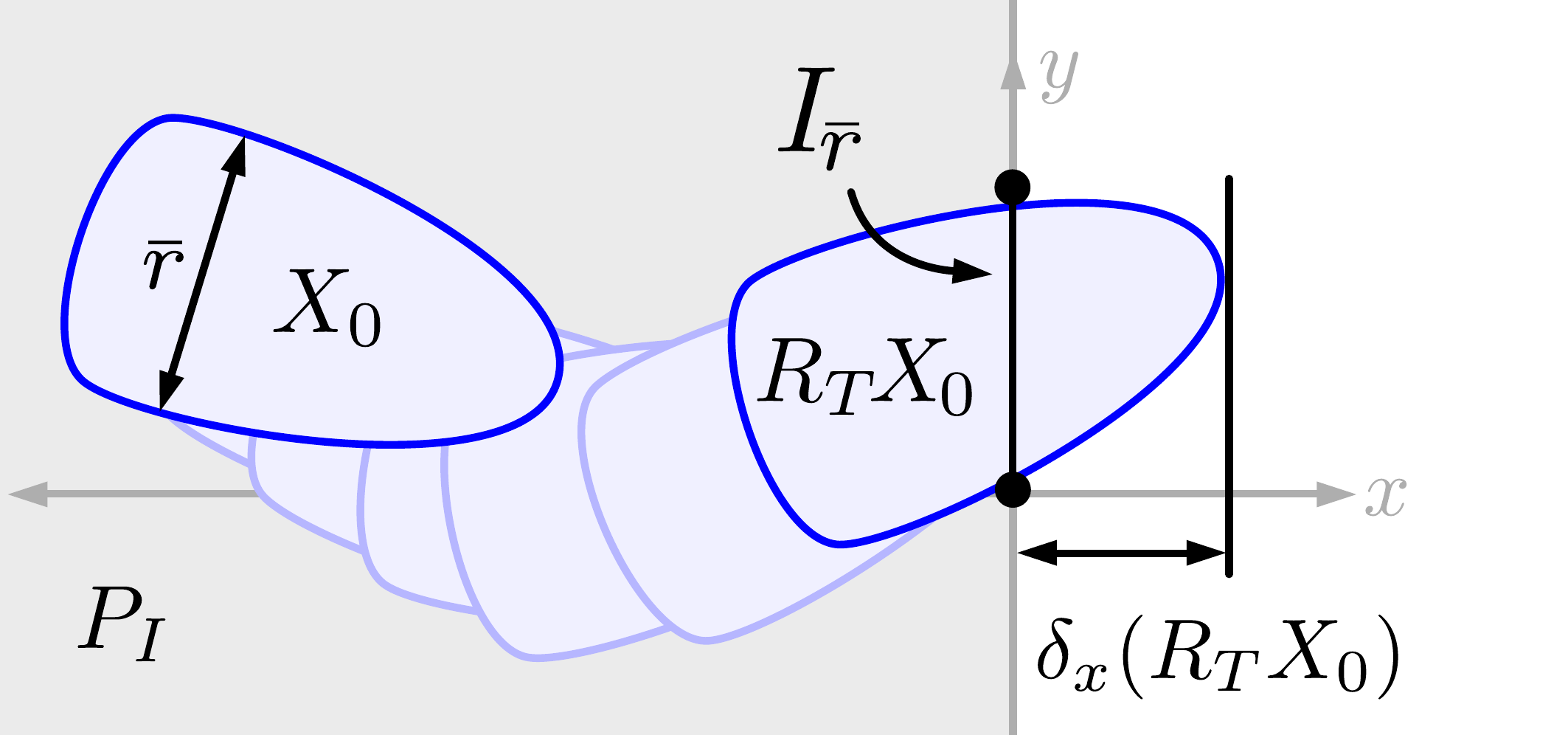}
    \caption{
    An arbitrary compact, convex set $X_0$ of width $\rbar$ penetrates a line segment $\irbar$ when a transformation family $\{R_t\}_{t \in [0,T]}$ is applied to pass $X_0$ through $\irbar$ as in Definition \ref{def:pass_through}.
    Since $\irbar$ is of length $\rbar$, $X_0$ cannot pass fully through by Lemma \ref{lem:largest_gap_rbar}.
    At the initial index $t = 0$ and the final index $t = T$, the sets $R_0X_0$ and $R_TX_0$ are shown with dark outlines.
    A sampling of intermediate indices $t \in (0,T)$ are shown with light outlines.
    In the present figure, $X_0$ penetrates through $\irbar$ by the distance $\delh(R_TX_0)$; this produces a suboptimal, feasible value to \eqref{prog:find_bbar_max_penetration}.
    Note, Figure \ref{fig:max_penetration_bbar} shows the optimal solution.}
    \label{fig:max_penetration_bbar_suboptimal}
\end{figure}

\subsection{Proofs from Section \ref{subsubsec:find_r}}

To find the point spacing $r$, we first prove two lemmas about chords (see Definition \ref{def:chord}).
Then, for $b \in (0,\bbar)$, we find $r \in (0,\rbar)$ with Lemma \ref{lem:find_r}.

\begin{lem}\label{lem:middle_chord_is_not_shortest}
Given any three distinct, parallel chords of a convex, compact set in $\R^2$, the middle chord is not the shortest of the three.

We now restate this more formally.
Let $A \subset \R^2$ be a convex, compact set with nonzero volume.
Let $\kp_1$, $\kp_2$, and $\kp_3$ be three chords of $A$ (as in Definition \ref{def:chord}) such that $\kp_1 \parallel \kp_2 \parallel \kp_3$ and $\kp_i \cap \kp_j = \emptyset$ for any $i \neq j$.
Suppose the chords have lengths $L_1$, $L_2$, and $L_3$, respectively.
Furthermore, assume that there exists at least one line segment (as in Definition \ref{def:line_segment_I}) within $A$ that intersects $\kp_2$, and that has one endpoint on $\kp_1$ and the other endpoint on $\kp_3$; in other words, $\kp_2$ lies between $\kp_1$ and $\kp_3$.
Then$L_1 \geq L_3$ implies that $L_2 \geq L_3$, and $L_1 > L_3$ implies that $L_2 > L_3$.
\end{lem}
\begin{proof}
Let $e_{i,1}$ and $e_{i,2}$ denote the endpoints (as in Definition \ref{def:chord}) of each chord $\kp_i$ where $i = 1,2,3$.
By definition, these endpoints lie in $\bd A$.
Without loss of generality, assume that all three chords are oriented vertically (rotating the chords and the shape $A$ does not change the relative position of the chords to each other or to $A$).
Also suppose without loss of generality that each $e_{i,1}$ is the ``upper'' endpoint (we can do this without loss of generality because each chord is a line segment by Definition \ref{def:chord}, and because we can swap the labels of the endpoints of a line segment without changing the set of points in the line segment).
Define the line segments $I_1$ from $e_{1,1}$ to $e_{3,1}$ and $I_2$ from $e_{1,2}$ to $e_{3,2}$.
Since $A$ is convex, $I_1, I_2 \subset A$.

Suppose $\kp_1$ and $\kp_3$ have the same length, so $L_1 = L_3$.
Then the quadrilateral with edges given by the line segments $\kp_1$, $I_1$, $\kp_3$, and $I_3$ is a parallelogram $Q_P$ (two of its sides are parallel and of equal length).
So, every line segment inside $Q_P$ that is parallel to $\kp_1$ has length $L_1 = L_3$.
Furthermore, $Q_P$ lies completely inside $A$ because $A$ is convex; this means that $\kp_2 \cap Q_P$ is a chord of $Q_P$ that is parallel to $\kp_1$, and $\kp_2 \cap Q_P \subseteq \kp_2$.
Then, since the length of $\kp_2 \cap Q_P = L_1$, the length of $\kp_2$ is $L_2 \geq L_1 \geq L_3$.

Now suppose $L_1 > L_3$.
Then the quadrilateral with edges $\kp_1$, $I_1$, $\kp_3$, and $I_3$ is a trapezoid $Q_T$ (two of its sides are parallel and of different lengths) that lies within $A$.
Since $L_1 > L_3$, every line segment inside $Q_T$ that is parallel to $\kp_1$ is strictly shorter than $\kp_1$.
So, similar to the logic for $Q_P$ above, the length of $\kappa_2 \cap Q_T$ is greater than $L_3$, meaning that $L_2 > L_3$.
\end{proof}

Next, we use Lemma \ref{lem:middle_chord_is_not_shortest} to understand the shape of the footprint as it passes through a line segment in Lemma \ref{lem:parallel_chords_are_shorter}.
In particular, Lemma \ref{lem:parallel_chords_are_shorter} shows that, as the robot penetrates farther through a line segment, the size of the intersection between the robot and the line segment increases.
We use this result in Lemma \ref{lem:find_r} to bound $r$ above and below.

\begin{lem}\label{lem:parallel_chords_are_shorter}
Let $X_0$ be the robot's footprint at time $0$ (as in Definition \ref{def:X_and_X_0}), with width $\rbar$ (as in Definition \ref{def:rbar}).
Let $\irbar \subset (X\setminus X_0)$ be a line segment (as in Definition \ref{def:line_segment_I}) of length ${\rbar}$.
Let $\pirbar$ be the closed half-plane defined by $\irbar$ (as in Definition \ref{def:halfplane_P_I}) and containing $X_0$, and suppose that $X_0 \subset \pirbar$.
Suppose the transformation family $\{R_t\}$ attempts to pass $X_0$ through $\irbar$ (as in Definition \ref{def:pass_through}).
Suppose $t_0 > 0$ such that, for each $t \in [t_0,T]$, the set $\kp_t := R_tX_0 \cap \irbar$ is nonempty and is a chord of $R_tX_0$.
Then, for any $t > t_0$, every chord of $R_tX_0$ that is parallel to $\irbar$ and lies in $\pirbar^{C}$ is shorter than $\kp_t$.
\end{lem}
\begin{proof}
This proof follows directly from Definition \ref{def:pass_through} of passing through and from Lemma \ref{lem:middle_chord_is_not_shortest}.
Recall that $X_0$ is convex and compact with nonzero volume as in Assumption \ref{ass:X0_cpt_cvx}.

As in Lemma \ref{lem:max_penetration_bbar}, without loss of generality assume $\irbar$ lies along the $y$-axis with its lower endpoint fixed to the origin, i.e. $\irbar = \{0\}\times[0,{\rbar}]$, and that $X_0$ lies in the closed left half-plane, which is $\pirbar$.
We can do this without loss of generality because $X$ contains the origin by Assumption \ref{ass:X_subset_R2_contains_origin_xy-axes}, so moving $X_0$ and $\irbar$ in this way is a translation and rotation that can be undone.

Let $t \in (t_0,T]$ be arbitrary and let $\kp_t$ denote the chord $R_tX_0 \cap \irbar$.
Note that $t_0$ exists by Definition \ref{def:pass_through}.
In addition, for any $t \in (t_0,T]$, the set $R_tX_0 \cap \irbar$ is a chord of $R_tX_0$ \citep[Theorem 1]{width_of_a_chair}.
Notice that the length of $\kp_t$ is less than or equal to ${\rbar}$ by Definition \ref{def:pass_through} of passing through.
By Lemma \ref{lem:largest_gap_rbar}, $X_0$ cannot pass fully through $\irbar$.
Therefore, there exists a chord $\kp^-$ of $R_tX_0$ that lies in $\pirbar$, is parallel to $\irbar$, and has length greater than or equal to ${\rbar}$.
Otherwise, $R_tX_0$ could pass fully through $\irbar$ by translation.
Since $t > t_0$, $R_tX_0 \cap \pirbar^C$ is nonempty by Definition \ref{def:pass_through} of passing through.
Therefore, there exist chords of $R_tX_0$ that lie in $\pirbar^C$ and are parallel to $\irbar$.
Let $\kp^+$ be any such chord.
The chords $\kp^-$, $\kp_t$, and $\kp^+$ are three parallel, distinct chords of the convex, compact set $R_tX_0$, and the length of $\kp^-$ is greater than the length of $\kp_t$.
Therefore, by Lemma \ref{lem:middle_chord_is_not_shortest}, $\kp^+$ is shorter than $\kp_t$.
Since $\kp^+$ was arbitrary, we are done.
\end{proof}

Now we find the point spacing $r$ using the previous two lemmas.
The procedure to find $r$ is shown in Figure \ref{fig:find_r_and_a}.

\begin{customlem}{\ref{lem:find_r}}
Let $X_0 \subset \R^2$ be the robot's footprint at time $0$ (as in Definition \ref{def:X_and_X_0}), with width $\rbar$ (as in Definition \ref{def:rbar}).
Let $\bbar$ be the maximum penetration depth corresponding to $X_0$ (as in Lemma \ref{lem:max_penetration_bbar}).
Pick $b \in (0,\bbar)$.
Then there exists $r \in (0,\rbar]$ such that, if $I_r$ is a line segment of length $r$ (as in Definition \ref{def:line_segment_I}), and if $\{R_t\}$ is any transformation family that attempts to pass $X_0$ through $I_r$ (as in Definition \ref{def:pass_through}), then the penetration distance of $X_0$ through $I_r$ (as in Definition \ref{def:penetrate}) is less than or equal to $b$.
\end{customlem}
\begin{proof}
We first sketch the intuition for the proof.
As in Lemma \ref{lem:max_penetration_bbar}, we attempt to pass $X_0$ through a line segment $\irbar$ of length $\rbar$, but $X_0$ cannot pass fully through $\irbar$ by Lemma \ref{lem:largest_gap_rbar}.
Each time we pass $X_0$ through $\irbar$, we halt passing it through when the penetration distance of $X_0$ through $\irbar$ is equal to $b$.
Then, we measure the length of the line segment $X_0 \cap \irbar$.
The length of the smallest such line segment is the desired point spacing $r$.

We now proceed rigorously.
Let $\irbar \subset (X\setminus X_0)$ be a line segment of length $\rbar$  (as in Definition \ref{def:line_segment_I}).
Without loss of generality, suppose that $\irbar$ is vertical with its lower endpoint at the origin, so $\irbar = \{0\}\times[0,\rbar]$; and suppose that $X_0 \subset X \subset \R^2$ lies entirely in the closed left half-plane.
See the proof of Lemma \ref{lem:max_penetration_bbar} for why $\irbar$ and $X_0$ can be placed this way without loss of generality; in brief, the rotations and translations required can be undone.

Next, we discuss how we measure horizontal distance (to constrain the penetration distance to $b$) and vertical span (to find the distance $r$).
Unlike in Lemma \ref{lem:max_penetration_bbar}, instead of letting $X_0$ penetrate through $\irbar$ by the distance $\bbar$, we limit the penetration distance to $b < \bbar$.
Since $\irbar$ is oriented vertically at the origin, we can measure the penetration distance through $\irbar$ using the horizontal distance given by $\delh$ from \eqref{eq:delh_find_horz_penetration_dist}, which returns the maximum $x$-coordinate over all points in a set in $\R^2$.
To measure vertical span, we define the map $\delv: \P(\R^2) \to \R_{\geq 0}$ as follows:
\begin{align}
    \delv(A)~=~ \sup\{a_y~|~a \in A \}~-~\inf\{a_y~|~a \in A\},\label{eq:delv_find_r_vertical_span}
\end{align}
where $a_y$ denotes the $y$-component of $a$.
% Note that, if $A = \emptyset$, then $\delv(A) = \emptyset$.

Now, we find $r$ by constructing the line segment $I_r$.
Let $\mc{R}_\rbar$ be the set of all transformation families $\{R_t\}$ (as in Definition \ref{def:R_t_translation_and_rotation_family}) that attempt to pass $X_0$ through $\irbar$ (as in Definition \ref{def:pass_through}).
Suppose that $\{R_t\} \in \mc{R}_\rbar$ is a transformation family for which, at $t = T$, the penetration distance of $X_0$ through $\irbar$ is $b$ (as in Definition \ref{def:penetrate}).
In other words, $\delh(R_TX_0) = b$.
Consider the line segment $I_r = R_TX_0 \cap I_\rbar$ (this is a line segment by Theorem 1 of \citet{width_of_a_chair}).
Then, under the transformation family $\{R_t\}$, $X_0$ penetrates through $I_r$ by the distance $b$, and the length of $I_r$ is given by $\delv(R_TX_0 \cap \irbar)$.
So, our goal is to find the shortest $I_r$ over all such $\{R_t\}$; the length of the shortest $I_r$ is the distance $r$ claimed by the premises.
Consider the following program to achieve this goal:
\begin{flalign}\label{prog:find_r}
    r\quad =\quad\inf_{\{R_t\}}\quad &\delv(R_T X_0 \cap \irbar)\\
    \mathrm{s.t.}\hspace{0.3cm} &\{R_t\} \in \mc{R}_\rbar, \\
    &\delh(R_TX_0) = b.\label{cons:delh_pentr_dist_equals_b}
\end{flalign}

We first check that feasible solutions exist for \eqref{prog:find_r}.
By Lemma \ref{lem:max_penetration_bbar}, there exist $\{R_t\} \in \mathcal{R}_\rbar$ for which $\delh(R_TX_0) = \bbar > b$.
For any such $\{R_t\}$, since $R_0X_0 = X_0$ (which lies in the left half-plane), we have that $\delh(R_0X_0) \leq 0$.
Then, since $\{R_t\}$ is continuous in $t$ by Definition \ref{def:R_t_translation_and_rotation_family}, there must exist some $\tau \in (0,T)$ for which $\delh(R_\tau X_0) = b$.
So, again using that $\{R_t\}$ is continuous, we can ``cut off'' the time index $t$ at $\tau$ and then rescale time so that $\tau$ becomes $T$ as follows.
For $t \in [0,\tau]$, let $t' = \frac{T}{\tau}t$.
Then the family $\{R_{t'}\ |\ t' \in [0,T]\}$ for which $R_{t'} = R_t$ is a family in $\mc{R}_\rbar$ for which $X_0$ penetrates through $\irbar$ by the distance $b$.

Now we check that $r \in (0,\rbar]$.
Suppose that $\{R_t\}$ is a feasible solution to \eqref{prog:find_r}.
Notice that $\{R_t\}$ cannot pass $X_0$ fully through $\irbar$ by Lemma \ref{lem:largest_gap_rbar}, so $\delv(R_TX_0\cap\irbar) \leq \rbar$ is immediate.
By Definition \ref{def:pass_through} of passing through, $R_TX_0 \cap \irbar$ must be nonempty, so $r = \delv(R_TX_0\cap\irbar) \geq 0$.

Finally, we show that \eqref{prog:find_r} achieves a minimum $r > 0$.
Let $\{R_t\}$ be a feasible solution.
Suppose for the sake of contradiction that there is no $\vep > 0$ for which $r \geq \vep$.
Let  $\kp_r = R_T X_0 \cap \irbar$, which is a chord (as in Definition \ref{def:chord}) of $R_TX_0$ \citep[Theorem 1]{width_of_a_chair}.
By Lemma \ref{lem:parallel_chords_are_shorter}, no chord parallel and to the right of $\kp_r$ can be longer than $\kp_r$, because $\irbar$ is of length $\rbar \geq r$ and parallel to $\kp_r$.
But then, if $\vep = 0$, since $X_0$ has nonzero volume by Assumption \ref{ass:X0_cpt_cvx}, there can be no nonempty chords to the right of $\kp_r$, which contradicts the fact that $\{R_t\}$ attempts to pass $X_0$ through $I$ and as a result violates \eqref{cons:delh_pentr_dist_equals_b}.
\end{proof}

A suboptimal, feasible solution to \eqref{prog:find_r} is shown in Figure \ref{subfig:find_r_suboptimal}; an optimal solution for the same $X_0$ is shown in Figure \ref{subfig:find_r_optimal}.
With Lemma \ref{lem:find_r}, and specifically \eqref{prog:find_r}, we find the \defemph{point spacing} $r$ as in Definition \ref{def:point_and_arc_spacing}.

\begin{figure}
    \centering
    \begin{subfigure}[t]{0.5\columnwidth}
        \centering
        \includegraphics[scale=0.25]{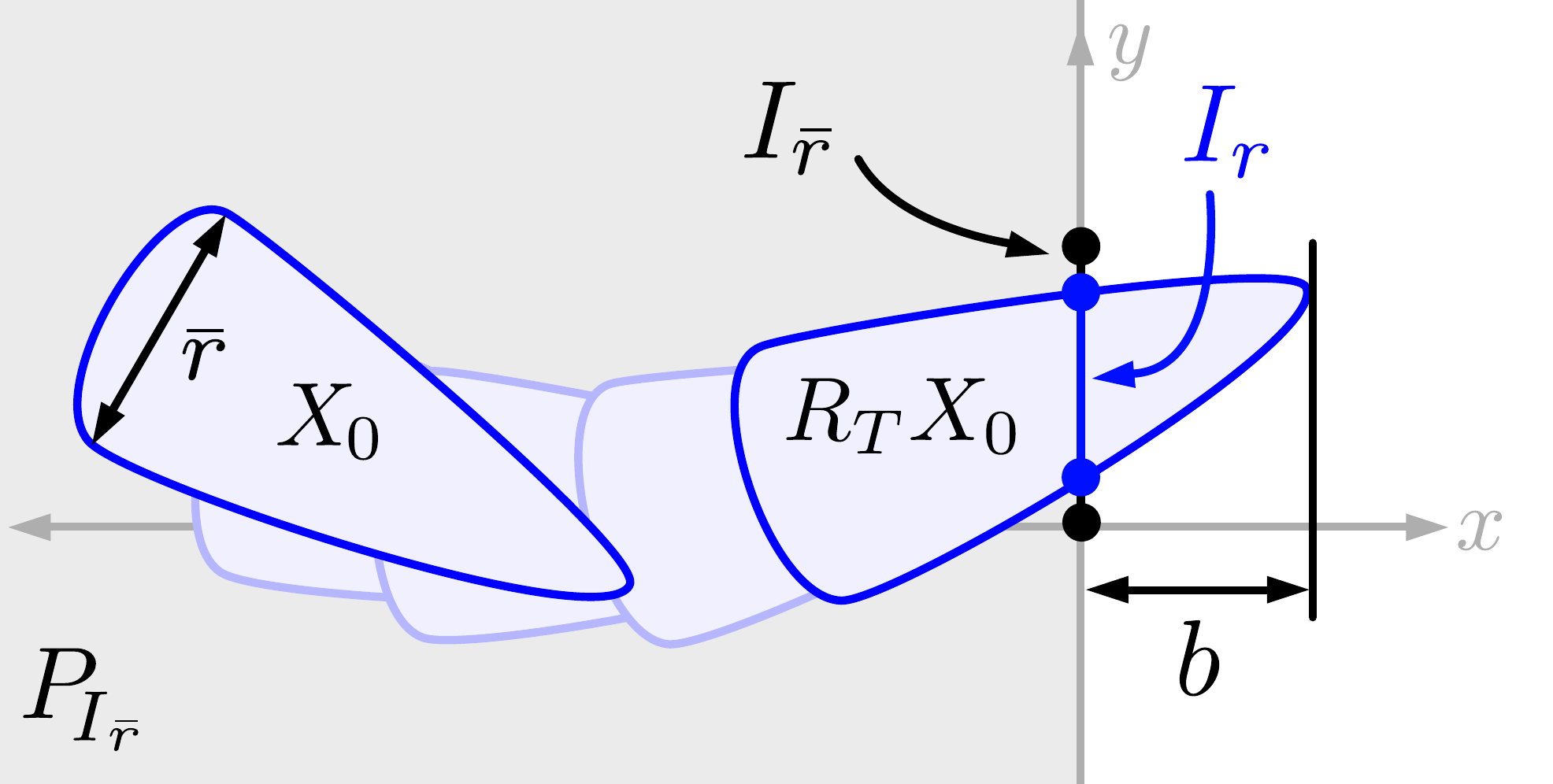}
        \caption{\centering}
        \label{subfig:find_r_suboptimal}
    \end{subfigure}%
    
    \begin{subfigure}[t]{0.5\columnwidth}
        \centering
        \includegraphics[scale=0.25]{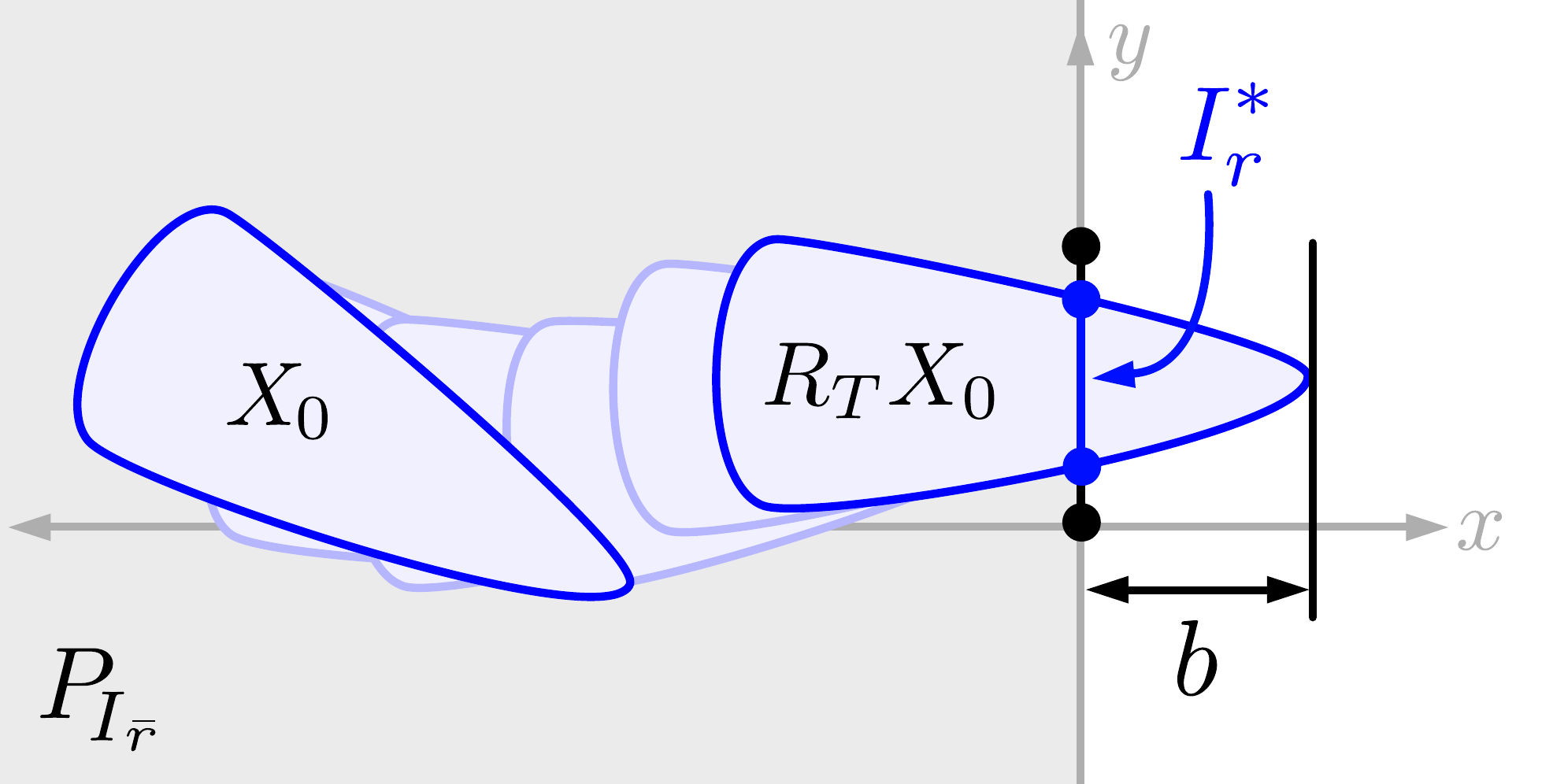}
        \caption{\centering}
        \label{subfig:find_r_optimal}
    \end{subfigure}
    
    \begin{subfigure}[t]{0.5\columnwidth}
        \centering
        \includegraphics[scale=0.25]{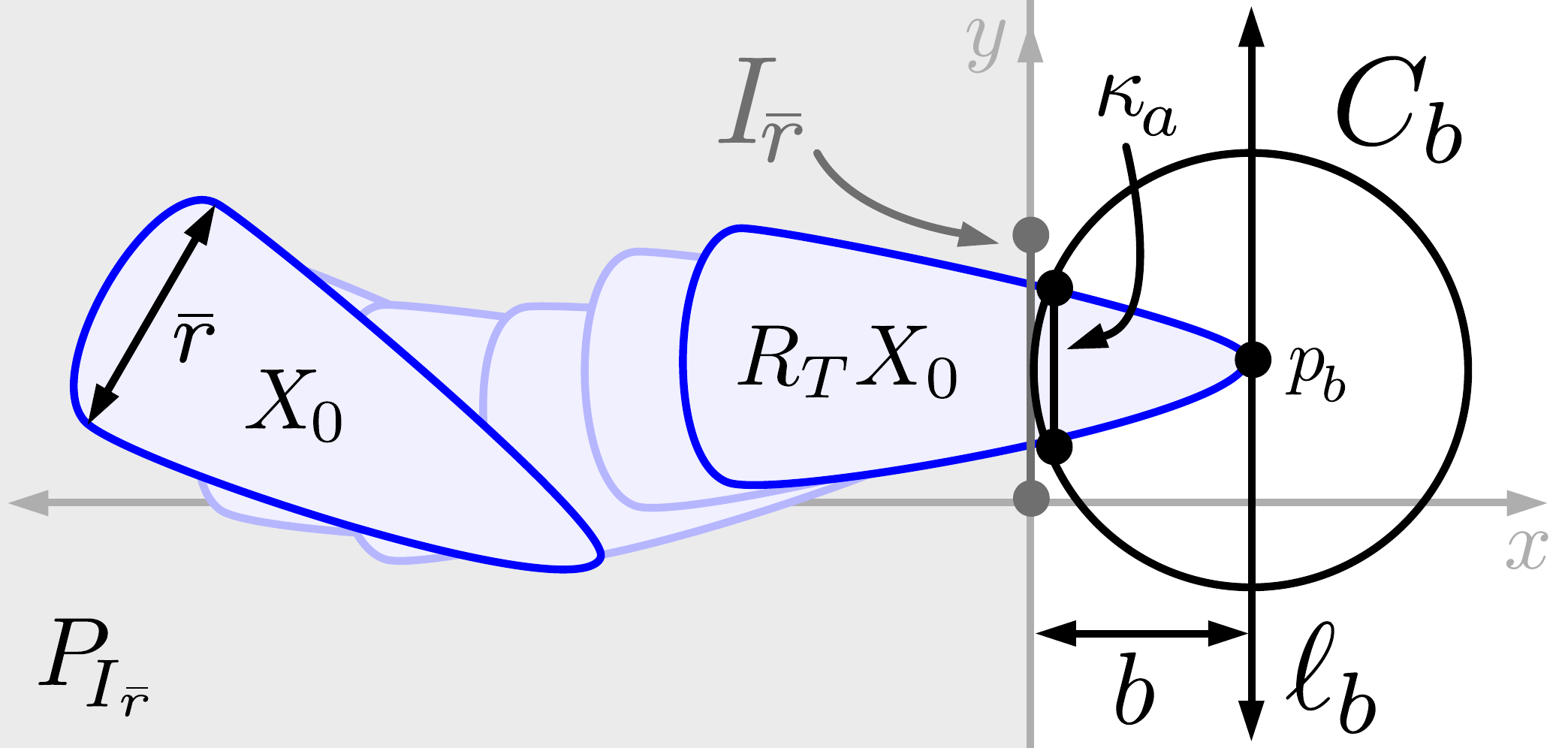}
        \caption{\centering}
        \label{subfig:find_a}
    \end{subfigure}
    \caption{
    An illustration of Program \eqref{prog:find_r} in Figures \ref{subfig:find_r_suboptimal} and \ref{subfig:find_r_optimal}, and Program \eqref{prog:find_a} in Figure \ref{subfig:find_a}.
    The set $X_0$ is an arbitrary convex, compact shape, and starts at $t = 0$ in the left half-plane $P_I$.
    The transformation family $\{R_t\ |\ t \in [0,T]\}$ attempts to pass $X_0$ through $\irbar$.
    At time $T$, $R_TX_0$ is stopped such that its penetration distance through $\irbar$ is the distance $b$.
    Program \eqref{prog:find_r} attempts to find the smallest line segment $I_r$ that can be created when passing $X_0$ through $\irbar$ up to the penetration distance $b$; a suboptimal, feasible solution is shown in Figure \ref{subfig:find_r_suboptimal}, and an optimal solution is shown in Figure \ref{subfig:find_r_optimal}.
    Program \eqref{prog:find_a} attempts to find the smallest chord $\kp_a$ of a circle $C_b$ for which $X_0$ cannot penetrate farther than $b$ into $C_b$ through $\kp_a$.
    This is shown in Figure \ref{subfig:find_a}, which starts from a feasible solution to \eqref{prog:find_r}, then centers the circle $C_b$ on a point of $R_TX_0$ that has penetrated to the distance $b$ past $\irbar$.
    The chord $\kp_a$ is defined by points in the intersection of $\bd R_TX_0$ with $C_b$, and is therefore also a chord of $R_TX_0$.
    In this case, the optimal $\kp_a$ is shown.}
    \label{fig:find_r_and_a}
\end{figure}

\subsection{Proof from Section \ref{subsubsec:find_a}}

\begin{customlem}{\ref{lem:find_a}}
Let $X_0$ be the robot's footprint at time $0$ (as in Definition \ref{def:X_and_X_0}), with width $\rbar$ (as in Definition \ref{def:rbar}).
Let $\bbar$ be the maximum penetration distance corresponding to $X_0$ (as in Lemma \ref{lem:max_penetration_bbar}).
Pick $b \in (0,\bbar)$, and let $C \subset (X\setminus X_0)$ be a circle of radius $b$ centered at a point $p \in X$ (as in Definition \ref{def:circle_and_arc}).
Then there exists a number $a \in (0,\rbar)$ such that, if $\kp_a$ is any chord of $C$ of length $a$ (as in Definition \ref{def:chord}), then the penetration of $X_0$ into $C$ through $\kp$ (as in Definition \ref{def:pass_and_penetrate_circles_and_arcs}) is no larger than $b$.
\end{customlem}

\begin{proof}
We begin with a sketch of the proof to build intuition.
This proof proceeds much as for Lemma \ref{lem:find_r} to find the point spacing $r$.
To prove that $a$ exists, we pass $X_0$ through a line segment $\irbar$ of length $\rbar$, up to a penetration distance of $b$.
Then, we translate the circle $C$ of radius $b$ such that $X_0$ is penetrating into this circle.
From the intersection of the circle with $X_0$, we find a chord $\kp_a$.
The length of $\kp_a$ depends on the transformation family $\{R_t\}$ used to pass $X_0$ through $\irbar$.
We search across all such transformation families to find the smallest $\kp_a$, the length of which is the desired arc point spacing $a$.

Now we proceed rigorously.
Recall by Assumption \ref{ass:X0_cpt_cvx} that $X_0$ is compact, convex, and has nonzero volume, and by Assumption \ref{ass:X_subset_R2_contains_origin_xy-axes} that $X \subset \R^2$ contains the origin of $\R^2$.
Let $\irbar \subset (X\setminus X_0)$ be a line segment of length $\rbar$ (as in Definition \ref{def:line_segment_I}).
As in Lemma \ref{lem:max_penetration_bbar} (used to find $\bbar$), suppose without loss of generality that $\irbar$ is oriented vertically, with its lower endpoint fixed at the origin, so $\irbar = \{0\}\times[0,\rbar]$.
Suppose without loss of generality that $X_0$ lies fully in the left half-plane, which is $\pirbar$, the half-plane defined by $\irbar$ (as in Definition \ref{def:halfplane_P_I}).
This can be done without loss of generality because it only requires rotation and translation of $X_0$ and $\irbar$, which can be undone.

Let $\mc{R}_\rbar$ be the set of all transformation families that attempt to pass $X_0$ through $\irbar$ (as in Definition \ref{def:pass_through}).
By Lemma \ref{lem:find_r}, there exist $\{R_t\} \in \mc{R}_\rbar$ for which the penetration distance of $X_0$ through $\irbar$ is equal to $b$.
Such $\{R_t\}$ are feasible solutions to \eqref{prog:find_r}.
Let $\ell_b = \{b\}\times\R$ be the vertical line at $x = b$.
Let $\{R_t\}$ be a feasible solution to \eqref{prog:find_r}.
Then, there exists at least one point in $R_TX_0$ that lies on $\ell_b$.
Let $\kp_b$ denote the set $R_T X_0 \cap \ell_b$, which is a chord of $R_T X_0$  \citep[Theorem 1]{width_of_a_chair}.
Note that $\kp_b$ may have length $0$, i.e. it is a point, and that $\kp_b$ is compact, because it is the intersection of two compact sets \citep[Theorem 17.1 and Theorem 26.2]{Munkres2000}.
Place the circle $C$ (with radius $b$) tangent to the $y$-axis, and centered at any point $p_b \in \kp_b$.
Let $C_b$ denote this translation of $C$.
Recall the function $\delh$ from \eqref{eq:delh_find_horz_penetration_dist}, which returns the right-most point of a set in $\R^2$.
With these objects, we pose following program to find the shortest chord $\kp_a$ for which $X_0$ penetrates into $C_b$ through $\kp_a$ by the distance $b$:
\begin{flalign}\label{prog:find_a}
    a\quad =\quad\inf_{\{R_t\}, p_b, p_1, p_2}\quad &\norm{p_1 - p_2}_2\\
    \mathrm{s.t.}\hspace{0.7cm} &\{R_t\} \in \mc{R}_\rbar \label{cons:find_a_with_feas_to_find_r} \\
    &\delh(R_TX_0) = b,\label{cons:delh_RTX0_is_b}\\
    &p_b \in \ell_b \cap R_T X_0, \label{cons:C_centered_on_bdRX_0} \\
    &p_1, p_2 \in C_b \cap \bd R_T X_0, \label{cons:endpoints_of_arc_for_a}
\end{flalign}
where $p_1$ and $p_2$ are the endpoints of $\kp_a$ (as in Definition \ref{def:circle_and_arc}).

We now construct a feasible solution to \eqref{prog:find_a}.
Let $\{R_t\}$ be a feasible solution to \eqref{prog:find_r}, so $\delh(R_TX_0) = b$, which satisfies \eqref{cons:find_a_with_feas_to_find_r} and \eqref{cons:delh_RTX0_is_b}.
Since $\ell_b \cap R_TX_0$ is nonempty as discussed above, we can pick $p_b$ to satisfy \eqref{cons:C_centered_on_bdRX_0}, and create $C_b$ centered at $p_b$.
Then $A_b = C_b \cap R_T X_0$ is an arc of radius $b$ (as in Definition \ref{def:circle_and_arc}); we justify that $A_b$ is indeed an arc in the next paragraph.
Let $p_1$ and $p_2$ be the endpoints (as in Definition \ref{def:circle_and_arc}) of $A_b$, satisfying \eqref{cons:endpoints_of_arc_for_a}.
Let $\kp_a$ be the chord that lies between the endpoints of $A_b$.
Then, $R_TX_0$ penetrates into $C_b$ through $\kp_a$ by the distance $b$ (as in Definition \ref{def:pass_and_penetrate_circles_and_arcs}).
This is illustrated in Figure \ref{subfig:find_a}.

Now we justify that $A_b$ is indeed an arc of radius $b$ with two endpoints.
First, notice that the intersection $C_b \cap R_T X_0$ is nonempty for two reasons.
One, because $C_b$ is centered on a point in $\bd R_T X_0$; and two, because $\delh(R_T X_0) = b$, which implies that there exists at least one line segment inside $R_TX_0$ that is in the open right half-plane and of length $b$.
Furthermore, because $R_TX_0$ has nonzero volume (Assumption \ref{ass:X0_cpt_cvx}), $A_b$ has exactly two endpoints, which lie on the boundary of $R_TX_0$.
Otherwise, there would exist a pair of points in $R_TX_0$ that are connected by a line segment that does not lie fully in $R_\tau X_0$, which would violate the convexity of $R_T X_0$.

Now, we check that $a \in (0,\rbar)$.
Let $\{R_t\}, p_b, p_1, p_2$ be a feasible solution to \eqref{prog:find_a}.
By construction, $X_0$ penetrates into $C_b$ through $\kp_a$ by $b < \bbar$.
Then the length $a$ of $\kp_a$ is less than $\rbar$, otherwise, by Lemma \ref{lem:max_penetration_bbar}, $X_0$ could penetrate into $C_b$ through $\kp_a$ by no more than $b$.
Now suppose that $a = 0$.
Then, by Lemma \ref{lem:parallel_chords_are_shorter}, there can be no nonempty chords of $R_TX_0$ between $\kp_a$ and the center of the circle $p_b$, but then $X_0$ does not penetrate into $C_b$ through $\kp_a$.
\end{proof}

\noindent Suppose $r$ is the point spacing found with Lemma \ref{lem:find_r}.
Then we can prove that $a \in (0,r)$ using the same techniques from the proof of Lemma \ref{lem:find_a}, by replacing $\irbar$ with $I_r$, a line segment of length $r$.

\subsection{Proof from Section \ref{subsec:proving_X_p_works}}

\begin{customthm}{\ref{thm:X_p}}
Let $X_0$ be the robot's footprint at time 0 as in Definition \ref{def:X_and_X_0}, with width $\rbar$ as in Definition \ref{def:rbar}.
Let $X\obs \subset (X\setminus X_0)$ be a set of obstacles as in Definition \ref{def:obs}.
Suppose that the maximum penetration depth $\bbar$ is found for $X_0$ as in Lemma \ref{lem:max_penetration_bbar}.
Pick $b \in (0,\bbar)$, and find the point spacing $r$ with \eqref{prog:find_r} and the arc point spacing $a$ with \eqref{prog:find_a}.
Construct the discretized obstacle $X_p$ in Algorithm \ref{alg:construct_X_p}.
Then, the set of all unsafe trajectory parameters corresponding to $X\obs$ is a subset of the trajectory parameters corresponding to $X_p$, i.e. $\pi_K(X_p) \supseteq \pi_K(X\obs)$.
\end{customthm}

\begin{proof}
We show that any trajectory parameter outside of those corresponding to $X_p$ cannot cause any point on the robot to enter the set $X\obs$ at any time $t \in [0,T]$.
If no $q \in \pi_K(X_p)^C$ can cause a collision, then $\pi_K(X_p)^C \subseteq K\safe$, which implies that $\pi_K(X_p) \supseteq \pi_K(X\obs)$.
First, recall that the robot's high-fidelity model in \eqref{eq:high-fidelity_model} produces continuous trajectories (by Assumption \ref{ass:dyn_are_lipschitz_cont}) of the robot's footprint in $\R^2$, so we can represent the motion of the robot over the time horizon $[0,T]$ using a transformation family $\{R_t\}$ as in Definition \ref{def:R_t_translation_and_rotation_family}.

Suppose $k \in \pi_K(X_p)^C$ is arbitrary, and the robot begins at an arbitrary $\z\hio \in Z\hio$.
Let $\{R_t\}$ be the transformation family that describes the robot's motion when tracking the trajectory parameterized by $k$.
Consider a pair $(p_1,p_2)$ of adjacent points (as in Definition \ref{def:adjacent_points}) of $X_p$.
Recall that the function \texttt{sample} returns the endpoints of any line segment (as in Definition \ref{def:line_segment_I}) or arc (as in Definition \ref{def:circle_and_arc}), in addition to points spaced along the line segment or arc if necessary.
Therefore, by Algorithm \ref{alg:construct_X_p}, $(p_1,p_2)$ is either from a line segment or from an arc of $\bd X\obs^b$.
Recall that, by Lemma \ref{lem:buffered_obs_arcs_and_lines}, $\bd X\obs^b$ consists exclusively of line segments and arcs.
By construction, if $p_1$ is on a line segment (resp. arc), then $p_2$ is within the distance $r$ (resp. $a$) along the line segment; this also holds if either point is an endpoint of a line segment or arc.

Consider the case when $(p_1,p_2)$ is from an arbitrary line segment $L_i$ of $\bd X\obs^b$.
By \eqref{eq:X_obs_b_buffered_obstacle}, the distance from $X\obs$ to any point on $L_i$ is $b$.
By Lemma \ref{lem:when_param_cant_cause_crash}, when tracking the trajectory parameterized by $k$, the robot can approach infinitesimally close to $p_1$ and/or $p_2$, but cannot contain them, for any $t \in [0,T]$.
So, by Lemma \ref{lem:find_r} and continuity of the robot's trajectory, no point in the robot can penetrate farther than $b$ through $L_i$.

Now consider when $(p_1,p_2)$ is from an arbitrary arc $A_i$ of $\bd X\obs^b$.
By Equation \eqref{eq:X_obs_b_buffered_obstacle}, the distance from $X\obs$ to any point on $A_i$ is $b$.
Each such arc is a section of a circle of radius $b$.
By Lemma \ref{lem:when_param_cant_cause_crash}, the robot cannot contain $p_1$ or $p_2$ for any $t \in [0,T]$.
So, by Lemma \ref{lem:find_a} and continuity of the robot's trajectory, the robot cannot pass farther than the distance $b$ into $A_i$ through the chord of $A_i$ with endpoints $p_1$ and $p_2$.

Since $L_i$ and $A_i$ were arbitrary, there does not exist any $t \in [0,T]$ for which $R_tX_0 \cap X\obs$ is nonempty.
In other words, the robot does not collide with $X\obs$ by passing through any line segment or arc of $\bd X\obs^b$.
Since $k$ was arbitrary, we conclude that there does not exist any $k \in \pi_K(X_p)^C$ for which the robot collides with any obstacle.
Therefore, $\pi_K(X_p)^C \subseteq K\safe$.
\end{proof}

\subsection{Finding the Maximum Penetration Distance}
To conclude this appendix, we present a geometric method for finding the maximum penetration distance $\bbar$ (Lemma \ref{lem:max_penetration_bbar}) for an arbitrary robot footprint $X_0$ with width $\rbar$ (Definition \ref{def:rbar}).

\begin{lem}\label{lem:rotate_then_translate_to_penetrate}
Let $X_0$ be the robot's footprint at time $0$ (as in Definition \ref{def:X_and_X_0}) with width $\rbar$ (as in Definition \ref{def:rbar}).
Let $I_{\rbar} \subset (X\setminus X_0)$ be a line segment (Definition \ref{def:line_segment_I}) of length ${\rbar}$.
Let $\bbar$ denote the maximum penetration distance of $X_0$ through $I_\rbar$ (as in Lemma \ref{lem:max_penetration_bbar}).
Then, there exists at least one angle of rotation $\ta \in [0,2\pi)$ for which, if $X_0$ is rotated by $\ta$, then passed through $I_{\rbar}$ by translation only (where passing through is as in Definition \ref{def:pass_through}), $X_0$ penetrates $I_{\rbar}$ by $\bbar$.
\end{lem}
\begin{proof}
Let $\pirbar$ be the closed half-plane defined by $\irbar$ as in Definition \ref{def:halfplane_P_I}.
Let $\{R_t\}$ be a transformation family as in Definition \ref{def:R_t_translation_and_rotation_family} such that the penetration of $X_0$ into $\irbar$ (as in Definition \ref{def:penetrate}) is the number $\bbar$; such an $\{R_t\}$ exists by Lemma \ref{lem:max_penetration_bbar}.
As per Definition \ref{def:R_t_translation_and_rotation_family}, every $R_t \in \{R_t\}$ has an associated translation $s_t$ and rotation $\ta_t$, both of which are continuous in $t$.
By \citep[Theorem 1]{width_of_a_chair}, the line segment $R_TX_0 \cap \irbar$ is a chord of $R_TX_0$ (as in Definition \ref{def:chord}).
By Lemma \ref{lem:parallel_chords_are_shorter}, every chord of $R_TX_0$ that is parallel to $\irbar$ and lies in $\pirbar^C$ is strictly shorter than $\irbar$.
Therefore, one can translate $R_TX_0$ ``out'' of $\pirbar^C$ (i.e., ``undo'' passing $X_0$ through $\irbar$), while leaving $R_TX_0$ rotated at the angle $\ta_T$ associated with $R_T$.
More precisely, there exists a family $\{s_t\}_{t \geq T}$ of translations such that the set $\{p + s_t~|~p \in R_TX_0\}$ does not intersect the endpoints of $\irbar$ for all $t > T$; otherwise, there exists a chord of $R_TX_0$ that is parallel to $\irbar$ and longer than $\irbar$ that lies in $\pirbar^C$, which contradicts Lemma \ref{lem:parallel_chords_are_shorter} and the fact that the family $\{R_t\}$ passes $X_0$ through $\irbar$.
So, we can cause $X_0$ to penetrate through $\irbar$ by the distance $\bbar$ by first rotating it to an angle $\ta_T$, then passing it through by translation only.
\end{proof}

\noindent Note that this lemma starts with the penetration distance of $\bbar$ and works ``backwards.''
However, we can find $\ta_T$ in a ``forward'' direction, which we now discuss informally.
Assume the premises of Lemma \ref{lem:rotate_then_translate_to_penetrate}, and recall that $X_0$ cannot pass fully through $\irbar$.
Rotate $X_0$ by an angle $\ta \in [0,2\pi)$.
Next, pass $X_0$ through $\irbar$ ``as far as possible'' into $\pirbar^C$ by translation only; this means that both endpoints of $\irbar$ lie in the boundary of the translated $X_0$.
Consider the following to see why the translated $X_0$ must contain both endpoints to be translated ``as far as possible.''
If $\bd X_0$ does not contain both endpoints of $\irbar$ after translation, then there are two possibilities.
In the first case, the translated $X_0$ contains neither endpoint, so it can be translated farther into $\pirbar^C$.
In the second case, the translated $X_0$ contains one endpoint, in which case it can be translated a small distance towards the other endpoint so that it no longer contains both endpoints, which means that we are backin the first case.
This procedure of rotation-then-translation need only be checked for $\theta \in [0,2\pi)$, and every $\theta$ is associated with a finite penetration distance, so at least one $\theta$ produces the maximum penetration distance $\bbar$.
\end{appendices}

\end{document}